\documentclass[letterpaper]{article}

\usepackage[margin=1in]{geometry}
\usepackage{times}
\usepackage{amsthm,amsmath,amssymb}
\usepackage{xspace}
\usepackage{enumitem}
\usepackage{xcolor}
\usepackage{graphicx}
\usepackage{dsfont}
\usepackage{lineno}
\usepackage[round]{natbib}
\usepackage[caption=false,font=footnotesize]{subfig}
\usepackage{thmtools,thm-restate}
\usepackage{placeins}
\usepackage{soul}
\usepackage[toc,page]{appendix}
\usepackage[hidelinks]{hyperref}
\usepackage{rotating}

\newcommand{\numpapers}{d}
\newcommand{\hyperparam}{\lambda}
\newcommand{\similarity}{S}
\newcommand{\numrev}{n}

\newcommand{\proxy}{h}
\newcommand{\blocksize}{q}
\newcommand{\rankfunction}{\sigma}
\newcommand{\paperset}{[d]}
\newcommand{\reviewerset}{[n]}
\newcommand{\numbids}{g}
\newcommand{\ordering}{\pi}
\newcommand{\bidfunction}{f}
\newcommand{\randombid}{\mathcal{B}}
\newcommand{\realbid}{b}
\newcommand{\history}{\mathcal{H}}
\newcommand{\paperindex}{j}
\newcommand{\revindex}{i}

\newcommand{\gainfunction}{\gamma_p}
\newcommand{\gainfunctionrev}{\gamma_r}
\newcommand{\bidprob}{p}

\newcommand{\super}{\texttt{SUPER}\text{$^*$}\xspace}

\newcommand{\bidbase}{\texttt{BID}\xspace}
\newcommand{\simbase}{\texttt{SIM}\xspace}
\newcommand{\randbase}{\texttt{RAND}\xspace}
\newcommand{\optbase}{\texttt{OPT}\xspace}

\newcommand{\algbase}{\texttt{ALG}\xspace}
\newcommand{\fpo}{\texttt{FindPaperOrder}\xspace}
\newcommand{\fpoe}{\texttt{FindPaperOrderEfficient}\xspace}

\newcommand{\gain}{\mathcal{G}}
\newcommand{\numblocks}{m}

\newcommand{\extravar}{k}
\newcommand{\optvar}{x}
\newcommand{\weight}{w}
\newcommand{\heuristicf}{h}
\newcommand{\sortvar}{\alpha}
\newcommand{\gainminrevs}{r}
\newcommand{\symgroup}{\Pi}

\newcommand{\reals}{\mathbb{R}}
\newcommand{\Exs}{\mathbb{E}}
\newcommand{\simscalar}{s}

\newcommand{\bidn}{1}
\newcommand{\bidnp}{2}

\newcommand{\constant}{c}
\newcommand{\diagset}{\mathcal{D}}
\newcommand{\diagsetp}{\mathcal{D}}
\newcommand{\permleft}{P}
\newcommand{\permright}{\widetilde{P}}
\newcommand{\permclass}{\mathcal{P}}
\newcommand{\blockm}{B}

\newcommand{\reviewers}{reviewers} 
\newcommand{\reviewer}{reviewer}
 
\newtheorem{property}{Property}

\newtheorem{lemma}{Lemma}
\newtheorem{theorem}{Theorem}
\newtheorem{corollary}{Corollary}
\newtheorem{proposition}{Proposition}
\newtheorem{definition}{Definition}

\makeatletter \renewcommand\paragraph{\@startsection{paragraph}{4}{\z@}%
{2mm}%
{-1em}%
{\normalfont\normalsize\bfseries}} 
\makeatother 

\DeclareMathOperator*{\argmax}{arg\,max}

\newcounter{alg}
\date{} 
\title{A SUPER* Algorithm to Optimize Paper Bidding in Peer Review}

\author{
  \textbf{Tanner Fiez} \\
  University of Washington \\
  \textup{fiezt@uw.edu}
  \and
  \textbf{Nihar B.~Shah} \\
  Carnegie Mellon University \\
  \textup{nihars@cs.cmu.edu}
  \and 
  \textbf{Lillian Ratliff} \\
  University of Washington \\
  \textup{ratliffl@uw.edu}
}

\begin{document}

\maketitle

\begin{abstract}
A number of applications involve  sequential arrival of users, and require showing each user an ordering of items. 
A prime example (which forms the focus of this paper) is the bidding process in conference peer review where reviewers enter the system sequentially, each reviewer needs to be shown the list of submitted papers, and the reviewer then ``bids'' to review some papers. The order of the papers shown has a significant impact on the bids due to primacy effects.  In deciding on the ordering of papers to show, there are two competing goals: (i) obtaining sufficiently many bids for each paper, and (ii) satisfying reviewers by showing them relevant items. 
In this paper, we begin by developing a framework to study this problem in a principled manner. We present an algorithm called \super, inspired by the A$^{\ast}$ algorithm, for this goal. 
Theoretically, we show a local optimality guarantee of our algorithm and prove that popular baselines are considerably suboptimal. Moreover, under a community model for the similarities, we prove that \super is near-optimal whereas the popular baselines are considerably suboptimal. 
In experiments on real data from ICLR 2018 and synthetic data, we find that \super considerably outperforms baselines deployed in existing systems, consistently reducing the number of papers with fewer than requisite bids by 50-75\% or more, and is also robust to various real world complexities. 
\end{abstract}

\addtocontents{toc}{\setcounter{tocdepth}{-10}}

\section{Introduction}
\label{sec:intro}
It is well known that peer review is essential for ensuring the quality and scientific value of research
~\citep{black:1998aa, thurner:2011aa,bianchi2015three}. 
A fundamental challenge in peer review is matching or assigning papers to qualified and willing {\reviewers}. 
Common methods to deal with this problem often rely on access to a \emph{similarity matrix} containing scores for each paper-{\reviewer} pair expressing the estimated match quality between them. The similarity matrix is often obtained using feature-based or profile-based matching mechanisms that leverage keywords and available {\reviewer} publications~\citep{charlin2013toronto,price2017computational}. A number of automated methods to match papers with reviewers using similarity scores have been proposed that optimize objectives such as cumulative similarity or fairness notions~\citep{karimzadehgan2008multi,garg:2010aa,tang2010expertise,long2013good,stelmakh:2018aa}. 

 A shortcoming of automating the paper matching process stems from the failure to actively incorporate {\reviewers} within the paper assignment phase of the review process. 
 The outright dependence on the similarity scores can be problematic since the preferences of {\reviewers} can change frequently and the similarity scores themselves can be noisy. 
 Bidding has emerged as an important mechanism for aiding in and improving the peer review process under the guise that active engagement of the {\reviewer} leads to assignments more aligned with their preferences and hence, enhanced review quality~\citep{di:2005aa}.

In typical peer review process, when the bidding process opens, {\reviewers} enter the system in an arbitrary sequential order. Upon entering, a list of papers is shown to them and they are asked to place bids on papers they would prefer to review. Following the bidding process, bids can be incorporated into the reviewer-paper assignment mechanism. 
It is known that the order of papers presented to {\reviewers} in the bidding stage can greatly impact the number of bids that a paper receives~\citep{cabanac:2013aa}. 
From the perspective of the platform, there are two competing goals: (i)  ensure that each paper has a sufficient number of bids, and (ii) ensure individual {\reviewer} satisfaction by showing relevant papers. 

With regard to goal (i), the platform  aims to select a display order for each {\reviewer} such that at the end of the bidding process, each paper has at least a certain number of bids. The main objective of ensuring a minimum number of bids on each paper is to improve review quality for all papers~\citep{shah:2018aa}. The well-documented primacy effect~\citep{murphy2006primacy} suggests that papers shown on top of the ordering are the ones on which {\reviewers} are more likely to bid. 
Consequently, this objective strongly suggests that papers with few bids should be placed higher in the list. 
Indeed,~\citet{cabanac:2013aa} make the following remark: 
 \begin{quotation}
 \noindent
 \emph{``It is advised to counterbalance order effects during the bidding phase of peer review by promoting the submissions with fewer bids to
potential referees. This manipulation intends to better share bids out among submissions in order to attract qualified referees for all submissions.
''} 
\end{quotation}

With regard to goal (ii), the platform aims to display `well-matched' papers to each {\reviewer}. 
That is, the set of papers to be displayed is composed of papers on which the {\reviewer} is most likely to bid. 
There are several reasons to select well-matched papers.
It is generally assumed that reviewers are more likely to place bids on papers they are qualified to review~\citep{rodriguez:2007aa}. Furthermore, reviewers that place positive bids on papers are more likely to give a review with high confidence and voice sharp opinions of acceptance or rejection that help guide final decisions on papers~\citep{cabanac:2013aa}. A number of comprehensive surveys also indicate that a primary motivation of reviewers is the ability to help and contribute to the work of colleagues~\citep{mulligan2013peer, ware2008peer}. Failing to display relevant papers to reviewers can result in several unintended negative consequences. If irrelevant papers are shown early in the order to a reviewer, it may cause the {\reviewer} to opt-out and disengage with the system even if further down the list there was an option that they would have happily bid on. Similarly, a poorly selected ordering may result in significantly fewer bids from a {\reviewer}. 

Competing objectives of this form are not unique to peer review systems and they appear in a number of applications. A fitting example is an intermediary between distinct user groups that seeks to facilitate interactions and satisfy each party. For example, in online labor markets, the platform must ensure each job obtains a sufficient number of applicants and that workers are presented with tasks they are qualified enough for to be considered. Similarly, in online e-commerce marketplaces, as the platform decides how to show products to users, there is a definite trade-off between satisfying merchants offering products that need to be sold and users that want to be shown relevant items. 
In this paper, we maintain peer review as a running example and comment further on relevant applications of our work in a concluding discussion section.

In peer review, it is recognized that actively engaging reviewers in the paper assignment process via bidding can greatly improve the review process. If administered inadequately, bidding can in fact have a significant negative impact on the quality of the review process. 
 In the words of~\citet{rodriguez:2007aa}, 
 \begin{quotation}
 \noindent
 \emph{``Since bidding is the preliminary component of the manuscript-to-referee matching algorithm, sloppy bidding can have dramatic effects on which referees actually review which submissions.''}
 \end{quotation} 
A study on the 2016 Neural Information Processing Systems (NeurIPS) conference revealed the distribution of bids arising from a typical bidding process leaves significant challenges to match papers with reviewers~\citep{shah:2018aa}. It was observed that a considerable number of reviewers do not place a sufficient number of bids and papers commonly fail to obtain as many bids as the number of reviewers needed. 
This phenomenon is detailed in \cite{shah:2018aa} amongst the 3,200 reviewers and 2,400 papers.
 \begin{quotation}
 \noindent
 \emph{``Moreover, there are 148 reviewers with no (positive or negative) bids and 1201 reviewers with at most 2 positive bids... We thus observe that a large number of reviewers do not even provide positive bids amounting to the number of papers they would review. As a consequence of the low number of bids by reviewers, we are left with 278 papers with at most 2 positive bids and 816 papers with at most 5 positive bids... There is thus a significant fraction of papers with fewer positive bids than the number of requisite reviewers.''
}
 \end{quotation} 
The study also found that there were 1090 papers with no positive bids from the area chairs. 
The inability to elicit meaningful bidding information in NeurIPS is far from an aberration. 
In a study of the 2005 Joint Conference on Digital Libraries, 146 out of the 264 submissions did not obtain any bids~\citep{rodriguez:2007aa}. 
The shortfalls of existing bidding systems shift the onus of the reviewing assignments away from the participants and to the paper matching mechanisms.

Despite the importance of the bidding process in peer review, there is not yet much fundamental research on the problem of optimizing the display order during the bidding process, and much less so in consideration of the two objectives identified in this paper. 
In practice, the display order is typically determined via heuristics such as a fixed ordering (e.g., order of submission), or in decreasing order of the relevance of the papers to that reviewer, or in increasing order of the number of bids received by the paper until then.

A key reason that bidding can fail is that papers are suboptimally displayed to the {\reviewers}. Consider a paper that is not an ideal match for any {\reviewer} in the system. If papers are ranked for display simply by how well-matched they are to {\reviewers}, this particular paper may be shown far down in the ranking for each reviewer and hence, not receive many, if any, bids. The risk of this scenario is elevated for interdisciplinary research, 
which is know to face significant 
impediments as a consequence of the lack of ideally matched peers~\citep{travis1991new, porter1985peer}.

On the other hand, if papers are inversely ranked by the number of bids they have obtained, then papers with fewer bids are more likely to be shown higher on the list regardless of how well-matched they are to any particular {\reviewer}. This display order may cause {\reviewer} dissatisfaction, which in the worst case could result in zero bids. Similarly,  ordering heuristics that are based on a fixed baseline may lead to bias in the review process. Indeed,  in the report of a study of $42$ peer-reviewed conferences in Computer Science, it was observed that under a fixed ordering (based on the submission time), the number of bids on papers is heavily influenced by the order of submission times of the papers~\citep{cabanac:2013aa}. 
It was concluded that the later the paper is submitted, the fewer bids it will receive.

Given the flaws of existing peer review bidding systems, we study the important problem of selecting the ordering of papers to display to each arriving {\reviewer} in a principled manner.

\subsection{Our contributions}
 The key {contributions} of this paper are summarized as follows. 

\paragraph*{Problem identification and formulation (Section~\ref{sec:problem_setup}).} 
The bidding process is highly consequential, yet one of the most understudied components of the conference peer-review process. We identify a key source of unfairness and inefficiency in the bidding process, and develop principled methods to address it. 
A key challenge is suitably formalizing the peer review bidding process, for which to the best of our knowledge there are no prior formulations.  
We formulate an objective function that captures the competing goals of the platform while reflecting the underlying decision-making process of {\reviewers}. The framework developed in this paper to analyze the problem is an important step toward future improvements on bidding systems.  

\paragraph*{Algorithm design (Section~\ref{sec:algorithm}).} 
We present a sequential decision-making algorithm called $\super$ to address this problem. The algorithm takes as input the ``similarities'' between each reviewer-paper pair and the bids made by all past reviewers, and outputs the ordering of papers to show to any current reviewer.  

\paragraph*{Theoretical results (Section~\ref{sec:theoretical_properties}).} 
We show two sets of theoretical results. We first consider a notion of `local' performance: the performance with respect to a single reviewer. We prove that \super is locally optimal whereas all popular baselines are considerably suboptimal. Our second set of theoretical results are based on a community model, where we prove that \super is near-optimal (globally) and all popular baselines are considerably suboptimal.

\paragraph*{Experiments on real and synthetic data (Section~\ref{sec:simulations}).} 
We run extensive experiments using similarity scores from ICLR 2018 and on synthetic data. The experiments reveal that the \super algorithm outperforms all popular baselines. For instance, it consistently reduces the number of papers with fewer than requisite bids by 50-75\% while maintaining individual reviewer satisfaction. In addition, we see that \super is very robust to model mismatches and complexities of the real-world review process.  

~\\The code for the algorithm is available at  {\href{https://github.com/fiezt/Peer-Review-Bidding}{\tt github.com/fiezt/Peer-Review-Bidding}}.

\subsection{Related Work}

The paper ordering problem for the bidding process in peer review bears a strong resemblance to the learning to rank problem~\citep{singh2019policy,yadav2019fair,svore2011learning,momma2019multiobjective,aslanyan2019position,cao2007learning}.
 Typically, the goal of learning to rank is to learn an overall ranking of items via  supervised methods or by querying users, where the latter provides further information on the relative ranking of items.
 In peer review, the objective of finding a ranking most suitable for an arriving reviewer during the bidding process is analogous to learning to rank methods that consider the utility of rankings for users along with the impact on the items being ranked~\citep{singh2019policy,yadav2019fair}. Moreover, the bidding model considered in this work is motivated from that which is commonly adopted in learning to rank models~\citep{aslanyan2019position}. 

As formulated in this paper, the goal for the design of the bidding process in peer review is to optimize for multiple criteria reflecting the objectives of the reviewers and the papers, respectively. This is not unlike the methods of~\cite{singh2019policy} and~\cite{yadav2019fair}, which consider a fairness objective in combination with a ranking quality objective, or the multi-objective learning to rank problems studied by~\cite{svore2011learning} and~\cite{momma2019multiobjective}. In the works of~\citet{singh2019policy} and~\cite{yadav2019fair}, the objective of ensuring fairness is encoded as a constraint in the optimization problems. Similarly,~\cite{svore2011learning} optimize a linear combination of ranking measures referred to as a `graded measure' and~\cite{momma2019multiobjective} convert a constrained optimization problem into an unconstrained problem by penalizing constraint violations in the objective. In each of the aforementioned works, the ranking measures are separable in the arriving users, meaning that the contribution of any individual user to the overall objective is independent of the other users. 

The problem of paper ordering in peer review given multiple objectives is also abstractly similar to online recommendation systems similarly facing competing objectives~\citep{rodriguez2012multiple,agarwal2011click,jambor2010opt}. However, a prevailing approach is to convert the multi-objective problem to a constrained optimization problem~\citep{rodriguez2012multiple,jambor2010opt}. Both the approach of incorporating objectives as constraints in the optimization problem formulations and combining objectives in a linear fashion is considered by~\citet{agarwal2011click}. Analogous to the learning to rank problem, the objectives are separable in the users.

The objective in the peer review problem as formulated in this paper presents unique challenges not addressed in the aforementioned works on learning to rank and recommendation systems. Notably, it is not separable between the reviewers since it depends on the number of bids on each paper after each reviewer has arrived and placed bids on the papers. 
Being applicable to more general multi-criteria settings, our approach to the design of the bidding processes in peer review may also be applied to the learning to rank problem. This is a direction worthy of further study.

Our work also contributes to a growing literature on improving various aspects of the peer review process
 such as reviewer assignment~\citep{charlin2013toronto,garg:2010aa,lian2018conference, stelmakh:2018aa,kobren19localfairness}, biases~\citep{tomkins2017reviewer,stelmakh2019testing}, subjectivity~\citep{noothigattu2018choosing}, miscalibration~\citep{roos2012statistical,wang2018your}, strategic behavior~\citep{aziz2019strategyproof,xu2018strategyproof}, and others~\citep{church2005reviewing,wing2011hypercriticality,nips14experiment,shah2017design,kang2018dataset,jecmen2020manipulation,ding2020privacy,stelmakh2020catch,stelmakh2020resubmissions}. The present paper addresses the bidding process in conference peer review, which has largely been unexplored in past literature. The concurrent work of~\cite{meirmarket}, which appeared after an initial  workshop version of our work~\citep{fiez2019super}, is the only work besides our own that we are aware of to focus on methods for improving bidding in peer review. However, their approach is to design a market for bidding, which is entirely different from ours.

\section{Problem Formulation}
\label{sec:problem_setup}
Consider $\numpapers \geq 2$ papers and $\numrev\geq 2$ reviewers indexed as $\{1,\ldots,\numpapers\}$ and $\{1,\ldots,\numrev\}$ respectively.\footnote{Henceforth, for any positive integer $\kappa$, we will use the standard shorthand $[\kappa]$ to denote the set $\{1,\ldots,\kappa\}$.} 
For each reviewer-paper pair, we have access to a \emph{similarity score} that captures the similarity between the reviewer and the paper.
We use the notation $\similarity_{\revindex, \paperindex} \in [0,1]$ to denote the given similarity between any reviewer $\revindex \in [\numrev]$ and paper $\paperindex \in [\numpapers]$.
A higher similarity score indicates a greater relevance of the paper to that reviewer. There are several systems in use today that compute similarities~\citep{price2010subsift, charlin2013toronto}, and in our work, we  treat them as being given.

In the bidding period, reviewers sequentially arrive into the system and place bids on the papers. In our work, for any reviewer and paper, we only consider the existence of a bid or not, and do not consider the possibility of multiple bidding options. 
We assume for simplicity that all $\numrev$ reviewers arrive exactly once, and that a reviewer arrives after the previous reviewer has completed their bidding.\footnote{However, in Section~\ref{sec:iclr_simulations},
we show that our algorithm is empirically robust to violations of these assumptions. } We do not make any assumptions on the arrival order of the reviewers. 
The problem is to determine the ordering of papers to show each reviewer on arrival in the interest of influencing the papers they decide to bid on while ensuring individual satisfaction. When deciding the paper ordering for any reviewer, the bids made by all reviewers who arrived in the past along with the paper orderings presented to them are known, but the bids made by the current or future reviewers are unknown. Let $\symgroup_\numpapers$ denote the set of all possible $\numpapers!$ permutations of the $\numpapers$ papers. In what follows, for any reviewer $\revindex \in \reviewerset$, we let $\ordering_\revindex \in \symgroup_\numpapers$  denote the ordering (permutation) of the papers shown to reviewer $\revindex$. We also use the notation $\ordering_\revindex(\paperindex)$ to denote the position of paper $\paperindex \in \paperset$ in the ordering $\ordering_\revindex$. 

\paragraph*{Gain function (objective).} 
Any algorithm to determine the ordering of papers must trade-off between two competing objectives: ensuring each paper receives a sufficient number of bids and ensuring each reviewer gets to see relevant papers early in the ordering. A combination of the objectives comprise our ``gain function,'' which is the objective we aim to optimize. We begin by discussing each objective component.

\textbf{Paper-side gain:} The paper-side gain is associated with a 
given function $\gainfunction:\reals_{\geq0}\rightarrow \reals_{\geq0}$. 
At the end of the entire bidding process, the paper-side gain $\gain_p$ is
\begin{equation*}
\gain_{p} = \sum_{\paperindex \in \paperset} \gainfunction(\numbids_{\paperindex}),
\end{equation*}
where $\numbids_{\paperindex}$ is the number of bids received by paper $\paperindex \in \paperset$.
We assume the function $\gainfunction$ is non-decreasing and concave.
The non-decreasing property represents an improved gain
if there are more bids, and the concavity property captures
diminishing returns.\footnote{Our algorithm easily adapts to paper-side gains that may also be a function of the similarity scores of the reviewers who bid; for example, a higher gain for bids from expert reviewers. We omit this detail for sake of brevity.}
An example of a choice for the paper-side gain is the square-root function $\gainfunction(\optvar) = \sqrt{\optvar}$. This function is increasing, smooth, and captures the diminishing returns property. The reader may keep this function in mind as a running example for concreteness. A second example is $\gainfunction(\optvar) = \min\{\optvar, \gainminrevs\}$ for a given parameter $\gainminrevs \geq 1$, which emphasizes having at least $\gainminrevs$ bids per paper. 

\textbf{Reviewer-side gain:} 
This objective captures the desideratum that the reviewers should be shown papers with high relevance early in the paper ordering. 
The reviewer-side gain is associated with some predetermined function $\gainfunctionrev: [\numpapers] \times [0,1] \rightarrow \reals_{\geq0}$. Given this function, the reviewer-side gain $\gain_r$ is defined as:
\begin{equation*}
\gain_{r} =  \sum_{\revindex \in \reviewerset} \sum_{\paperindex \in [\numpapers]}  \gainfunctionrev(\ordering_{\revindex}(\paperindex),\similarity_{\revindex,\paperindex}). 
\end{equation*}
The function $\gainfunctionrev$ is assumed to be non-increasing in the position (its first argument) and non-decreasing in the similarity (its second argument). One example choice of this function, which the reader may choose to keep in mind as a running example, is the  Discounted Cumulative Gain or DCG used commonly in data mining~\citep{jarvelin2000ir}. In our setting, the function is given by
\begin{equation}
\gainfunctionrev(\ordering_{\revindex}(\paperindex),\similarity_{\revindex,\paperindex}) =  \frac{2^{\similarity_{\revindex, \paperindex}}-1}{\log_2(\ordering_{\revindex}(\paperindex)+1)},
\label{eq:dcg}
\end{equation}
where we have set the ``relevance'' parameter in DCG to be the similarity $\similarity_{\revindex, \paperindex}$.

\textbf{Overall gain function:} 
Finally, we assume there is a trade-off parameter $\hyperparam \geq 0$, chosen by the program chairs, which trades off between these two objectives so that the overall gain function is given by
\begin{equation}
\gain = \gain_{p} + \hyperparam \gain_{r}.
\label{EqnDefnOverallgain}
\end{equation}
The goal is to determine the orderings of papers to show each reviewer to maximize the expected overall gain, $\Exs \big[ \gain \big]$, 
where the expectation is taken over the randomness in the bids made by the reviewers (see reviewer bidding model below) and any randomness in the algorithm.\footnote{For the pedantic reader, a (deterministic) algorithm is a mapping from $[0,1]^{\numrev \times \numpapers} \times {([\numrev]\times 2^{[\numrev]} \times \{0,\ldots,\numrev\}^\numpapers)}^\numrev$ to  $(\symgroup_\numpapers)^\numrev$. In this representation, the first input argument is the similarity matrix. The second argument represents, for each of the arriving reviewers, the identity of the current reviewer, the identities of the past reviewers, and the number of bids so far for each paper. The output space is simply an ordering of the $\numpapers$ papers for each reviewer. A stochastic algorithm outputs a probability distribution over the output space.} 

\paragraph*{Reviewer bidding model.} 
An important aspect of any system that displays a list to users is the presence of primacy effects. In the context of our problem, the primacy effect means a reviewer is more likely to bid on a paper shown at the top of the list rather than later~\citep{murphy2006primacy}. A second aspect of bidding is that a reviewer is more likely to bid on papers with greater similarity, although the reviewer may not bid on exactly the papers with the highest similarity since the similarities are noisy representations of their reviewing interests. 

Thus in order to model reviewer bidding, we revert to literature on position-based click models that have a nearly identical setting (where clicks are analogous to our bids). We model the bidding via a given function $\bidfunction : [\numpapers]\times [0, 1] \rightarrow [0,1]$, where $\bidfunction(\ordering_{\revindex}(\paperindex), \similarity_{\revindex, \paperindex})$ is non-increasing in the position that a paper is shown (the first argument) and non-decreasing in the similarity score (the second argument).  Any reviewer $\revindex \in \reviewerset$ bids on paper $\paperindex \in \paperset$ independently with probability
\begin{equation*}
\bidprob_{\revindex, \paperindex} = \bidfunction(\ordering_{\revindex}(\paperindex), \similarity_{\revindex, \paperindex}).
\end{equation*}
As a running example throughout the paper, note that in position-based click models, the click probability decomposes into a product of relevance and position bias~\citep{chuklin2015click}. Moreover, the literature considers the click probability to decay logarithmically 
as a function of the position~\citep{aslanyan2019position}. 
The translation of these models into our setting gives rise to the example bidding function  
\begin{equation}
\bidfunction(\ordering_{\revindex}(\paperindex), \similarity_{\revindex, \paperindex}) = \frac{\similarity_{\revindex, \paperindex}}{\log_2(\ordering_{\revindex}(\paperindex)+1)}.
\label{eq:bidding_func_example}
\end{equation} 

\paragraph*{Baselines.} 
We consider the following three methods of ordering papers as baselines.

{\bf Random baseline}~(\randbase): A commonplace practice~\citep{cabanac:2013aa} is to show papers to reviewers in some fixed order, such as in order of submission of the papers. As a baseline, we consider a better variant of this practice, in which each reviewer is shown an independently and randomly selected paper ordering.

{\bf Similarity baseline}~(\simbase): A second common practice, followed in several conference management systems today, is to order the papers according to their similarities. In other words, any reviewer $\revindex \in \reviewerset$ is shown the papers in order of the values in $\{\similarity_{\revindex,\paperindex}\}_{\paperindex\in \paperset}$ (where the paper with maximum similarity is shown at the top, and so on). Any ties are broken by showing papers with fewer bids higher, and further ties are broken uniformly at random.

{\bf Bid baseline}~(\bidbase): A third baseline shows papers to greedily optimize the minimum bid count. Each reviewer is shown papers in increasing order of the number of bids received so far (from the reviewers who arrived previously). Any ties are broken in favor of the paper with a higher similarity, and further ties are broken uniformly at random.

\section{Algorithm}
\label{sec:algorithm}
The key challenge in designing a suitable algorithm for the problem at hand stems from the fact that the paper-side gain is coupled (non-separable) across the orderings of papers presented to all reviewers so the impact of each individual paper ordering cannot be fully realized until the entire bidding process is complete. 
Conversely, the reviewer-side gain is decoupled (separable) across reviewers. This means the reviewer-side gain that can be obtained from any given reviewer is independent of the ordering of papers presented to any other reviewer. Thus, an algorithm for this problem is required to make local decisions, where the effect of the decision on the global gain (or cost) is only partially known. This perspective is reminiscent of the classical A$^{\ast}$ algorithm~\citep{hart1968formal}, and using A$^{\ast}$ as an inspiration, we now present an algorithm which we call \super for our problem\footnote{The name \super stands for SUperior PERmutations and also indicates the inspiration from A$^{\ast}$.}.

The A$^{\ast}$ algorithm operates with a goal of finding the minimum cost path between a pair of vertices in a cost-weighted graph. For any node in consideration, it considers two functions: a function which captures the cost so far and a second function---called the ``heuristic''---which captures some estimate of the cost from the current node to the destination. The A$^{\ast}$ algorithm then finds a path based on these two functions. 
Before moving to a description of \super, we discuss such a heuristic in the context of the problem at hand.

\subsection{Heuristic for Future Bids} 
In a manner analogous to the A$^{\ast}$ algorithm, at any point in time \super keeps track of the gains so far and also takes as input a heuristic that captures the ``unseen'' events. The heuristic in A$^{\ast}$ provides, for every vertex in the given graph, an estimate of the cost incurred in the future. Analogously, the heuristic in \super provides, for every arrival of a reviewer, an estimate of the number of bids each paper will receive in the future. Formally, let us index the reviewers as $\revindex\in \reviewerset$ in the order of arrival (note that this order is unknown a priori). The heuristic comprises a collection of vectors $\{\heuristicf_1, \ldots, \heuristicf_\numrev\}$, where each $\heuristicf_\revindex \in [0,\numrev-\revindex]^\numpapers$ represents an estimate of the number of bids each of the $\numpapers$ papers will receive from all future reviewers $\{\revindex+1,\ldots,\numrev\}$. The vector $\heuristicf_\revindex$ is provided to the \super algorithm on arrival of the $\revindex^{th}$ reviewer.  
Two examples of heuristic functions that we consider in the subsequent narrative are described as follows.
\begin{itemize}[itemsep=5pt, topsep=5pt]
    \item {\em Zero heuristic:} $\heuristicf_\revindex = 0$ for every $\revindex \in [\numrev]$.
    \item {\em Mean heuristic:} This function computes the expected number of bids each paper will receive if the permutations shown to all future reviewers are chosen independently and uniformly at random. Formally: $\heuristicf_{\revindex, \paperindex} =  \frac{1}{\numpapers} \sum_{\revindex' = \revindex+1}^{\numrev} \sum_{\paperindex' \in [\numpapers]} \bidfunction(\paperindex', \similarity_{\revindex', \paperindex})$  $\forall\ \revindex \in [\numrev-1], \paperindex \in \paperset$.
\end{itemize}
We set $\heuristicf_\numrev = 0$ for any heuristic, implying there are no bids placed after the last reviewer. This is analogous to setting the heuristic value to zero for the target vertex in the A$^{\ast}$ algorithm.

\subsection{Intuition Behind the Algorithm} 
 We first provide some intuition about the \super algorithm, and subsequently present a formal description. 
Since a primary impediment to designing an algorithm is the inability to fully realize the impact of a paper ordering on the paper-side gain until the end of the bidding process, we begin by considering the scenario where $(\numrev-1)$ reviewers have already departed, and the problem is to determine the ordering of papers to show the final reviewer. In this scenario, we have access to the bids of all $(\numrev-1)$ reviewers that have already arrived and the orderings of papers presented to them. We use the notation $\numbids_{\numrev-1, \paperindex} \in \{0,\ldots,\numrev-1\}$ to denote the number of bids received by any paper $\paperindex \in \paperset$ at the time of arrival of the last reviewer. The values $\{\numbids_{\numrev-1,1},\ldots,\numbids_{\numrev-1, \numpapers}\}$ are thus known at the time when the final reviewer arrives. As a result, we can formulate an optimization problem for the final reviewer $\numrev$ to maximize the expected gain from~\eqref{EqnDefnOverallgain} in the following manner. For every $\paperindex \in \paperset$, let $\randombid_{\numrev, \paperindex}$ denote a Bernoulli random variable with mean $\bidprob_{\revindex, \paperindex} = \bidfunction(\ordering_{\numrev}(\paperindex), \similarity_{\numrev, \paperindex})$, independent of all else. The random variable $\randombid_{\numrev, \paperindex}$ represents the bid of the final reviewer on paper $\paperindex \in\paperset$. The optimization problem can be written as
\begin{equation}
 \max_{\ordering_{\numrev}\in \symgroup_\numpapers}  \sum_{\paperindex\in \paperset} \mathbb{E}[\gainfunction(\numbids_{\numrev-1, \paperindex} + \randombid_{\numrev,\paperindex})]  + \hyperparam \sum_{\paperindex\in \paperset}\gainfunctionrev(\ordering_{\numrev}(\paperindex), \similarity_{\numrev, \paperindex}) ,
 \label{eq:subproblem}
 \end{equation}
where the expectation is taken over the distribution of the random variables $\randombid_{\numrev, 1},\ldots,\randombid_{\numrev, \numpapers}$.

Observe that the constraint set for the optimization problem in~\eqref{eq:subproblem} is the set $\symgroup_\numpapers$ of all permutations. This set is, in general, not very well behaved~\citep{ailon2008aggregating,shah2016stochastically}, which makes even this one-step optimization a challenge. As we discuss later in the formal algorithm description along with Theorem~\ref{prop:local} and its proof, \super for the final reviewer optimally solves~\eqref{eq:subproblem} and it is computationally efficient manner (see Proposition~\ref{prop:time} in Appendix~\ref{sec:proof_time}). 
The aforementioned subproblem forms the starting point for the \super algorithm. Now that we know to handle a single (last) reviewer in an optimal fashion, we now describe the \super algorithm for a general reviewer, say, $\revindex \in \reviewerset$. When reviewer $\revindex$ arrives, we have access to the number of bids made by all past reviewers on any paper $\paperindex \in [\numpapers]$, which we denote by $\numbids_{\revindex-1, \paperindex} \in \{0,\ldots,\revindex-1\}$.

 \begin{figure}[t!]
\begin{minipage}{\textwidth}\refstepcounter{alg}\label{alg:super}
 \vspace{0pt}%
 \centering
\fbox{
\begin{minipage}{.95\textwidth}
\textbf{Algorithm~\ref{alg:super}}: $\super$
\hrule
\vspace{2mm}
\textbf{Input:} $\gainfunction: \reals_{\geq0} \rightarrow \reals_{\geq0}$,  paper-side gain function  \\
\hspace*{.43in}$\gainfunctionrev: \paperset \times [0,1] \rightarrow \reals_{\geq0}$,  reviewer-side gain function  \\
\hspace*{.43in} $\bidfunction: \paperset \times [0,1] \rightarrow [0, 1]$, bidding model  \\
\hspace*{.43in} $\lambda \geq 0$, trade-off parameter \\
\hspace*{.43in} $\similarity \in [0, 1]^{\numrev\times\numpapers}$, similarity matrix. \\ 
\textbf{Algorithm:} 
\begin{enumerate}[topsep=6pt,itemsep=0ex, align=left, leftmargin=0pt, labelindent=.05in, labelwidth=0pt, itemindent=!]
\item Initialize bids on each paper to zero: $\numbids_{0}\gets 0_{\numpapers}$
\item For each reviewer arrival $\revindex \in \reviewerset$
\begin{enumerate}[topsep=0pt,itemsep=0ex, align=left, leftmargin=0pt, labelindent=.1in, labelwidth=0pt, itemindent=!]
\item Compute or input heuristic $\heuristicf_\revindex \in [0,\numrev-\revindex]^\numpapers$
\item $\ordering_{\revindex}\gets$ \fpo
\item  Present papers in the order $\ordering_{\revindex}$ and observe bids $\realbid_{\revindex} \in \{0, 1\}^{\numpapers}$ 
\item Update paper bid counts: $\numbids_{\revindex}= \numbids_{\revindex-1}+\realbid_{\revindex}$
\end{enumerate}
\end{enumerate}
\end{minipage}}
\end{minipage}
\
\vspace{8mm}

\begin{minipage}{\textwidth}\refstepcounter{alg}\label{alg:subprocedure}
\vspace{0pt}%
\centering
\fbox{
\begin{minipage}{.95\textwidth}
\textbf{Algorithm~\ref{alg:subprocedure}}: \fpo
\hrule
\vspace{.5mm}
\begin{enumerate}[topsep=6pt,itemsep=0ex, align=left, leftmargin=0pt, labelindent=.05in, labelwidth=0pt, itemindent=!]
\item Compute weight matrix $\weight \in \reals^{\numpapers \times \numpapers}$ such that 
\begin{align*}
\weight_{\paperindex, \extravar} = \hyperparam \gainfunctionrev(\extravar, \similarity_{\revindex, \paperindex}) + \bidfunction(\extravar, \similarity_{\revindex, \paperindex})(\gainfunction(\numbids_{\revindex-1, \paperindex}+\proxy_{\revindex, \paperindex} + 1) - \gainfunction(\numbids_{\revindex-1, \paperindex}+\proxy_{\revindex, \paperindex})) \ \forall \ \paperindex \in \paperset, \extravar \in \paperset
\end{align*}
\item Solve linear program to obtain $\optvar^{\ast} \in \reals^{\numpapers \times \numpapers}$:
\begin{align*}
&\optvar^{\ast} \in \argmax_{\optvar\in [0, 1]^{\numpapers \times \numpapers}} \quad  \sum_{\paperindex\in \paperset}\sum_{\extravar\in \paperset} \weight_{\paperindex,\extravar}\optvar_{\paperindex,\extravar}\\
&\text{s.t.}  \sum_{\extravar \in \paperset} \optvar_{\paperindex,\extravar} = 1 \ \forall \ \paperindex \in \paperset, \ \sum_{\paperindex \in \paperset} \optvar_{\paperindex,\extravar} = 1 \ \forall \ \extravar \in \paperset
\end{align*}
\qquad with ties broken arbitrarily between the set of maximizing solutions
\item $\ordering_{\revindex}(\paperindex) = \extravar$ such that $\optvar_{\paperindex, \extravar}^{\ast} = 1$ for each $\paperindex \in \paperset$
\end{enumerate}
\begin{minipage}{3.5in} \
\textbf{Output:} $\ordering_{\revindex}$
\end{minipage}
\end{minipage}}
\end{minipage}
\
\vspace{8mm}

\begin{minipage}{\textwidth}\refstepcounter{alg}\label{alg:subprocedure_efficient}
\vspace{0pt}%
\centering
\fbox{
\begin{minipage}{.95\textwidth}
\textbf{Algorithm~\ref{alg:subprocedure_efficient}}: \fpoe
\hrule
\vspace{.5mm}
\begin{enumerate}[topsep=6pt,itemsep=0ex, align=left, leftmargin=0pt, labelindent=.05in, labelwidth=0pt, itemindent=!]
\item Compute weights for each $\paperindex \in \paperset$:
\begin{align*}
\sortvar_{\revindex, \paperindex} =   \bidfunction^{\similarity}(\similarity_{\revindex, \paperindex})(\gainfunction(\numbids_{\revindex-1, \paperindex} +\proxy_{\revindex, \paperindex}+ 1) - \gainfunction(\numbids_{\revindex-1, \paperindex}+\proxy_{\revindex, \paperindex}))+\hyperparam \gainfunctionrev^{\similarity}(\similarity_{\revindex, \paperindex})
\end{align*}
\item $\ordering_{\revindex} = \rankfunction(\sortvar_{\revindex})$, where $\rankfunction: \reals^{\numpapers}\rightarrow \paperset^{\numpapers}$ returns the rank from maximum to minimum of each input in place and breaks ties arbitrarily.
\end{enumerate}
\begin{minipage}{3.5in} \
\textbf{Output:} $\ordering_{\revindex}$
\end{minipage}
\end{minipage}}
\end{minipage}
\end{figure}

We now recall the A$^{\ast}$ algorithm: for any vertex, A$^{\ast}$ considers the cost ``$g$'' so far and a heuristic estimate ``$h$'' of the subsequent cost.
Then, considering the cost of any vertex as ``$g+h$'', the A$^{\ast}$ algorithm takes the one-step optimal action given by selecting the neighboring vertex with the smallest value of ``$g+h$''. In an analogous fashion, \super considers the number of bids so far ($\numbids_{\revindex-1}$) and takes as input a heuristic ($\heuristicf_\revindex$) for the number of bids in the future. Then, considering the number of bids from all other reviewers as ``$\numbids_{\revindex-1}+\heuristicf_\revindex$'', the \super algorithm  takes the action which is the one-step optimal action. In other words, \super solves for each paper ordering using:
\begin{equation}
\max_{\ordering_{\revindex} \in \symgroup_\numpapers}\quad \sum_{\paperindex\in \paperset} \mathbb{E}[\gainfunction(\numbids_{\revindex-1, \paperindex}+\heuristicf_{\revindex, \paperindex} + \randombid_{\revindex,\paperindex})] \\
+ \hyperparam \sum_{\paperindex\in \paperset}\gainfunctionrev(\ordering_{\revindex}(\paperindex), \similarity_{\revindex, \paperindex}) ,
\label{eq:subproblem_genreviewer}
\end{equation}
where $\randombid_{\revindex, \paperindex}$ is a Bernoulli random variable with mean $\bidprob_{\revindex, \paperindex} = \bidfunction(\ordering_{\revindex}(\paperindex), \similarity_{\revindex, \paperindex})$ and is independent of all else. As for the final reviewer, \super solves this problem in an efficient manner for any arbitrary reviewer (see Proposition~\ref{prop:time} in Appendix~\ref{sec:proof_time}).

\subsection{Formal Algorithm Description}
\label{sec:formal_alg}
The $\super$ algorithm is presented in Algorithm~\ref{alg:super}. To determine a paper ordering to show any reviewer, \super calls a procedure to efficiently solve~\eqref{eq:subproblem_genreviewer}. We give a general method in Algorithm~\ref{alg:subprocedure} and a faster method in Algorithm~\ref{alg:subprocedure_efficient} that is applicable for a special class of reviewer-side gain and bidding functions.

\paragraph*{General version.} In the general version of \super, 
Algorithm~\ref{alg:subprocedure} is called to return a paper ordering 
that is a solution to~\eqref{eq:subproblem_genreviewer}
each time a {\reviewer} arrives. In the proof of Theorem~\ref{prop:local}, we show that the optimization problem over the set of permutations given in~\eqref{eq:subproblem} to find the optimal paper ordering for the final reviewer can be reformulated as an integer linear programming problem with a totally unimodular constraint set. The totally unimodular property of the constraint set guarantees that the solution of a relaxed linear program is in fact the integer optimal solution. The application of this reduction from an optimization problem over permutations to a linear programming problem for any given reviewer forms the technique given in Algorithm~\ref{alg:subprocedure} to efficiently obtain a solution to~\eqref{eq:subproblem_genreviewer}.
Finally, the per-reviewer time complexity of the general version of \super given the evaluations of the heuristic is $\mathcal{O}(\numpapers^3)$ (see Proposition~\ref{prop:time} in Appendix~\ref{sec:proof_time}) as a consequence of the call to solve a linear assignment problem in Algorithm~\ref{alg:subprocedure}.

\paragraph*{Faster specialized version.} 
Given a bidding model that can be decomposed as $\bidfunction(\ordering_{\revindex}(\paperindex), \similarity_{\revindex, \paperindex}) = \bidfunction^{\similarity}(\similarity_{\revindex, \paperindex})\bidfunction^{\ordering}(\ordering_{\revindex}(\paperindex))$ where $\bidfunction^{\similarity}: [0, 1] \rightarrow [0, 1]$ is non-decreasing and $\bidfunction^{\ordering}: \paperset \rightarrow [0, 1]$ is non-increasing, along with a reviewer-side gain function that can be decomposed as $\gainfunctionrev(\ordering_{\revindex}(\paperindex), \similarity_{\revindex, \paperindex}) = \gainfunctionrev^{\similarity}(\similarity_{\revindex, \paperindex})\bidfunction^{\ordering}(\ordering_{\revindex}(\paperindex))$ where $\gainfunctionrev^{\similarity}: [0, 1]\rightarrow \reals_{\geq 0}$ is non-decreasing, \super calls Algorithm~\ref{alg:subprocedure_efficient} to return a paper ordering 
that is a solution to~\eqref{eq:subproblem_genreviewer}
each time a {\reviewer} arrives.
In the proof of Proposition~\ref{prop:time} in Appendix~\ref{sec:proof_time}, we show for this model class that the problem from~\eqref{eq:subproblem} to find the optimal paper ordering for the final reviewer after evaluating the expectation can be reformulated as
\begin{equation}
\max_{\ordering_{\numrev}\in \symgroup_\numpapers} \quad \sum_{\paperindex\in \paperset} \sortvar_{\numrev, \paperindex}\bidfunction^{\ordering}(\ordering_{\numrev}(\paperindex)) 
\label{eq:super_efficient_main}
\end{equation}
for some non-negative weights $\{\sortvar_{\numrev, \paperindex}\}_{\paperindex\in \paperset}$.
The problem in~\eqref{eq:super_efficient_main} admits a simple solution: $\bidfunction^{\ordering}$ is non-increasing on the domain, so the objective is maximized by presenting papers in decreasing order of the weights $\{\sortvar_{\numrev, \paperindex}\}_{\paperindex\in \paperset}$. 
Obtaining this solution only requires sorting the weights, which has a time complexity of $\mathcal{O}(\numpapers\log(\numpapers))$. The application of this problem reformulation for the given model class and any reviewer forms the technique given in Algorithm~\ref{alg:subprocedure_efficient} to obtain a solution to~\eqref{eq:subproblem_genreviewer}.

Before moving on to present our theoretical results, we comment on the relevance of this model class. Importantly, the DCG reviewer-side gain function and bidding model $\bidfunction(\similarity_{\revindex,\paperindex}, \ordering_{\revindex}(\paperindex)) = \similarity_{\revindex, \paperindex}/\log_2(\ordering_{\revindex}(\paperindex)+1)$, which we have mentioned as running examples that can be kept in mind, satisfy the decomposition for which \super is computationally efficient. This choice of functions is standard in the past literature on ranking models and click behavior~\citep{jarvelin2000ir,aslanyan2019position}, meaning that the time complexity result for this model class is quite relevant.

\section{Theoretical Results}
\label{sec:theoretical_properties}
We now present the main theoretical results of this paper. 
Complete proofs of all results are in Appendix~\ref{sec:proofs}. 

\subsection{Local Optimality}
The property of local optimality, as the name suggests, means that the algorithm is optimal with respect to the reviewer under consideration. Achieving even a good local performance in a computationally efficient manner is challenging due to the optimization over permutations in~\eqref{eq:subproblem}. The following results show that \super, which is computationally efficient, is locally optimal. 

The result is first presented in terms of the final reviewer for simplicity and extended to a general reviewer subsequently. In the following theorem, since we consider only the final reviewer, note that the heuristic for \super is irrelevant because the heuristic value for the final reviewer is always set to zero.

\begin{theorem}\label{prop:local}
Given any history of paper orderings and bids from reviewers that arrived previously, the paper ordering given by \super to the final reviewer maximizes the expected gain conditioned on the history.
\end{theorem}
In other words, the expected amount by which the gain is increased from the final reviewer is maximized. To generalize the previous result to a local optimality result for any reviewer, let the immediate gain from a reviewer be defined as the difference between the gain after and before the reviewer arrived. 
\begin{corollary}\label{cor:local}
Given any history of paper orderings and bids from reviewers that arrived previously, the paper ordering given to any reviewer by \super with zero heuristic maximizes the expected immediate gain 
from that reviewer conditioned on the history.
\end{corollary}

The property of local optimality also implies optimality of \super (with any heuristic) when the paper-side gain function is linear.  
We refer the reader to  Appendix~\ref{sec:proof_linear} for more details. 

We now show that an analogous statement cannot be made regarding the other baseline methods. In fact, in contrast to \super, all the popular baselines are considerably suboptimal.

\begin{theorem}\label{prop:localworst}
Consider a model with the paper-side gain function $\gainfunction(\numbids_{\paperindex})=\sqrt{\numbids_{\paperindex}}$, the reviewer-side gain function $\gainfunctionrev(\ordering_{\revindex}(\paperindex), \similarity_{\revindex, \paperindex}) = (2^{\similarity_{\revindex, \paperindex}}-1)/\log_2(\ordering_{\revindex}(\paperindex) + 1)$, and the bidding function $\bidfunction(\ordering_{\revindex}(\paperindex), \similarity_{\revindex, \paperindex})=\similarity_{\revindex, \paperindex}/\log_2(\ordering_{\revindex}(\paperindex)+1)$. There exists a constant $\constant >0$ such that for every $\numpapers\geq 2$ and $\hyperparam\geq 0$, in the worst case for the final reviewer:\\
\indent (a) \simbase is suboptimal by an additive factor of at least $\constant\numpapers/\log_2^2(\numpapers)$;\\
\indent (b) \bidbase is suboptimal by an additive factor of at least $\constant\numpapers\max\{1, \hyperparam\}/\log_2^2(\numpapers)$;\\
\indent (c) \randbase is suboptimal by an additive factor of at least $\constant\numpapers\max\{1, \hyperparam\}/\log_2^2(\numpapers)$.
\end{theorem}
Theorems~\ref{prop:local} and~\ref{prop:localworst} in tandem show that \super not only is locally optimal but can outperform currently popular algorithms by a wide margin.

\subsection{Global Optimality Under a Community Model}
\label{sec:global_opt_model}
We now transition to consider the global performance of the algorithms. Given our focus on the application of peer review, we are motivated to give guarantees on the performance of \super for similarity matrix classes that would be encountered in a real conference. 

A common characteristic of networks is community structure~\citep{newman2004finding,
porter2009communities}, where nodes can be grouped into clusters and links between groups are not as common. This phenomena has been documented in social and biological networks among others~\citep{girvan2002community}. Pertinent to this work, empirical investigations have revealed community structures in scientific collaboration networks~\citep{newman2001structure}. Given this close connection, and the fact that scientific research is highly specialized, it is intuitive that communities exist in major conferences pertaining to different subtopics.

We explore the possible existence of such structure in the ICLR 2018 similarity matrix that was reconstructed by~\citet{xu2018strategyproof} and is of size $\numrev=2435$ and $\numpapers=935$. Recall that the ICLR similarity matrix is of size $(2435 \times 935)$. To begin, we investigate the spectral properties of the similarity matrix from ICLR 2018, and in particular, whether it is low rank. We plot the singular values of the similarity matrix in Figure~\ref{fig:iclr_analysis_b}, where the (heuristic) elbow method suggests a low rank ($\approx 10$). In Figure~\ref{fig:iclr_analysis_c} we plot the entries of the similarity matrix after permuting its rows and columns according to the spectral co-clustering algorithm~\citep{dhillon2001co}.
The result suggests the ICLR 2018 similarity matrix exhibits some characteristics of a noisy block diagonal structure.
\begin{figure}[ht]
    \centering
    \subfloat[][]{\includegraphics[width=.35\linewidth]{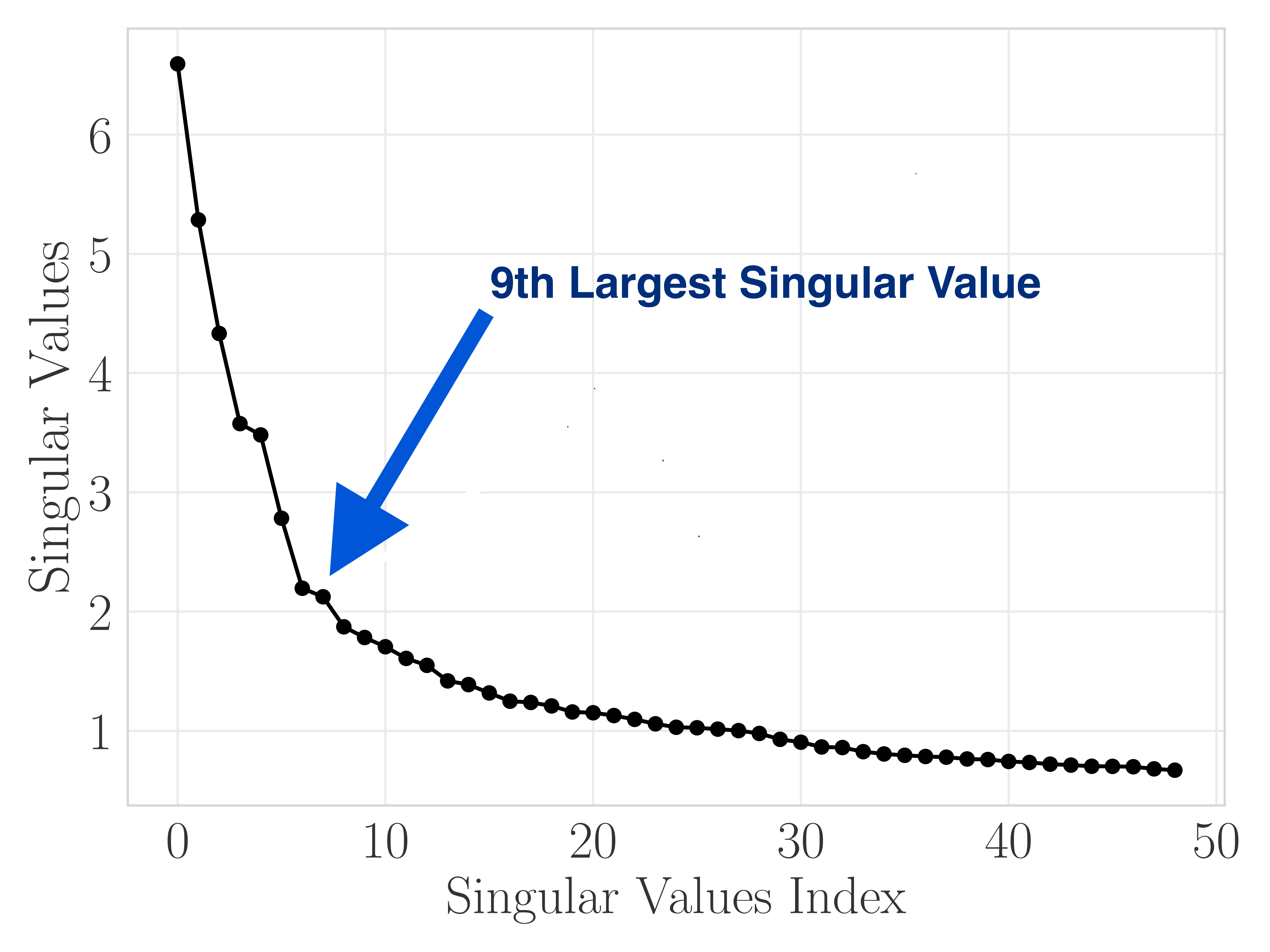}\label{fig:iclr_analysis_b}}\hfill
    \subfloat[][]{\includegraphics[width=.6\linewidth]{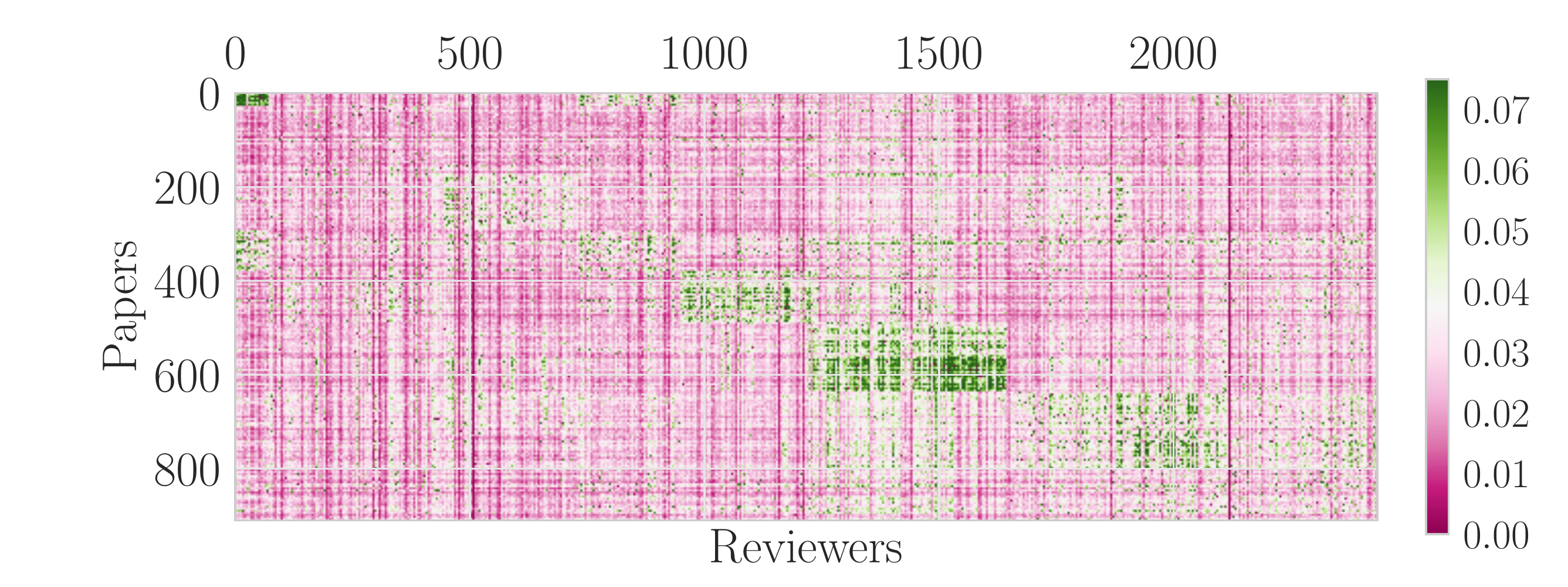}\label{fig:iclr_analysis_c}}\hfill
    \caption{\small (a) The 50 singular values (excluding the maximum singular value) of ICLR 2018 similarity matrix, which shows low-rank structure. (b) Similarity scores of the permuted ICLR matrix as a heatmap indicating the block diagonal structure.}
    \label{fig:iclr_analysis}
\end{figure}

In what follows, we now perform an associated theoretical analysis of the algorithms under such community structures of the similarity matrix. We begin by proposing a simple model which we call the `noiseless community model'.

\paragraph*{Noiseless community model.}
Informally, the noiseless community model we study is a set of similarity matrices that up to a permutation of rows and columns belong to a subclass of block diagonal matrices. 
Formally, let $\mathbf{0}_{\blocksize\times \blocksize}$ and $\mathbf{1}_{\blocksize\times \blocksize}$ denote $\blocksize \times \blocksize$ matrices of all zeros and  all ones respectively. Define an $\numblocks\blocksize \times \numblocks\blocksize$ block diagonal matrix $\blockm$ as:
\begin{equation*}
\blockm = \begin{bmatrix}  \mathbf{1}_{\blocksize\times \blocksize} & \mathbf{0}_{\blocksize\times \blocksize} & \cdots & \mathbf{0}_{\blocksize\times \blocksize}\\\mathbf{0}_{\blocksize\times \blocksize} & \mathbf{1}_{\blocksize\times \blocksize}  &\cdots & \mathbf{0}_{\blocksize\times \blocksize} \\ \vdots  & \vdots &\ddots&\vdots  \\ \mathbf{0}_{\blocksize\times \blocksize} & \mathbf{0}_{\blocksize\times \blocksize} & \cdots & \mathbf{1}_{\blocksize\times \blocksize}  \end{bmatrix}.
\end{equation*}
Finally, denote by $\permclass_{\numblocks\blocksize\times \numblocks\blocksize}$ the set of all $\numblocks\blocksize \times \numblocks\blocksize$ permutation matrices. Recall that a permutation matrix is a matrix obtained by permuting the rows of an identity matrix. Also recall that left multiplying a matrix by a permutation matrix permutes the rows of the matrix and right multiplying a matrix by a permutation matrix permutes the columns of the matrix. With this background, the noiseless community model is defined as the following set of similarity matrices for $\numblocks\geq 2$ and $\blocksize\geq 2$:
\begin{equation}
\text{Noiseless Community Model} = \{\similarity \in \mathbb{R}^{\numblocks \blocksize \times \numblocks\blocksize}: \similarity = 
\permleft (\simscalar{\blockm})\permright,~~~ \simscalar \in [0.01, 1],~~~ \permleft,\permright \in \permclass_{\numblocks\blocksize\times \numblocks\blocksize}\}.
\label{eq:noiseless_model}
\end{equation}
 The number of reviewers is given by $\numrev=\numblocks\blocksize$ and the number of papers is by $\numpapers=\numblocks\blocksize$.
In words, this is the set of all similarity matrices obtained via a permutation of the rows and columns of the block matrix $\blockm$. 

We begin our theoretical results for this section by showing that under the noiseless community formulation, both \super and \simbase are optimal, whereas \bidbase and \randbase fare poorly. Recall that $\numpapers=\numrev=\numblocks\blocksize$ in the noiseless community model.
\begin{theorem}\label{prop:diagonal}
Consider a model with a paper-side gain function  $\gainfunction(\numbids_{\paperindex})=\sqrt{\numbids_{\paperindex}}$,  the reviewer-side gain function  $\gainfunctionrev(\ordering_{\revindex}(\paperindex), \similarity_{\revindex, \paperindex}) = (2^{\similarity_{\revindex, \paperindex}}-1)/\log_2(\ordering_{\revindex}(\paperindex) + 1)$, and the bidding function $\bidfunction(\ordering_{\revindex}(\paperindex), \similarity_{\revindex, \paperindex}) = \mathds{1}\{\ordering_{\revindex}(\paperindex)=1\}\mathds{1}\{\similarity_{\revindex, \paperindex} > \simscalar/2\}$. Then, under the noiseless community model from~\eqref{eq:noiseless_model}, for all $\numblocks\geq 2$, $\blocksize\geq 2$ and $\hyperparam\geq 0$:\\
\indent (a) \super with zero heuristic is optimal; \\
\indent (b) \simbase is optimal.\\
In contrast, there exists a constant $\constant > 0$ such that for every $\numblocks\geq 2$, $\blocksize\geq 2$ and $\hyperparam\geq 0$:\\
\indent (c) \bidbase is suboptimal by an additive factor of at least $\constant\hyperparam \numpapers/\log_2^2(\numpapers)$; \\
\indent (d) \randbase is suboptimal by an additive factor of at least $\constant\numpapers$.
\end{theorem}

Although \simbase is optimal in the noiseless community model, this optimality turns out to be quite brittle. As we show below, even an infinitesimally small amount of noise makes \simbase considerably suboptimal. In contrast, \super is robust enough and suffers by only a small amount.

\paragraph*{Noisy community model.}
More formally, we first define a `noisy community model'. Under this model, we assume that the similarity matrix is generated by first selecting any similarity matrix $\similarity'$ from the noiseless community model defined  in~\eqref{eq:noiseless_model}, and then adding noise to its entries as:
\newcommand{\unifnoise}{\xi}
\begin{equation}
\similarity_{\revindex, \paperindex} = 
\begin{cases}
\simscalar -\nu_{\revindex, \paperindex}&  \text{if} \ \similarity_{\revindex, \paperindex}' = \simscalar\\
\nu_{\revindex, \paperindex}& \text{if} \ \similarity_{\revindex, \paperindex}'=0,
\end{cases}
\label{eq:noisy_model}
\end{equation}
where $\nu_{\revindex, \paperindex}$ is drawn independently and uniformly from $(0, \unifnoise)$ for each reviewer-paper pair, for some small value $\unifnoise$ to be defined subsequently. 

The next result shows that even under an arbitrarily small perturbation $\unifnoise$ from a noiseless community model, the baselines become significantly suboptimal. In contrast, \super is robust to the noise and is near-optimal. Recall that $\numpapers=\numrev=\numblocks\blocksize$ in the noisy community model.

\begin{theorem}\label{thm:diagonal}
Consider the gain and bidding functions from Theorem~\ref{prop:diagonal} and the noisy community model given in~\eqref{eq:noisy_model} with any noise bound satisfying $\unifnoise \leq (1+\hyperparam)^{-1} e^{-e\numblocks\blocksize}$. Then, for all $\numblocks\geq 2$, $\blocksize\geq 2$, and $\hyperparam\geq 0$:  \\
 \indent (a) \super with zero heuristic is within at least an additive factor of $0.0001$ of the optimal. \\
Moreover, there exists a constant $\constant>0$ such that for every $\numblocks\geq 2$, $\blocksize\geq 2$, and $\hyperparam\geq 0$, with respect to \super with zero heuristic:  \\
 \indent (b) \simbase is suboptimal by an additive factor of at least $\constant\numpapers$; \\
 \indent (c) \bidbase is suboptimal by an additive factor of at least $\constant\hyperparam\numpapers/\log_2^2(\numpapers)$; \\
 \indent (d) \randbase is suboptimal by an additive factor of at least $\constant\numpapers$.
\end{theorem}

This result thus establishes the global optimality of the proposed \super algorithm under the community model, while in contrast all popular baselines are considerably suboptimal.

\section{Experimental Results}
\label{sec:simulations}
\noindent
We now empirically evaluate \super (with zero and mean heuristics) and compare it with the baselines $\simbase$, $\bidbase$, and $\randbase$ (discussed earlier in Section~\ref{sec:problem_setup}). The experimentation methodology is as follows for any chosen set of model parameters including the gain functions, bidding probability, trade-off parameter, and number of reviewers and papers. Given a fixed similarity matrix, we shuffle the rows of the matrix to randomize the sequence of reviewer arrivals and simulate each of the algorithms. Then, for each algorithm, we record the gain along with the number of papers that end up with bid counts in the intervals $\{0,1,2\}$, $\{3,4,5\}, \{6,7,8\}$, and $\{9+\}$. We repeat this process 20 times for a given similarity matrix if it is fixed and draw a similarity matrix at random for each run if the score structure being simulated is a distribution. 
To evaluate performance, we show the means of the relative gains (additive gains relative to the gain of a baseline) across the runs and include error bars representing the standard error of the mean. Moreover, we present the mean number of papers across the repeated simulations that finish with bid counts in each of the previously given bid count intervals. The code and data to reproduce each of the experiments is available at  {\href{https://github.com/fiezt/Peer-Review-Bidding}{\tt github.com/fiezt/Peer-Review-Bidding}}.

\subsection{ICLR Similarity Matrix}
\label{sec:iclr_simulations}
To begin our experiments, we perform evaluations on a similarity matrix from the ICLR 2018 conference discussed earlier in Section~\ref{sec:theoretical_properties}. 
Recall that the similarity matrix consists of $\numrev=2435$ reviewers and $\numpapers=935$ papers. 
In the following experiments, we run a simulation using a default model configuration, then we explore the impact of changing components of the model, and we finish by exploring the robustness of the algorithm to various real-world complexities.

\paragraph*{Default model configuration.} 
We begin by evaluating a default model configuration that is considered throughout the the remainder of the experiments unless otherwise specified. The model consists of the paper-side gain function $\gainfunction(\numbids_{\paperindex}) = \min\{\numbids_{\paperindex}, 6\}$, the reviewer-side gain function $\gainfunctionrev(\ordering_{\revindex}(\paperindex),\similarity_{\revindex,\paperindex}) =  (2^{\similarity_{\revindex, \paperindex}}-1)/\log_2(\ordering_{\revindex}(\paperindex)+1)$, and the bidding probability model $\bidfunction(\ordering_{\revindex}(\paperindex), \similarity_{\revindex, \paperindex}) = \similarity_{\revindex, \paperindex}/\log_2(\ordering_{\revindex}(\paperindex)+1)$. We remark that the paper-side gain function is a natural choice given that conferences often assign three reviewers to each paper and as such they may seek twice the number of bids per paper. Moreover, recall that for this pair of reviewer-side gain and bidding functions, the efficient routine in Algorithm~\ref{alg:subprocedure_efficient} can be called in place of Algorithm~\ref{alg:subprocedure} in \super to retrieve a paper ordering, which is what we implement. 

\begin{figure*}[t]
    \centering
    \subfloat[][]{\includegraphics[width=0.25\textwidth]{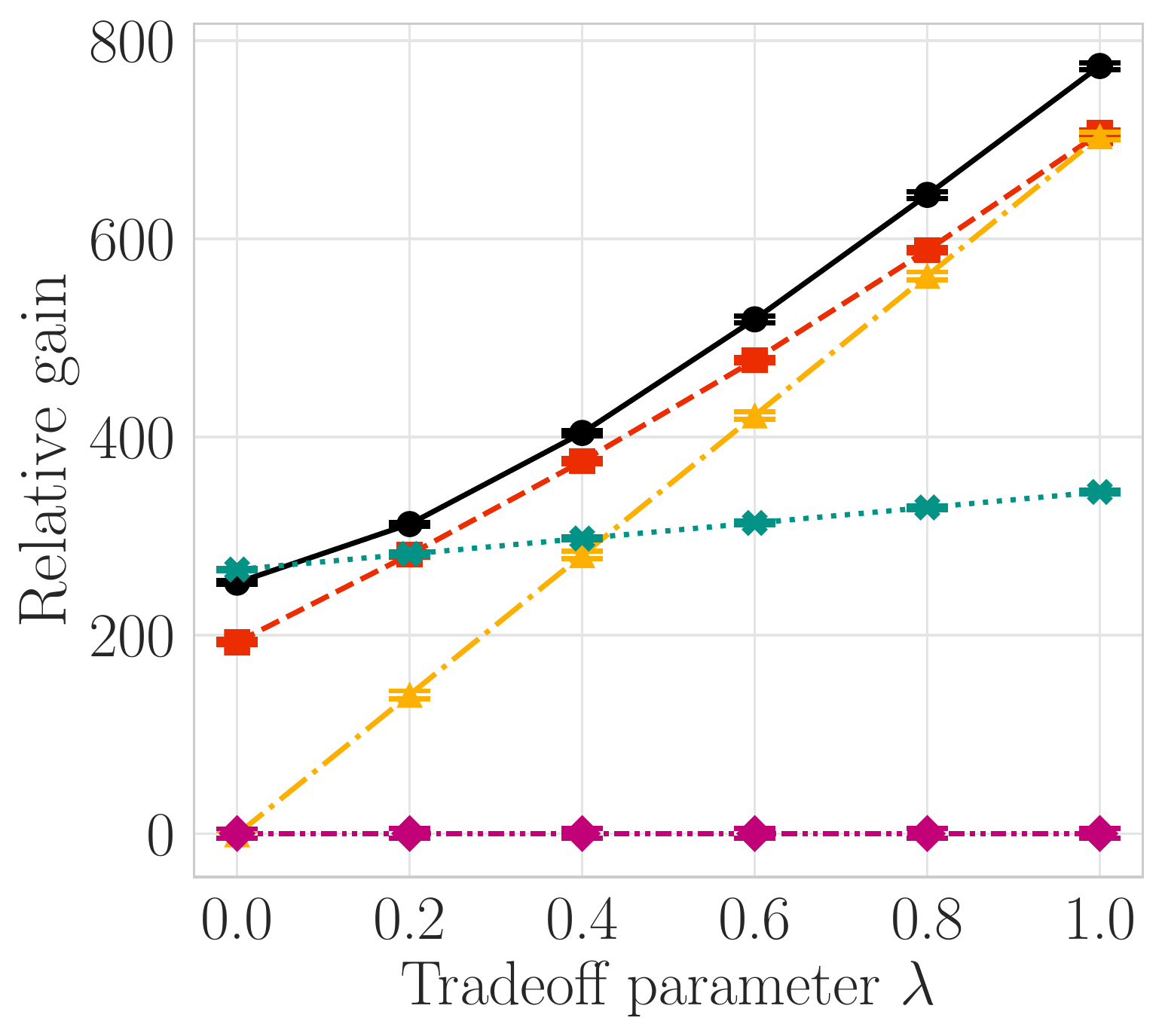}\label{fig:iclr20181a}}\hfill 
    \subfloat[][]{\includegraphics[width=0.25\textwidth]{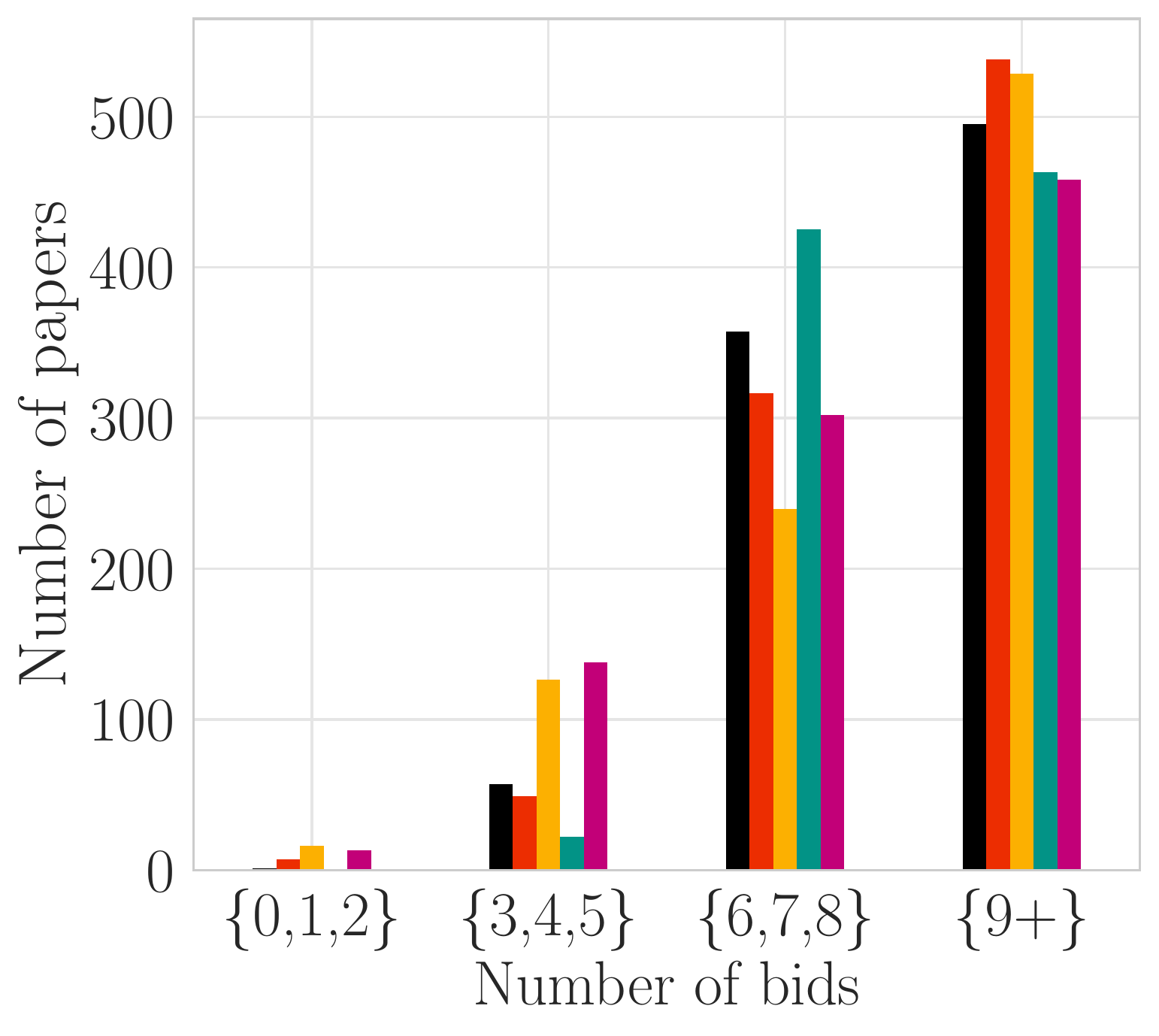}\label{fig:iclr20181b}}\hfill
    \subfloat[][]{\includegraphics[width=0.25\textwidth]{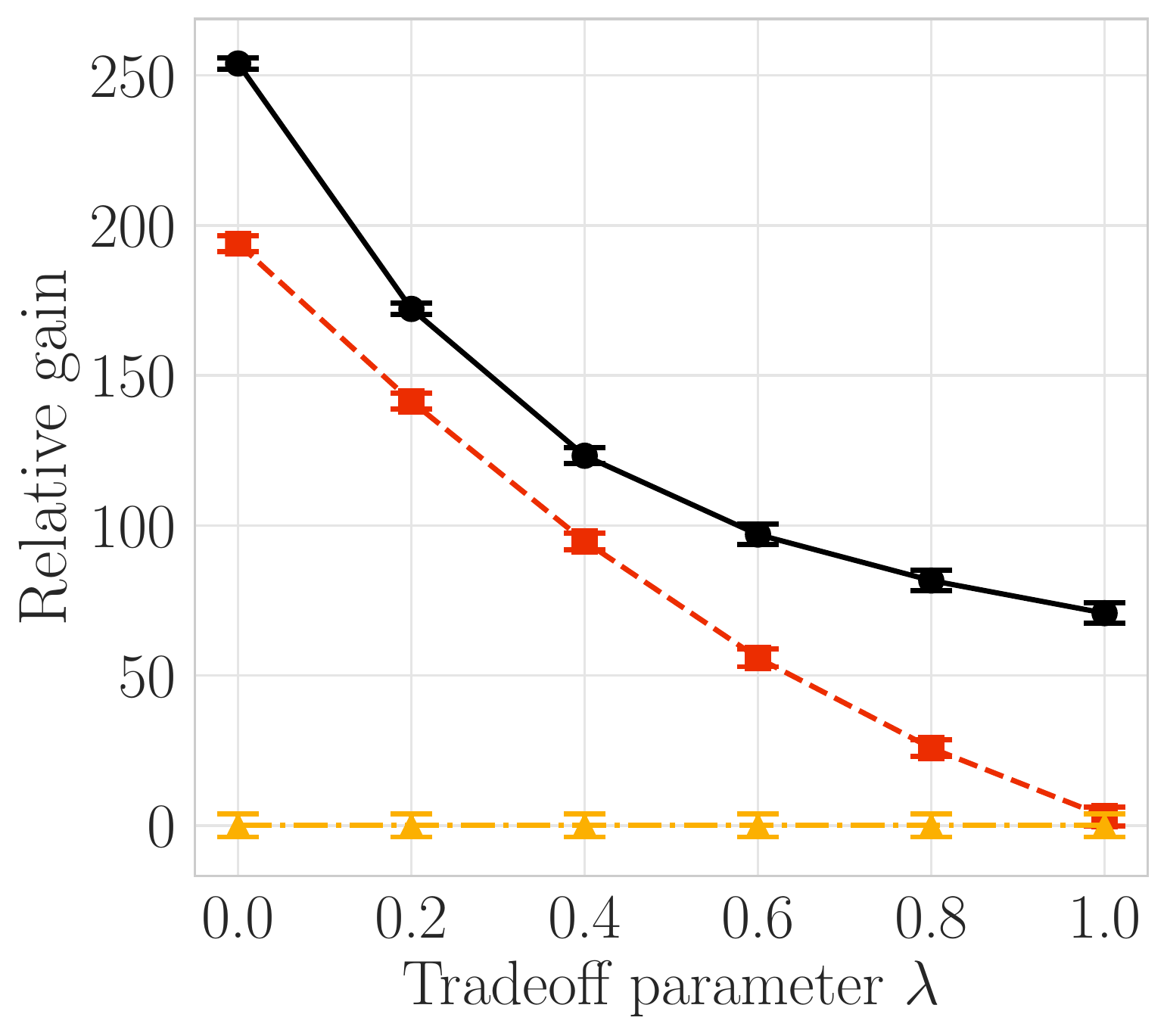}\label{fig:iclr20181c}}\hfill
    \subfloat[][]{\includegraphics[width=0.25\textwidth]{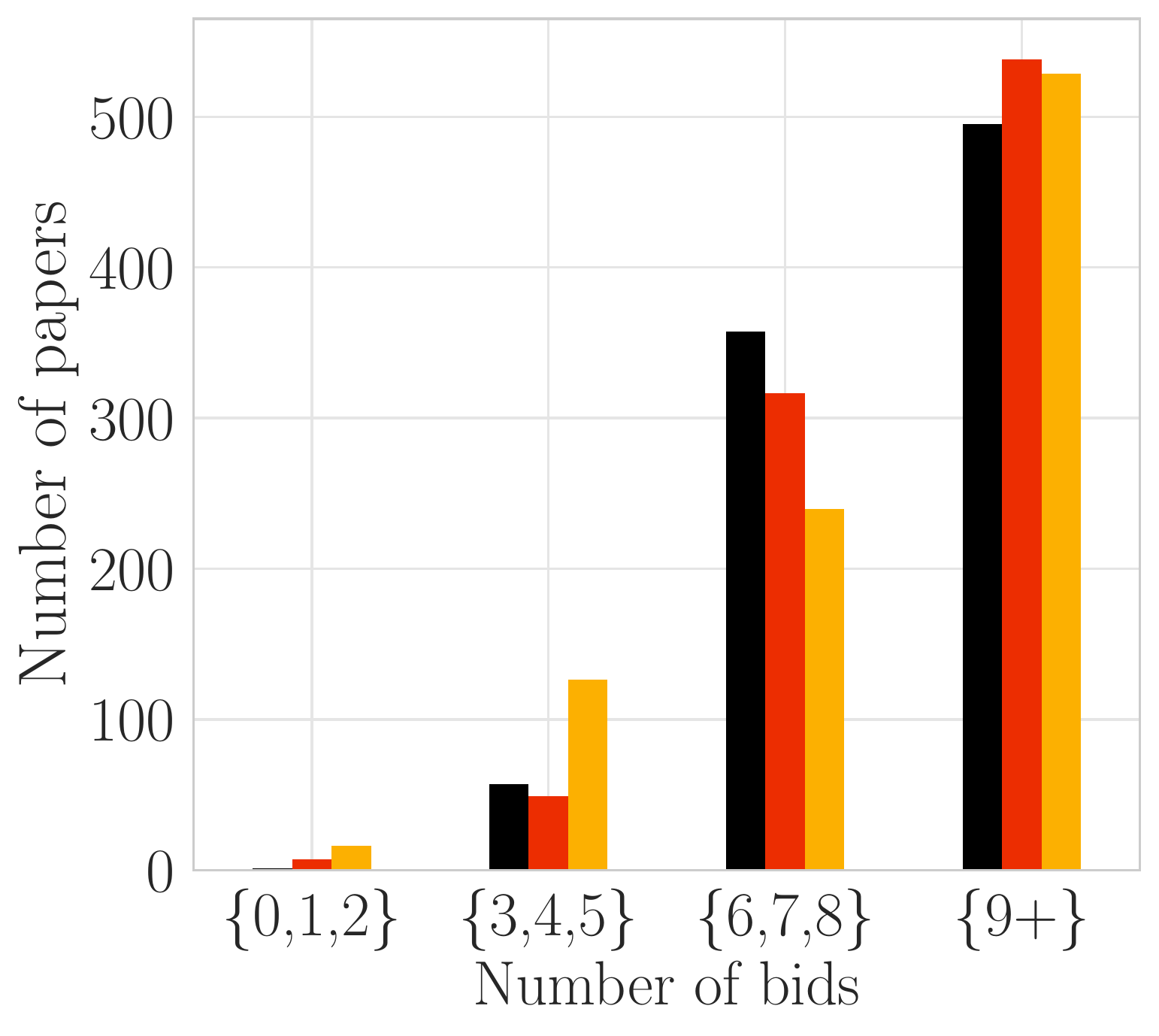}\label{fig:iclr20181d}}\hfill
    \subfloat{\includegraphics[width=.9\textwidth]{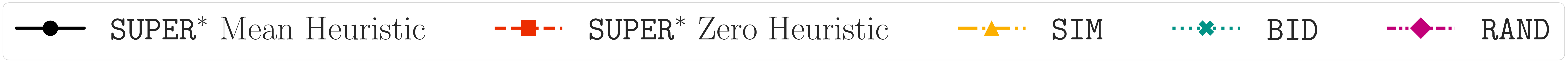}}
    \caption{\small ICLR experiment with the default model configuration. Legend for each bid distribution plot (b, d): within each bid-count interval, the bars
    correspond to the algorithms given in the legend and are presented in an order consistent with the legend itself when read from left to right and (d) only includes  \super with mean heuristic, \super with zero heuristic, and \simbase.}
    \label{fig:iclr20181}
\end{figure*}

The results of the experiment are presented in Figure~\ref{fig:iclr20181}. In Figures~\ref{fig:iclr20181a}--\ref{fig:iclr20181b} we compare \super to each baseline and in Figures~\ref{fig:iclr20181c}--\ref{fig:iclr20181d} we zoom in and only show the results for \super and \simbase. 
In terms of the gain results shown in Figures~\ref{fig:iclr20181a} and \ref{fig:iclr20181c}, each version of \super outperforms the baseline algorithms, while \bidbase outperforms \simbase when minimal weight $\hyperparam$ is given to the reviewer-side gain and vice versa when a significant amount of weight $\hyperparam$ is given to the reviewer-side gain. In Figures~\ref{fig:iclr20181b} and~\ref{fig:iclr20181d}, the distribution of the bid counts obtained for the algorithms are shown with $\hyperparam=0.8$, which was chosen since this parameter choice gave nearly equal paper-side and weighted reviewer-side gain for \randbase. While \bidbase has a similar number of papers with fewer than the minimum number of desired bids as each version of \super, the gain demonstrates why it is not a generally adopted method. As a result of showing papers of limited relevance early in the paper orderings to elicit bids on papers with few bids, the overall gain is significantly smaller than \simbase and \super since the reviewer-side gain is worse. The distributions also illustrate that both versions of \super end the bidding process with approximately a 60\% reduction of the number of papers with fewer than the desired minimum number of bids compared to \simbase and \randbase.

\begin{figure*}[t]
    \subfloat[][]{\includegraphics[width=0.24\textwidth]{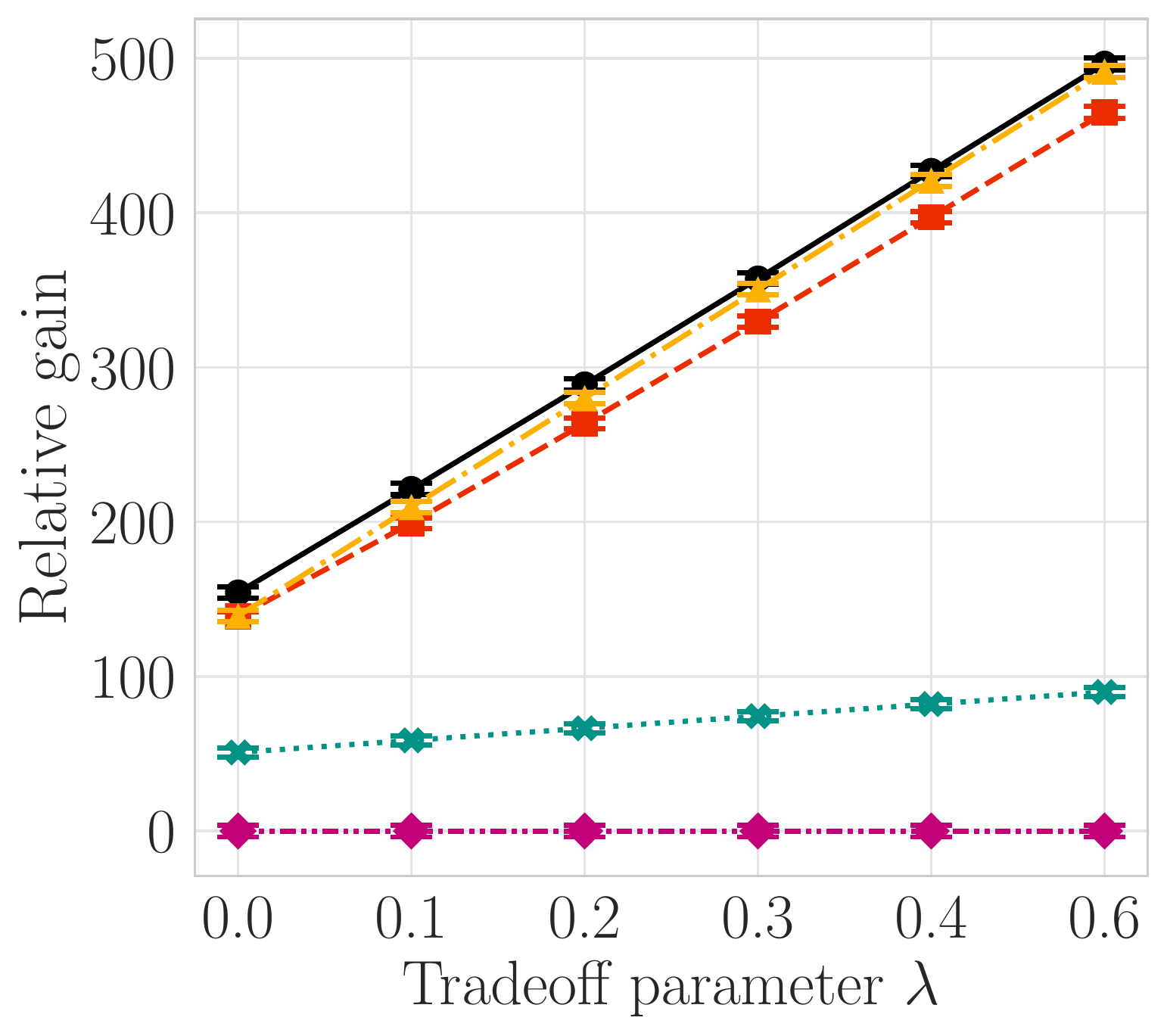}\label{fig:iclr20182a}}\hfill 
    \subfloat[][]{\includegraphics[width=0.24\textwidth]{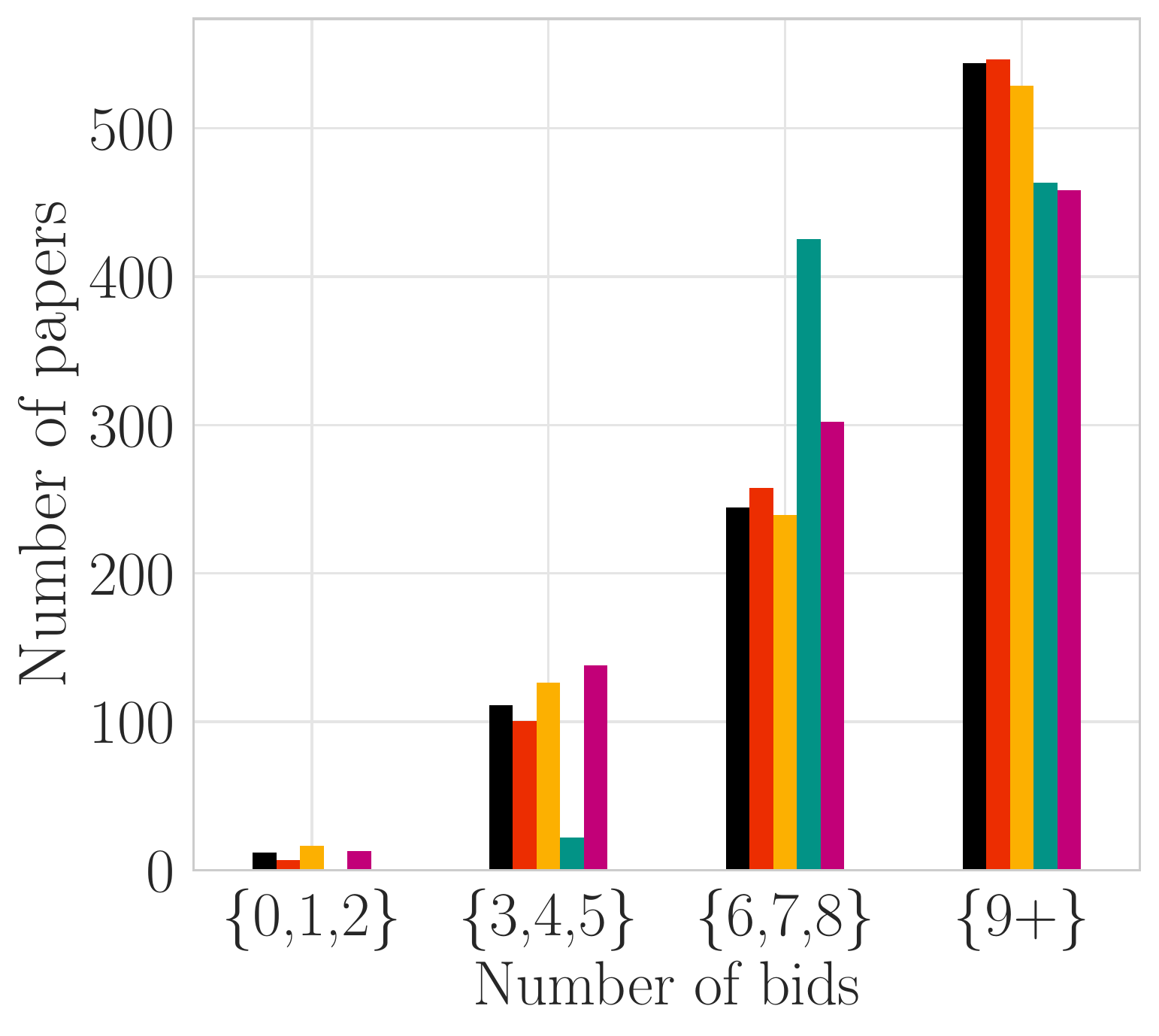}\label{fig:iclr20182b}}\hfill
    \subfloat[][]{\includegraphics[width=0.24\textwidth]{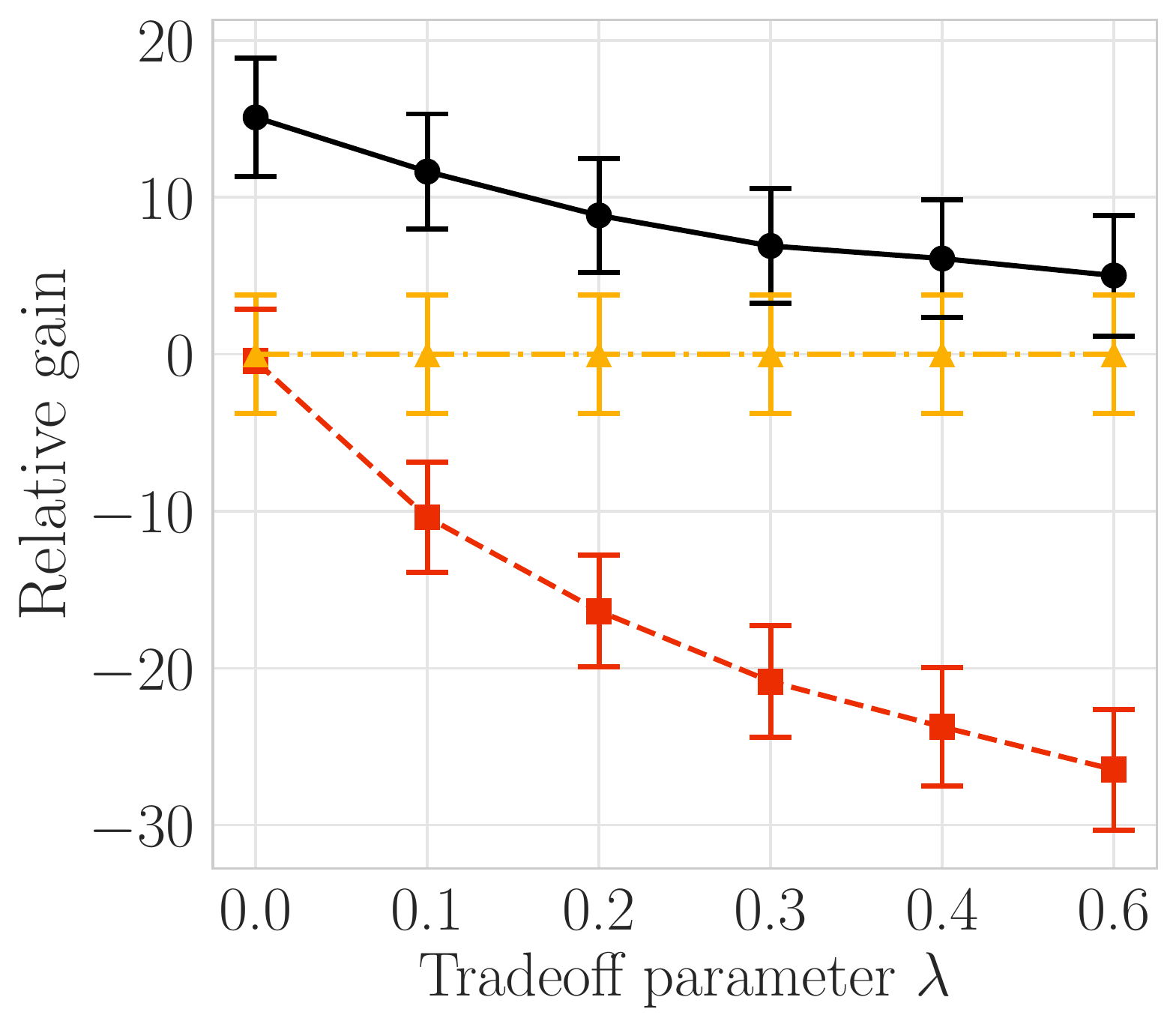}\label{fig:iclr20182c}}\hfill
    \subfloat[][]{\includegraphics[width=0.24\textwidth]{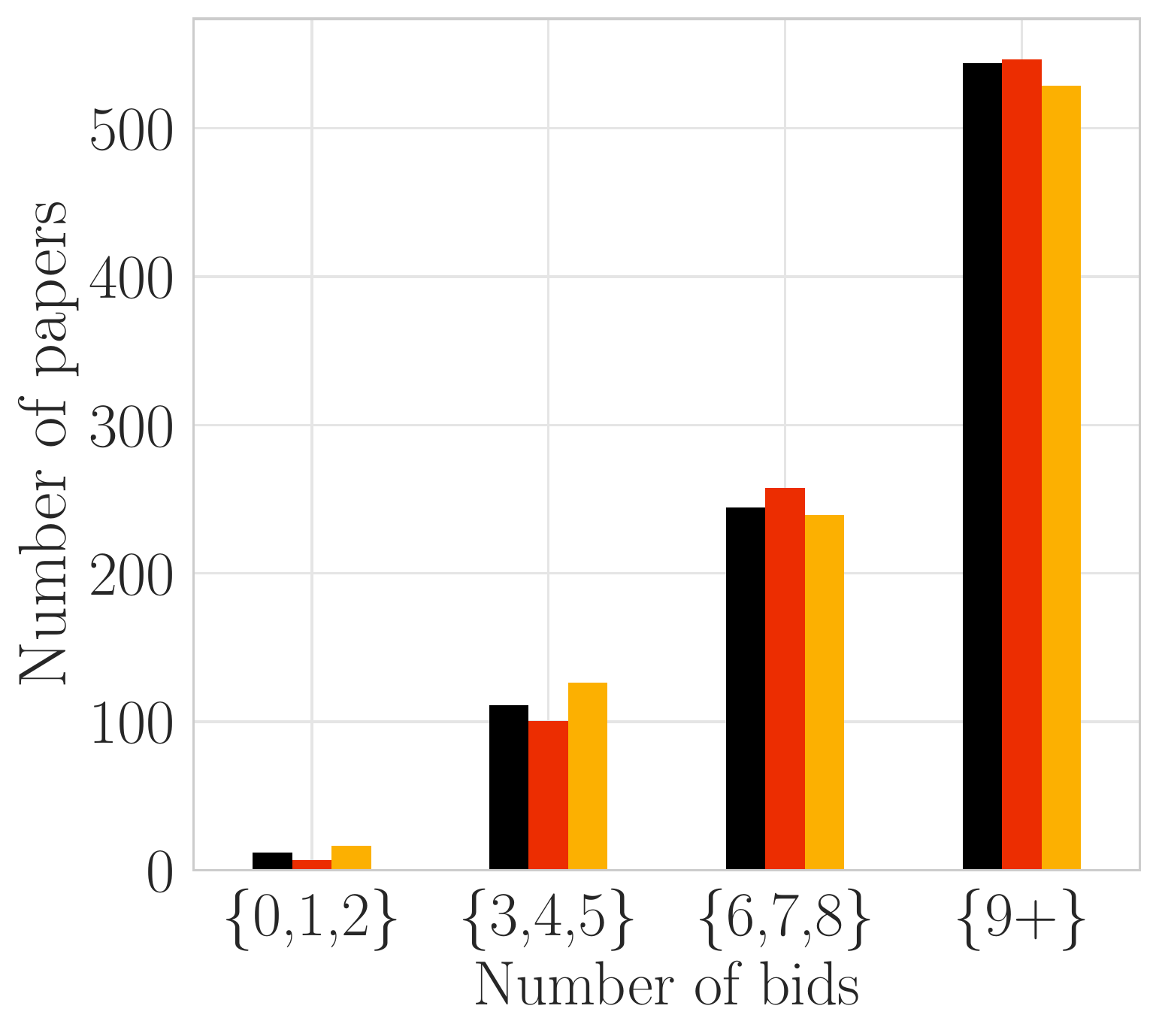}\label{fig:iclr20182d}}\hfill

    \subfloat[][]{\includegraphics[width=0.25\textwidth]{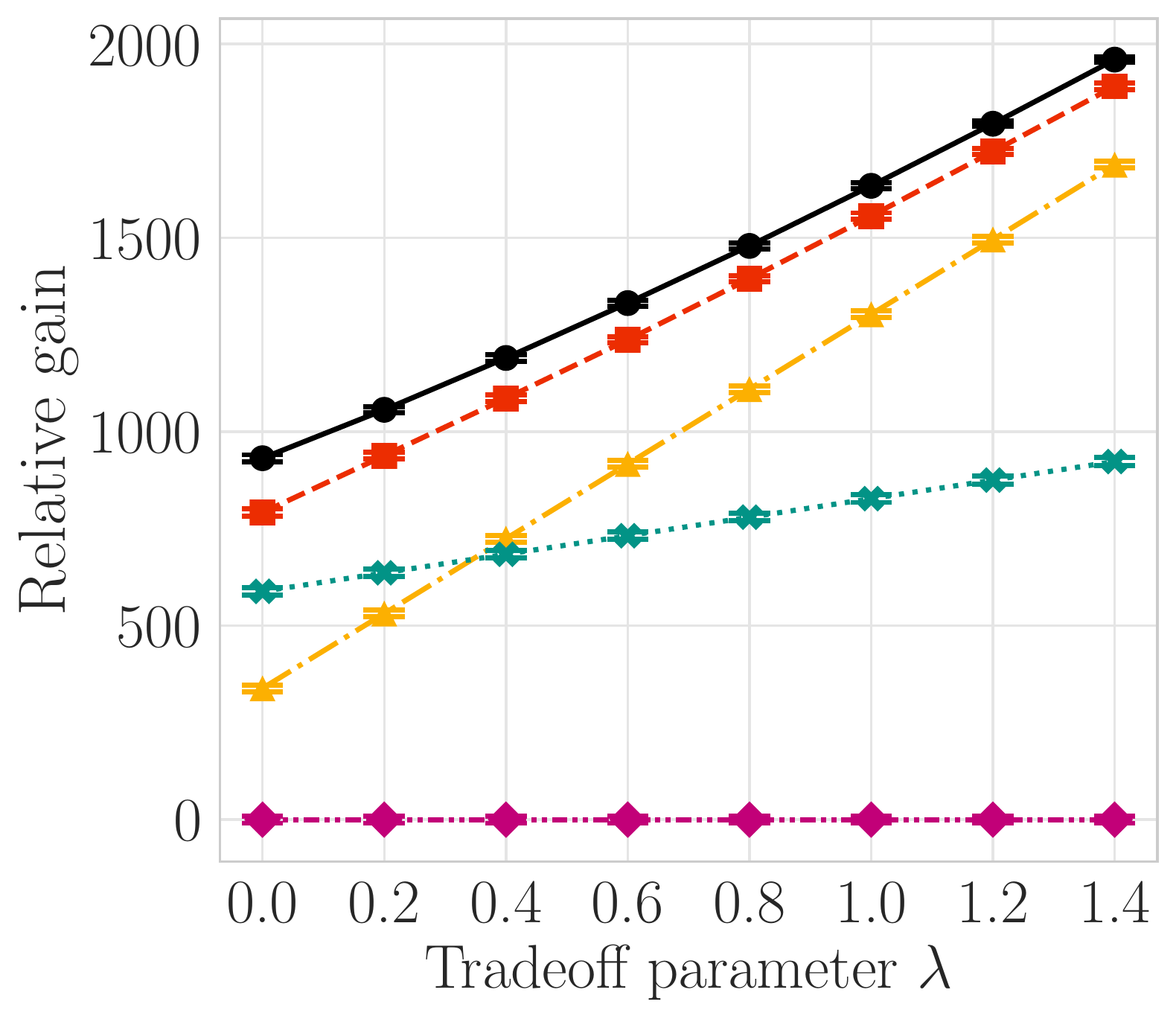}\label{fig:iclr20182e}}\hfill 
    \subfloat[][]{\includegraphics[width=0.25\textwidth]{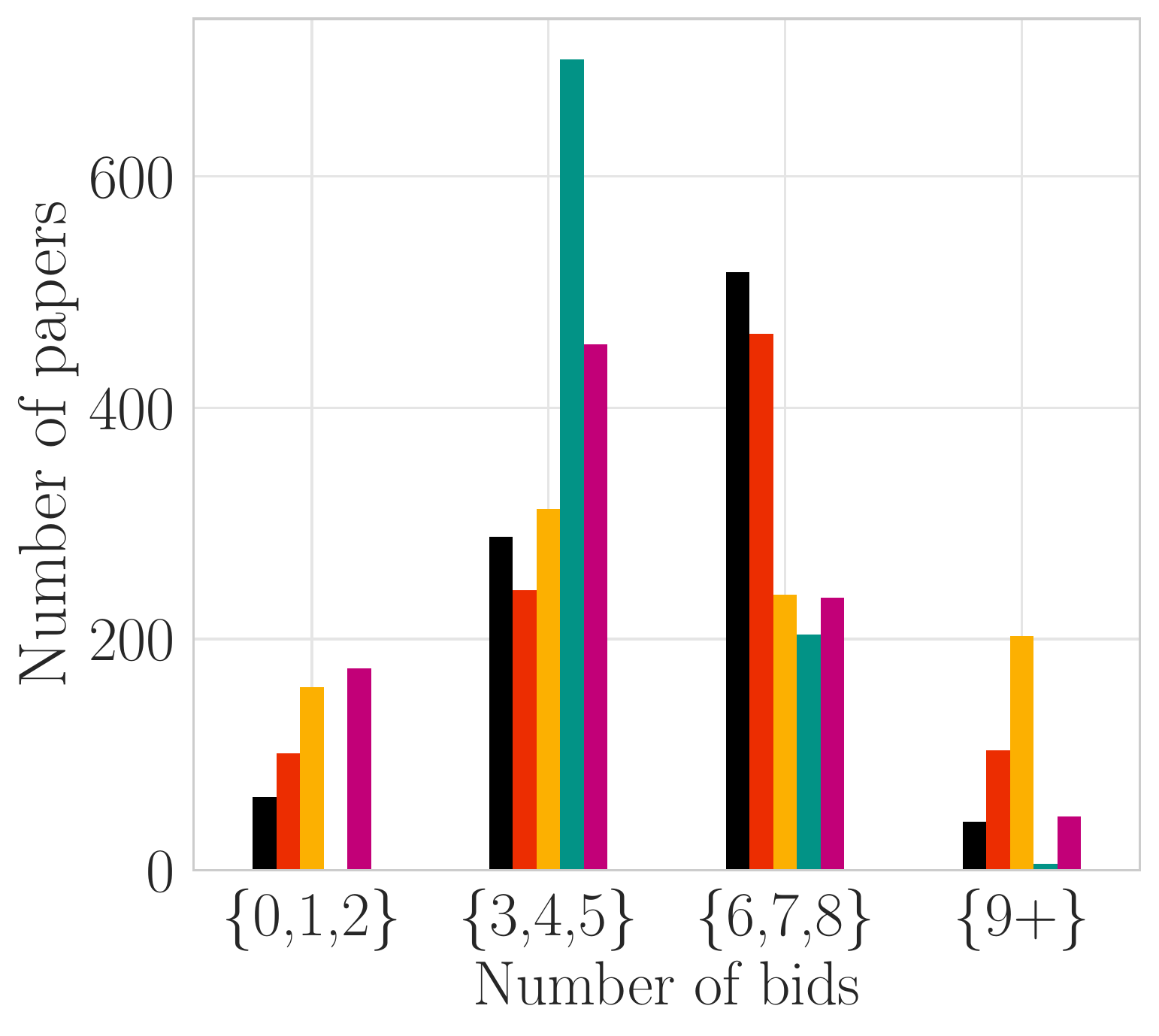}\label{fig:iclr20182f}}\hfill
    \subfloat[][]{\includegraphics[width=0.25\textwidth]{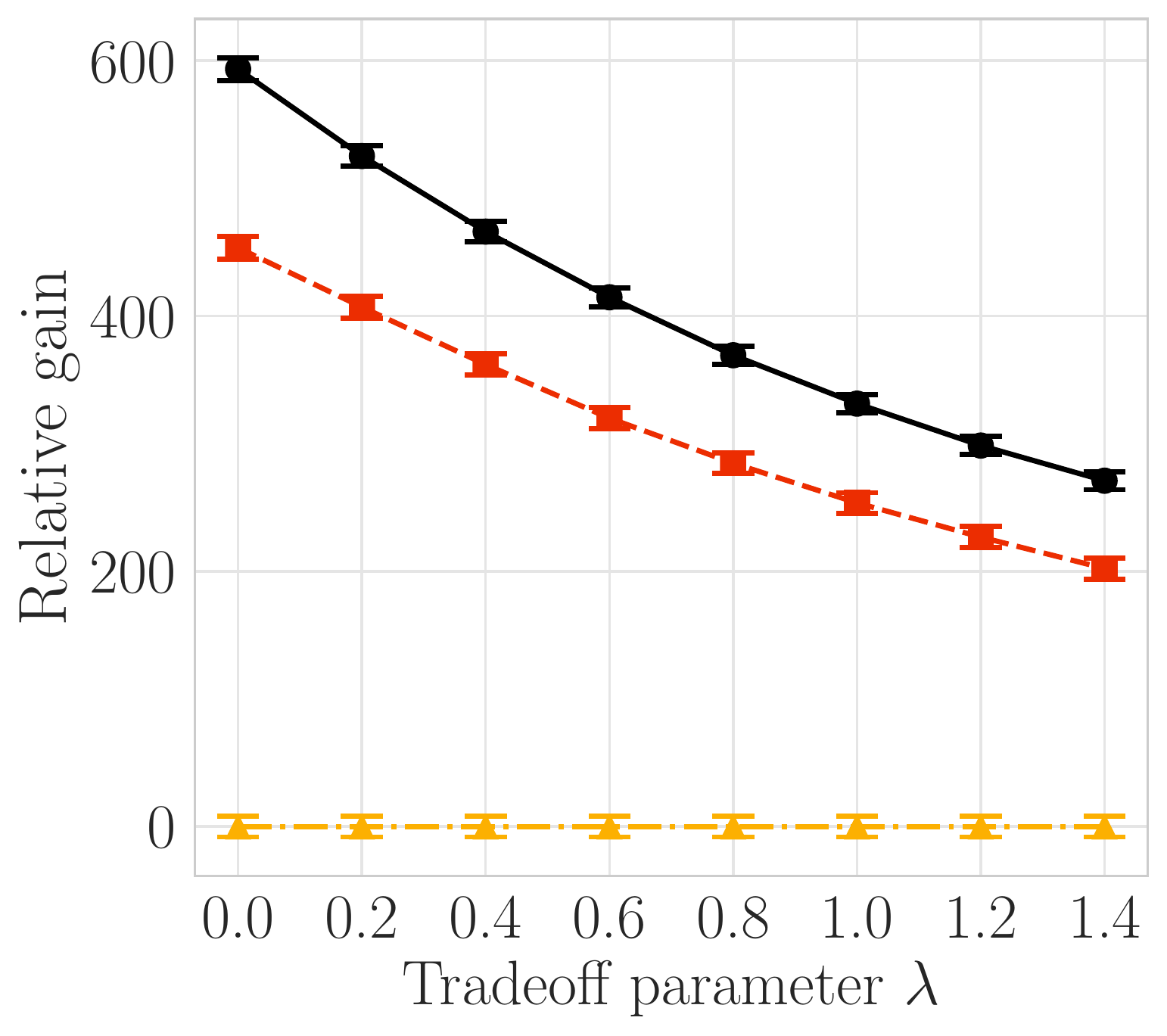}\label{fig:iclr20182g}}\hfill
    \subfloat[][]{\includegraphics[width=0.25\textwidth]{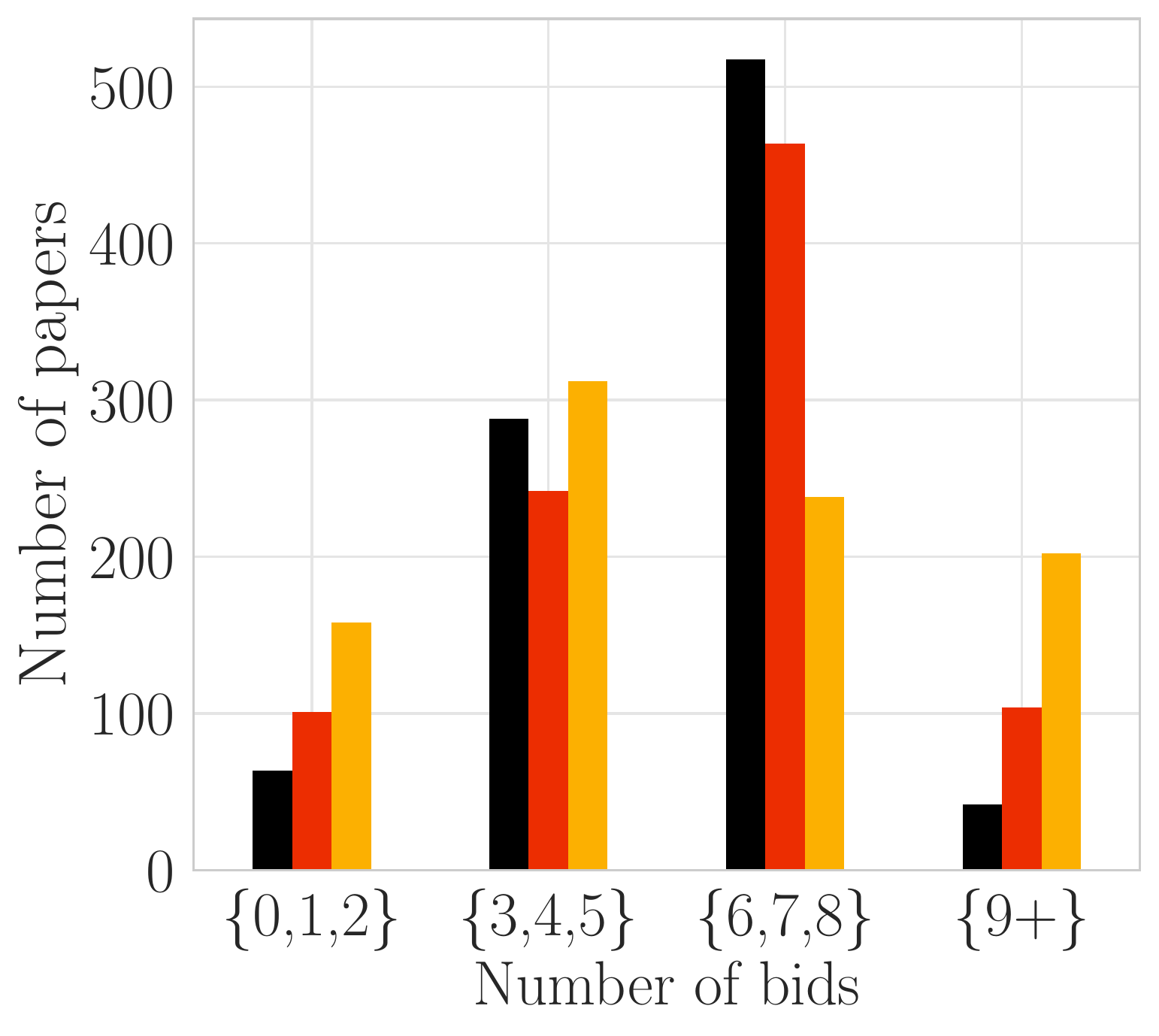}\label{fig:iclr20182h}}\hfill
    \subfloat{\includegraphics[width=.9\textwidth]{Figs/legend}}
    \caption{\small ICLR 2018 experiment with variations of the default model configuration. 
        In Figures~\ref{fig:iclr20182a}--\ref{fig:iclr20182d}, the paper-side gain function is changed.  In Figures~\ref{fig:iclr20182e}--\ref{fig:iclr20182h}, the reviewer-side gain and bidding function are changed. Legend for each bid distribution plot (b, d, f, h): within each bid-count interval, the bars 
    correspond to the algorithms given in the legend and are presented in an order consistent with the legend itself when read from left to right and (d, h) only includes \super with mean heuristic, \super with zero heuristic, and \simbase.
    }
    \label{fig:iclr20182}
\end{figure*}

\paragraph*{Varying model parameters.} We now consider variations of the default model consideration. In Figures~\ref{fig:iclr20182a}--\ref{fig:iclr20182d}, results are shown when the paper-side gain function is changed from $\gainfunction(\numbids_{\paperindex})=\min\{\numbids_{\paperindex}, 6\}$ to $\gainfunction(\numbids_{\paperindex})=\sqrt{\numbids_{\paperindex}}$. In Figures~\ref{fig:iclr20182a}--\ref{fig:iclr20182b} we compare \super to each baseline and in Figures~\ref{fig:iclr20182c}--\ref{fig:iclr20182d} we zoom in and only show the results for \super and \simbase. In Figures~\ref{fig:iclr20182b} and~\ref{fig:iclr20182d}, the distribution of the bid counts obtained for the algorithms are shown with $\hyperparam=0.4$, which was chosen since this parameter choice gave nearly equal paper-side and weighted reviewer-side gain for \randbase.
For this model configuration, \bidbase and \randbase are significantly suboptimal in terms of the gain. \super with the mean heuristic outperforms \simbase by a marginal amount in terms of the gain, while \simbase outperforms \super with the zero heuristic by a marginal amount in terms of the gain. The bid distributions show that each version of \super ends the bidding process with a smaller number of papers obtaining fewer than six bids compared to \simbase. Compared to the default paper-side gain function, the discrepancy is not as significant since \super is optimizing an objective that rewards getting more than five bids per paper even though the returns are diminishing.

In Figures~\ref{fig:iclr20182a}--\ref{fig:iclr20182d}, the default paper-side gain function is considered, but the reviewer-side gain function is changed to $\gainfunctionrev(\ordering_{\revindex}(\paperindex),\similarity_{\revindex,\paperindex}) =  (2^{\similarity_{\revindex, \paperindex}}-1)/\sqrt{\ordering_{\revindex}(\paperindex)}$ and the bidding function is changed to $\bidfunction(\ordering_{\revindex}(\paperindex), \similarity_{\revindex, \paperindex}) = \similarity_{\revindex, \paperindex}/\sqrt{\ordering_{\revindex}(\paperindex)}$. 
The effect of this change is that the probability of obtaining a bid on a paper from a reviewer decays faster with the position the paper is shown and similarly the reviewer-side gain from a paper diminishes faster as a function of the position the paper is shown to a reviewer. In Figures~\ref{fig:iclr20182e}--\ref{fig:iclr20182f} we compare \super to each baseline and in Figures~\ref{fig:iclr20182g}--\ref{fig:iclr20182h} we zoom in and only show the results for \super and \simbase. In Figures~\ref{fig:iclr20182f} and~\ref{fig:iclr20182h}, the distribution of the bid counts obtained for the algorithms are shown with $\hyperparam=1.2$, which was chosen since this parameter choice gave nearly equal paper-side and weighted reviewer-side gain for \randbase.
For this model configuration, each version of \super significantly outperforms each of the baselines in terms of the gain. The bid count distributions show \super with zero and mean heuristic reduce the number of papers with fewer than three bids by 35\% and 60\% compared to \simbase, respectively.
Moreover, both versions of \super end up with half as many papers obtaining fewer than six bids compared to \bidbase.

\paragraph*{Robustness to real-world complexities.}
\label{SecRobustSimulations}
The previous experiments were performed under settings faithful to the model described earlier in Section~\ref{sec:problem_setup}. We now evaluate the robustness of \super to the models by evaluating the performance under various vagaries and complexities of real-world peer review. For this set of experiments, we consider that \super is optimizing the default model configuration described previously in this section. The results of the following simulations that consider deviations from the model being optimized are given in Figure~\ref{fig:iclr20183}. For each experiment we show the gains of each of the algorithms relative to \randbase across a sweep of the parameter $\hyperparam$ and the bid count distributions for each of algorithms with $\hyperparam=0.8$ as selected for the default model previously.

\begin{figure*}[t]

    \begin{minipage}{\textwidth}
    \begin{minipage}{.2\textwidth}
    \centering
    \footnotesize \textbf{Bid probability mismatch}
    \end{minipage}%
    \begin{minipage}{.2\textwidth}
    \centering
    \footnotesize \textbf{Similarity mismatch}
    \end{minipage}%
    \begin{minipage}{.2\textwidth}
    \centering
    \vspace{-.5mm}
    \footnotesize \textbf{Subset of reviewers arrive}
    \end{minipage}%
    \begin{minipage}{.2\textwidth}
    \centering
    \vspace{-.5mm}
    \footnotesize \textbf{Concurrent arrivals}
    \end{minipage}%
    \begin{minipage}{.2\textwidth}
    \centering
    \footnotesize \textbf{Subset of papers bid on}
    \end{minipage}
    \end{minipage}

    \centering
    \subfloat[][]{\includegraphics[width=0.2\textwidth]{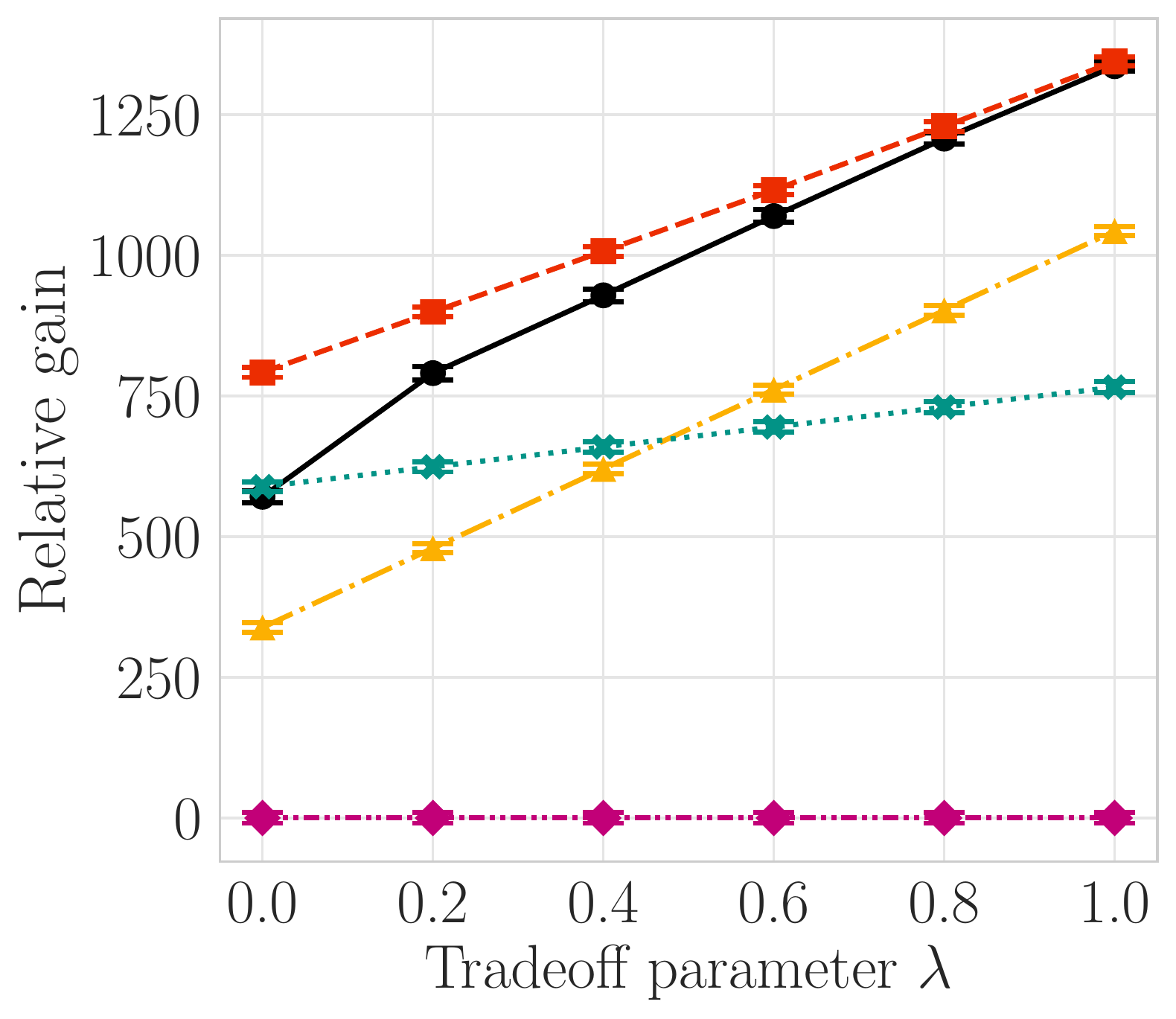}\label{fig:iclr20184a}}\hfill 
    \subfloat[][]{\includegraphics[width=0.2\textwidth]{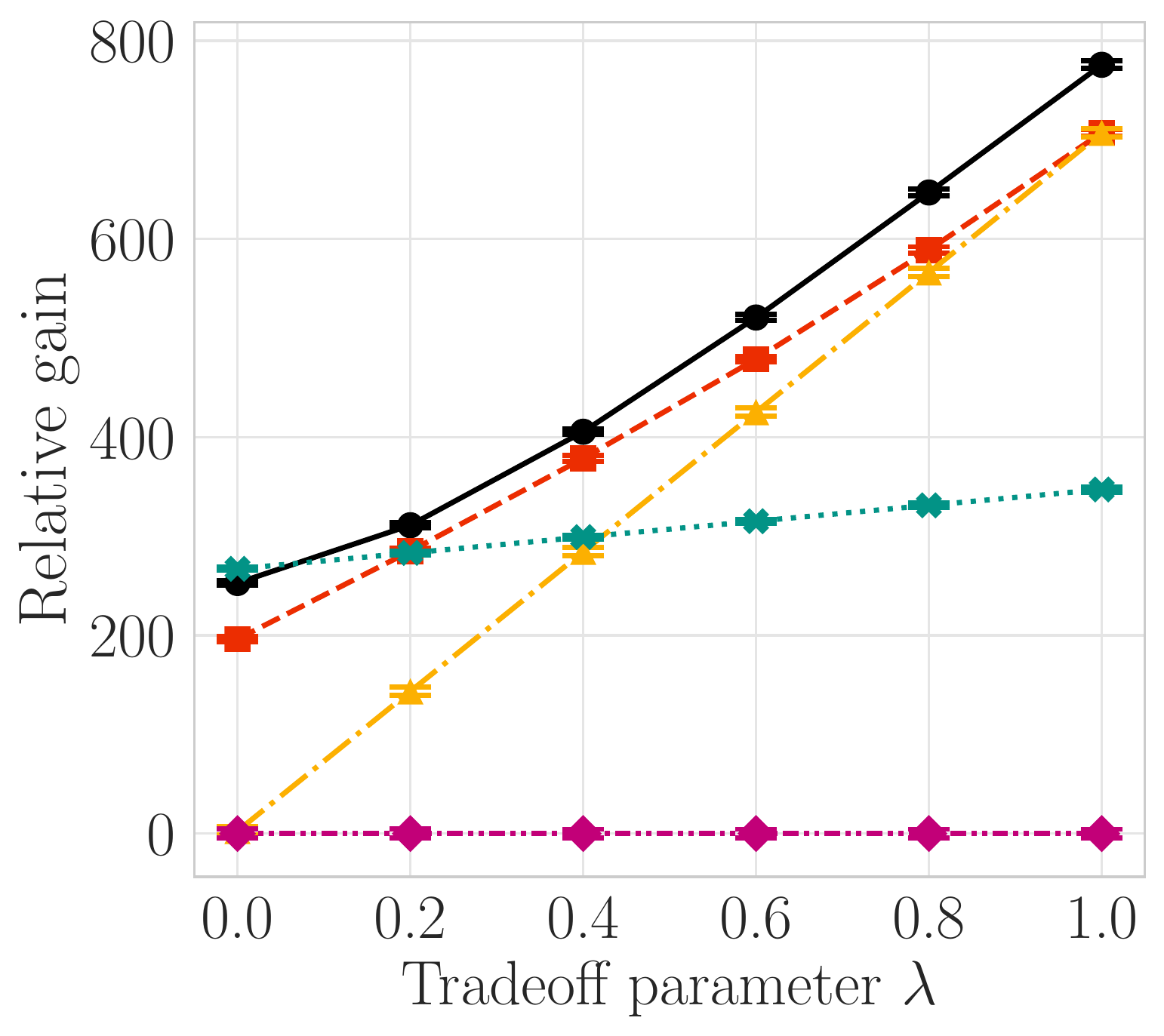}\label{fig:iclr20184b}}\hfill
    \subfloat[][]{\includegraphics[width=0.2\textwidth]{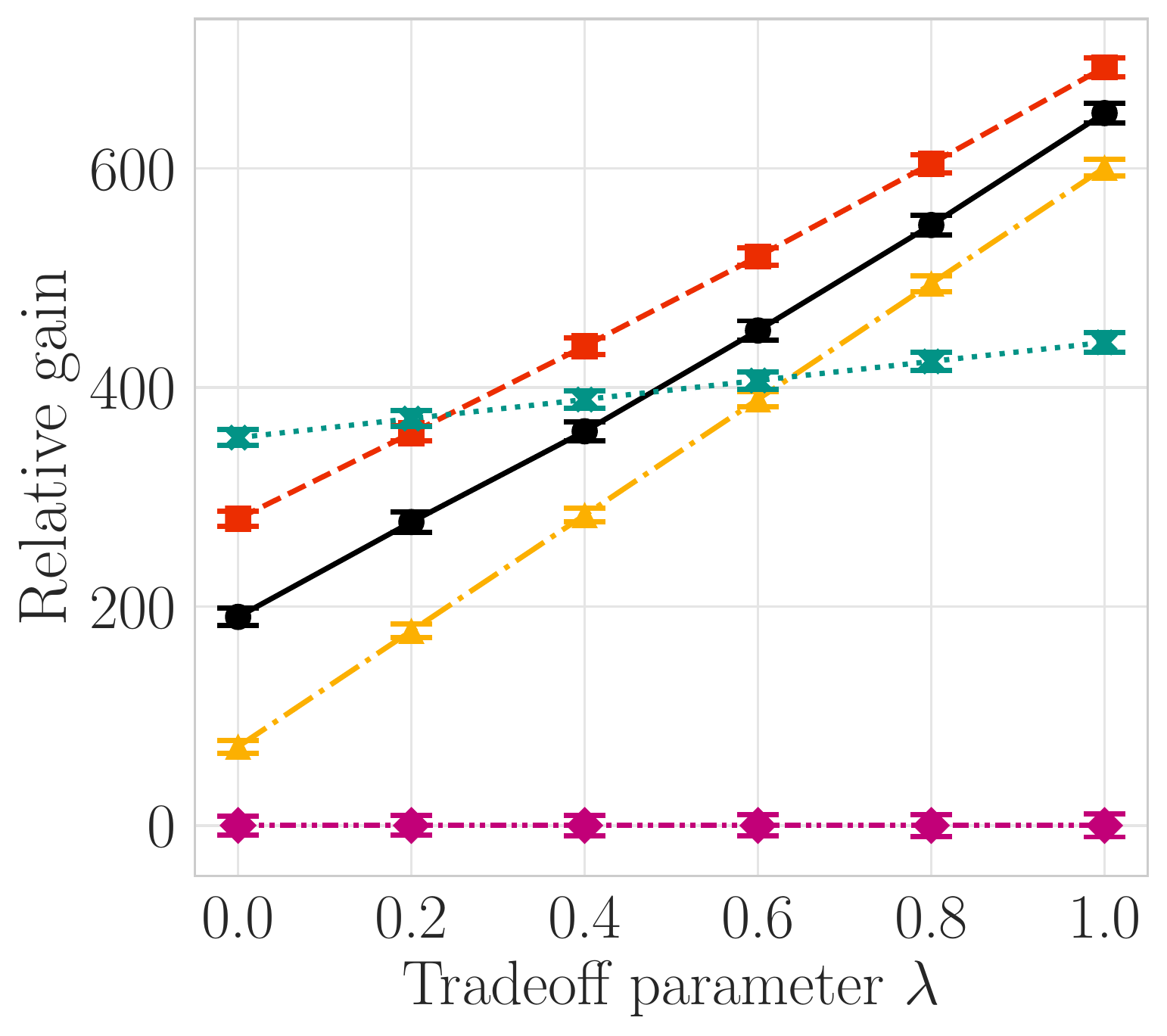}\label{fig:iclr20184c}}\hfill
    \subfloat[][]{\includegraphics[width=0.2\textwidth]{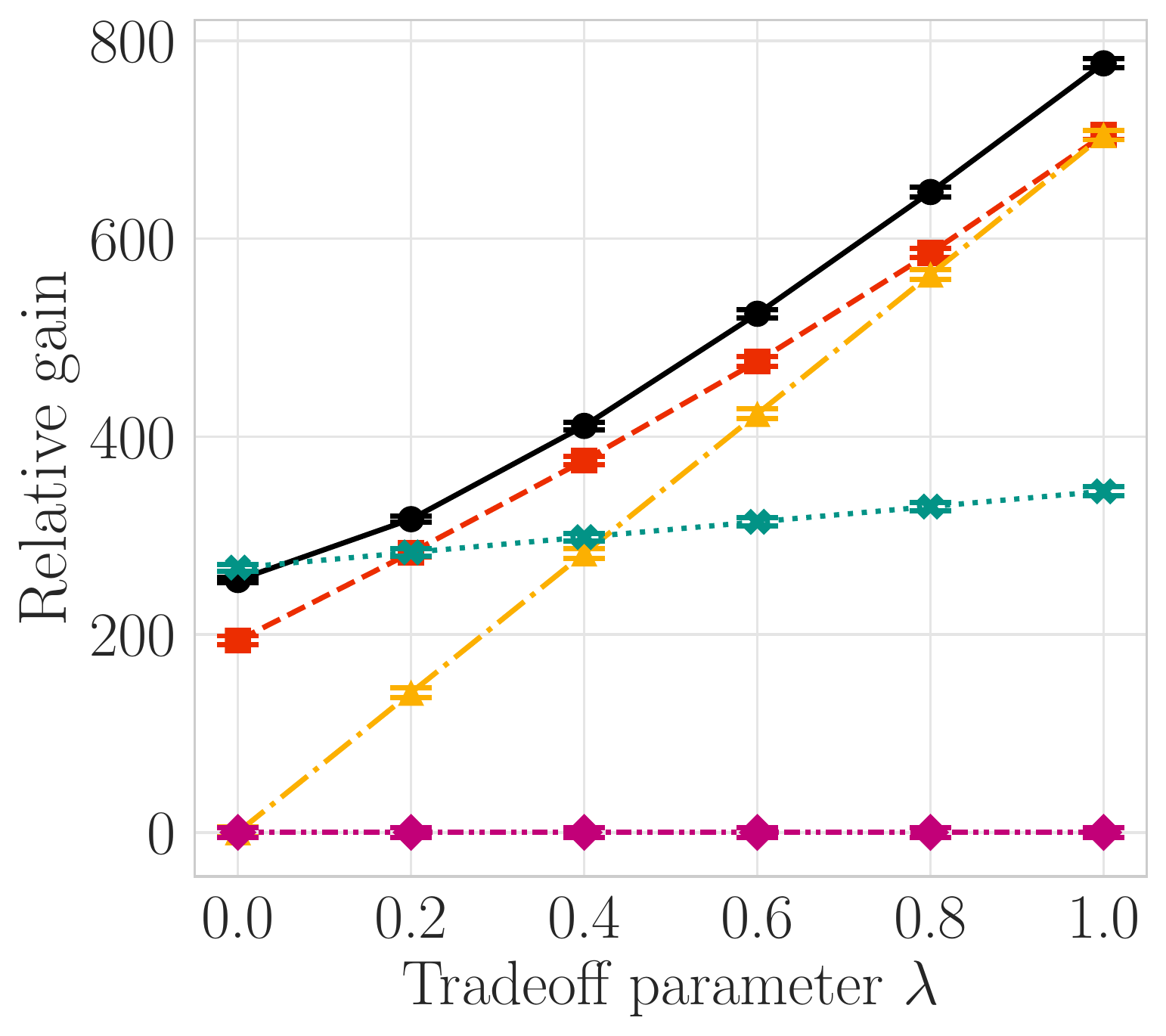}\label{fig:iclr20184d}}\hfill
    \subfloat[][]{\includegraphics[width=0.2\textwidth]{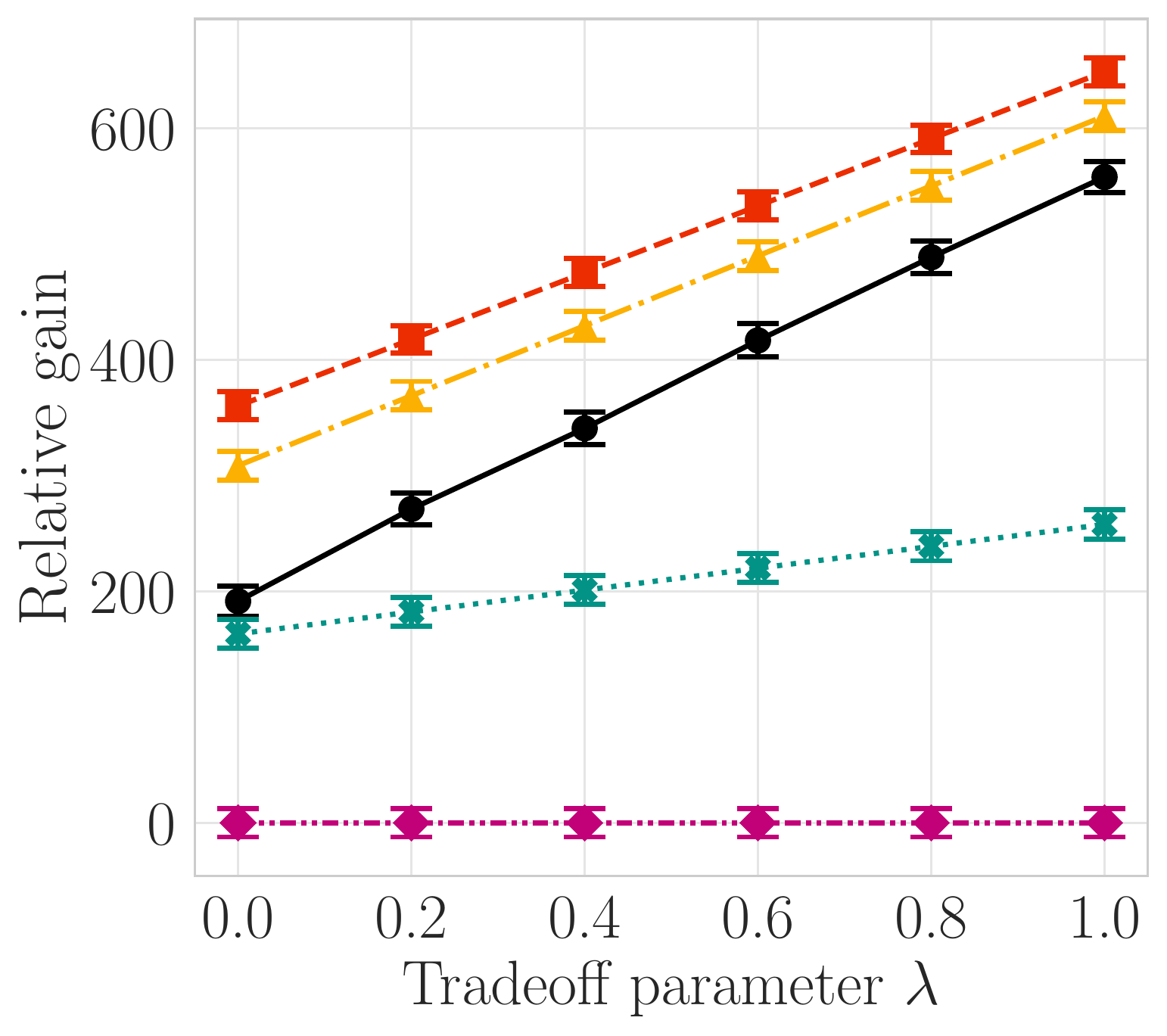}\label{fig:iclr20184dp}}\hfill
    
    \subfloat[][]{\includegraphics[width=0.2\textwidth]{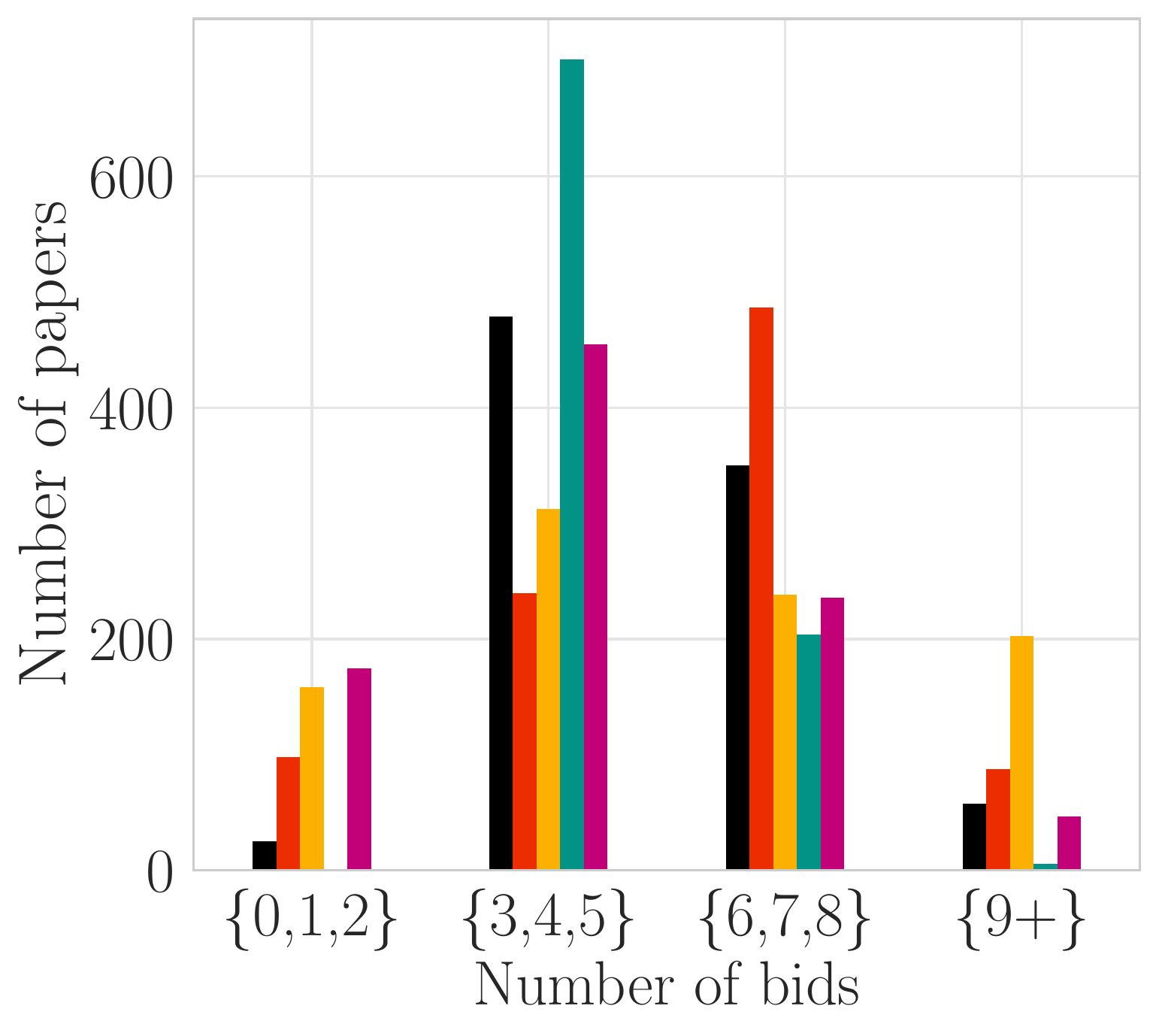}\label{fig:iclr20184e}}\hfill 
    \subfloat[][]{\includegraphics[width=0.2\textwidth]{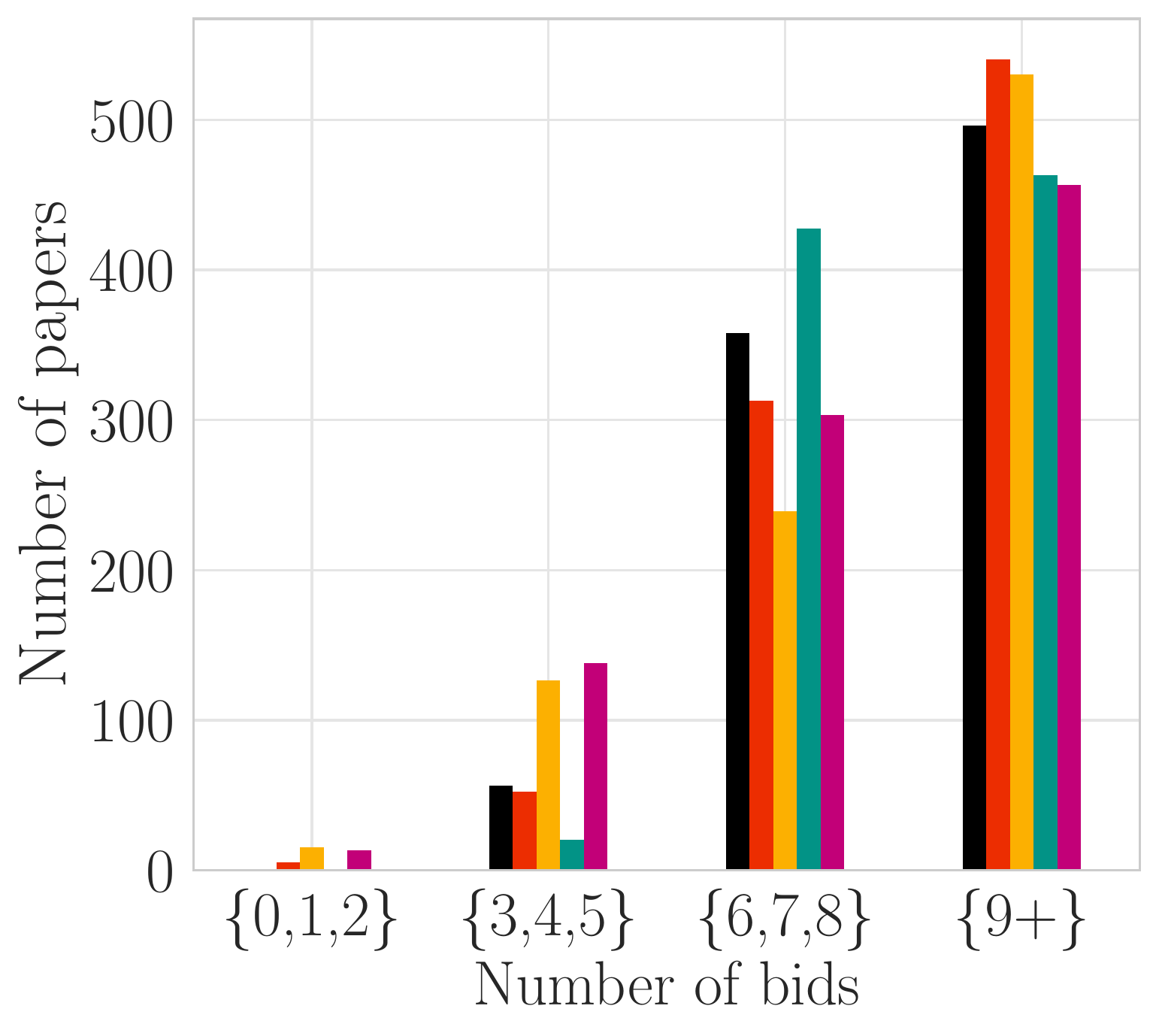}\label{fig:iclr20184f}}\hfill
    \subfloat[][]{\includegraphics[width=0.2\textwidth]{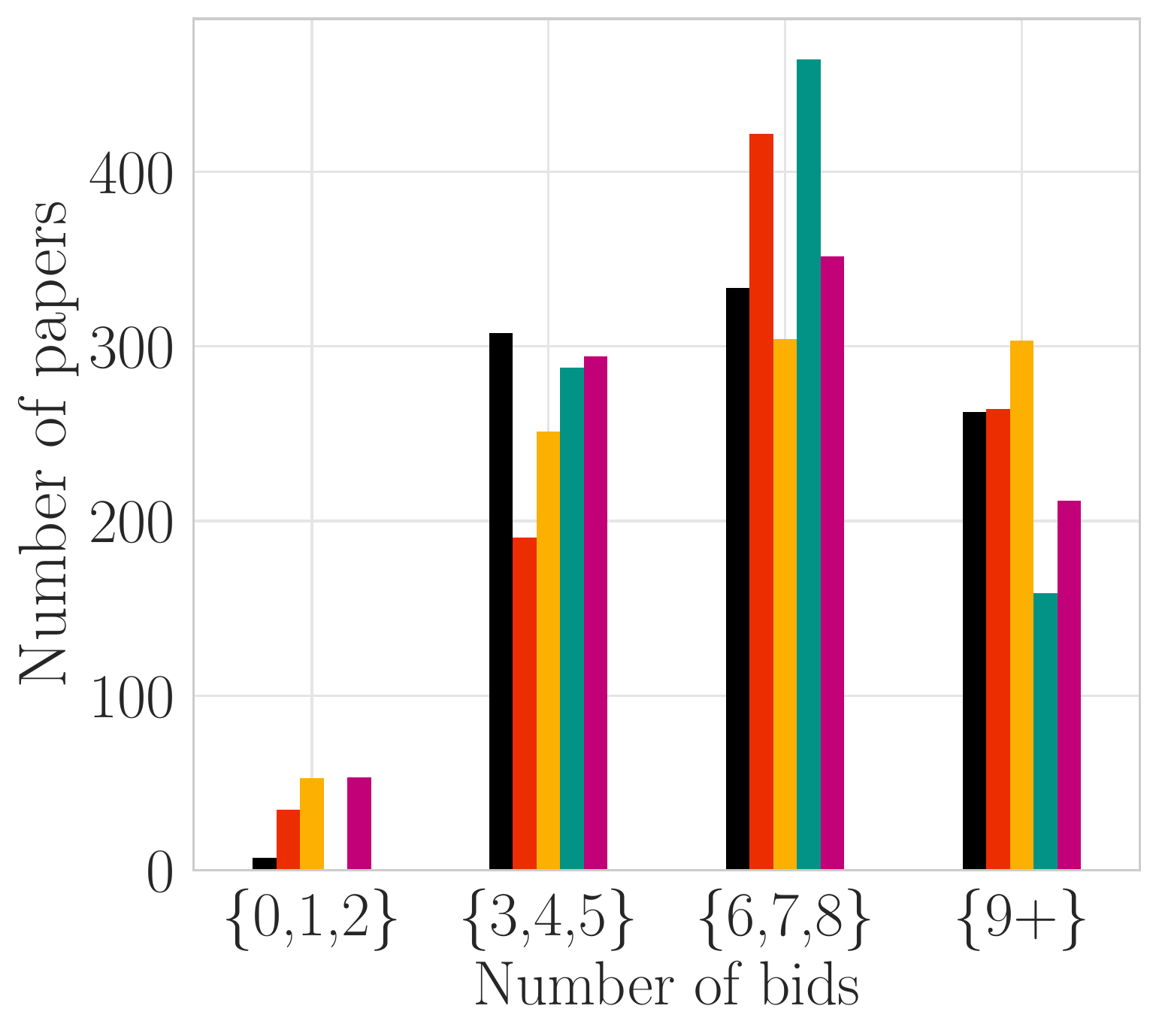}\label{fig:iclr20184g}}\hfill
    \subfloat[][]{\includegraphics[width=0.2\textwidth]{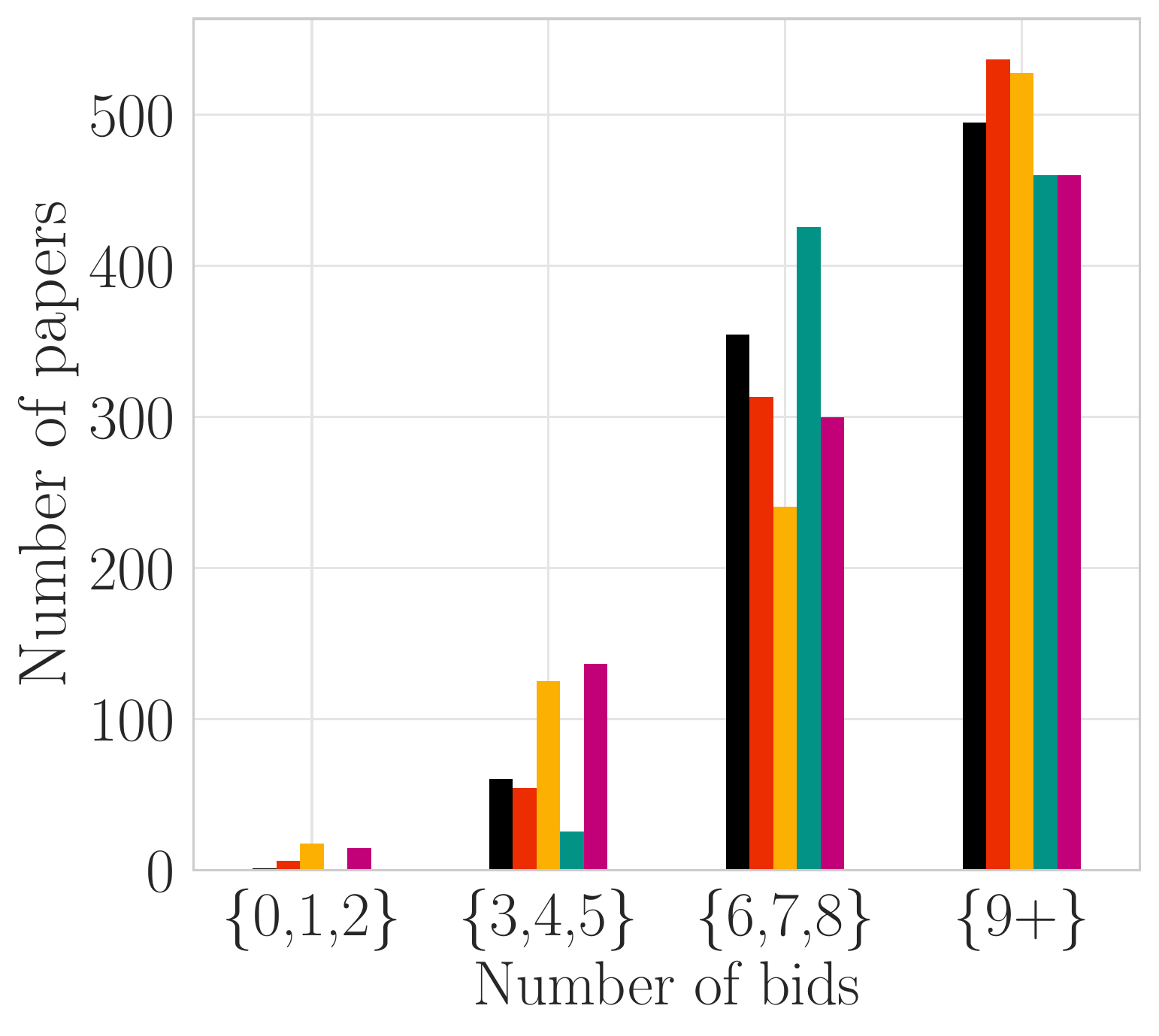}\label{fig:iclr20184h}}\hfill
    \subfloat[][]{\includegraphics[width=0.2\textwidth]{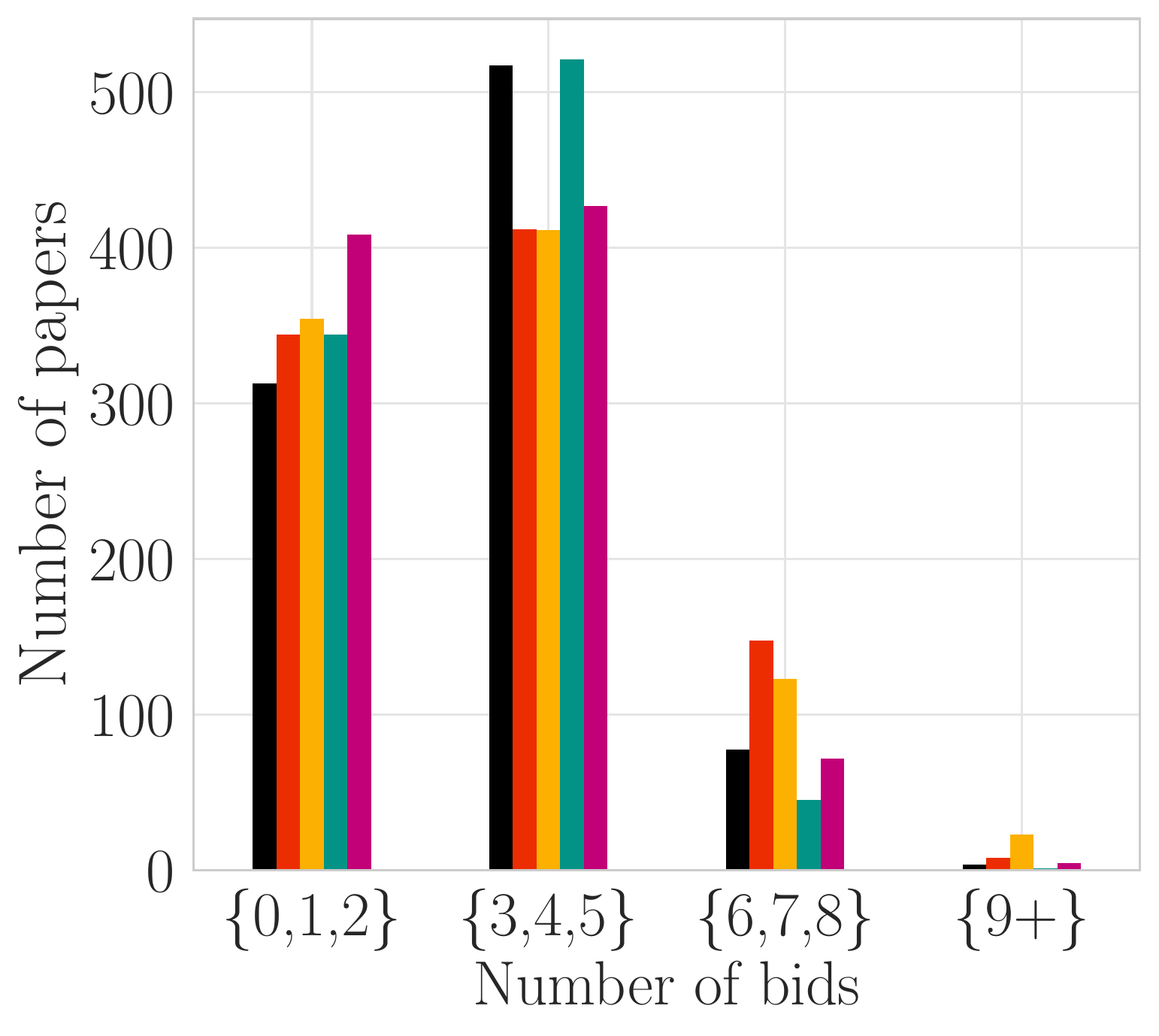}\label{fig:iclr20184hp}}\hfill
    \subfloat{\includegraphics[width=.9\textwidth]{Figs/legend}}
    \caption{\small ICLR 2018 experiments of robustness to real-world complexities. Legend for each bid distribution plot (f)--(j): within each bid-count interval, the bars 
    correspond to the algorithms given in the legend and are presented in an order consistent with the legend itself when read from left to right.
    }
    \label{fig:iclr20183}
\end{figure*}

\begin{figure*}[t]
    \begin{minipage}{\textwidth}
    \begin{minipage}{.25\textwidth}
    \centering
    \vspace{1mm}
    \footnotesize \textbf{Homogeneous}
    \end{minipage}%
    \begin{minipage}{.25\textwidth}
    \centering
    \footnotesize \textbf{Low rank}
    \end{minipage}%
    \begin{minipage}{.25\textwidth}
    \centering
    \vspace{1mm}
    \footnotesize \textbf{Community}
    \end{minipage}%
    \begin{minipage}{.25\textwidth}
    \centering
    \vspace{1mm}
    \footnotesize \textbf{Interdisciplinary}
    \end{minipage}%
    \end{minipage}

    \centering
    \subfloat[][]{\includegraphics[width=0.25\textwidth]{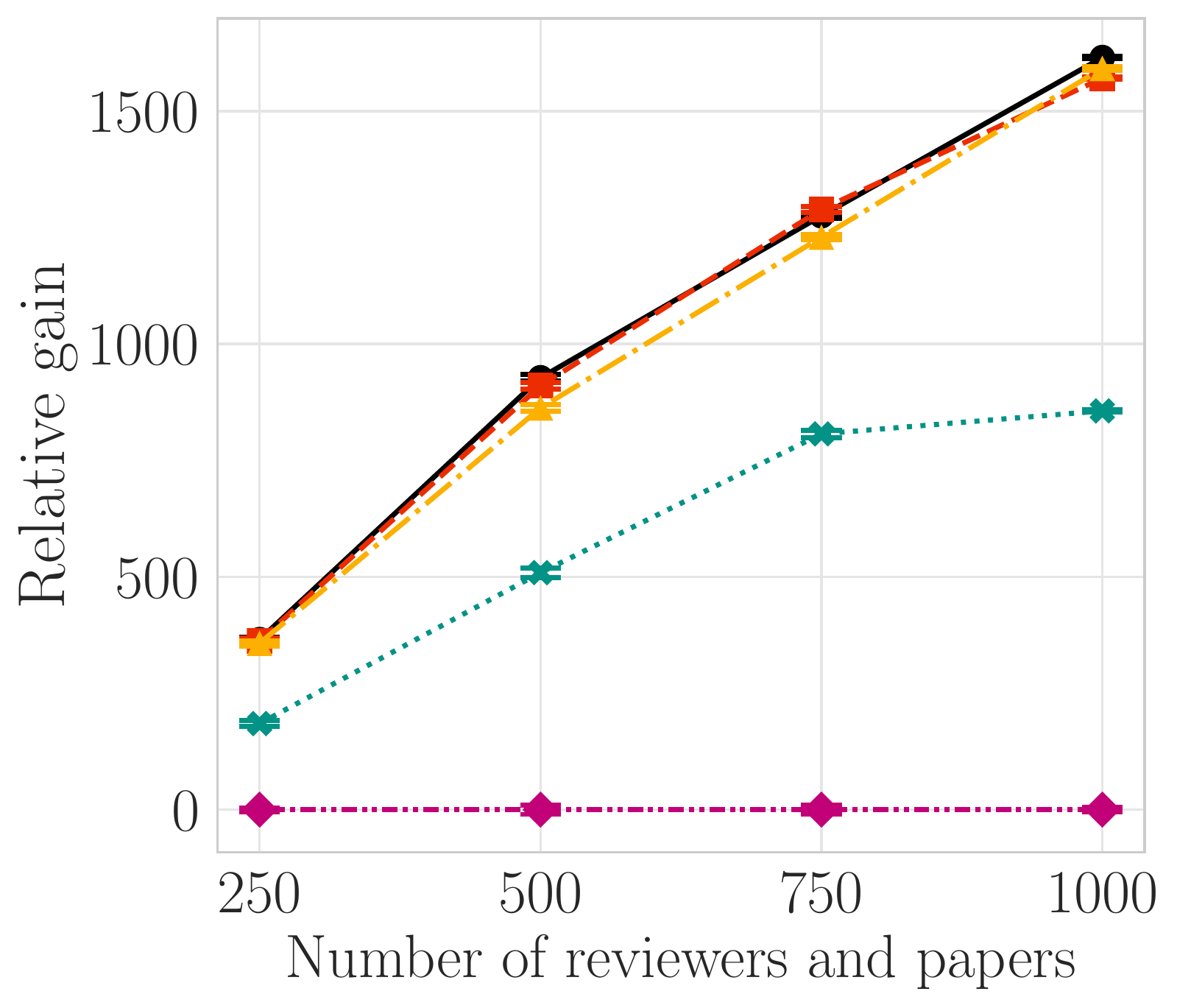}\label{fig:svary1a}}\hfill 
    \subfloat[][]{\includegraphics[width=0.25\textwidth]{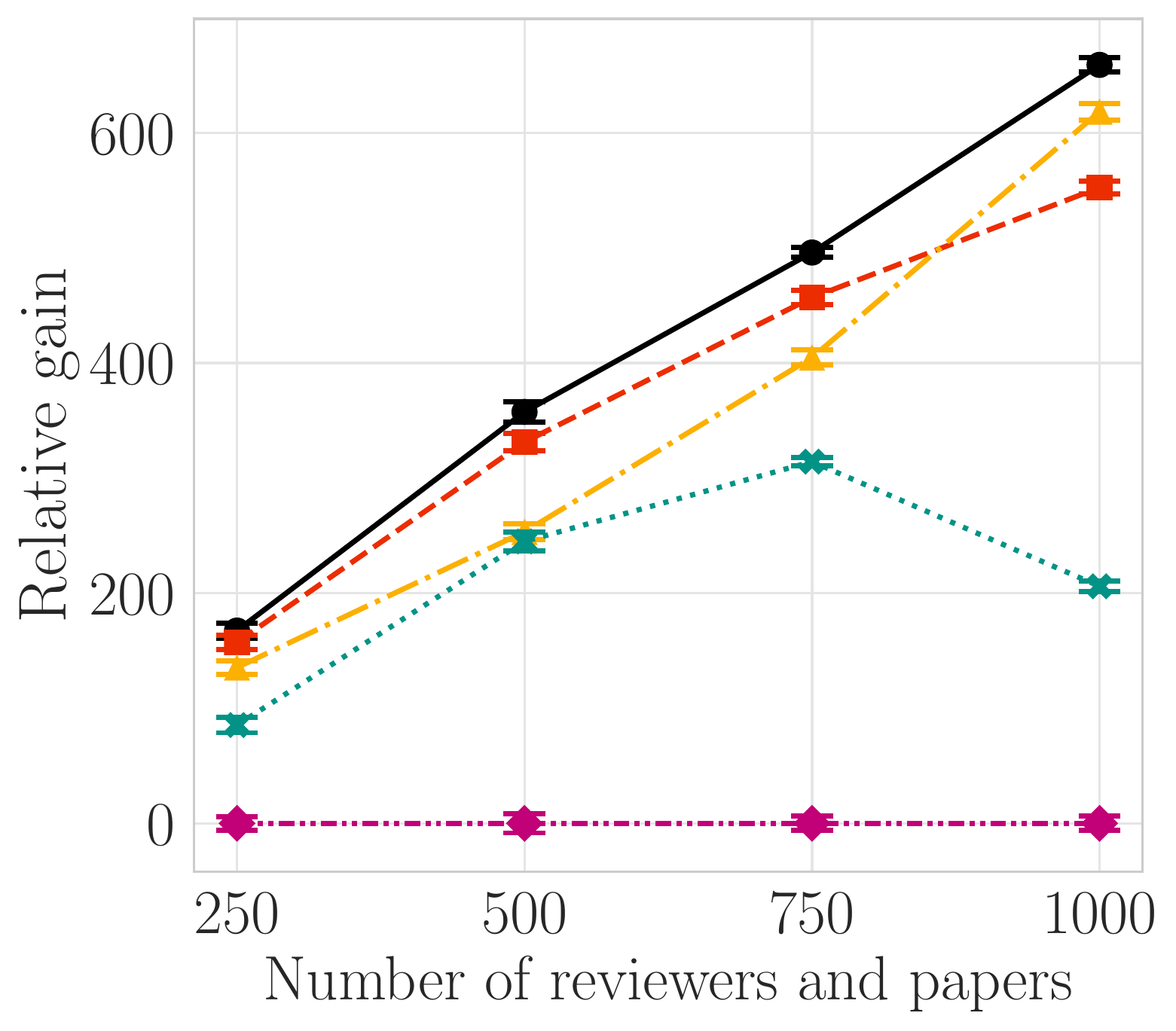}\label{fig:svary1b}}\hfill
    \subfloat[][]{\includegraphics[width=0.25\textwidth]{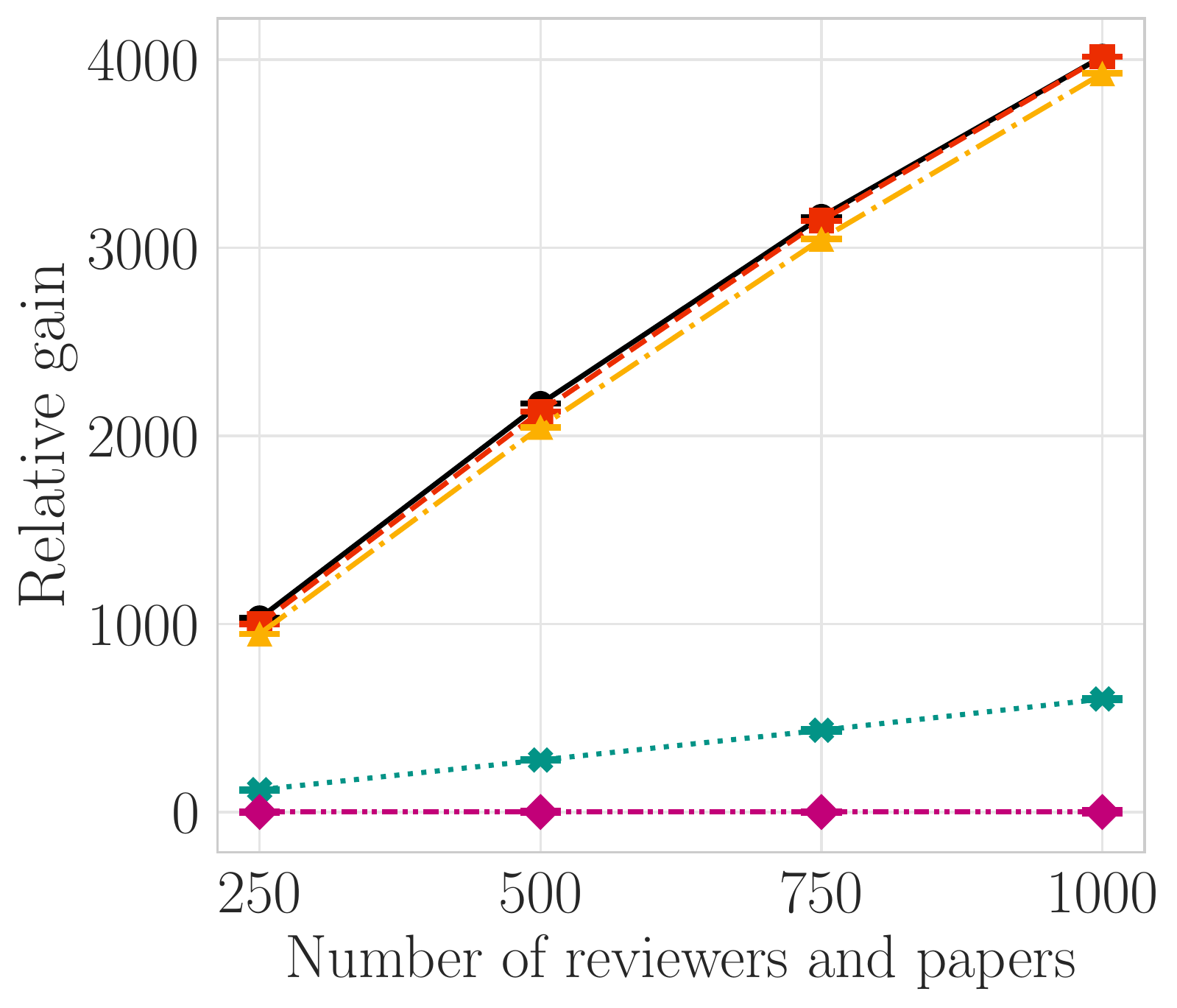}\label{fig:svary1c}}\hfill
    \subfloat[][]{\includegraphics[width=0.25\textwidth]{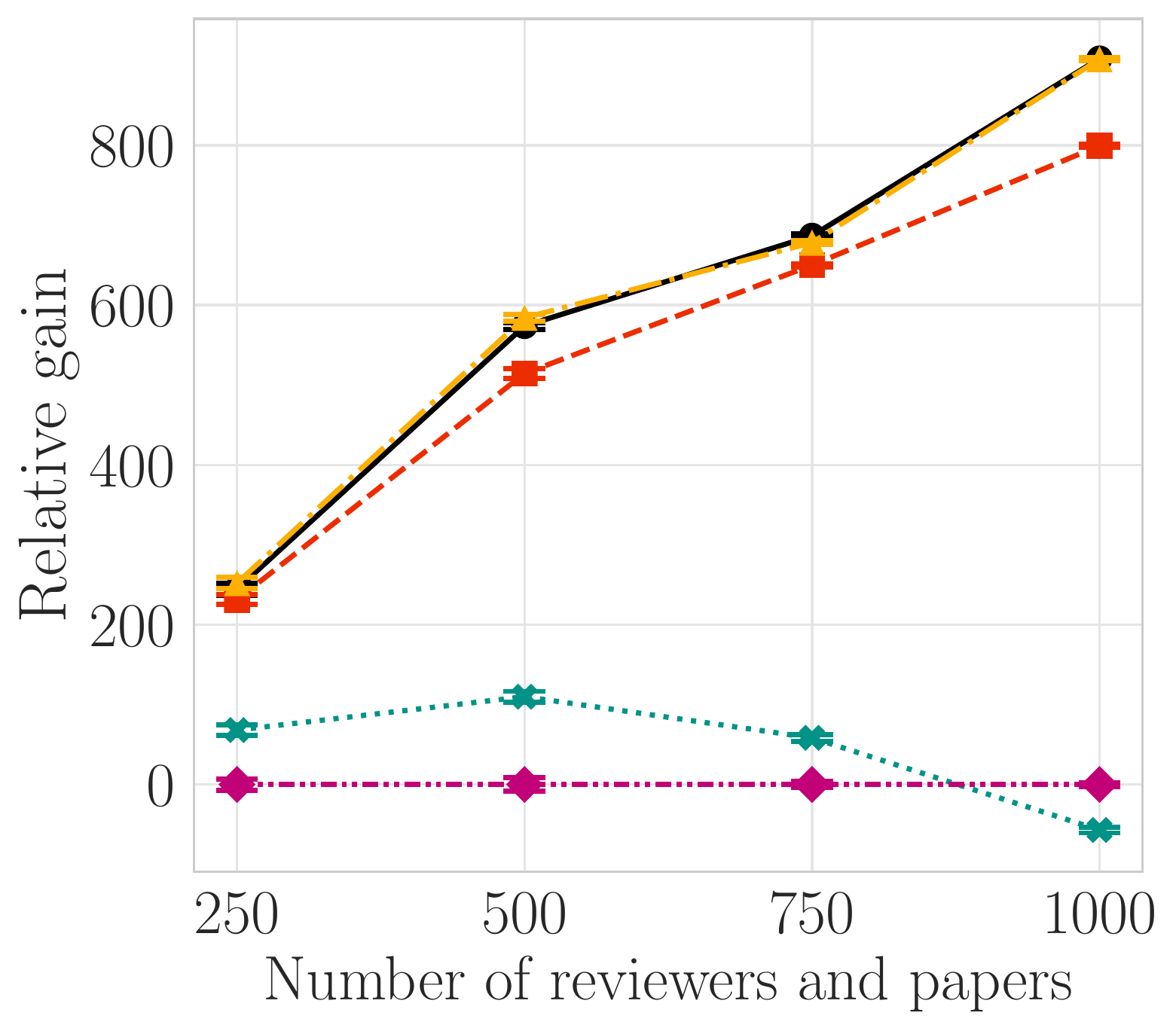}\label{fig:svary1d}}\hfill
    
    \subfloat[][]{\includegraphics[width=0.25\textwidth]{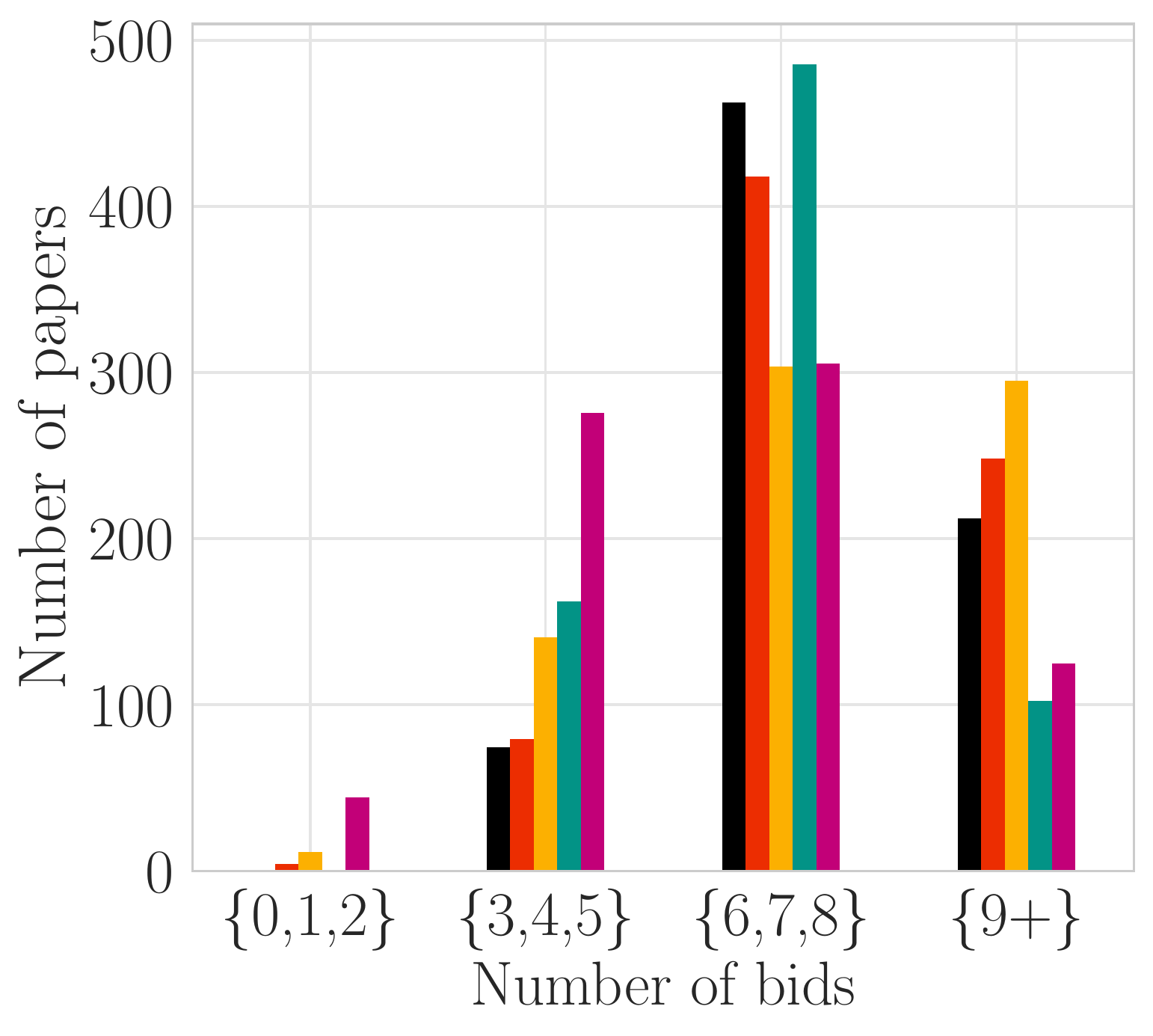}\label{fig:svary1e}}\hfill 
    \subfloat[][]{\includegraphics[width=0.25\textwidth]{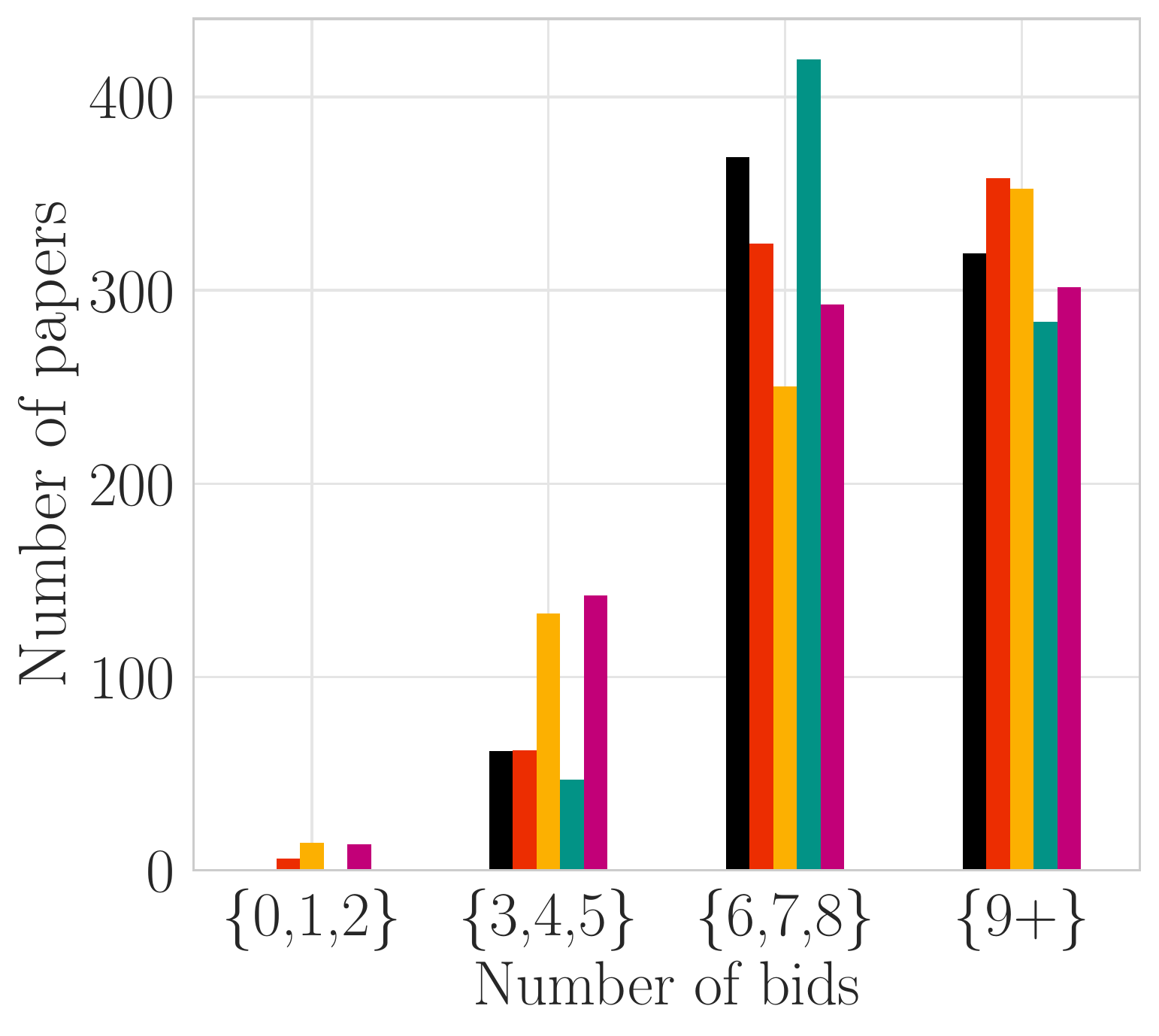}\label{fig:svary1f}}\hfill
    \subfloat[][]{\includegraphics[width=0.25\textwidth]{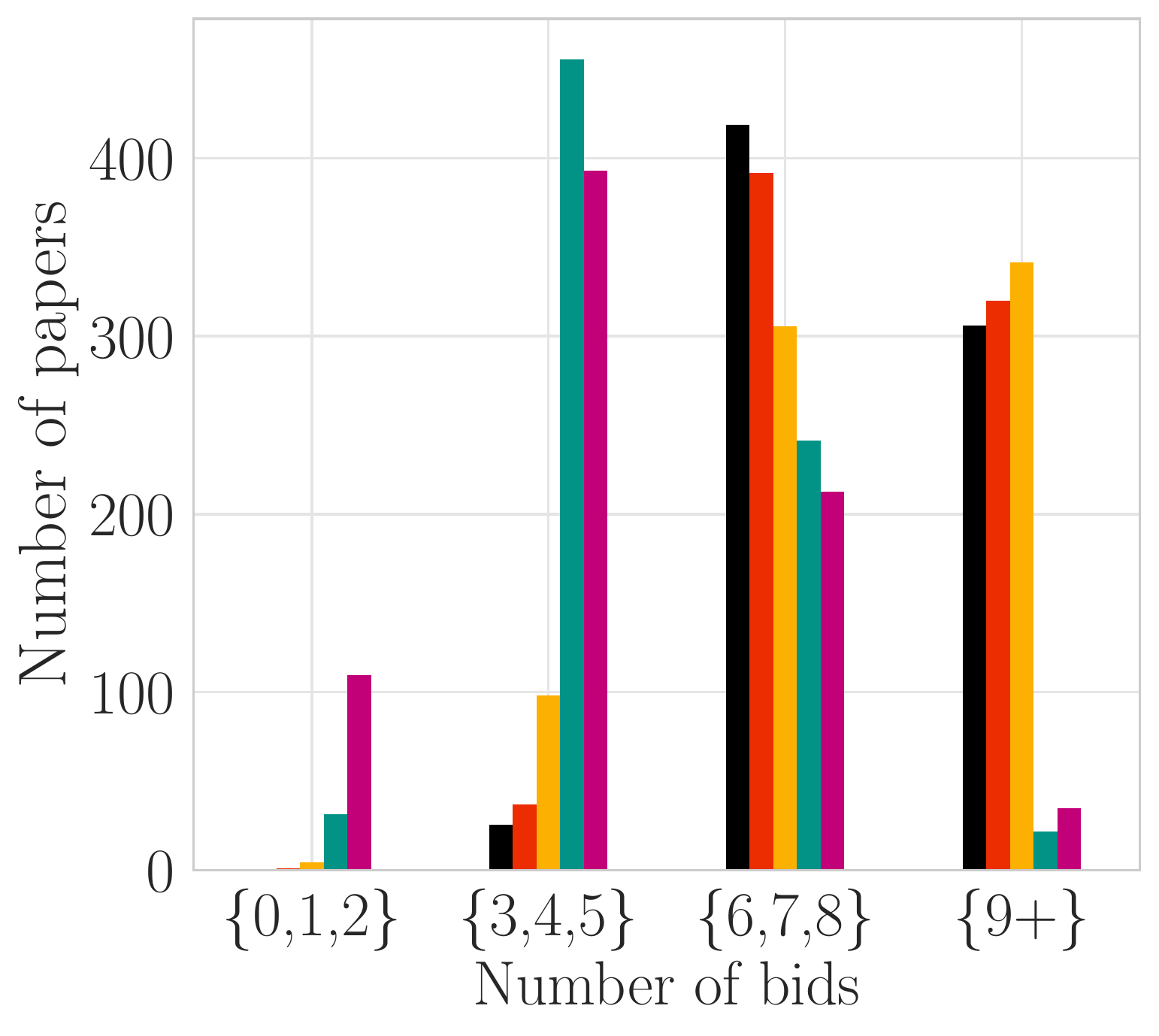}\label{fig:svary1g}}\hfill
    \subfloat[][]{\includegraphics[width=0.25\textwidth]{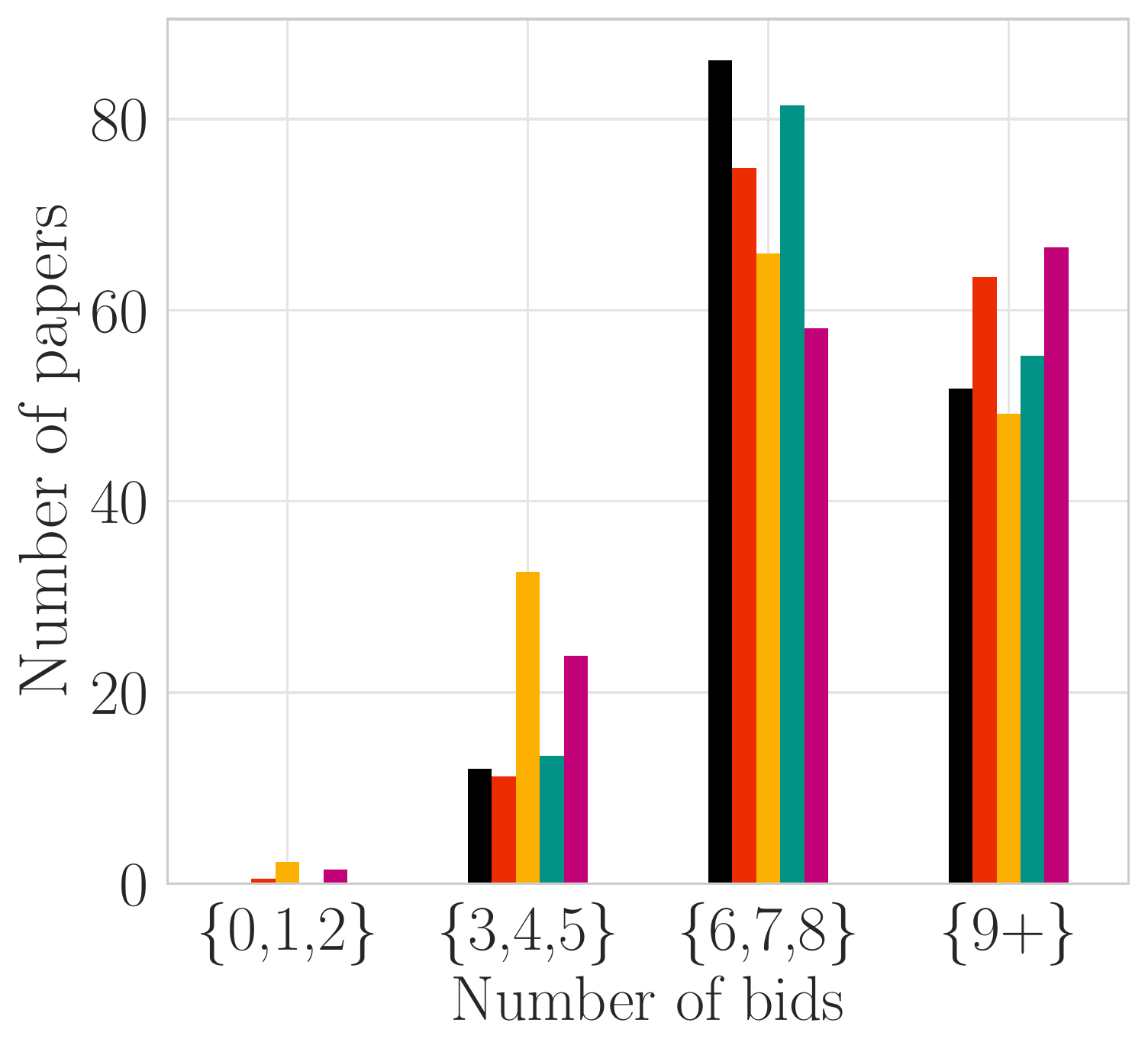}\label{fig:svary1h}}\hfill
    \subfloat{\includegraphics[width=.9\textwidth]{Figs/legend}}
    \caption{\small Experiments on synthetic similarity scores under the default model configuration. Legend for each bid distribution plot (e)--(h): within each bid-count interval, the bars 
    correspond to the algorithms given in the legend and are presented in an order consistent with the legend itself when read from left to right.
    }
    \label{fig:svary}
\end{figure*}

We begin looking at model mismatch for the bidding function. In Figures~\ref{fig:iclr20184a} and~\ref{fig:iclr20184e}, the situation is considered in which the actual bids are performed under the bidding function $\bidfunction(\ordering_{\revindex}(\paperindex), \similarity_{\revindex, \paperindex})=\similarity_{\revindex, \paperindex}/\sqrt{\ordering_{\revindex}(\paperindex)}$, whereas \super assumes $\bidfunction(\ordering_{\revindex}(\paperindex), \similarity_{\revindex, \paperindex}) = \similarity_{\revindex, \paperindex}/\log_2(\ordering_{\revindex}(\paperindex)+1)$. The results show each version of \super is robust to this deviation and outperforms the baselines in terms of the gain. Moreover, \super with zero heuristic is especially robust since it is not as dependent on the model of bids as \super with mean heuristic. 
The bid count distributions show that \super with mean heuristic reduces the number of papers with fewer than three bids by 85\% relative to \simbase, and \super with zero heuristic reduces the number of papers with fewer than six bids by 50\% compared to \bidbase.
In Figures~\ref{fig:iclr20184b} and~\ref{fig:iclr20184f}, we consider that the probability of a reviewer bidding on a paper is actually given by $\bidfunction(\ordering_{\revindex}(\paperindex), \similarity_{\revindex, \paperindex}) = (\similarity_{\revindex, \paperindex}+\mathcal{N}(0, \sigma^2))/\log_2(\ordering_{\revindex}(\paperindex)+1)$ where $\sigma=0.01$ and we remark that the mean of the similarity scores is approximately $0.03$ so the noise magnitude is not insignificant. We clip the noisy bid probabilities to guarantee that they remain in the interval $[0, 1]$. We observe that \super is again robust to this model mismatch and outperforms the baselines in terms of the gain and each version ends up reducing the fewer the number of papers having below the minimum desired number of bids by approximately 60\% compared to \simbase and \randbase.

In practice, not all reviewers may participate in the bidding process. We consider that only three quarters of the reviewers arrive, but this is unknown a priori to the algorithms. This proportion is roughly based on the number of reviewers that were found to not place any positive bids during the NeurIPS 2016 review process~\citep{shah:2018aa}. The results under this real-world complexity are shown in Figures~\ref{fig:iclr20184c} and~\ref{fig:iclr20184g}.
Moreover, reviewers may not actually arrive sequentially. 
Figures~\ref{fig:iclr20184d} and~\ref{fig:iclr20184h} consider the setting where Poisson(1) reviewers arrive at each time, and the algorithm must present paper orderings to all these reviewers simultaneously. For this pair of real-world complexities, \super remains quite robust and generally performs favorably compared to the baselines in terms of both the gain and the bid count distributions. It is worth noting that when not all reviewers arrive, \super with zero heuristic outperforms \super with zero heuristic since it does attempt to account for bids that may come from reviewers that have not arrived.

A common feature in peer review bidding systems is the ability for a reviewer to search papers by subject area or keyword and then only bid within the resulting subset of papers.  We now evaluate the robustness of \super to this type of reviewer behavior in the following manner. In our simulations, on arrival of any reviewer, one quarter of the papers are randomly selected and required to be shown to the reviewer (these are the papers that are assumed to meet the search query). The remaining papers are not presented to the reviewer, are bid on with probability zero, and do not factor into the reviewer-side gain. The algorithms compute the paper orderings over only the selected subset of papers. The results of this experiment are shown in Figures~\ref{fig:iclr20184dp} and~\ref{fig:iclr20184hp}. We observe that \super with zero heuristic outperforms the rest of the algorithms in terms of the gain while obtaining a favorable bid distribution. \super with mean heuristic is not quite as robust since fewer bids come from future reviewers than anticipated when computing the mean heuristic -- if one has estimates of the amount of selection done by reviewers via search, then this issue may be mitigated by appropriately scaling down the heuristic value.

\subsection{Synthetic Similarities}
\label{SecSyntheticSimulations}
 We perform several simulations on synthetic similarity scores comparing the algorithms as presented in Figure~\ref{fig:svary}. For this set of simulations, we consider the default model configuration from the previous section. 
 Moreover, for each similarity structure we let the number of reviewers and the number of papers $(\numrev, \numpapers)$ be among the set of pairs $\{(250, 250), (500,500),(750, 750), (1000,1000)\}$ and fix the trade-off parameter to be $\hyperparam = 0.8$ since this gave roughly equal paper-side and weighted reviewer-side gains for \randbase with $\numrev=\numpapers=750$. 
 
 \textbf{Homogeneous similarity scores.} 
 We consider a synthetic similarity matrix structure where each entry is drawn from at random from a Beta distribution with parameters $\alpha=1$ and $\beta = 15$. This distribution is highly skewed and the expected value of a draw from it is $0.0625$. The results of this experiment are given in Figures~\ref{fig:svary1a} and~\ref{fig:svary1e}. The gain of \super with each heuristic exceeds that of each of the baselines and similarly the bid count distributions show \super with each heuristic ends up with at least 50\% fewer papers obtaining under the minimum desired by the paper-side gain function compared to the baselines. We tried other homogeneous similarity structures and observed similar trends throughout.

 \textbf{Low rank structure.}
 We experimented with the following low rank structure to generate the similarity scores. The similarity matrix is split into 10 groups of reviewers which correspond to a block of rows in the matrix. The similarity scores for block $\ell \in \{1, \dots, 10\}$ are given by the rank-1 matrix $u(v^{\ell})^{\top}$ where $u$ is a $\numrev/10$-dimensional vector of ones and $v^{\ell}$ is a $\numpapers$-dimensional vector with each entry drawn at random from a Beta distribution with parameters $\alpha = \ell$ and $\beta =60$. This set of parameter choices for the distribution was made so that the scores were near the scale of the ICLR similarity matrix and to create a disparity in similarity scores between the blocks. Combining the blocks forms a similarity matrix of rank 10 where within each block the reviewers are identical and between the blocks the similarity score distribution changes. The structure can be viewed as if there are 10 types of reviewers with varying levels of relevance to the papers.
The result of the experiment with this similarity matrix structure is shown in Figures~\ref{fig:svary1b} and~\ref{fig:svary1f}. We again see that \super with each heuristic outperforms the baselines in terms of the gain and they obtain a favorable distribution of the bid counts with a 50\% reduction in the number of papers with fewer than six bids compared to \simbase and \randbase. 

 \textbf{Community model structure.} We now consider a community model type block structure as motivated in Section~\ref{sec:theoretical_properties}.
 To form this similarity matrix, we create a block matrix where each submatrix is of dimension $25\times 25$. The similarity score of each entry in the submatrix $(\ell, k)$ is 0 if $\ell \neq k$ and $0.7$ if $\ell = k$. In other words the matrix is block diagonal. We then add noise to each similarity score drawn uniformly at random from the interval $[0, 0.05]$. The results are given in Figures~\ref{fig:svary1c} and~\ref{fig:svary1g}. \bidbase and \randbase are highly suboptimal in terms of the gain and bid count distribution since they end up showing papers with similarity scores near zero early in the paper orderings even if later on there are reviewers to arrive which are closely matched to the papers. Remarkably, \super with each heuristic reduces the number of papers with fewer than the minimum number of desired bids by at least 90\% compared to \bidbase and \randbase. Moreover, \super with each heuristic outperforms \simbase in terms of the gain and reduces the number of papers with fewer than six bids by over 60\%. This happens since for papers with close similarity scores, \super shows the paper with fewer bids ahead if the similarity score is only marginally smaller.

 \textbf{Interdisciplinary papers.} To conclude our simulations, we consider the impact our algorithm could have on interdisciplinary papers. As mentioned in Section~\ref{sec:intro}, such papers face additional challenges in the peer review process owing to the lack of ideally matched peers. To simulate this phenomena, we consider a similarity matrix structure where there are two groups of reviewers of equal size and then three groups of papers which make up 40\%, 40\%, and 20\% of the papers, respectively. Each reviewer in group one has similarity scores of $0.17$, $0.005$, and $0.085$ with the respective paper groups and each reviewer in group two has similarity scores of $0.005$, $0.17$, and $0.085$ with the respective paper groups. This reflects a scenario where the reviewer pool has two distinct areas of expertise and there is a set of interdisciplinary papers (paper group with similarity score of $0.075$ with all reviewers). We show the results of the experiment in Figures~\ref{fig:svary1d} and~\ref{fig:svary1h}. In terms of the gain, \super and \simbase perform significantly outperform \bidbase and \randbase. For the bid count distribution in Figure~\ref{fig:svary1h}, we only consider the interdisciplinary papers. \super with each heuristic mitigates negative impacts on the interdisciplinary papers as the number with an insufficient number of bids is curtailed by 65\% and 50\% compared to \simbase and \randbase.
 This is a result of the fact that \super works to trade-off the paper-side and reviewer-side objectives so the interdisciplinary papers end up not always being shown after the papers matching the reviewers expertise as occurs for \simbase.

\section{Discussion}
\label{sec:discussion}
This work develops a principled framework to improve bidding in peer review. We develop a sequential decision-making algorithm called \super to optimize the process and show that it empirically outperforms baseline methods on real conference data and has several compelling theoretical guarantees. 

This work leads to several interesting and potentially impactful open problems:
\begin{itemize}
\item An obvious open problem is that of developing an online algorithm that is globally optimal for every similarity matrix. Conversely, showing possible computational hardness of the problem as a negative result could be a path of future work.
\item In several automated reviewer-paper assignment methods, bids and and similarities are combined to form the scores used to compute the assignment~\citep{shah:2018aa}. Given the tight connection between the bidding and matching systems, it is natural to design methods for jointly optimizing the components that govern the assignment process.
\item Finally, given the online nature of reviewer arrivals and the need to immediately show the paper ordering to an arriving reviewer, it is of interest to solve the passive problem. That is, develop an algorithm which selects the paper orderings to present to each reviewer before any of the reviewers arrive, which can be presented to a reviewer in the event of insufficient computational time. We also think that a solution to the passive problem would serve as an effective heuristic function for \super. 
\end{itemize}

Although our work primarily focuses on peer review, there are a number of other applications for which this work is relevant. One such application is crowdsourcing, where a common goal of the crowdsourcing platform is to ensure that each task gets sufficiently many qualified workers. From the perspective of the worker, it has been documented quite extensively that workers put a non-trivial emphasis on a task if it is of interest to them~\citep{kaufmann2011more, hossain2012users}. This means there is a trade-off of ensuring each worker is satisfied while ensuring a minimum number of workers for each job. As such, it is reasonable to formulate the crowdsourcing problem as a multi-objective optimization problem using analogous approach to the one presented in this paper for the bidding process in peer review. Indeed, the crowdsourcing platform in this formulation seeks to optimize both for a task-side objective, which ensures each task gets a sufficient number of workers selecting it, and for a user-side objective, which is to  present relevant tasks to users. 

Another potentially viable application is crowdfunding and microlending platforms such as Kiva or KickStarter. In crowdfunding, users pledge toward funding a project, and the project is only funded if the cumulative contributions of the crowd meet a known target threshold. The platforms seek to maximize the number of funded projects by deciding, when, how often, and to which users the projects are displayed. Past work~\citep{jain2018firing} has modeled the optimization as a multi-armed bandit problem where the goal is to maximize the number of funded projects with a minimum amount of user impressions. This problem has a fundamental trade-off between showing relevant projects to users while ensuring that the projects themselves are given a fair shot to be funded. As such, a model-dependent approach such of ours could be of potential interest. 

A final potential application outside of peer review for our work is in recommender systems and online advertising where there is the common trade-off between exploration (gaining sufficient feedback on all items) and exploitation (showing the most relevant items to users). In recommender systems, the cold start problem refers to the situation where the system is just beginning to interact with users or items are freshly included in the catalogue and no past user interaction information is available. The common trade-off arises again where there is a need to show relevant items to users, while the system benefits from gaining feedback on items for which the utility is unknown. Our framework easily adapts to a changing action set and could be relevant to this task.

\section*{Acknowledgments}
This work was supported by NSF grants CIF 1763734, CAREER: CIF 1942124	and CRII: CIF 1755656. Tanner Fiez was also supported by the DoD NDSEG Fellowship Program.

\bibliographystyle{abbrvnat}
\bibliography{bibtex}

\begin{appendices}
\renewcommand\contentsname{Table of Contents}
\tableofcontents
\addtocontents{toc}{\setcounter{tocdepth}{3}} 

\section{Proofs}\label{sec:proofs}
In this section, we present proofs of the theoretical results presented in the main text.  We begin by formally defining the history of information available to an algorithm when a reviewer arrives and restating important notation that will be found throughout the proofs.

\begin{definition}[History]
The history of information available to an algorithm when reviewer $\revindex \in \reviewerset$ arrives is defined as $\history_{\revindex-1}= \{\ordering_{\ell}, \realbid_{\ell}: \ell=1,\dots,\revindex-1\}$, where $\ordering_{\ell} \in \symgroup_{\numpapers}$ is the paper ordering that was presented to reviewer $\ell$ and $\realbid_{\ell} \in \{0, 1\}^{\numpapers}$ contains the bid realizations on each paper from reviewer $\ell$. 
\label{def:history}
\end{definition}

\paragraph*{Notation.}
We let $\numpapers\geq 2$ denote the number of papers and $\numrev\geq 2$ denote the number of reviewers. We use the notation $\similarity_{\revindex, \paperindex} \in [0,1]$ to denote the given similarity score between any reviewer $\revindex \in [\numrev]$ and paper $\paperindex \in [\numpapers]$. In general, we let $\revindex$ index a reviewer and $\paperindex$ index a paper. We commonly use the set notation $[\kappa] = \{1,2,\dots, \kappa\}$ for any positive integer $\kappa$.
$\symgroup_\numpapers$ denotes the set of $\numpapers!$ permutations of $\numpapers$ papers. For any reviewer $\revindex \in \reviewerset$, we let $\ordering_\revindex \in \symgroup_\numpapers$  denote the ordering (permutation) of the papers shown to reviewer $\revindex$. We use the notation $\ordering_\revindex(\paperindex)$ to denote the position of paper $\paperindex \in \paperset$ in the ordering $\ordering_\revindex$.  
The notation $\randombid_{\revindex, \paperindex}$ represents the random variable of reviewer $\revindex \in \reviewerset$ biding on paper $\paperindex \in \paperset$ which follows a Bernoulli distribution with parameter $\bidprob_{\revindex, \paperindex}= \bidfunction(\ordering_{\revindex}(\paperindex), \similarity_{\revindex, \paperindex})$. We use the notation $\numbids_{\revindex-1, \paperindex} \in \{0,\ldots,\revindex-1\}$ to denote the number of bids received by any paper $\paperindex \in \paperset$ at the time of arrival of reviewer $\revindex \in \reviewerset$ and $\numbids_{\paperindex}$ to denote the number of bids at the end of the bidding process on paper $\paperindex \in \paperset$. The heuristic estimating the number of bids paper $\paperindex \in \paperset$ will obtain from reviewers $\{\revindex+1, \dots, \numrev\}$ is denoted as $\proxy_{\revindex, \paperindex}$ and it is provided to the \super algorithm on arrival of reviewer $\revindex \in \reviewerset$. Finally, we abbreviate `with probability' by w.p and use the terminology almost surely to mean with probability one and almost never to mean with probability zero.

\subsection{Proof of Theorem~\ref{prop:local}: Local Optimality of \super for Final Reviewer}
\label{sec:proof_local}
In this proof, we show that selecting the optimal ordering to present to the final reviewer can be simplified to a tractable optimization problem. The $\super$ algorithm solves exactly this problem to determine an ordering of papers to present the final reviewer, and is hence an optimal algorithm for the final reviewer.

The optimization problem for the final reviewer $\numrev$ to maximize the expected gain conditioned on the history is
\begin{equation}
\max_{\ordering_{\numrev} \in \symgroup_\numpapers} \quad \sum_{\paperindex\in \paperset} \mathbb{E}[\gainfunction(\numbids_{\paperindex})|\history_{\numrev-1}]+\hyperparam\sum_{\revindex \in \reviewerset}\sum_{\paperindex\in \paperset}\mathbb{E}[\gainfunctionrev(\ordering_{\revindex}(\paperindex), \similarity_{\revindex,\paperindex})|\history_{\numrev-1}] 
\label{eq:app_opt}
\end{equation}
where the expectation is taken over the randomness in the bid to be placed by the reviewer. Conditioned on the history $\history_{\numrev-1}$, the final number of bids on any paper $\paperindex \in \paperset$ given by $\numbids_{\paperindex}$ is the sum of the deterministic number of bids prior to the final reviewer denoted as $\numbids_{\numrev-1, \paperindex}$ and a Bernoulli random variable $\randombid_{\numrev, \paperindex}$ with parameter $\bidprob_{\numrev, \paperindex} = \bidfunction(\ordering_{\numrev}(\paperindex), \similarity_{\numrev, \paperindex})$ representing the random bid of the final reviewer. We incorporate this fact and remove terms independent of the optimization variable from the problem in~\eqref{eq:app_opt} to equivalently obtain
\begin{equation}
\max_{\ordering_{\numrev} \in \symgroup_\numpapers} \quad \sum_{\paperindex\in \paperset} \mathbb{E}[\gainfunction(\numbids_{\numrev-1, \paperindex} + \randombid_{\numrev,\paperindex})]+ \hyperparam \sum_{\paperindex\in \paperset}\gainfunctionrev(\ordering_{\numrev}(\paperindex), \similarity_{\numrev, \paperindex})
\label{eq:app_opt2}
\end{equation}
where the expectation on the reviewer-side gain is removed as it is independent of the random reviewer bids.

We now simplify the paper-side gain term by expanding the expectation for each $\paperindex \in \paperset$. Observe that 
\begin{equation*}
\gainfunction(\numbids_{\numrev-1, \paperindex} + \randombid_{\numrev,\paperindex}) = 
\begin{cases}
\gainfunction(\numbids_{\numrev-1, \paperindex} + 1) & \text{w.p.}\quad \bidfunction(\ordering_{\numrev}(\paperindex), \similarity_{\numrev, \paperindex})\\
\gainfunction(\numbids_{\numrev-1, \paperindex}) & \text{w.p.}\quad 1- \bidfunction(\ordering_{\numrev}(\paperindex), \similarity_{\numrev, \paperindex}).
\end{cases}
\end{equation*}
Therefore,
\begin{equation}
\mathbb{E}[\gainfunction(\numbids_{\numrev-1, \paperindex} + \randombid_{\numrev,\paperindex})] 
= \bidfunction(\ordering_{\numrev}(\paperindex), \similarity_{\numrev, \paperindex})(\gainfunction(\numbids_{\numrev-1, \paperindex} + 1) - \gainfunction(\numbids_{\numrev-1, \paperindex})) + \gainfunction(\numbids_{\numrev-1, \paperindex}).\label{eq:prob_simplify}
\end{equation}
Substituting~\eqref{eq:prob_simplify} into~\eqref{eq:app_opt2} and removing the term independent of the optimization variable gives the problem
\begin{equation}
\max_{\ordering_{\numrev} \in \symgroup_\numpapers} \quad \sum_{\paperindex\in \paperset} \bidfunction(\ordering_{\numrev}(\paperindex), \similarity_{\numrev, \paperindex})(\gainfunction(\numbids_{\numrev-1, \paperindex} + 1) - \gainfunction(\numbids_{\numrev-1, \paperindex})) + \hyperparam \sum_{\paperindex \in \paperset}\gainfunctionrev(\ordering_{\numrev}(\paperindex), \similarity_{\numrev, \paperindex}).
\label{eq:opt_preinteger}
\end{equation}

We now reformulate~\eqref{eq:opt_preinteger} into the following equivalent integer linear programming problem:  
\begin{equation*}
\begin{split}
\max_{\optvar\in \reals^{\numpapers \times \numpapers}} &\quad  \sum_{\paperindex\in \paperset}\sum_{\extravar \in \paperset}\bidfunction(\extravar, \similarity_{\numrev, \paperindex})(\gainfunction(\numbids_{\numrev-1, \paperindex} + 1) - \gainfunction(\numbids_{\numrev-1, \paperindex})) \optvar_{\paperindex,\extravar}  + \hyperparam\sum_{\paperindex\in \paperset}\sum_{\extravar \in \paperset}\gainfunctionrev(\extravar, \similarity_{\numrev, \paperindex}) \optvar_{\paperindex, \extravar} \\
\text{s.t.} &\quad \sum_{\extravar \in \paperset} \optvar_{\paperindex,\extravar} = 1 \ \forall \ \paperindex \in \paperset,  \quad \sum_{\paperindex \in \paperset} \optvar_{\paperindex,\extravar} = 1 \ \forall \ \extravar \in \paperset, \quad \optvar_{\paperindex,\extravar} \in \{0,1\} \ \forall \ \paperindex,\extravar \in \paperset.
\end{split}
\label{eq:opt_prelinear}
\end{equation*}
In this formulation, $\optvar$ is a $\numpapers \times \numpapers$ matrix for which $\optvar_{\paperindex,\extravar}$ is an indicator of paper $\paperindex \in \paperset$ being shown at position $\extravar \in \paperset$. The constraint $\sum_{\extravar \in \paperset} \optvar_{\paperindex,\extravar} = 1 \ \forall \ \paperindex \in \paperset$ ensures each paper is included strictly once in the ordering shown to the reviewer. The constraint $\sum_{\paperindex \in \paperset} \optvar_{\paperindex,\extravar} = 1 \ \forall \ \extravar \in \paperset$ ensures strictly one paper is selected to be shown at each position. The final constraint ensures that each index of $\optvar$ is integer valued. This integer linear programming problem is known as a linear sum assignment problem.

The key step of this proof is to recall that the linear sum assignment problem can be solved as a linear program. Indeed, the final constraint ensuring an integer solution can be relaxed to $0 \leq \optvar_{\paperindex,\extravar} \leq 1 \ \forall \ \paperindex,\extravar \in \paperset$ and the optimal solution of the relaxed linear program will be the integer optimal solution. 
This property of the linear sum assignment problem is a consequence of the relaxed linear program containing a totally unimodular constraint set which guarantees the optimal solution to be the integral solution (see, e.g., Chapter 4 in~\citealp{burkard2012assignment}). 

The optimization problem arising from recognizing that the integer constraint can be relaxed is given by
\begin{equation}
\begin{split}
\max_{\optvar\in \reals^{\numpapers \times \numpapers}} &\quad  \sum_{\paperindex\in \paperset}\sum_{\extravar\in \paperset} \weight_{\paperindex,\extravar}\optvar_{\paperindex,\extravar}\\
\text{s.t.} &\quad \sum_{\extravar \in \paperset} \optvar_{\paperindex,\extravar} = 1 \ \forall \ \paperindex \in \paperset, \quad \sum_{\paperindex \in \paperset} \optvar_{\paperindex,\extravar} = 1 \ \forall \ \extravar \in \paperset, \quad  0\leq \optvar_{\paperindex,\extravar} \leq 1 \ \forall \ \paperindex,\extravar \in \paperset,
\end{split}
\label{eq:opt_linear}
\end{equation}
where 
\begin{equation}
\weight_{\paperindex, \extravar} = \bidfunction(\extravar, \similarity_{\numrev, \paperindex})(\gainfunction(\numbids_{\numrev-1, \paperindex} + 1) - \gainfunction(\numbids_{\numrev-1, \paperindex})) + \hyperparam \gainfunctionrev(\extravar, \similarity_{\numrev, \paperindex}) \ \forall \ \paperindex,\extravar \in \paperset.
\label{eq:opt_final_weights}
\end{equation}
This formulation shows that the paper ordering for the final reviewer that maximizes the expected gain conditioned on the history can be obtained efficiently by solving a linear program.

\paragraph*{Local optimality of \super for final reviewer.}
To determine the ordering of papers to show any reviewer $\revindex \in \reviewerset$ for the general class of assumed gain and bidding functions, \super calls Algorithm~\ref{alg:subprocedure}. Algorithm~\ref{alg:subprocedure} solves the optimization problem in~\eqref{eq:opt_linear} using the weights 
\begin{equation}
\weight_{\paperindex, \extravar} = \bidfunction(\extravar, \similarity_{\revindex, \paperindex})(\gainfunction(\numbids_{\revindex-1, \paperindex} + \proxy_{\revindex, \paperindex} + 1) - \gainfunction(\numbids_{\revindex-1, \paperindex} + \proxy_{\revindex, \paperindex})) + \hyperparam \gainfunctionrev(\extravar, \similarity_{\revindex, \paperindex}) \ \forall \ \paperindex,\extravar \in \paperset.
\label{eq:superweights}
\end{equation}
Given any heuristic function, the heuristic for the final reviewer is such that $\proxy_{\numrev} =0$. This means \super selects the paper ordering for the final reviewer by solving the optimization problem from~(\ref{eq:opt_linear}--\ref{eq:opt_final_weights}) to maximize the expected gain conditioned on the history. As a result, we conclude \super is locally optimal for the final reviewer with any heuristic function. 

\subsection{Proof of Corollary~\ref{cor:local}: Local Optimality of \super for Any Reviewer}
\label{sec:proof_local_col}
The immediate gain from any reviewer $\revindex \in \reviewerset$ is the difference between the gain after and before
the reviewer arrived. Formally, the immediate gain from any reviewer $\revindex \in \reviewerset$ is given by the quantity
\begin{equation*}
\gain_{\revindex} =  \sum_{\paperindex \in \paperset} \Big(\gainfunction\Big(\sum_{\ell\in [\revindex]}\randombid_{\ell, \paperindex}\Big) - \gainfunction\Big(\sum_{\ell\in [\revindex-1]}\randombid_{\ell, \paperindex}\Big)\Big) + \hyperparam\sum_{\paperindex \in [\numpapers]}  \gainfunctionrev(\ordering_{\revindex}(\paperindex),\similarity_{\revindex,\paperindex}),
\end{equation*}
where $\randombid_{\ell, \paperindex}$ is a Bernoulli random variable representing the random bid of a reviewer $\ell \in \reviewerset$ on a paper $\paperindex \in \paperset$ that depends on the position the paper was shown to the reviewer. 

The optimization problem to maximize the immediate expected gain from reviewer $\revindex \in \reviewerset$ conditioned on the history (see Definition~\ref{def:history}) is thus given by 
\begin{equation*}
\max_{\ordering_{\revindex} \in \symgroup_\numpapers} \quad \sum_{\paperindex \in \paperset} \mathbb{E}\Big[\gainfunction\Big(\sum_{\ell\in [\revindex]}\randombid_{\ell, \paperindex}\Big) - \gainfunction\Big(\sum_{\ell\in [\revindex-1]}\randombid_{\ell, \paperindex}\Big)\big|\history_{\revindex-1}\Big] + \hyperparam\sum_{\paperindex \in [\numpapers]}  \mathbb{E}[\gainfunctionrev(\ordering_{\revindex}(\paperindex),\similarity_{\revindex,\paperindex})].
\end{equation*}
Since $\sum_{\ell\in [\revindex-1]}\randombid_{\ell, \paperindex}$ is the deterministic bid count $\numbids_{\revindex-1, \paperindex}$ conditioned on the history $\history_{\revindex-1}$ for each paper $\paperindex\in \paperset$ and the reviewer-side gain from reviewer $\revindex \in \reviewerset$ is deterministic given a fixed paper ordering $\ordering_{\revindex} \in \symgroup_{\numpapers}$, the previous optimization problem is equivalently given by
\begin{equation}
\max_{\ordering_{\revindex} \in \symgroup_\numpapers} \quad \sum_{\paperindex \in \paperset} \mathbb{E}[\gainfunction(\numbids_{\revindex-1, \paperindex} + \randombid_{\revindex, \paperindex}) - \gainfunction(\numbids_{\revindex-1, \paperindex})] + \hyperparam\sum_{\paperindex \in [\numpapers]}  \gainfunctionrev(\ordering_{\revindex}(\paperindex),\similarity_{\revindex,\paperindex}).
\label{eq:cor_opt}
\end{equation}
We now evaluate the expectation in~\eqref{eq:cor_opt} using~\eqref{eq:prob_simplify} and then simplify to obtain the optimization problem
\begin{equation}
\max_{\ordering_{\revindex} \in \symgroup_\numpapers} \quad \sum_{\paperindex\in \paperset} \bidfunction(\ordering_{\revindex}(\paperindex), \similarity_{\revindex, \paperindex})(\gainfunction(\numbids_{\revindex-1, \paperindex} + 1) - \gainfunction(\numbids_{\revindex-1, \paperindex})) + \hyperparam \sum_{\paperindex \in \paperset}\gainfunctionrev(\ordering_{\revindex}(\paperindex), \similarity_{\revindex, \paperindex}).
\label{eq:cor_opt2}
\end{equation}
The problem in~\eqref{eq:cor_opt2} is equivalent to that given in~\eqref{eq:opt_preinteger} from the proof of Theorem~\ref{prop:local} up to the reviewer index. Consequently, we 
follow the steps after~\eqref{eq:opt_preinteger} in the proof of Theorem~\ref{prop:local} to simplify~\eqref{eq:cor_opt2} into a tractable representation. In doing so, we get that the problem in~\eqref{eq:cor_opt2} is equivalent to the linear program in~\eqref{eq:opt_linear} with the weights
\begin{equation*}
\weight_{\paperindex, \extravar} = \bidfunction(\extravar, \similarity_{\revindex, \paperindex})(\gainfunction(\numbids_{\revindex-1, \paperindex} + 1) - \gainfunction(\numbids_{\revindex-1, \paperindex})) + \hyperparam \gainfunctionrev(\extravar, \similarity_{\revindex, \paperindex}) \ \forall \ \paperindex,\extravar \in \paperset.
\end{equation*}
\paragraph*{Local optimality of \super with zero heuristic for any reviewer.}
To determine the ordering of papers to show any reviewer $\revindex \in \reviewerset$ for the general class of assumed gain and bidding functions, \super calls Algorithm~\ref{alg:subprocedure}. Algorithm~\ref{alg:subprocedure} solves the optimization problem in~\eqref{eq:opt_linear} using the weights 
\begin{equation*}
\weight_{\paperindex, \extravar} = \bidfunction(\extravar, \similarity_{\revindex, \paperindex})(\gainfunction(\numbids_{\revindex-1, \paperindex} + \proxy_{\revindex, \paperindex} + 1) - \gainfunction(\numbids_{\revindex-1, \paperindex} + \proxy_{\revindex, \paperindex})) + \hyperparam \gainfunctionrev(\extravar, \similarity_{\revindex, \paperindex}) \ \forall \ \paperindex,\extravar \in \paperset.
\end{equation*}
For any reviewer $\revindex \in \reviewerset$, given the zero heuristic function, $\proxy_{\revindex} =0$ by definition. This means \super with zero heuristic selects the paper ordering for any reviewer by solving the optimization problem to maximize the immediate expected gain conditioned on the history, so we conclude it is locally optimal for any reviewer.

\subsection{Proof of Theorem~\ref{prop:localworst}: Suboptimality of Baselines for Final Reviewer}
The organization of this proof is as follows. In Section~\ref{sec:localworst_prelim}, we present notation and preliminary information common to the analysis for each of the baselines. We prove the suboptimality bounds for the \simbase, \bidbase, and \randbase baselines separately in Sections~\ref{sec:localworst_sim},~\ref{sec:localworst_bid}, and~\ref{sec:localworst_rand}, respectively. Combining the results for each of the baselines proves the theorem statement. We conclude in Section~\ref{sec:localworst_lemmas} with proofs of technical lemmas invoked in the analysis of the baselines. 

\subsubsection{Notation and Preliminaries}\label{sec:localworst_prelim}
We denote the gain from the final reviewer $\numrev$ of an arbitrary algorithm \algbase presenting a potentially random paper ordering $\ordering_{\numrev}^{\algbase}$ to the reviewer as
\begin{equation*}
\gain_{\numrev}^{\algbase} = \gain_{p, \numrev}^{\algbase}  + \hyperparam \gain_{r, \numrev}^{\algbase}.
\end{equation*}
The paper-side gain from the final reviewer $\gain_{p, \numrev}^{\algbase}$ is given by 
\begin{equation*}
\gain_{p, \numrev}^{\algbase}  = \sum_{\paperindex \in \paperset} \Big(\gainfunction\Big(\sum_{\revindex\in \reviewerset}\randombid_{\revindex, \paperindex}\Big) - \gainfunction\Big(\sum_{\revindex\in [\numrev-1]}\randombid_{\revindex, \paperindex}\Big)\Big),
\end{equation*}
where again $\randombid_{\revindex, \paperindex}$ is a Bernoulli random variable representing the random bid of a reviewer $\revindex \in \reviewerset$ on a paper $\paperindex \in \paperset$ that depends on the position the paper was shown to the reviewer. The reviewer-side gain from the final reviewer $\gain_{r, \numrev}^{\algbase}$ is given by 
\begin{equation*}
\gain_{r, \numrev}^{\algbase} =   \sum_{\paperindex \in [\numpapers]}  \gainfunctionrev(\ordering_{\numrev}^{\algbase}(\paperindex),\similarity_{\numrev,\paperindex}). 
\end{equation*}
Accordingly, 
\begin{equation*}
\gain_{\numrev}^{\algbase} =  \sum_{\paperindex \in \paperset} \Big(\gainfunction\Big(\sum_{\revindex\in \reviewerset}\randombid_{\revindex, \paperindex}\Big) - \gainfunction\Big(\sum_{\revindex\in [\numrev-1]}\randombid_{\revindex, \paperindex}\Big)\Big) + \hyperparam\sum_{\paperindex \in [\numpapers]}  \gainfunctionrev(\ordering_{\numrev}^{\algbase}(\paperindex),\similarity_{\numrev,\paperindex}).
\end{equation*}
The expected gain 
from the final reviewer conditioned on the history of bids and paper orderings $\history_{\numrev-1}$ (see Definition~\ref{def:history}) is given by
\[\mathbb{E}[\gain_{\numrev}^{\algbase}|\history_{\numrev-1}] = \mathbb{E}[\gain_{p, \numrev}^{\algbase}|\history_{\numrev-1}] + \hyperparam \mathbb{E}[\gain_{r, \numrev}^{\algbase}|\history_{\numrev-1}]\]
where the expectation is with respect to the randomness in the algorithm and the bids from the final reviewer.
Observe that 
\begin{equation*}
\mathbb{E}[\gain_{p, \numrev}^{\algbase}|\history_{\numrev-1}] = \mathbb{E}_{\ordering_{\numrev}^{\algbase}}\Big[\sum_{\paperindex\in \paperset} \bidfunction(\ordering_{\numrev}^{\algbase}(\paperindex), \similarity_{\numrev, \paperindex})(\gainfunction(\numbids_{\numrev-1, \paperindex} + 1) - \gainfunction(\numbids_{\numrev-1, \paperindex}))\Big]
\end{equation*} 
since $\sum_{\revindex\in [\numrev-1]}\randombid_{\revindex, \paperindex}$ is the deterministic quantity $\numbids_{\numrev-1, \paperindex}$ for each $\paperindex\in \paperset$ conditioned on $\history_{\numrev-1}$ and $\randombid_{\numrev, \paperindex}$ is a Bernoulli random variable with parameter $\bidprob_{\numrev, \paperindex} = \bidfunction(\ordering_{\numrev}^{\algbase}(\paperindex), \similarity_{\numrev, \paperindex})$ for each $\paperindex\in \paperset$ given the fixed paper ordering $\ordering_{\numrev}^{\algbase}$. Moreover,
\begin{equation*}
\mathbb{E}[\gain_{r, \numrev}^{\algbase}|\history_{\numrev-1}] = \mathbb{E}_{\ordering_{\numrev}^{\algbase}}\Big[\sum_{\paperindex \in \paperset}\gainfunctionrev(\ordering_{\numrev}^{\algbase}(\paperindex), \similarity_{\numrev, \paperindex})\Big].
\end{equation*} 
It follows that
\begin{equation*}
\mathbb{E}[\gain_{\numrev}^{\algbase}|\history_{\numrev-1}] = \mathbb{E}_{\ordering_{\numrev}^{\algbase}}\Big[\sum_{\paperindex\in \paperset} \bidfunction(\ordering_{\numrev}^{\algbase}(\paperindex), \similarity_{\numrev, \paperindex})(\gainfunction(\numbids_{\numrev-1, \paperindex} + 1) - \gainfunction(\numbids_{\numrev-1, \paperindex})) + \hyperparam \sum_{\paperindex \in \paperset}\gainfunctionrev(\ordering_{\numrev}^{\algbase}(\paperindex), \similarity_{\numrev, \paperindex})\Big].
\end{equation*} 
The given biding function can be decomposed into the form
\begin{equation}
\bidfunction(\ordering_{\revindex}(\paperindex), \similarity_{\revindex, \paperindex})=\frac{\similarity_{\revindex, \paperindex}}{\log_2(\ordering_{\revindex}(\paperindex)+1)} = \similarity_{\revindex, \paperindex}\bidfunction^{\ordering}(\ordering_{\revindex}(\paperindex))
\label{eq:bidf_decomposed}
\end{equation}
where \[\bidfunction^{\ordering}(\ordering_{\revindex}(\paperindex)) = \frac{1}{\log_2(\ordering_{\revindex}(\paperindex)+1)}\] 
denotes the component of the bidding function $\bidfunction$ that only depends on the paper ordering and is independent of the similarity score. The reviewer-side gain function can similarly be decomposed into the form 
\begin{equation}
\gainfunctionrev(\ordering_{\revindex}(\paperindex), \similarity_{\revindex, \paperindex}) = (2^{\similarity_{\revindex, \paperindex}}-1)\bidfunction^{\ordering}(\ordering_{\revindex}(\paperindex)).
\label{eq:revgain_decomposed}
\end{equation} 

Using the decomposed forms of the bidding function and the reviewer-side gain function from~\eqref{eq:bidf_decomposed} and~\eqref{eq:revgain_decomposed}, the expected gain from the final reviewer $\numrev$ of an arbitrary algorithm \algbase is given by
\begin{align}
\mathbb{E}[\gain_{\numrev}^{\algbase}|\history_{\numrev-1}] &=\mathbb{E}[\gain_{p, \numrev}^{\algbase}|\history_{\numrev-1}]+ \hyperparam\mathbb{E}[\gain_{r, \numrev}^{\algbase}|\history_{\numrev-1}] \notag \\
&\mkern-96mu= \mathbb{E}_{\ordering_{\numrev}^{\algbase}}\Big[\sum_{\paperindex\in \paperset} \similarity_{\numrev, \paperindex}(\gainfunction(\numbids_{\numrev-1, \paperindex} + 1) - \gainfunction(\numbids_{\numrev-1, \paperindex}))\bidfunction^{\ordering}(\ordering_{\numrev}^{\algbase}(\paperindex)) + \hyperparam \sum_{\paperindex \in \paperset}(2^{\similarity_{\numrev, \paperindex}}-1)\bidfunction^{\ordering}(\ordering_{\numrev}^{\algbase}(\paperindex))\Big].\label{eq:gain_new}
\end{align} 
The optimal paper ordering for the final reviewer is thus given by the solution to the following optimization problem
\begin{equation}
\ordering_{\numrev}^{\ast} =  \argmax_{\ordering_{\numrev} \in \symgroup_\numpapers} \quad \sum_{\paperindex\in \paperset} \sortvar_{\numrev, \paperindex}\bidfunction^{\ordering}(\ordering_{\numrev}(\paperindex))
\label{eq:baseline_subopt_opt_problem}
\end{equation}
where
\begin{equation}
\sortvar_{\numrev, \paperindex} = \similarity_{\numrev, \paperindex}(\gainfunction(\numbids_{\numrev-1, \paperindex} + 1) - \gainfunction(\numbids_{\numrev-1, \paperindex})) + \hyperparam(2^{\similarity_{\numrev, \paperindex}}-1)  \ \forall \ \paperindex \in \paperset.
\label{eq:baseline_subopt_opt_problem_weights}
\end{equation}
The optimal solution to~\eqref{eq:baseline_subopt_opt_problem} ranks papers in a decreasing order of their corresponding values in \{$\alpha_{\numrev, \paperindex}\}_{\paperindex\in \paperset}$ since the function $\bidfunction^{\ordering}$ is decreasing in the decision variable. This observation will be used to obtain an explicit form of $\ordering_{\numrev}^{\ast}$ for each problem we subsequently construct to show the suboptimality of the baselines. From Theorem~\ref{prop:local}, \super with any heuristic is optimal for the final reviewer, which means $\ordering^{\super}_{\numrev}=\ordering^{\ast}_{\numrev}$. See that $\ordering^{\super}_{\numrev}$ is a non-random quantity given a deterministic tie-breaking mechanism and the expected gain is independent of the tie-breaking mechanism. Thus, without loss of generality, we assume ties are broken by the paper indexes in favor of $\paperindex<\paperindex'$ for \super.

We need to compare the expected gain obtained from the final reviewer using the \super algorithm presenting the optimal paper ordering $\ordering^{\super}_{\numrev}=\ordering^{\ast}_{\numrev}$ with any baseline \algbase $\in \{\simbase, \bidbase, \randbase\}$ presenting a potentially random ordering $\ordering^{\algbase}_{\numrev}$. Therefore, we analyze the quantity
\begin{align}
\mathbb{E}[\gain_{\numrev}^{\super}-\gain_{\numrev}^{\algbase}|\history_{\numrev-1}] &= \mathbb{E}[\gain_{p, \numrev}^{\super}-\gain_{p, \numrev}^{\algbase}|\history_{\numrev-1}] + \hyperparam\mathbb{E}[\gain_{r, \numrev}^{\super}-\gain_{r, \numrev}^{\algbase}|\history_{\numrev-1}]\notag \\
&=\mathbb{E}_{\ordering_{\numrev}^{\algbase}}\Big[\sum_{\paperindex \in \paperset} \similarity_{\numrev, \paperindex}(\gainfunction(\numbids_{\numrev-1, \paperindex}+1)-\gainfunction(\numbids_{\numrev-1, \paperindex}))(\bidfunction^{\ordering}(\ordering_{\numrev}^{\super}(\paperindex))-\bidfunction^{\ordering}(\ordering_{\numrev}^{\algbase}(\paperindex))) \notag\\
&\qquad\qquad+ \hyperparam\sum_{\paperindex\in \paperset} (2^{\similarity_{\numrev, \paperindex}}-1)(\bidfunction^{\ordering}(\ordering_{\numrev}^{\super}(\paperindex))-\bidfunction^{\ordering}(\ordering_{\numrev}^{\algbase}(\paperindex)))\Big]
\label{eq:diff}
\end{align}
for each baseline \algbase $\in \{\simbase, \bidbase, \randbase\}$. The \simbase and \bidbase algorithms are deterministic for the final reviewer conditioned on the history, up to the tie-breaking mechanism.
The \randbase algorithm is random, so the expectation over the paper ordering in~\eqref{eq:diff} is necessary when analyzing \randbase.
In the remainder of the proof, we analyze \simbase, then \bidbase, and finish with \randbase.

\subsubsection{Suboptimality of \simbase for Final Reviewer}\label{sec:localworst_sim}
In this section, we prove the worst case performance of the \simbase baseline for the final reviewer.
\paragraph*{Intuition.} The \simbase algorithm directly optimizes the expected reviewer-side gain since it shows papers in a decreasing order of the similarity scores. Consequently, it obtains the maximum expected reviewer-side gain that can be achieved. However, the algorithm gives no attention to the number of bids on each paper, which play an important role in the expected paper-side gain that can be obtained upon the arrival of the final reviewer. This point suggests that \simbase may be suboptimal for the combined objective.

To build intuition for when this can occur, consider that there is only a pair of papers $\paperindex$ and $\paperindex'$. Moreover, suppose paper $\paperindex$ has only marginally higher similarity score than paper $\paperindex'$, but paper $\paperindex$ has significantly more bids than paper $\paperindex'$. 
In this scenario, the expected reviewer-side gain of any paper ordering is nearly equal. However, showing paper $\paperindex'$ ahead of paper $\paperindex$ results in significantly higher expected paper-side gain owing to the diminishing returns of bids from the paper-side gain function. Since \simbase instead shows paper $\paperindex$ ahead of paper $\paperindex'$, it would be suboptimal. The following construction now generalizes this observation.

\paragraph*{Construction.}
We now construct a problem instance that will be used to prove \simbase is significantly suboptimal for the final reviewer in the worst case.
Consider the similarity scores for the final reviewer to be $\similarity_{\numrev, \paperindex}=1-1/(\paperindex\epsilon)$ for each paper $\paperindex \in \paperset$, where $\epsilon = (1+\hyperparam)e^{e^{e}}$ and $\hyperparam\geq 0$ is the fixed and given trade-off parameter. In this construction, the similarity scores for the final reviewer are nearly equal, but they are increasing in the paper index. For the time being, assume the number of papers 
$\numpapers$ is even. At the end of this section, we handle when the number of papers $\numpapers$ is odd.
Let the number of bids on the papers from previous reviewers be 
\begin{equation*}
\numbids_{\numrev-1, \paperindex} =
\begin{cases}
0, & \text{if} \ \paperindex \in \{1,\ldots,\numpapers/2\} \\
\bidn, & \text{if} \ \paperindex \in \{\numpapers/2+1,\ldots,\numpapers\}.
\end{cases}
\end{equation*}
The bid counts are such that papers among the top half of the similarity scores obtained a bid in the past, and papers among the bottom half of the similarity scores did not obtain any bids from previous reviewers. 

We now derive the explicit form of the optimal paper ordering for the final reviewer. Recall from~\eqref{eq:baseline_subopt_opt_problem_weights} that the weights of the optimization problem for the final reviewer given in~\eqref{eq:baseline_subopt_opt_problem} are defined by
\begin{equation}
\sortvar_{\numrev, \paperindex} = \similarity_{\numrev, \paperindex}(\gainfunction(\numbids_{\numrev-1, \paperindex} + 1) - \gainfunction(\numbids_{\numrev-1, \paperindex})) + \hyperparam(2^{\similarity_{\numrev, \paperindex}}-1)  \ \forall \ \paperindex \in \paperset.
\label{eq:opt_weights}
\end{equation} 
Moreover, from the structure of the optimization problem in~\eqref{eq:baseline_subopt_opt_problem}, if $\sortvar_{\numrev, \paperindex}>\sortvar_{\numrev, \paperindex'}$, then $\ordering_{\numrev}^{\super}(\paperindex)<\ordering_{\numrev}^{\super}(\paperindex')$ so that paper $\paperindex$ is shown ahead of paper $\paperindex'$ in the ranking. Observe that $\sortvar_{\numrev, \paperindex}$ is increasing in the similarity score $\similarity_{\numrev, \paperindex}$ and decreasing in the number of bids $\numbids_{\numrev-1, \paperindex}$ for each $\paperindex \in \paperset$. Consequently, if a pair of papers $\paperindex, \paperindex' \in \paperset$ are such that $\numbids_{\numrev-1, \paperindex} = \numbids_{\numrev-1, \paperindex'}$ and $\similarity_{\numrev, \paperindex}>\similarity_{\numrev, \paperindex'}$, then $\sortvar_{\numrev, \paperindex}>\sortvar_{\numrev, \paperindex'}$ and in turn, $\ordering_{\numrev}^{\super}(\paperindex)<\ordering_{\numrev}^{\super}(\paperindex')$.

We now show that in the optimal ordering for this instance, each paper with zero bids is shown ahead of each paper with a non-zero number of bids. Among the set of papers with zero bids, namely those indexed by $\{1,\dots, \numpapers/2\}$, the minimum similarity score $\similarity_{\numrev, \paperindex}$ and minimum weight $\sortvar_{\numrev, \paperindex}$ occur at $\paperindex=1$. Among the set of papers with a non-zero number of bids, namely those indexed by $\{\numpapers/2+1,\dots, \numpapers\}$, the maximum similarity score $\similarity_{\revindex, \paperindex'}$ and maximum weight $\sortvar_{\numrev, \paperindex'}$ occur at $\paperindex'=\numpapers$. Thus, if $\sortvar_{\numrev, 1}-\sortvar_{\numrev, \numpapers}>0$, then we can conclude each paper with zero bids is shown ahead of each paper with a non-zero number of bids. 
To prove this, we need the following lemma, the proof of which can be found in Section~\ref{sec:multiplicative}. 
\begin{lemma}
For $\gainfunction(x) = \sqrt{x}$, any $\hyperparam\geq 0$, $\numpapers\geq2$, and $\epsilon=(1+\hyperparam)e^{e^e}$, it must be that 
\begin{equation}
\setlength\belowdisplayskip{-0pt}
(\gainfunction(1)-\gainfunction(0))(1-1/\epsilon)-(\gainfunction(\bidnp)-\gainfunction(\bidn))(1-1/(\numpapers\epsilon))+\hyperparam (2^{(1-1/\epsilon)}-2^{(1-1/(\numpapers\epsilon))}) \geq 1/2.
\label{eq:lem_mult}
\end{equation}
\label{lem:multiplicative}
\end{lemma}
From the given similarity scores and bid counts, and then applying Lemma~\ref{lem:multiplicative}, we obtain
\begin{equation*}
\sortvar_{\numrev, 1}-\sortvar_{\numrev, \numpapers} = (\gainfunction(1)-\gainfunction(0))(1-1/\epsilon)-(\gainfunction(2)-\gainfunction(1))(1-1/(\numpapers\epsilon))+\hyperparam (2^{(1-1/\epsilon)}-2^{(1-1/(\numpapers\epsilon))}) \geq 1/2.
\end{equation*}
Therefore, the optimal paper ordering shows all papers with zero bids  ahead of every paper with a non-zero number of bids. 
Moreover, recall if a pair of papers $\paperindex, \paperindex' \in \paperset$ are such that $\numbids_{\numrev-1, \paperindex} = \numbids_{\numrev-1, \paperindex'}$ and $\similarity_{\numrev, \paperindex}>\similarity_{\numrev, \paperindex'}$, then $\sortvar_{\numrev, \paperindex}>\sortvar_{\numrev, \paperindex'}$. This fact allows us to determine that among each group of papers (zero and non-zero bids), the optimal paper ordering presents the papers in decreasing order of the similarity scores. 

Combining the previous conclusions, \super shows the optimal paper ordering 
\begin{equation}
\ordering_{\numrev}^{\super}(j) = 
\begin{cases}
\numpapers/2-\paperindex + 1, & \text{if} \ \paperindex \in \{1,\dots,\numpapers/2\} \\
\numpapers + \numpapers/2+1-\paperindex, & \text{if} \ \paperindex \in \{\numpapers/2+1,\dots,\numpapers\}.
\end{cases}
\label{eq:sim_construct}
\end{equation}
The \simbase algorithm shows papers in a decreasing order of the similarity scores so that $\ordering_{\numrev}^{\simbase}(\paperindex) = \numpapers-\paperindex+1$ for $\paperindex \in \paperset$. Observe that for this problem, \super shows the papers with zero bids much earlier in the paper ordering than \simbase. We now move on to lower bounding~\eqref{eq:diff} for this construction and begin by considering the expected paper-side gain.

\paragraph*{Bounding the expected paper-side gain.}
Substituting the similarity scores, the number of bids on each paper, and the (deterministic) paper orderings presented by each algorithm for this construction into the paper-side component of~\eqref{eq:diff}, we obtain
\begin{align*}
\mathbb{E}[\gain_{p, \numrev}^{\super}-\gain_{p, \numrev}^{\simbase}|\history_{\numrev-1}] &=  (\gainfunction(1)-\gainfunction(0))\sum_{\paperindex=1}^{\numpapers/2}(1-1/(\paperindex \epsilon)) (\bidfunction^{\ordering}(\numpapers/2 - \paperindex + 1)-\bidfunction^{\ordering}(\numpapers - \paperindex+1)) \\
&+ (\gainfunction(\bidnp)-\gainfunction(\bidn))\sum_{\paperindex=\numpapers/2+1}^{\numpapers}(1-1/(\paperindex \epsilon))(\bidfunction^{\ordering}(\numpapers + \numpapers/2 +1 -\paperindex)-\bidfunction^{\ordering}(\numpapers - \paperindex+1)).
\end{align*}
Manipulating the indexing of the sum over the last half of the papers gives
\begin{align*}
\mathbb{E}[\gain_{p, \numrev}^{\super}-\gain_{p, \numrev}^{\simbase}|\history_{\numrev-1}] &=  (\gainfunction(1)-\gainfunction(0))\sum_{\paperindex=1}^{\numpapers/2}(1-1/(\paperindex \epsilon)) (\bidfunction^{\ordering}(\numpapers/2 - \paperindex + 1)-\bidfunction^{\ordering}(\numpapers - \paperindex+1)) \\
&+ (\gainfunction(\bidnp)-\gainfunction(\bidn))\sum_{\paperindex=1}^{\numpapers/2}(1-1/((\paperindex + \numpapers/2) \epsilon))(\bidfunction^{\ordering}(\numpapers-\paperindex+1)-\bidfunction^{\ordering}(\numpapers/2 - \paperindex+1)).
\end{align*}
Noting that $(\bidfunction^{\ordering}(\numpapers/2 - \paperindex + 1)-\bidfunction^{\ordering}(\numpapers - \paperindex+1)) = -(\bidfunction^{\ordering}(\numpapers - \paperindex+1)-\bidfunction^{\ordering}(\numpapers/2 - \paperindex + 1))$, we obtain
\begin{align*}
\mathbb{E}[\gain_{p, \numrev}^{\super}-\gain_{p, \numrev}^{\simbase}|\history_{\numrev-1}] &=(\gainfunction(1)-\gainfunction(0))\sum_{\paperindex=1}^{\numpapers/2}(1-1/(\paperindex \epsilon)) (\bidfunction^{\ordering}(\numpapers/2 - \paperindex + 1)-\bidfunction^{\ordering}(\numpapers - \paperindex+1)) \notag \\
&- (\gainfunction(\bidnp)-\gainfunction(\bidn))\sum_{\paperindex=1}^{\numpapers/2}(1-1/((\paperindex + \numpapers/2) \epsilon))(\bidfunction^{\ordering}(\numpapers/2 - \paperindex+1)-\bidfunction^{\ordering}(\numpapers-\paperindex+1)). \notag 
\end{align*}
For every $\paperindex \in [\numpapers/2]$, we have $\bidfunction^{\ordering}(\numpapers/2 - \paperindex + 1)-\bidfunction^{\ordering}(\numpapers - \paperindex+1) >0$ since $\numpapers/2 - \paperindex + 1<\numpapers - \paperindex+1$ and $\bidfunction^{\ordering}(\ordering_{\revindex}(\paperindex)) = 1/\log_2(\ordering_{\revindex}(\paperindex)+1)$ 
is a decreasing function on the domain $\mathbb{R}_{>0}$. Moreover, for every $\paperindex \in [\numpapers/2]$,  $1-1/(\paperindex \epsilon)\geq 1-1/\epsilon$ and $1-1/((\paperindex + \numpapers/2) \epsilon)\leq 1-1/(\numpapers \epsilon)$, and the given paper-side gain function $\gainfunction$ is increasing on the domain $\reals_{\geq0}$. Thus, 
\begin{align*}
\mathbb{E}[\gain_{p, \numrev}^{\super}-\gain_{p, \numrev}^{\simbase}|\history_{\numrev-1}]&\geq ((\gainfunction(1)-\gainfunction(0))(1-1/\epsilon)-(\gainfunction(\bidnp)-\gainfunction(\bidn))(1-1/(\numpapers\epsilon))) \\
&\qquad \sum_{\paperindex=1}^{\numpapers/2} (\bidfunction^{\ordering}(\numpapers/2 - \paperindex + 1)-\bidfunction^{\ordering}(\numpapers - \paperindex+1)).
\end{align*}
Finally, manipulating the indexing of the sum gives 
\begin{equation}
\begin{split}
\mathbb{E}[\gain_{p, \numrev}^{\super}-\gain_{p, \numrev}^{\simbase}|\history_{\numrev-1}]
&\geq ((\gainfunction(1)-\gainfunction(0))(1-1/\epsilon)-(\gainfunction(\bidnp)-\gainfunction(\bidn))(1-1/(\numpapers\epsilon))) \\
&\qquad \sum_{\paperindex=1}^{\numpapers/2} (\bidfunction^{\ordering}(\paperindex)-\bidfunction^{\ordering}(\paperindex+\numpapers/2)). 
\end{split}
\label{eq:sim_pside_bound}
\end{equation}

We now bound the expected reviewer-side gain including the trade-off parameter $\hyperparam$ from~\eqref{eq:diff}. The steps that follow are analogous to those exercised in bounding the expected paper-side gain.
\paragraph*{Bounding the expected reviewer-side gain.}
Substituting the values of the similarity scores, the number of bids on each paper, and the (deterministic) paper orderings presented by each algorithm into the reviewer-side component of~\eqref{eq:diff}, we obtain
\begin{align*}
\hyperparam\mathbb{E}[\gain_{r, \numrev}^{\super}-\gain_{r, \numrev}^{\simbase}|\history_{\numrev-1}] &= \hyperparam\sum_{\paperindex=1}^{\numpapers/2}(2^{(1-1/(\paperindex \epsilon))}-1) (\bidfunction^{\ordering}(\numpapers/2 - \paperindex + 1)-\bidfunction^{\ordering}(\numpapers - \paperindex+1)) \\
&\quad+ \hyperparam \sum_{\paperindex=\numpapers/2+1}^{\numpapers}(2^{(1-1/(\paperindex \epsilon))}-1)(\bidfunction^{\ordering}(\numpapers + \numpapers/2 +1 -\paperindex)-\bidfunction^{\ordering}(\numpapers - \paperindex+1)).
\end{align*}
Manipulating the indexing of the sum over the last half of the papers results in
\begin{align*}
\hyperparam\mathbb{E}[\gain_{r, \numrev}^{\super}-\gain_{r, \numrev}^{\simbase}|\history_{\numrev-1}]  &= \hyperparam \sum_{\paperindex=1}^{\numpapers/2}(2^{(1-1/(\paperindex \epsilon))}-1) (\bidfunction^{\ordering}(\numpapers/2 - \paperindex + 1)-\bidfunction^{\ordering}(\numpapers - \paperindex+1)) \\
&\quad+ \hyperparam \sum_{\paperindex=1}^{\numpapers/2}(2^{(1-1/((\paperindex + \numpapers/2) \epsilon))}-1)(\bidfunction^{\ordering}(\numpapers-\paperindex+1)-\bidfunction^{\ordering}(\numpapers/2 - \paperindex+1)). 
\end{align*}
Noting that $\bidfunction^{\ordering}(\numpapers/2 - \paperindex + 1)-\bidfunction^{\ordering}(\numpapers - \paperindex+1) = -(\bidfunction^{\ordering}(\numpapers - \paperindex+1)-\bidfunction^{\ordering}(\numpapers/2 - \paperindex + 1))$, we obtain
\begin{align*}
\hyperparam\mathbb{E}[\gain_{r, \numrev}^{\super}-\gain_{r, \numrev}^{\algbase}|\history_{\numrev-1}]  &= \hyperparam \sum_{\paperindex=1}^{\numpapers/2}(2^{(1-1/(\paperindex \epsilon))}-1) (\bidfunction^{\ordering}(\numpapers/2 - \paperindex + 1)-\bidfunction^{\ordering}(\numpapers - \paperindex+1)) \\
&\quad- \hyperparam \sum_{\paperindex=1}^{\numpapers/2}(2^{(1-1/((\paperindex + \numpapers/2) \epsilon))}-1)(\bidfunction^{\ordering}(\numpapers/2 - \paperindex+1)-\bidfunction^{\ordering}(\numpapers-\paperindex+1)). 
\end{align*}
For every $\paperindex \in [\numpapers/2]$, we have $\bidfunction^{\ordering}(\numpapers/2 - \paperindex + 1)-\bidfunction^{\ordering}(\numpapers - \paperindex+1) >0$ since $\numpapers/2 - \paperindex + 1<\numpapers - \paperindex+1$ and $\bidfunction^{\ordering}(\ordering_{\revindex}(\paperindex)) = 1/\log_2(\ordering_{\revindex}(\paperindex)+1)$ is a decreasing function on the domain $\mathbb{R}_{>0}$. 
Furthermore, observe that for every $\paperindex \in [\numpapers/2]$,  $1-1/(\paperindex \epsilon)\geq 1-1/(\epsilon)$ and $1-1/((\paperindex + \numpapers/2) \epsilon)\leq 1-1/(\numpapers \epsilon)$. This set of facts leads to the bound
\begin{equation*}
\hyperparam\mathbb{E}[\gain_{r, \numrev}^{\super}-\gain_{r, \numrev}^{\simbase}|\history_{\numrev-1}]  \geq \hyperparam\big(2^{(1-1/\epsilon)}-2^{(1-1/(\numpapers\epsilon))}\big)\sum_{\paperindex=1}^{\numpapers/2} (\bidfunction^{\ordering}(\numpapers/2 - \paperindex + 1)-\bidfunction^{\ordering}(\numpapers - \paperindex+1)).
\end{equation*}
To finish this sequence of steps, we manipulate the indexing of the sum to conclude 
\begin{equation}
\mathbb{E}[\gain_{r, \numrev}^{\super}-\gain_{r, \numrev}^{\simbase}|\history_{\numrev-1}] \geq \hyperparam\big(2^{(1-1/\epsilon)}-2^{(1-1/(\numpapers\epsilon))}\big)\sum_{\paperindex=1}^{\numpapers/2} (\bidfunction^{\ordering}(\paperindex)-\bidfunction^{\ordering}(\paperindex+\numpapers/2)). 
\label{eq:sim_rside_bound}
\end{equation}
\paragraph*{Completing the lower bound.}
Combining the bounds on the expected paper-side and reviewer-side gain terms from~\eqref{eq:sim_pside_bound} and~\eqref{eq:sim_rside_bound}, we obtain an initial lower bound given by 
\begin{equation}
\mathbb{E}[\gain^{\super}_{\numrev} -\gain^{\simbase}_{\numrev}|\history_{\numrev-1}] \geq \mathcal{C}\sum_{\paperindex=1}^{\numpapers/2} (\bidfunction^{\ordering}(\paperindex)-\bidfunction^{\ordering}(\paperindex+\numpapers/2)),
 \label{eq:sim_intermediate}
\end{equation}
where for ease of notation we define
\begin{equation}
\mathcal{C} = (\gainfunction(1)-\gainfunction(0))(1-1/\epsilon)-(\gainfunction(\bidnp)-\gainfunction(\bidn))(1-1/(\numpapers\epsilon))+\hyperparam (2^{(1-1/\epsilon)}-2^{(1-1/(\numpapers\epsilon))}).
\label{eq:sim_intermediate_mult}
\end{equation}
We apply Lemma~\ref{lem:multiplicative} to get that $\mathcal{C}\geq 1/2$.
The following lemma, the proof of which can be found in Section~\ref{sec:sum_bound}, provides a bound on the sum in~\eqref{eq:sim_intermediate}.
\begin{lemma}
Let $\bidfunction^{\ordering}(x) = 1/\log_2(x+1)$. Fix $\numpapers$ to be an even integer such that $\numpapers\geq 2$. Then, 
\begin{equation*}
\setlength\belowdisplayskip{-0pt}
\sum_{\paperindex=1}^{\numpapers/2} (\bidfunction^{\ordering}(\paperindex)-\bidfunction^{\ordering}(\paperindex+\numpapers/2)) \geq \frac{\numpapers}{16\log_2^2(\numpapers)}.
\end{equation*}
\label{lemma:sum_bound}
\end{lemma}
\noindent
From~\eqref{eq:sim_intermediate} along with the fact that $\mathcal{C}\geq 1/2$ and Lemma~\ref{lemma:sum_bound}, we obtain
\begin{equation}
\mathbb{E}[\gain^{\super}_{\numrev} -\gain^{\simbase}_{\numrev}|\history_{\numrev-1}] \geq \frac{\numpapers}{32\log_2^2(\numpapers)}.
 \label{eq:sim_div2}
\end{equation}

Lemma~\ref{lemma:sum_bound}, and consequently the bound in \eqref{eq:sim_div2}, are applicable when $\numpapers$ is even.  The following lemma shows that an equivalent result (up to constants) holds when the number of papers $\numpapers$ is odd. The bound is obtained by looking at an identical problem construction for $\numpapers' = \numpapers-1$ papers, and then including an additional paper that has a similarity score of zero with the final reviewer and one previous bid. This change is such that both \super and \simbase show the additional paper last, and moreover, the expected gain from paper $\numpapers$ is zero since the similarity score is zero. 
\begin{lemma}
If $\numpapers$ is odd, then in the worst case for the final reviewer under the assumptions of Theorem~\ref{prop:localworst},
\begin{equation*}
\setlength\belowdisplayskip{-0pt}
\mathbb{E}[\gain^{\super}_{\numrev} -\gain^{\simbase}_{\numrev}|\history_{\numrev-1}] \geq \frac{\numpapers}{64\log_2^2(\numpapers)}.
\end{equation*}
\label{lemma:sim_odd_bound}
\end{lemma}
\noindent
The proof of Lemma~\ref{lemma:sim_odd_bound} is in Section~\ref{sec:sim_odd_bound}.

Combining the bound from~\eqref{eq:sim_div2} which holds for $\numpapers$ even with the bound from Lemma~\ref{lemma:sim_odd_bound} which holds for $\numpapers$ odd, we find that for every $\numpapers\geq 2$ and $\hyperparam\geq 0$,
\begin{equation*}
\mathbb{E}[\gain^{\super}_{\numrev} -\gain^{\simbase}_{\numrev}|\history_{\numrev-1}] \geq \frac{\numpapers}{64\log_2^2(\numpapers)}.
\end{equation*}
This proves the claim in Theorem~\ref{prop:localworst} stating that there exists a constant $\constant >0$ such that for every $\numpapers\geq2$ and $\hyperparam\geq 0$, \simbase is suboptimal by an additive factor of at least $\constant\numpapers /\log_2^2(\numpapers)$ in the worst case for the final reviewer.

\subsubsection{Suboptimality of \bidbase for Final Reviewer}
\label{sec:localworst_bid}
In this section, we prove the worst case performance of the \bidbase baseline for the final reviewer. 

\paragraph*{Intuition.} 
The \bidbase algorithm greedily optimizes the minimum bid count since it shows papers in an increasing order of the number of bids received previously. 
The underlying problem with this method is that the paper ordering selected for any reviewer is independent of the similarity scores for that reviewer (up to serving as a tie-breaking mechanism). Since the paper-side gain function and the reviewer-side gain function both depend on the similarity scores, this property of \bidbase leads to suboptimality in terms of both the expected paper-side gain and the reviewer-side gain.

To build some intuition for when \bidbase is suboptimal, consider there is only a pair of papers $\paperindex$ and $\paperindex'$. Moreover, suppose paper $\paperindex$ has a much higher similarity score than paper $\paperindex'$ and paper $\paperindex$ has only one more bid than $\paperindex'$. In this scenario, the expected reviewer-side gain from showing paper $\paperindex$ ahead of paper $\paperindex'$ is significantly higher than showing paper $\paperindex'$ ahead of paper $\paperindex$. Moreover, since the probability of obtaining a bid on paper $\paperindex$ is significantly higher at a given position in the paper ordering than for paper $\paperindex'$ and the number of bids on the papers are nearly equal, the expected paper-side gain is also maximized if paper $\paperindex$ is shown ahead of paper $\paperindex'$. Since \bidbase instead shows paper $\paperindex'$ ahead of paper $\paperindex$, it would be suboptimal for both paper-side and reviewer-side gain. The following construction now generalizes this observation.

\paragraph*{Construction.}
We now construct a problem instance that will be used to prove \bidbase is significantly suboptimal for the final reviewer in the worst case.
Consider the similarity scores for the final reviewer as
\begin{equation*}
\similarity_{\numrev, \paperindex}=
\begin{cases}
1, & \text{if} \ \paperindex \in \{1,\ldots,\numpapers/2\} \\
0, & \text{if} \ \paperindex \in \{\numpapers/2+1,\ldots,\numpapers\}.
\end{cases}
\end{equation*}
For now, assume the number of papers $\numpapers$ is even. We handle when the number of papers $\numpapers$ is odd at the end of this proof. Let the number of bids on the papers from previous reviewers be such that 
\begin{equation*}
\numbids_{\numrev-1, \paperindex} =
\begin{cases}
\bidn, & \text{if} \ \paperindex \in \{1,\ldots,\numpapers/2\} \\
0, & \text{if} \ \paperindex \in \{\numpapers/2+1,\ldots,\numpapers\}.
\end{cases}
\end{equation*}
In words, half of the papers have a similarity score of one and have recieved bids, and the other half of the papers have a similarity score of zero and have obtained no bids.

We now derive the optimal paper ordering for the final reviewer. Recall  from~\eqref{eq:baseline_subopt_opt_problem_weights} that the weights of the optimization problem for the final reviewer given in~\eqref{eq:baseline_subopt_opt_problem} are defined by
\begin{equation*}
\sortvar_{\numrev, \paperindex} = \similarity_{\numrev, \paperindex}(\gainfunction(\numbids_{\numrev-1, \paperindex} + 1) - \gainfunction(\numbids_{\numrev-1, \paperindex})) + \hyperparam(2^{\similarity_{\numrev, \paperindex}}-1)  \ \forall \ \paperindex \in \paperset.
\end{equation*} 
Moreover, from the structure of the optimization problem in~\eqref{eq:baseline_subopt_opt_problem}, if $\sortvar_{\numrev, \paperindex}>\sortvar_{\numrev, \paperindex'}$, then $\ordering_{\numrev}^{\super}(\paperindex)<\ordering_{\numrev}^{\super}(\paperindex')$ so that paper $\paperindex$ is shown ahead of paper $\paperindex'$ in the ranking. Observe that for each $\paperindex \in \{1,\dots, \numpapers/2\}$, $\sortvar_{\numrev, \paperindex}$ is a fixed number. Similarly, for each $\paperindex' \in \{\numpapers/2+1,\dots, \numpapers\}$, $\sortvar_{\numrev, \paperindex'}$ is a fixed number. If $\sortvar_{\numrev, \paperindex} - \sortvar_{\numrev, \paperindex'}>0$ for any $\paperindex \in \{1,\dots, \numpapers/2\}$ and $\paperindex' \in \{\numpapers/2+1,\dots, \numpapers\}$, then we can conclude each paper with a bid is shown ahead of each paper without a bid. We consider $\paperindex=1$ and $\paperindex'=\numpapers$. Since $\gainfunction(x)=\sqrt{x}$ and $\sortvar_{\numrev, \numpapers}=0$,
\begin{equation}
\sortvar_{\numrev, 1}-\sortvar_{\numrev, \numpapers} = \gainfunction(2) - \gainfunction(1) + \hyperparam = \sqrt{2}-1+\hyperparam \geq 1/3 +\hyperparam.
\label{eq:simple_lem_eq}
\end{equation}
Consequent of the fact $\hyperparam\geq0$, we conclude  $\sortvar_{\numrev, 1}-\sortvar_{\numrev, \numpapers}>0$, which means \super shows each paper with a bid ahead of each paper without a bid. Finally, since $\sortvar_{\numrev, \paperindex}$ is a fixed number for each $\paperindex \in \{1,\dots, \numpapers/2\}$ and $\sortvar_{\numrev, \paperindex'}$ is a fixed number for each $\paperindex' \in \{\numpapers/2+1,\dots, \numpapers\}$, as long as $\ordering_{\numrev}^{\super}(\paperindex)<\ordering_{\numrev}^{\super}(\paperindex')$ for every $\paperindex, \paperindex'$ pair, then the paper ordering is optimal. In other words, any paper ordering which shows the papers with a bid in an arbitrary order followed by the papers without a bid in an arbitrary order is optimal. 

We conclude $\ordering^{\super}_{\numrev}(j) = \paperindex$ for each $\paperindex \in \paperset$, where without loss of generality, to simplify the analysis, we assume if a pair of papers have equal weights in the optimization problem, then ties are broken in order of the paper indexes since the tie-breaking mechanism will not change the expected gain the paper ordering obtains from the final reviewer.

The \bidbase baseline will show papers in an increasing order of the number of bids so that 
\begin{equation*}
\ordering_{\numrev}^{\bidbase}(j) = 
\begin{cases}
\paperindex + \numpapers/2, & \text{if} \ \paperindex \in \{1,\ldots,\numpapers/2\} \\
\paperindex - \numpapers/2, & \text{if} \ \paperindex \in \{\numpapers/2+1,\ldots,\numpapers\}.
\end{cases}
\end{equation*}
This paper ordering is derived from recalling that \bidbase breaks ties by the similarity scores and further ties are broken uniformly at random. However, without loss of generality, to simplify the analysis, we assume if a pair of papers have equal similarity scores and bid counts, then ties are broken in order of the paper indexes since the tie-breaking mechanism among this set of papers will not impact the expected gain. We now move on to lower bounding~\eqref{eq:diff} for this construction.

\paragraph*{Bounding the expected gain.}
Substituting the similarity scores, the number of bids on each paper, and the (deterministic) paper orderings presented by each algorithm for this construction into~\eqref{eq:diff}, we obtain
\begin{equation*}
\mathbb{E}[\gain^{\super}_{\numrev} -\gain^{\bidbase}_{\numrev}|\history_{\numrev-1}]
= \sum_{\paperindex=1}^{\numpapers/2} (\gainfunction(2)-\gainfunction(1))(\bidfunction^{\ordering}(\paperindex)-\bidfunction^{\ordering}(\paperindex+\numpapers/2)) + \hyperparam \sum_{\paperindex=1}^{\numpapers/2} (\bidfunction^{\ordering}(\paperindex)-\bidfunction^{\ordering}(\paperindex+\numpapers/2))
\end{equation*}
where the terms for papers in the set $\{\numpapers/2+1,\dots, \numpapers\}$ dropped out since the similarity scores are zero. Simplifying the expression, we obtain 
\begin{equation}
\mathbb{E}[\gain^{\super}_{\numrev} -\gain^{\bidbase}_{\numrev}|\history_{\numrev-1}] = (\gainfunction(\bidnp)-\gainfunction(\bidn)+\hyperparam)\sum_{\paperindex=1}^{\numpapers/2} (\bidfunction^{\ordering}(\paperindex)-\bidfunction^{\ordering}(\paperindex+\numpapers/2)).
\label{eq:bid_subopt_mid}
\end{equation}
From~\eqref{eq:simple_lem_eq}, we get 
\begin{equation}
\gainfunction(\bidnp)-\gainfunction(\bidn)+\hyperparam \geq 1/3 +\hyperparam.
\label{eq:bid_mult_bound}
\end{equation}
Moreover, from Lemma~\ref{lemma:sum_bound}, 
\begin{equation}
\sum_{\paperindex=1}^{\numpapers/2} (\bidfunction^{\ordering}(\paperindex)-\bidfunction^{\ordering}(\paperindex+\numpapers/2)) \geq \frac{\numpapers}{16\log_2^2(\numpapers)}.
\label{eq:bid_sum_bound}
\end{equation}
Combining~\eqref{eq:bid_subopt_mid},~\eqref{eq:bid_mult_bound}, and~\eqref{eq:bid_sum_bound}, we obtain 
\begin{equation}
\mathbb{E}[\gain^{\super}_{\numrev} -\gain^{\bidbase}_{\numrev}|\history_{\numrev-1}] \geq \Big(\frac{1}{48} + \frac{\hyperparam}{16}\Big)\Big(\frac{\numpapers}{\log_2^2(\numpapers)}\Big)
\label{eq:bid_even}
\end{equation}
whenever the number of papers $\numpapers$ is even.

Lemma~\ref{lemma:sum_bound} and the bound in \eqref{eq:bid_even} are applicable when $\numpapers$ is even. The following lemma shows that an equivalent result (up to constants) holds when the number of papers $\numpapers$ is odd. The approach to obtain the result is similar to that for deriving Lemma~\ref{lemma:sim_odd_bound}. We obtain the bound for an odd number of papers $\numpapers$ by looking at an identical problem construction for $\numpapers' = \numpapers-1$ papers, and then include an additional paper that has a similarity score of zero with the final reviewer and one previous bid. This change is such that both \super and \bidbase show the additional paper last, and moreover, the expected gain from paper $\numpapers$ is zero since the similarity score is zero. 
\begin{lemma}
If $\numpapers$ is odd, then in the worst case for the final reviewer under the assumptions of Theorem~\ref{prop:localworst},
\begin{equation*}
\setlength\belowdisplayskip{-0pt}
\mathbb{E}[\gain^{\super}_{\numrev} -\gain^{\bidbase}_{\numrev}|\history_{\numrev-1}] \geq \Big(\frac{1}{96}+\frac{\hyperparam}{32}\Big)\Big(\frac{\numpapers}{\log_2^2(\numpapers)}\Big).
\end{equation*}
\label{lemma:bid_odd}
\end{lemma}
\noindent
The proof of Lemma~\ref{lemma:bid_odd} is in Section~\ref{sec:bid_odd}.

Combining the bound from~\eqref{eq:bid_even} which holds for $\numpapers$ even with the bound from Lemma~\ref{lemma:bid_odd} which holds for $\numpapers$ odd, we find that for every $\numpapers\geq 2$ and $\hyperparam\geq0$,
\begin{equation}
\mathbb{E}[\gain^{\super}_{\numrev} -\gain^{\bidbase}_{\numrev}|\history_{\numrev-1}] \geq \Big(\frac{1}{96}+\frac{\hyperparam}{32}\Big)\Big(\frac{\numpapers}{\log_2^2(\numpapers)}\Big).
\label{eq:bid_final}
\end{equation}
This proves the claim in Theorem~\ref{prop:localworst} that there exists a constant $\constant >0$ such that for all $\numpapers\geq2$ and $\hyperparam\geq 0$, \bidbase is suboptimal by an additive factor of at least $\constant\numpapers\max\{1,\hyperparam\}/\log_2^2(\numpapers)$ in the worst case for the final reviewer. 

\subsubsection{Suboptimality of \randbase for Final Reviewer}
\label{sec:localworst_rand}
In this section, we prove the the worst case performance of the \randbase baseline for the final reviewer.

\paragraph*{Intuition.}
The \randbase algorithm selects an ordering of papers to show a reviewer uniformly at random from the set of permutations. Since this method is agnostic to the similarity scores and the number of bids, \randbase can select highly suboptimal paper orderings with some non-zero probability. 

To see when this can occur, consider the example that provided intuition for the suboptimality of \bidbase in Section~\ref{sec:localworst_bid} that consisted of only a pair of papers $\paperindex$ and $\paperindex'$. In this example, paper $\paperindex$ has a much higher similarity score than paper $\paperindex'$ and paper $\paperindex$ has only one more bid than $\paperindex'$. The expected reviewer-side gain from showing paper $\paperindex$ ahead of paper $\paperindex'$ is significantly higher than showing paper $\paperindex'$ ahead of paper $\paperindex$. Moreover, since the probability of obtaining a bid on paper $\paperindex$ is significantly higher at a given position in the paper ordering than for paper $\paperindex'$ and the number of bids on the papers are nearly equal, the expected paper-side gain is also maximized if paper $\paperindex$ is shown ahead of paper $\paperindex'$. Since there are only two permutations of the papers that can be selected, with probability $1/2$, \randbase would show paper $\paperindex'$ ahead of paper $\paperindex$ and be suboptimal for both paper-side and reviewer-side gain. The problem construction from Section~\ref{sec:localworst_bid} is sufficient to generalize this observation. For completeness, we repeat the construction below. 

\paragraph*{Construction.}
In the remainder of the proof, we show the problem construction from Section~\ref{sec:localworst_bid} can be used to prove \randbase is significantly suboptimal for the final reviewer in the worst case. In this construction, the similarity scores for the final reviewer are
\begin{equation*}
\similarity_{\numrev, \paperindex}=
\begin{cases}
1, & \text{if} \ \paperindex \in \{1,\ldots,\numpapers/2\} \\
0, & \text{if} \ \paperindex \in \{\numpapers/2+1,\ldots,\numpapers\}.
\end{cases}
\end{equation*}
For now, assume the number of papers $\numpapers$ is divisible by four. We deal with a number of papers $\numpapers$ that is not divisible by four at the end of this section. The number of bids on the papers from previous reviewers are
\begin{equation*}
\numbids_{\numrev-1, \paperindex} =
\begin{cases}
\bidn, & \text{if} \ \paperindex \in \{1,\ldots,\numpapers/2\} \\
0, & \text{if} \ \paperindex \in \{\numpapers/2+1,\ldots,\numpapers\}.
\end{cases}
\end{equation*}
In Section~\ref{sec:localworst_bid}, we showed that \super selects the optimal ordering $\ordering^{\super}_{\numrev}(j) = \paperindex$ for each $\paperindex \in \paperset$. The \randbase baseline will select a paper ordering $\ordering_{\numrev}^{\randbase}$ uniformly at random from the set of permutations $\symgroup_{\numpapers}$.

\paragraph*{Bounding the expected gain.}
For this construction, we need to lower bound~\eqref{eq:diff}.
As an initial step, we simplify the quantity by substituting the similarity scores, the number of bids on each paper, and the paper ordering presented by \super to obtain
\begin{align*}
\mathbb{E}[\gain_{\numrev}^{\super}-\gain_{\numrev}^{\randbase}|\history_{\numrev-1}] 
=\mathbb{E}_{\ordering_{\numrev}^{\randbase}}\Big[&\sum_{\paperindex =1}^{\numpapers/2} (\gainfunction(\bidnp)-\gainfunction(\bidn))(\bidfunction^{\ordering}(\paperindex)-\bidfunction^{\ordering}(\ordering_{\numrev}^{\randbase}(\paperindex))) \notag\\
&+ \hyperparam\sum_{\paperindex=1}^{\numpapers/2} (\bidfunction^{\ordering}(\paperindex)-\bidfunction^{\ordering}(\ordering_{\numrev}^{\randbase}(\paperindex)))\Big],
\end{align*}
where the terms for papers in the set $\{\numpapers/2+1,\dots, \numpapers\}$ dropped out since the similarity scores are zero. Combining the sums and using the fact that $\gainfunction(2)-\gainfunction(1)=\sqrt{2}-1\geq 1/3$ gives
\begin{align}
\mathbb{E}[\gain_{\numrev}^{\super}-\gain_{\numrev}^{\randbase}|\history_{\numrev-1}] 
\geq\mathbb{E}_{\ordering_{\numrev}^{\randbase}}\Big[&(1/3+\hyperparam)\sum_{\paperindex =1}^{\numpapers/2} (\bidfunction^{\ordering}(\paperindex)-\bidfunction^{\ordering}(\ordering_{\numrev}^{\randbase}(\paperindex))) \Big].
\label{eq:rand_sum}
\end{align}

Before proceeding, we provide some intuition that guides the remainder of the proof.
The expression in~\eqref{eq:rand_sum} only depends on the positions \randbase shows the papers in the set $\{1,\dots, \numpapers/2\}$ to the final reviewer.
Recalling that the given function $\bidfunction^{\ordering}$ is decreasing on the domain $\mathbb{R}_{>0}$, we can observe that the number of positive summand in $\sum_{\paperindex =1}^{\numpapers/2} (\bidfunction^{\ordering}(\paperindex)-\bidfunction^{\ordering}(\ordering_{\numrev}^{\randbase}(\paperindex)))$ increases with the number of papers from the set $\{1,\dots, \numpapers/2\}$ that are not presented in the set of positions $\{1,\dots, \numpapers/2\}$ from a selection $\ordering_{\numrev}^{\randbase}$ and the remaining summand are zero. This point suggests if with probability bounded away from zero, sufficiently many papers from the set $\{1,\dots, \numpapers/2\}$ are not presented in the set of positions $\{1,\dots, \numpapers/2\}$ in the ordering selected by \randbase, then it should be suboptimal in expectation. 

Toward formalizing this line of reasoning, the following lemma provides a lower bound on the probability that $\randbase$ selects a paper ordering that shows fewer than $\numpapers/4$ papers from the set $\{1,\dots, \numpapers/2\}$ in the set of positions $\{1,\dots, \numpapers/2\}$. 
The proof is given in Section~\ref{sec:prob_bound}.
\begin{lemma}
Assume $\numpapers$ is divisible by four and consider a set of papers $\paperset$. Let $\mathcal{E}$ be the event that a permutation $\ordering$ of the paper set $\paperset$ drawn uniformly at random from $\symgroup_{\numpapers}$ has fewer than $\numpapers/4$ of the papers $\{1,\dots, \numpapers/2\}$ in the positions $\{1,\dots, \numpapers/2\}$. Then, $\mathbb{P}(\mathcal{E})\geq 1/6$.
\label{lem:prob_bound}
\end{lemma}

Define $T_1\subset \symgroup_{\numpapers}$ as the set of paper orderings with fewer than $\numpapers/4$ of the papers from the set $\{1,\dots, \numpapers/2\}$ in the set of positions $\{1,\dots, \numpapers/2\}$ and $T_2\subset \symgroup_{\numpapers}$ as the set containing the remaining paper orderings so that $T_1\cup T_2 = \symgroup_{\numpapers}$. Now, beginning from~\eqref{eq:rand_sum}, we evaluate and bound the expectation as follows:
\begin{align*}
\mathbb{E}[\gain_{\numrev}^{\super}-\gain_{\numrev}^{\randbase}|\history_{\numrev-1}] 
&=(1/3+\hyperparam)\sum_{\ordering_{\numrev}\in T_1}\mathbb{P}(\ordering_{\numrev}^{\randbase}=\ordering_{\numrev})\sum_{\paperindex =1}^{\numpapers/2} (\bidfunction^{\ordering}(\paperindex)-\bidfunction^{\ordering}(\ordering_{\numrev}(\paperindex)))  \\
&\quad +(1/3+\hyperparam)\sum_{\ordering_{\numrev}\in T_2}\mathbb{P}(\ordering_{\numrev}^{\randbase}=\ordering_{\numrev})\sum_{\paperindex =1}^{\numpapers/2} (\bidfunction^{\ordering}(\paperindex)-\bidfunction^{\ordering}(\ordering_{\numrev}(\paperindex)))  \\
&\geq (1/3+\hyperparam) \mathbb{P}(\ordering_{\numrev}^{\randbase} \in T_1)\min_{\ordering_{\numrev} \in T_1}\sum_{\paperindex =1}^{\numpapers/2} (\bidfunction^{\ordering}(\paperindex)-\bidfunction^{\ordering}(\ordering_{\numrev}(\paperindex))) \\
&\quad+(1/3+\hyperparam) \mathbb{P}(\ordering_{\numrev}^{\randbase} \in T_2)\min_{\ordering_{\numrev} \in T_2}\sum_{\paperindex =1}^{\numpapers/2} (\bidfunction^{\ordering}(\paperindex)-\bidfunction^{\ordering}(\ordering_{\numrev}(\paperindex))).
\end{align*}
Observe that $\min_{\ordering_{\numrev} \in T_2}\sum_{\paperindex =1}^{\numpapers/2} (\bidfunction^{\ordering}(\paperindex)-\bidfunction^{\ordering}(\ordering_{\numrev}(\paperindex)))=0$ since the optimal paper ordering that shows each paper in the set $\{1,\dots, \numpapers/2\}$ in the set of positions $\{1, \dots, \numpapers/2\}$ is contained in $T_2$. 
Moreover, from Lemma~\ref{lem:prob_bound},
$\mathbb{P}(\ordering_{\numrev}^{\randbase} \in T_1)\geq 1/6$. This results in the bound
\begin{equation}
\mathbb{E}[\gain_{\numrev}^{\super}-\gain_{\numrev}^{\randbase}|\history_{\numrev-1}] 
\geq \Big(\frac{1}{18}+\frac{\hyperparam}{6}\Big) \min_{\ordering_{\numrev} \in T_1}\sum_{\paperindex =1}^{\numpapers/2} (\bidfunction^{\ordering}(\paperindex)-\bidfunction^{\ordering}(\ordering_{\numrev}(\paperindex))). 
\label{eq:rand_point} 
\end{equation}
We now need to reason about the minimizer of $\min_{\ordering_{\numrev} \in T_1}\sum_{\paperindex =1}^{\numpapers/2} (\bidfunction^{\ordering}(\paperindex)-\bidfunction^{\ordering}(\ordering_{\numrev}(\paperindex)))$. Equivalently, we can find the maximizer of $\max_{\ordering_{\numrev} \in T_1}\sum_{\paperindex =1}^{\numpapers/2} \bidfunction^{\ordering}(\ordering_{\numrev}(\paperindex))$. Since the given function $\bidfunction^{\ordering}$ is decreasing on the domain $\mathbb{R}_{>0}$, the quantity $\sum_{\paperindex =1}^{\numpapers/2} \bidfunction^{\ordering}(\ordering_{\numrev}(\paperindex))$ is maximized when the papers in the set $\{1,\dots, \numpapers/2\}$ are shown the earliest in the ordering $\ordering_{\numrev}$ that is feasible subject to the constraint that fewer than $\numpapers/4$ papers from the set $\{1,\dots, \numpapers/2\}$ are presented in the set of positions $\{1,\dots, \numpapers/2\}$. This means  
\begin{equation}
\ordering_{\numrev}(j) = 
\begin{cases}
\paperindex, & \text{if} \ \paperindex \in \{1,\ldots,\numpapers/4-1\} \\
\paperindex + \numpapers/4+1, & \text{if} \ \paperindex \in \{\numpapers/4, \dots, \numpapers/2\} \\
\paperindex - \numpapers/4-1, & \text{if} \ \paperindex \in \{\numpapers/2+1, \dots, 3\numpapers/4+1\} \\
\paperindex, & \text{if} \ \paperindex \in \{3\numpapers/4+2,\ldots,\numpapers\}
\end{cases}
\label{eq:rand_ordering}
\end{equation}
is a minimizer of $\sum_{\paperindex =1}^{\numpapers/2} (\bidfunction^{\ordering}(\paperindex)-\bidfunction^{\ordering}(\ordering_{\numrev}(\paperindex)))$ among the set $T_1$. 
Substituting the paper ordering from~\eqref{eq:rand_ordering} as the minimizer into~\eqref{eq:rand_point}, we obtain
\begin{equation}
\mathbb{E}[\gain_{\numrev}^{\super}-\gain_{\numrev}^{\randbase}|\history_{\numrev-1}] 
\geq \Big(\frac{1}{18}+\frac{\hyperparam}{6}\Big) \sum_{\paperindex=\numpapers/4}^{\numpapers/2} (\bidfunction^{\ordering}(\paperindex)-\bidfunction^{\ordering}(\paperindex+\numpapers/4+1)).
\label{eq:rand_subopt_mid}
\end{equation}
The following lemma provides a bound on the sum in~\eqref{eq:rand_subopt_mid}.
\begin{lemma}
Let $\bidfunction^{\ordering}(x) = 1/\log_2(x+1)$ and fix $\numpapers\geq 4$ and divisible by four. Then, 
\begin{equation*}
\setlength\belowdisplayskip{-0pt}
\sum_{\paperindex=\numpapers/4}^{\numpapers/2} (\bidfunction^{\ordering}(\paperindex)-\bidfunction^{\ordering}(\paperindex+\numpapers/4+1)) \geq \frac{\numpapers}{32\log_2^2(\numpapers)}.
\end{equation*}
\label{lemma:sum_bound2}
\end{lemma}
\noindent
The proof of Lemma~\ref{lemma:sum_bound2} can be found in Section~\ref{sec:sum_bound2}.

Combining~\eqref{eq:rand_subopt_mid} and Lemma~\ref{lemma:sum_bound2} results in the following bound whenever $\numpapers$ is divisible by four:
\begin{equation}
\mathbb{E}[\gain_{\numrev}^{\super}-\gain_{\numrev}^{\randbase}|\history_{\numrev-1}]  \geq \Big(\frac{1}{576}+\frac{\hyperparam}{192}\Big)\Big(\frac{\numpapers}{\log_2^2(\numpapers)}\Big).
\label{eq:rand_div4}
\end{equation}

The next lemma shows that if the number of papers $\numpapers$ is not divisible by four, an equivalent result (up to constants) holds. For $\numpapers \in \{2, 3\}$, the bound is rather immediate since we can compute the probability that \randbase selects the paper ordering \bidbase shows for this construction and then apply the bound from~\eqref{eq:bid_final} on the suboptimality of \bidbase that holds for any $\numpapers$. For $\numpapers>3$ and not divisible by four, the bound is obtained by looking at an identical problem construction for the maximum $\numpapers'<\numpapers$ divisible by four and then including $\numpapers-\numpapers'$ papers with a similarity score of zero and one previous bid. This change is such that the bound from~\eqref{eq:rand_div4} applies as a function of $\numpapers'$, so the result then follows immediately. 
\begin{lemma}
If $\numpapers$ is not divisible by four, then in the worst case for the final reviewer under the assumptions of Theorem~\ref{prop:localworst},
\begin{equation*}
\mathbb{E}[\gain_{\numrev}^{\super}-\gain_{\numrev}^{\randbase}|\history_{\numrev-1}]  \geq \Big(\frac{1}{1728} + \frac{\hyperparam}{576}\Big)\Big(\frac{\numpapers}{\log_2^2(\numpapers)}\Big).
\end{equation*}
\label{lemma:rand_nodiv4}
\end{lemma}
\noindent
The proof of Lemma~\ref{lemma:rand_nodiv4} can be found in Section~\ref{sec:rand_nodiv4}.

Combining the bound from~\eqref{eq:rand_div4} which holds for $\numpapers$ divisible by four with the bound from Lemma~\ref{lemma:rand_nodiv4} which holds for $\numpapers$ not divisible by four, we find that for every $\numpapers\geq 2$ and $\hyperparam\geq 0$,
\begin{equation*}
\mathbb{E}[\gain_{\numrev}^{\super}-\gain_{\numrev}^{\randbase}|\history_{\numrev-1}]  \geq \Big(\frac{1}{1728} + \frac{\hyperparam}{576}\Big)\Big(\frac{\numpapers}{\log_2^2(\numpapers)}\Big).
\end{equation*}
This proves there is a constant $\constant >0$ such that for all $\numpapers\geq2$ and $\hyperparam\geq 0$, \randbase is suboptimal by an additive factor of at least $\constant\numpapers\max\{1,\hyperparam\}/\log_2^2(\numpapers)$ in the worst case for the final reviewer as claimed in Theorem~\ref{prop:localworst}. 

\subsubsection{Proofs of Lemmas~\ref{lem:multiplicative}--\ref{lemma:rand_nodiv4}}
\label{sec:localworst_lemmas}
In this section, we present the proofs of technical lemmas stated in the primary proof of Theorem~\ref{prop:localworst}.
\paragraph{Proof of Lemma~\ref{lem:multiplicative}.} \label{sec:multiplicative}
Recall from the lemma statement, $\gainfunction(x) =\sqrt{x}$. Moreover, the fixed and given quantities are $\hyperparam\geq 0$, $\numpapers\geq2$, and $\epsilon=(1+\hyperparam)e^{e^e}$. We derive the following bound justified below:
\begin{align}
(\gainfunction(1)-&\gainfunction(0))(1-1/\epsilon)-(\gainfunction(\bidnp)-\gainfunction(\bidn))(1-1/(\numpapers\epsilon))+\hyperparam (2^{(1-1/\epsilon)}-2^{(1-1/(\numpapers\epsilon))}) \notag \\
& \geq (\gainfunction(1)-\gainfunction(0))(1-1/\epsilon)-(\gainfunction(\bidnp)-\gainfunction(\bidn))+\hyperparam(2^{(1-1/\epsilon)}-2) \label{eq:sim_final_leading1} \\
& = (\gainfunction(1)-\gainfunction(0))(1-((1+\hyperparam)e^{e^e})^{-1})-(\gainfunction(\bidnp)-\gainfunction(\bidn)) +\hyperparam(2^{(1-1/((1+\hyperparam)e^{e^e})}-2)  \label{eq:sim_final_leading2} \\
& \geq (0.99)(\gainfunction(1)-\gainfunction(0))-(\gainfunction(\bidnp)-\gainfunction(\bidn)) -0.01 \label{eq:sim_final_leading3} \\
& \geq 1/2. \label{eq:sim_final_leading4}
\end{align}
We obtain~\eqref{eq:sim_final_leading1} using the fact that $(1-1/(\numpapers\epsilon))\leq1$ for any given $\numpapers$. Equation~\eqref{eq:sim_final_leading2} follows from plugging in the explicit form of $\epsilon=(1+\hyperparam)e^{e^e}$.
To see the inequality in~\eqref{eq:sim_final_leading3}, observe that $(1-((1+\hyperparam)e^{e^e})^{-1})$ is an increasing function of $\hyperparam$. Consequently, $(1-((1+\hyperparam)e^{e^e})^{-1})\geq (1-(e^{e^e})^{-1}) \geq 0.99$. 
Furthermore, the quantity $\hyperparam(2^{(1-1/((1+\hyperparam)e^{e^e}))}-2)$ is a decreasing function of $\hyperparam$, from which we determine 
\[\hyperparam(2^{(1-1/((1+\hyperparam)e^{e^e}))}-2) \geq \lim_{\hyperparam'\rightarrow\infty}\hyperparam'(2^{(1-1/((1+\hyperparam')e^{e^e}))}-2) = -e^{-e^e}\log(4) \geq -0.01.\]   
Then, we obtain the final bound in~\eqref{eq:sim_final_leading4} as follows:
\[(0.99)(\gainfunction(1)-\gainfunction(0))-(\gainfunction(\bidnp)-\gainfunction(\bidn)) -0.01 =0.99-(\sqrt{2}-1) -0.01 \geq 1/2.\]

\paragraph{Proof of Lemma~\ref{lemma:sum_bound}.}\label{sec:sum_bound}
Recall that $\bidfunction^{\ordering}(x) = 1/\log_2(x+1)$.
Fixing $\numpapers=2$, we obtain 
\begin{align}
\sum_{\paperindex=1}^{\numpapers/2} (\bidfunction^{\ordering}(\paperindex)-\bidfunction^{\ordering}(\paperindex+\numpapers/2)) &= \bidfunction^{\ordering}(1)-\bidfunction^{\ordering}(2) \notag \\
&= \frac{1}{\log_2(2)}-\frac{1}{\log_2(3)} \notag \\
&\geq \Big(\frac{1}{4}\Big)\Big(\frac{1}{\log_2^2(2)}\Big) \label{eq:d2seq1}\\
&= \Big(\frac{1}{8}\Big)\Big(\frac{\numpapers}{\log_2^2(\numpapers)}\Big).\label{eq:d2seq2}
\end{align}
The inequality in~\eqref{eq:d2seq1} follows from the fact that $\log_2(3)-1\geq 1/2$ and $\log_2(3) \leq 2\log_2(2)$.

Now consider $\numpapers\geq 4$ and $\numpapers$ even.
We derive a bound as follows that is justified below:
\begin{align}
\sum_{\paperindex=1}^{\numpapers/2} (\bidfunction^{\ordering}(\paperindex)-\bidfunction^{\ordering}(\paperindex+\numpapers/2)) &\geq \sum_{\paperindex=1}^{\lfloor\numpapers/4\rfloor} (\bidfunction^{\ordering}(\paperindex)-\bidfunction^{\ordering}(\paperindex+\numpapers/2)) \label{eq:ssb_1}\\
&\geq \Big(\Big\lfloor\frac{\numpapers}{4}\Big\rfloor\Big)  (\bidfunction^{\ordering}(\lfloor\numpapers/4\rfloor)-\bidfunction^{\ordering}(\numpapers/2)) \label{eq:ssb_2}\\
& = \Big(\Big\lfloor\frac{\numpapers}{4}\Big\rfloor\Big) \Big (\frac{1}{\log_2(\lfloor\numpapers/4\rfloor+1)}-\frac{1}{\log_2(\numpapers/2 +1)}\Big)  \notag\\
&= \Big(\Big\lfloor\frac{\numpapers}{4}\Big\rfloor\Big) \Big (\frac{\log_2(\numpapers/2+1)-\log_2(\lfloor\numpapers/4\rfloor+1)}{\log_2(\numpapers/2+1)\log_2(\lfloor\numpapers/4\rfloor+1)}\Big) \notag  \\
& \geq \Big(\Big\lfloor\frac{\numpapers}{4}\Big\rfloor\Big) \Big (\frac{\log_2(3)-1}{\log_2^2(\numpapers)}\Big)\label{eq:ssb_3}\\
&\geq \Big (\frac{1}{16}\Big)\Big(\frac{\numpapers}{\log_2^2(\numpapers)}\Big).  \label{eq:ssb_4}
\end{align}
The inequality in~\eqref{eq:ssb_1} follows from the fact that $\bidfunction^{\ordering}(\paperindex) \geq \bidfunction^{\ordering}(\paperindex+\numpapers/2)$ for every $\paperindex \in [\numpapers/2]$ since $\bidfunction^{\ordering}$ is a decreasing function on the domain $\mathbb{R}_{>0}$. Similarly, we obtain~\eqref{eq:ssb_2} using the observation that $\bidfunction^{\ordering}(\paperindex)-\bidfunction^{\ordering}(\paperindex+\numpapers/2) \geq \bidfunction^{\ordering}(\lfloor\numpapers/4\rfloor)-\bidfunction^{\ordering}(\numpapers/2)$ for every $\paperindex \in [\lfloor\numpapers/4\rfloor]$ since $\bidfunction^{\ordering}$ is a decreasing function on the domain $\mathbb{R}_{>0}$. To see how~\eqref{eq:ssb_3} is derived, observe that
\begin{equation}
    \log_2(\numpapers/2+1)-\log_2(\lfloor\numpapers/4\rfloor+1) \geq \log_2(\numpapers/2+1)-\log_2(\numpapers/4+1)\geq \log_2(3)-1
    \label{eq:lblog}
\end{equation}
where the final inequality in \eqref{eq:lblog} follows since $\log_2(\numpapers/2+1)-\log_2(\numpapers/4+1)$ is increasing in $\numpapers$ and $\numpapers\geq 4$ by assumption. Moreover, $\log_2(\numpapers/2+1)\log_2(\lfloor\numpapers/4\rfloor+1)\leq \log_2^2(\numpapers)$ for $\numpapers\geq 4$. 
The final inequality in~\eqref{eq:ssb_4} holds since $\log_2(3)-1\geq 1/2$ and $\lfloor \numpapers/4\rfloor = \numpapers/4 - (\numpapers \ \mathrm{mod} \ 4)/4 \geq \numpapers/8$ for $\numpapers\geq 4$. 

Combining the bound for $\numpapers=2$ from~\eqref{eq:d2seq2} and the bound for $\numpapers\geq 4$ and even from~\eqref{eq:ssb_4}, we conclude 
\begin{equation*}
\sum_{\paperindex=1}^{\numpapers/2} (\bidfunction^{\ordering}(\paperindex)-\bidfunction^{\ordering}(\paperindex+\numpapers/2)) \geq \Big(\frac{1}{16}\Big)\Big(\frac{\numpapers}{\log_2^2(\numpapers)}\Big),
\end{equation*}
whenever $\numpapers\geq 2$ and even.

\paragraph{Proof of Lemma~\ref{lemma:sim_odd_bound}.}\label{sec:sim_odd_bound}
In this proof, we show a simple adaptation of the problem construction from Section~\ref{sec:localworst_sim} results in a suboptimality bound on \simbase for the final reviewer when the number of papers $\numpapers$ is odd that matches (up to constants) the bound given in~\eqref{eq:sim_div2} that holds whenever the paper count $\numpapers$ is even.

Fix $\numpapers$ odd and let $\numpapers' = \numpapers - 1$ (and hence $\numpapers'$ is an even number).
We consider an identical problem construction for the papers $\paperindex \in [\numpapers']$ as from Section~\ref{sec:localworst_sim} and then include a paper $\numpapers$ that has a similarity score of zero and one bid. This change is such that for the given class of functions in the model, \super and \simbase show the paper $\numpapers$ after the papers in the set $[\numpapers']$ and the expected gain from paper $\numpapers$ is deterministically zero since the similarity score is zero. 

In particular, let the similarity scores for the final reviewer be $\similarity_{\numrev, \paperindex}=1-1/(\paperindex\epsilon)$ for each paper $\paperindex \in [\numpapers']$, where $\epsilon = (1+\hyperparam)e^{e^{e}}$ and $\hyperparam\geq 0$ is the fixed and given trade-off parameter. Moreover, let $\similarity_{\numrev, \numpapers}=0$.
Set the number of bids on the papers from previous reviewers to be
\begin{equation*}
\numbids_{\numrev-1, \paperindex} =
\begin{cases}
0, & \text{if} \ \paperindex \in \{1,\ldots,\numpapers'/2\} \\
\bidn, & \text{if} \ \paperindex \in \{\numpapers'/2+1,\ldots,\numpapers\}.
\end{cases}
\end{equation*}

In Section~\ref{sec:localworst_sim}, we showed if a pair of papers $\paperindex, \paperindex'$ are such that $\numbids_{\numrev-1, \paperindex} = \numbids_{\numrev-1, \paperindex'}$ and $\similarity_{\numrev, \paperindex}>\similarity_{\numrev, \paperindex'}$, then $\ordering_{\numrev}^{\super}(\paperindex)<\ordering_{\numrev}^{\super}(\paperindex')$. Since paper $\paperindex'=\numpapers$ is such that for every paper $\paperindex \in \{\numpapers'/2+1,\ldots,\numpapers'\}$, $\numbids_{\numrev-1, \paperindex} = \numbids_{\numrev-1, \paperindex'}$ and $\similarity_{\numrev, \paperindex}>\similarity_{\numrev, \paperindex'}$, we find $\ordering_{\numrev}^{\super}(\paperindex') < \ordering_{\numrev}^{\super}(\paperindex')$. From~\eqref{eq:sim_construct}, this means for this construction
$\ordering_{\numrev}^{\super}(\numpapers)=\numpapers$
and that \super shows the paper ordering 
\begin{equation*}
\ordering_{\numrev}^{\super}(j) = 
\begin{cases}
\numpapers'/2-\paperindex + 1, & \text{if} \ \paperindex \in \{1,\dots,\numpapers'/2\} \\
\numpapers' + \numpapers'/2+1-\paperindex, & \text{if} \ \paperindex \in \{\numpapers'/2+1,\dots,\numpapers'\} \\
\paperindex, & \text{if} \ \paperindex \in \{\numpapers\}.
\end{cases}
\end{equation*}
The \simbase algorithm shows papers in a decreasing order of the similarity scores so that 
\begin{equation*}
\ordering_{\numrev}^{\simbase}(\paperindex) = 
\begin{cases}
\numpapers'-\paperindex+1, & \text{if} \ \paperindex \in \{1,\dots, \numpapers'\} \\
\paperindex, & \text{if} \ \paperindex \in \{\numpapers\}.
\end{cases}
\end{equation*}

Observe that this construction is identical to the problem construction from Section~\ref{sec:localworst_sim} for papers in the set $[\numpapers']$. 
Moreover, from~\eqref{eq:gain_new} it is clear that the expected gain from papers with zero similarity score is zero independent of the paper ordering. This allows us to conclude that the bound given in~\eqref{eq:sim_div2} for $\numpapers$ even applies to this construction as a function of $\numpapers'$. In other words, given $\numpapers$ odd, we obtain
\begin{equation*}
\mathbb{E}[\gain^{\super}_{\numrev} -\gain^{\simbase}_{\numrev}|\history_{\numrev-1}] \geq \Big(\frac{1}{32}\Big)\Big(\frac{\numpapers'}{\log_2^2(\numpapers')}\Big).
\end{equation*}
Moreover, since $\numpapers' = \numpapers-1$, whenever $\numpapers$ is odd,
\begin{equation*}
\mathbb{E}[\gain^{\super}_{\numrev} -\gain^{\simbase}_{\numrev}|\history_{\numrev-1}] \geq \Big(\frac{1}{32}\Big)\Big(\frac{\numpapers-1}{\log_2^2(\numpapers-1)}\Big) \geq \Big(\frac{1}{64}\Big)\Big(\frac{\numpapers}{\log_2^2(\numpapers)}\Big),
\end{equation*}
where the final inequality follows from the facts that $\log_2^2(\numpapers-1)\leq \log_2^2(\numpapers)$ and $\numpapers-1\geq \numpapers/2$ for $\numpapers\geq 2$.

\paragraph{Proof of Lemma~\ref{lemma:bid_odd}.}\label{sec:bid_odd}
In this proof, we show a simple adaptation of the problem construction from Section~\ref{sec:localworst_bid} leads to a suboptimality bound on \bidbase for the final reviewer when the number of papers is odd that matches (up to constants) the bound given in~\eqref{eq:bid_even} that holds whenever the number of papers $\numpapers$ is even. This approach is analogous to the method to obtain a suboptimality bound for \simbase with an odd number of papers using the bound that held for an even number of papers from Lemma~\ref{lemma:sim_odd_bound}. 

Fix $\numpapers$ odd and let $\numpapers' = \numpapers - 1$ (and hence $\numpapers'$ is an even number). 
We consider an identical problem construction for the papers $\paperindex \in [\numpapers']$ as from Section~\ref{sec:localworst_bid} and then include a paper $\numpapers$ that has a similarity score of zero and one bid. This change is such that for the given class of functions in the model, \super and \bidbase show the paper $\numpapers$ after the papers in the set $[\numpapers']$ and the expected gain from paper $\numpapers$ is deterministically zero since the similarity score is zero. 

Let the similarity scores for the final reviewer be 
\begin{equation*}
\similarity_{\numrev, \paperindex}=
\begin{cases}
1, & \text{if} \ \paperindex \in \{1,\ldots,\numpapers'/2\} \\
0, & \text{if} \ \paperindex \in \{\numpapers'/2+1,\ldots,\numpapers\}
\end{cases}
\end{equation*}
and the number of bids on the papers from previous reviewers be  
\begin{equation*}
\numbids_{\numrev-1, \paperindex} =
\begin{cases}
1, & \text{if} \ \paperindex \in \{1,\ldots,\numpapers'/2\} \\
0, & \text{if} \ \paperindex \in \{\numpapers'/2+1,\ldots,\numpapers'\}\\
1, & \text{if} \ \paperindex=\numpapers.
\end{cases}
\end{equation*}

We now derive the optimal paper ordering for the final reviewer. Recall  from~\eqref{eq:baseline_subopt_opt_problem_weights} that the weights of the optimization problem for the final reviewer given in~\eqref{eq:baseline_subopt_opt_problem} are defined by
\begin{equation*}
\sortvar_{\numrev, \paperindex} = \similarity_{\numrev, \paperindex}(\gainfunction(\numbids_{\numrev-1, \paperindex} + 1) - \gainfunction(\numbids_{\numrev-1, \paperindex})) + \hyperparam(2^{\similarity_{\numrev, \paperindex}}-1)  \ \forall \ \paperindex \in \paperset.
\end{equation*} 
Moreover, from the structure of the optimization problem in~\eqref{eq:baseline_subopt_opt_problem}, if $\sortvar_{\numrev, \paperindex}>\sortvar_{\numrev, \paperindex'}$, then $\ordering_{\numrev}^{\super}(\paperindex)<\ordering_{\numrev}^{\super}(\paperindex')$ so that paper $\paperindex$ is shown ahead of paper $\paperindex'$ in the ranking. Observe that for each paper $\paperindex \in \{1,\dots, \numpapers'/2\}$, $\sortvar_{\numrev, \paperindex}$ is a fixed number. Similarly, for each $\paperindex' \in \{\numpapers'/2+1,\dots, \numpapers'\}\cup \{\numpapers\}$, $\sortvar_{\numrev, \paperindex'}$ is a fixed number. 
In Section~\ref{sec:localworst_bid}, we showed if a pair of papers $\paperindex, \paperindex'$ are such that $\numbids_{\numrev-1, \paperindex} =1$ and $\numbids_{\numrev-1, \paperindex'}=0$ along with $\similarity_{\numrev, \paperindex}=1$ and $\similarity_{\numrev, \paperindex'}=0$, then $\sortvar_{\numrev, \paperindex} > \sortvar_{\numrev, \paperindex'}$ so that $\ordering_{\numrev}^{\super}(\paperindex)<\ordering_{\numrev}^{\super}(\paperindex')$. 
This immediately guarantees $\ordering_{\numrev}^{\super}(\paperindex)<\ordering_{\numrev}^{\super}(\paperindex')$ for each pair of papers $\paperindex \in \{1,\dots, \numpapers'/2\}$, $\paperindex' \in \{\numpapers'/2+1,\dots, \numpapers'\}\cup \{\numpapers\}$. 
We conclude $\ordering^{\super}_{\numrev}(j) = \paperindex$ for each $\paperindex \in \paperset$, where without loss of generality, to simplify the analysis, we assume if a pair of papers have equal weights in the optimization problem, then ties are broken in order of the paper indexes since the tie-breaking mechanism will not change the expected gain the paper ordering obtains from the final reviewer.

The \bidbase baseline will show papers in an increasing order of the number of bids so that 
\begin{equation*}
\ordering_{\numrev}^{\bidbase}(j) = 
\begin{cases}
\paperindex + \numpapers'/2, & \text{if} \ \paperindex \in \{1,\ldots,\numpapers'/2\} \\
\paperindex - \numpapers'/2, & \text{if} \ \paperindex \in \{\numpapers'/2+1,\ldots,\numpapers'\} \\
\paperindex & \text{if} \ \paperindex=\numpapers.
\end{cases}
\end{equation*}
This paper ordering is derived from recalling that \bidbase breaks ties by the similarity scores and further ties are broken uniformly at random. However, without loss of generality, to simplify the analysis, we assume if a pair of papers have equal similarity scores and bid counts, then ties are broken in order of the paper indexes since the tie-breaking mechanism among this set of papers will not impact the expected gain. 

The construction in this proof and the paper orderings selected by \super and \bidbase are identical to the problem construction and paper orderings selected by \super and \bidbase from Section~\ref{sec:localworst_bid} for papers in the set $[\numpapers']$. 
From~\eqref{eq:gain_new} it is clear that the expected gain from papers with zero similarity score is zero independent of the paper ordering.
This allows us to conclude that the bound given in~\eqref{eq:bid_even} for $\numpapers$ even applies to this construction as a function of $\numpapers'$. In other words, given $\numpapers$ odd, we obtain
\begin{equation*}
\mathbb{E}[\gain^{\super}_{\numrev} -\gain^{\bidbase}_{\numrev}|\history_{\numrev-1}] \geq \Big(\frac{1}{48} + \frac{\hyperparam}{16}\Big)\Big(\frac{\numpapers'}{\log_2^2(\numpapers')}\Big).
\end{equation*}
Since $\numpapers' = \numpapers-1$, whenever $\numpapers$ is odd,
\begin{equation*}
\mathbb{E}[\gain^{\super}_{\numrev} -\gain^{\bidbase}_{\numrev}|\history_{\numrev-1}] \geq \Big(\frac{1}{48} + \frac{\hyperparam}{16}\Big)\Big(\frac{\numpapers-1}{\log_2^2(\numpapers-1)}\Big) \geq \Big(\frac{1}{96}+\frac{\hyperparam}{32}\Big)\Big(\frac{\numpapers}{\log_2^2(\numpapers)}\Big),
\end{equation*}
where the final inequality follows from the facts that $\log_2^2(\numpapers-1)\leq \log_2^2(\numpapers)$ and $\numpapers-1\geq \numpapers/2$ for $\numpapers\geq 2$.

\paragraph{Proof of Lemma~\ref{lem:prob_bound}.}\label{sec:prob_bound}
Recall that from the lemma statement, the number of papers $\numpapers$ is assumed to be divisible by four. Let $\mathcal{E}$ denote the event that a permutation $\ordering$ of the paper set $\paperset$ drawn uniformly at random from $\symgroup_{\numpapers}$ has fewer than $\numpapers/4$ of the papers from the set $[\numpapers/2]$ in the position set $[\numpapers/2]$. 

The probability of the event $\mathcal{E}$ can be explained in the following manner. 
The number of outcomes presenting $\paperindex$ papers from the paper set $[\numpapers/2]$ in the position set $[\numpapers/2]$ consists of
${\numpapers/2 \choose \paperindex}$ 
combinations of potential papers that can be selected from the paper set 
$[\numpapers/2]$ and ${\numpapers/2 \choose \paperindex}$ combinations of potential positions in the position set $[\numpapers/2]$. Moreover, there are $\paperindex!$ permutations of the selected papers in the chosen positions. 
Given that there are $\paperindex$ papers selected from the paper set $[\numpapers/2]$ placed in the position set $[\numpapers/2]$, there are ${\numpapers/2 \choose \numpapers/2-\paperindex}$ combinations of papers from the paper set $\{\numpapers/2+1,\dots,\numpapers\}$ that can be placed in the remaining spots in the position set $[\numpapers/2]$. This set of papers can be permuted $(\numpapers/2-\paperindex)!$ ways in the given set of positions, and the remaining $\numpapers/2$ papers can be permuted $(\numpapers/2)!$ ways in the position set $\{\numpapers/2+1,\dots, \numpapers\}$. To obtain the final probability of the event $\mathcal{E}$, we sum the number of outcomes for each $\paperindex < \numpapers/4$ and then normalize by the total number of outcomes $\numpapers!$. Accordingly, 
\begin{align}
\mathbb{P}(\mathcal{E}) &=\frac{1}{\numpapers!}\sum_{\paperindex =0}^{\numpapers/4-1} {\numpapers/2 \choose \paperindex}{\numpapers/2 \choose \paperindex}(\paperindex!) {\numpapers/2 \choose \numpapers/2 - \paperindex}(\numpapers/2 - \paperindex)!(\numpapers/2)! \notag \\
&= \frac{1}{\numpapers!}\sum_{\paperindex =0}^{\numpapers/4-1} {\numpapers/2 \choose \paperindex}{\numpapers/2 \choose \paperindex}(\paperindex!) \Big(\frac{(\numpapers/2)!}{(\numpapers/2-\paperindex)!\paperindex!}\Big)(\numpapers/2 - \paperindex)!(\numpapers/2)! \notag \\
&= \frac{(\numpapers/2)!(\numpapers/2)!}{\numpapers!}\sum_{\paperindex =0}^{\numpapers/4-1}{\numpapers/2 \choose \paperindex}^2 \label{eq:prob_bound}.
\end{align}

We now recall some facts about the binomial coefficients. The symmetry property of the binomial coefficients implies ${n \choose k} = {n \choose n - k}$ for $0\leq k \leq n$ and Vandermonde's identity says that ${m+n \choose r} = \sum_{k=0}^r {m \choose k}{n \choose r - k}$ and as a corollary ${m \choose n} = \sum_{k=0}^m {2m \choose n}$. Using this set of facts, we work toward a lower bound on $\mathbb{P}(\mathcal{E})$ by obtaining a simplified form of the sum $\sum_{\paperindex =0}^{\numpapers/4-1}{\numpapers/2 \choose \paperindex}^2$. Observe that 
\begin{align*}
\sum_{\paperindex =0}^{\numpapers/4-1}{\numpapers/2 \choose \paperindex}^2 &= \sum_{\paperindex =0}^{\numpapers/4}{\numpapers/2 \choose \paperindex}^2 - {\numpapers/2 \choose \numpapers/4}^2 \\
&= \frac{1}{2}\sum_{\paperindex =0}^{\numpapers/4}{\numpapers/2 \choose \paperindex}^2 + \frac{1}{2}\sum_{\paperindex =0}^{\numpapers/4}{\numpapers/2 \choose \paperindex}^2 - {\numpapers/2 \choose \numpapers/4}^2.
\end{align*}
From the symmetry property, ${\numpapers/2 \choose \paperindex}^2={\numpapers/2 \choose \numpapers/2 - \paperindex}^2$, so we get
\begin{equation*}
\sum_{\paperindex =0}^{\numpapers/4-1}{\numpapers/2 \choose \paperindex}^2= \frac{1}{2}\sum_{\paperindex =0}^{\numpapers/4}{\numpapers/2 \choose \paperindex}^2 + \frac{1}{2}\sum_{\paperindex =0}^{\numpapers/4}{\numpapers/2 \choose \numpapers/2 - \paperindex}^2 - {\numpapers/2 \choose \numpapers/4}^2.
\end{equation*}
Manipulating the indexing of the sum $\sum_{\paperindex =0}^{\numpapers/4}{\numpapers/2 \choose \numpapers/2 - \paperindex}^2$, we obtain
\begin{equation*}
\sum_{\paperindex =0}^{\numpapers/4-1}{\numpapers/2 \choose \paperindex}^2= \frac{1}{2}\sum_{\paperindex =0}^{\numpapers/4}{\numpapers/2 \choose \paperindex}^2 + \frac{1}{2}\sum_{\paperindex =\numpapers/4}^{\numpapers/2}{\numpapers/2 \choose \paperindex}^2 - {\numpapers/2 \choose \numpapers/4}^2.
\end{equation*}
Now, moving the term $\frac{1}{2}{\numpapers/2 \choose \numpapers/4}^2$ out of the sum $\frac{1}{2}\sum_{\paperindex =\numpapers/4}^{\numpapers/2}{\numpapers/2 \choose \paperindex}^2$ results in
\begin{equation*}
\sum_{\paperindex =0}^{\numpapers/4-1}{\numpapers/2 \choose \paperindex}^2= \frac{1}{2}\sum_{\paperindex =0}^{\numpapers/4}{\numpapers/2 \choose \paperindex}^2 + \frac{1}{2}\sum_{\paperindex =\numpapers/4+1}^{\numpapers/2}{\numpapers/2 \choose \paperindex}^2 - \frac{1}{2}{\numpapers/2 \choose \numpapers/4}^2.
\end{equation*}
Furthermore,
\begin{equation*}
\sum_{\paperindex =0}^{\numpapers/4-1}{\numpapers/2 \choose \paperindex}^2= \frac{1}{2}\sum_{\paperindex =0}^{\numpapers/2}{\numpapers/2 \choose \paperindex}^2 - \frac{1}{2}{\numpapers/2 \choose \numpapers/4}^2.
\end{equation*}
Finally, applying Vandermonde's identity as given above, we get
\begin{equation}
\sum_{\paperindex =0}^{\numpapers/4-1}{\numpapers/2 \choose \paperindex}^2= \frac{1}{2}{\numpapers \choose \numpapers/2}^2 - \frac{1}{2}{\numpapers/2 \choose \numpapers/4}^2.
\label{eq:prob_simple}
\end{equation}
Combing~\eqref{eq:prob_bound} with~\eqref{eq:prob_simple} and then simplifying, we get 
\begin{align*}
\mathbb{P}(\mathcal{E}) 
&= \Big(\frac{(\numpapers/2)!(\numpapers/2)!}{2\numpapers!}\Big)\Big({\numpapers \choose \numpapers/2} - {\numpapers/2 \choose \numpapers/4}^2\Big) \\
&= \Big(\frac{(\numpapers/2)!(\numpapers/2)!}{2\numpapers!}\Big)\Big(\frac{\numpapers!}{(\numpapers/2)!(\numpapers/2)!}\Big) - \Big(\frac{(\numpapers/2)!(\numpapers/2)!}{2\numpapers!}\Big)\Big(\frac{(\numpapers/2)!}{(\numpapers/4)!(\numpapers/4)!}\Big)^2\\
&= \frac{1}{2} - \Big(\frac{(\numpapers/2)!(\numpapers/2)!}{2\numpapers!}\Big)\Big(\frac{(\numpapers/2)!}{(\numpapers/4)!(\numpapers/4)!}\Big)^2.
\end{align*}
The quantity $\Big(\frac{(\numpapers/2)!(\numpapers/2)!}{2\numpapers!}\Big)\Big(\frac{(\numpapers/2)!}{(\numpapers/4)!(\numpapers/4)!}\Big)^2$ is decreasing in $\numpapers$. Consequently, for every $\numpapers\geq 4$, 
\begin{equation*}
\Big(\frac{(\numpapers/2)!(\numpapers/2)!}{2\numpapers!}\Big)\Big(\frac{(\numpapers/2)!}{(\numpapers/4)!(\numpapers/4)!}\Big)^2\leq \frac{1}{3}.
\end{equation*}
This allows us to conclude 
\begin{equation*}
\mathbb{P}(\mathcal{E})\geq 1/2-1/3 = 1/6.
\end{equation*}

\paragraph{Proof of Lemma~\ref{lemma:sum_bound2}.} \label{sec:sum_bound2}
This proof follows in a similar manner to the proof of Lemma~\ref{lemma:sum_bound}.
Recall that $\bidfunction^{\ordering}(x) = 1/\log_2(x+1)$.
Fixing $\numpapers\geq 4$ and divisible by four, we obtain the following bound justified below: 
\begin{align}
\sum_{\paperindex=\numpapers/4}^{\numpapers/2} (\bidfunction^{\ordering}(\paperindex)-\bidfunction^{\ordering}(\paperindex+\numpapers/4+1)) &\geq \sum_{\paperindex=\numpapers/4}^{\lfloor3\numpapers/8\rfloor} (\bidfunction^{\ordering}(\paperindex)-\bidfunction^{\ordering}(\paperindex+\numpapers/4+1)) \label{eq:rsb_1}\\
& \geq \Big(\Big\lfloor\frac{3\numpapers}{8}\Big\rfloor+1-\frac{\numpapers}{4}\Big)  (\bidfunction^{\ordering}(\lfloor3\numpapers/8\rfloor)-\bidfunction^{\ordering}(\numpapers/2)) \label{eq:rsb_2}\\
& = \Big(\Big\lfloor\frac{3\numpapers}{8}\Big\rfloor+1-\frac{\numpapers}{4}\Big) \Big (\frac{1}{\log_2(\lfloor3\numpapers/8\rfloor+1)}-\frac{1}{\log_2(\numpapers/2 +1)}\Big)  \notag \\
&= \Big(\Big\lfloor\frac{3\numpapers}{8}\Big\rfloor+1-\frac{\numpapers}{4}\Big) \Big (\frac{\log_2(\numpapers/2+1)-\log_2(\lfloor 3\numpapers/8\rfloor+1)}{\log_2(\numpapers/2+1)\log_2(\lfloor3\numpapers/8\rfloor+1)}\Big) \notag \\
& \geq \Big(\Big\lfloor\frac{3\numpapers}{8}\Big\rfloor+1-\frac{\numpapers}{4}\Big) \Big (\frac{\log_2(3)-\log_2(12/8+1)}{\log_2^2(\numpapers)}\Big) \label{eq:rsb_3} \\
&= \Big (\frac{1}{32}\Big) \Big(\frac{\numpapers}{\log_2^2(\numpapers)}\Big).   \label{eq:rsb_4}
\end{align}
The inequality in~\eqref{eq:rsb_1} follows from the fact that $\bidfunction^{\ordering}(\paperindex) \geq \bidfunction^{\ordering}(\paperindex+\numpapers/4+1)$ for every $\paperindex \in \{\numpapers/4,\dots, \lfloor 3\numpapers/8\rfloor\}$ since $\bidfunction^{\ordering}$ is a decreasing function on the domain $\mathbb{R}_{>0}$. Similarly, we obtain~\eqref{eq:rsb_2} using the observation that $\bidfunction^{\ordering}(\paperindex)-\bidfunction^{\ordering}(\paperindex+\numpapers/4+1) \geq \bidfunction^{\ordering}(\lfloor3\numpapers/8\rfloor)-\bidfunction^{\ordering}(\numpapers/2)$ for every $\paperindex \in \{\numpapers/4,\dots, \lfloor 3\numpapers/8\rfloor\}$ since $\bidfunction^{\ordering}$ is a decreasing function on the domain $\mathbb{R}_{>0}$. To see how~\eqref{eq:rsb_3} is derived, observe that
\begin{equation}
    \log_2(\numpapers/2+1)-\log_2(\lfloor3\numpapers/8\rfloor+1) \geq \log_2(\numpapers/2+1)-\log_2(3\numpapers/8+1)\geq \log_2(3)-\log_2(12/8+1),
    \label{eq:lblb2}
\end{equation}
where the final inequality in \eqref{eq:lblb2} follows since $\log_2(\numpapers/2+1)-\log_2(3\numpapers/8+1)$ is increasing in $\numpapers$ and $\numpapers\geq 4$ by assumption. Moreover, $\log_2(\numpapers/2+1)\log_2(\lfloor3\numpapers/8\rfloor+1)\leq \log_2^2(\numpapers)$ for $\numpapers\geq 4$. 
The final inequality in~\eqref{eq:rsb_4} holds since $\log_2(3)-\log_2(12/8+1)\geq 1/4$ and 
\[\lfloor3\numpapers/8\rfloor+1-\numpapers/4\geq 3\numpapers/8-\numpapers/4 \geq \numpapers/8.\]

\paragraph{Proof of Lemma~\ref{lemma:rand_nodiv4}.}\label{sec:rand_nodiv4}
Let us begin with $\numpapers\in \{2, 3\}$. It is immediate that \randbase selects the paper ordering of \bidbase with probability at least $1/6$ since $|\symgroup_{\numpapers}|\leq 6$. Moreover, any paper ordering selected by \randbase cannot obtain higher expected gain from the final reviewer than \super since it is optimal for the final reviewer. Accordingly, combined with the bound on \bidbase from~\eqref{eq:bid_final} which holds for any $\numpapers\geq 2$, we obtain for $\numpapers\in \{2, 3\}$,
\begin{equation}
\mathbb{E}[\gain^{\super}_{\numrev} -\gain^{\randbase}_{\numrev}|\history_{\numrev-1}] \geq \Big(\frac{1}{488}+\frac{\hyperparam}{192}\Big)\Big(\frac{\numpapers}{\log_2^2(\numpapers)}\Big).
\label{eq:rand_23}
\end{equation}

We now focus on $\numpapers>3$ and not divisible by four. The problem construction we consider is similar to that from Lemma~\ref{lemma:bid_odd} where the number of papers was odd and it was derived from including a paper with zero similarity score and a bid as an extra paper to the problem construction from Section~\ref{sec:localworst_bid}. We follow the same approach, but let $\numpapers'$ be the maximum number divisible by four such that $\numpapers'< \numpapers$ and consider an identical problem construction for the papers in the set $[\numpapers']$, but then include $\numpapers-\numpapers'$ extra papers with zero similarity and a bid. This change is such that the papers in the set $\{\numpapers'+1, \dots, \numpapers\}$ are shown after the papers in the set $[\numpapers']$ by \super and the expected gain from them is deterministically zero since the similarity scores are zero. 

In particular, let the similarity scores for the final reviewer be 
\begin{equation*}
\similarity_{\numrev, \paperindex}=
\begin{cases}
1, & \text{if} \ \paperindex \in \{1,\ldots,\numpapers'/2\} \\
0, & \text{if} \ \paperindex \in \{\numpapers'/2+1,\ldots,\numpapers\}
\end{cases}
\end{equation*}
and the number of bids on the papers from previous reviewers be  
\begin{equation*}
\numbids_{\numrev-1, \paperindex} =
\begin{cases}
1, & \text{if} \ \paperindex \in \{1,\ldots,\numpapers'/2\} \\
0, & \text{if} \ \paperindex \in \{\numpapers'/2+1,\ldots,\numpapers'\}\\
1, & \text{if} \ \paperindex \in \{\numpapers'+1,\dots, \numpapers\}.
\end{cases}
\end{equation*}
Following the exact reasoning from the proof of Lemma~\ref{lemma:bid_odd}, we conclude $\ordering^{\super}_{\numrev}(j) = \paperindex$ for each $\paperindex \in \paperset$.

For this construction and \randbase, we need to lower bound~\eqref{eq:diff}.
We simplify the expression by substituting the similarity scores and the number of bids on each paper, and the paper ordering presented by \super for this construction to obtain
\begin{equation*}
\mathbb{E}[\gain_{\numrev}^{\super}-\gain_{\numrev}^{\randbase}|\history_{\numrev-1}] 
=\mathbb{E}_{\ordering_{\numrev}^{\randbase}}\big[(\gainfunction(\bidnp)-\gainfunction(\bidn)+\hyperparam)\sum_{\paperindex=1}^{\numpapers'/2}(\bidfunction^{\ordering}(\paperindex)-\bidfunction^{\ordering}(\ordering_{\numrev}^{\randbase}(\paperindex)))\big].
\end{equation*}
where the terms for papers in the set $\{\numpapers'/2+1,\dots, \numpapers\}$ dropped out since the similarity scores are zero. From this point, it is clear that the analysis beginning from~\eqref{eq:rand_sum} in Section~\ref{sec:localworst_rand} can be repeated as a function of $\numpapers'$ to obtain the bound 
\begin{equation*}
\mathbb{E}[\gain_{\numrev}^{\super}-\gain_{\numrev}^{\randbase}|\history_{\numrev-1}]  \geq \Big(\frac{1}{576}+\frac{\hyperparam}{192}\Big)\Big(\frac{\numpapers'}{\log_2^2(\numpapers')}\Big).
\end{equation*}
Since $\numpapers' \geq \numpapers-3$, we obtain for $\numpapers>3$ and not divisible by four,
\begin{equation}
\mathbb{E}[\gain_{\numrev}^{\super}-\gain_{\numrev}^{\randbase}|\history_{\numrev-1}]  \geq \Big(\frac{1}{576}+\frac{\hyperparam}{192}\Big)\Big(\frac{\numpapers-3}{\log_2^2(\numpapers-3)}\Big) \geq \Big(\frac{1}{1728}+\frac{\hyperparam}{576}\Big)\Big(\frac{\numpapers}{\log_2^2(\numpapers)}\Big)
\label{eq:rand_nodiv4}
\end{equation}
where the final inequality follows from the facts that $\log_2^2(\numpapers-3)\leq \log_2^2(\numpapers)$ and $\numpapers-3\geq \numpapers/3$ for $\numpapers\geq 5$. 

Combining the bound from~\eqref{eq:rand_23} for $\numpapers \in \{2, 3\}$ with the bound from~\eqref{eq:rand_nodiv4} for $\numpapers>3$ and not divisible by four, we conclude for every $\numpapers\geq 2$ and not divisible by four, 
\begin{equation*}
\mathbb{E}[\gain_{\numrev}^{\super}-\gain_{\numrev}^{\randbase}|\history_{\numrev-1}] \geq \Big(\frac{1}{1728}+\frac{\hyperparam}{576}\Big)\Big(\frac{\numpapers}{\log_2^2(\numpapers)}\Big).
\end{equation*}

\subsection{Proof of Theorem~\ref{prop:diagonal}: Noiseless Community Model Result}
In this proof, we show for the noiseless community model defined in Section~\ref{sec:global_opt_model} that \super with zero heuristic and \simbase are optimal. Moreover, we prove that \bidbase and \randbase are significantly suboptimal. 
The organization of this proof is as follows. In Section~\ref{sec:noiseless_notation}, we present additional notation that is needed in the proof. Section~\ref{sec:noiseless_prelim} presents simplifying preliminary analysis that is needed throughout the proof to analyze the expected paper-side and reviewer-side gains of the algorithms. In Section~\ref{sec:noiseless_optimal}, we characterize the optimal policy for the noiseless community model. We show in Sections~\ref{sec:noiseless_super} and~\ref{sec:noiseless_sim} that \super with zero heuristic and \simbase are equivalent to the optimal policy, respectively. We prove the suboptimality bounds for \bidbase and \randbase in Sections~\ref{sec:noiseless_bid} and~\ref{sec:noiseless_rand}, respectively. Combining the results from the sections of this proof gives the stated result of Theorem~\ref{prop:diagonal}. 
We relegate the proofs of technical lemmas needed for this result to Section~\ref{sec:noiseless_lemmas}. 

\subsubsection{Notation}
\label{sec:noiseless_notation}
Theorem~\ref{prop:diagonal} holds for any similarity matrix $\similarity$ belonging to the noiseless community model defined in Section~\ref{sec:global_opt_model} and formally in~\eqref{eq:noiseless_model}. From this point on in the proof, any reference to a similarity matrix $\similarity$ is such that it belongs to the noiseless community model. Recall that the number of reviewers is given by $\numrev=\numblocks\blocksize$ and the number of papers is given by $\numpapers=\numblocks\blocksize$ where $\numblocks\geq 2$ and $\blocksize\geq 2$.

We now state some additional notation for the proof and recall the class of gain and bidding functions assumed in this claim. Let us define for each reviewer $\revindex \in \reviewerset$ the set 
\begin{equation}
\diagset_{\revindex} = \{\paperindex \in \paperset: \similarity_{\revindex, \paperindex}=\simscalar\},
\label{eq:diagset_i}
\end{equation}
which comprises the papers on the block diagonal of the noiseless community model similarity matrix for the reviewer up to a permutation of rows and columns. Similarly, define for each paper $\paperindex \in \paperset$ the set 
\begin{equation}
\diagsetp_{\paperindex} = \{\revindex \in \reviewerset: \similarity_{\revindex, \paperindex}=\simscalar\},
\label{eq:diagset_j}
\end{equation}
which comprises the reviewers on the block diagonal of the noiseless community model similarity matrix for the paper up to a permutation of rows and columns. Observe that $|\diagset_{\revindex}|=\blocksize$ for each reviewer $\revindex \in \reviewerset$ and $|\diagset_{\paperindex}|=\blocksize$ for each paper $\paperindex \in \paperset$.
In the remainder of the proof, we simply refer to the set $\diagset_{\revindex}$ as the papers on the block diagonal for a reviewer $\revindex \in \reviewerset$ and the set $\diagsetp_{\paperindex}$ as the reviewers on the block diagonal for a paper $\paperindex \in \paperset$, and omit the wording of up to a permutation of rows and columns for brevity. Similarly, if a reviewer-paper pair $(\revindex, \paperindex)$ is such that $\similarity_{\revindex, \paperindex}=\simscalar$ so that $\revindex\in \diagset_{\revindex}$ and $\paperindex\in \diagsetp_{\revindex}$, we say the reviewer-paper pair is on the block diagonal and omit that this is up to a permutation of rows and columns.
Moreover, we denote by $\diagset_{\revindex}^c$ the complement of the set $\diagset_{\revindex}$ for any reviewer $\revindex \in \reviewerset$, which contains each paper $\paperindex \in \paperset$ not in the set $\diagset_{\revindex}$ and corresponds to the papers not on the block diagonal for the reviewer. Similarly, we let $\diagsetp_{\paperindex}^c$ denote the complement of the set $\diagsetp_{\paperindex}$ for any paper $\paperindex \in \paperset$, which contains each reviewer $\revindex \in \reviewerset$ not in the set $\diagset_{\paperindex}$ and corresponds to the reviewers not on the block diagonal for the paper. Finally, if a reviewer-paper pair $(\revindex, \paperindex)$ is such that $\similarity_{\revindex, \paperindex}=\simscalar$ so that $\revindex\in \diagset_{\revindex}^c$ and $\paperindex\in \diagsetp_{\revindex}^c$, we say the reviewer-paper pair is off the block diagonal. For each complement set, we again omit the wording of up to a permutation of rows and columns.

We denote the expected gain of any algorithm $\algbase$ presenting a potentially random sequence of paper orderings $\ordering_1^{\algbase}, \dots, \ordering_{\numrev}^{\algbase}$ as 
\begin{equation*}
\mathbb{E}[\gain^{\algbase}] = \mathbb{E}[\gain_p^{\algbase}] + \hyperparam\mathbb{E}[\gain_r^{\algbase}],
\end{equation*}
where the expectation is with respect to the randomness in the bids placed by reviewers and any randomness in the algorithm and $\hyperparam\geq 0$ is the trade-off parameter. The expected paper-side gain is given by the quantity 
\begin{equation}
\mathbb{E}[\gain_p^{\algbase}] = \mathbb{E}\Big[\sum_{\paperindex \in \paperset}\gainfunction(\numbids_{\paperindex})\Big],
\label{eq:noiseless_pgain}
\end{equation}
where $\numbids_{\paperindex}= \sum_{\revindex\in \reviewerset}\randombid_{\revindex, \paperindex}$ is random the number of bids on paper $\paperindex \in \paperset$ at the end of the bidding process and the paper-side gain function for this result is $\gainfunction(x)=\sqrt{x}$. Recall that $\randombid_{\revindex, \paperindex}$ is a Bernoulli random variable denoting the random bid of reviewer $\revindex\in \reviewerset$ on paper $\paperindex \in \paperset$. From the assumptions of Theorem~\ref{prop:diagonal}, the success probability of $\randombid_{\revindex, \paperindex}$ for any reviewer $\revindex \in \reviewerset$ and paper $\paperindex \in \paperset$ is given by the bidding function 
\begin{equation}
\bidfunction(\ordering_{\revindex}^{\algbase}(\paperindex), \similarity_{\revindex, \paperindex}) = \mathds{1}\{\ordering_{\revindex}^{\algbase}(\paperindex)=1\}\mathds{1}\{\similarity_{\revindex, \paperindex} > \simscalar/2\}.
\label{eq:noiseless_bidmodel}
\end{equation}
In the remainder of the proof, if $\ordering_{\revindex}^{\algbase}(\paperindex)=1$, we often say the paper is shown in the highest or top position of the paper ordering.
The expected reviewer-side gain is given by
\begin{equation}
\mathbb{E}[\gain_r^{\algbase}] = \mathbb{E}\Big[\sum_{\revindex\in \reviewerset}\sum_{\paperindex\in \paperset}\gainfunctionrev(\ordering_{\revindex}^{\algbase}(\paperindex), \similarity_{\revindex, \paperindex})\Big],
\label{eq:noiseless_rgain}
\end{equation} 
where for this result the reviewer-side gain function is
\begin{equation}
\gainfunctionrev(\ordering^{\algbase}(\paperindex), \similarity_{\revindex, \paperindex}) = \frac{2^{\similarity_{\revindex, \paperindex}}-1}{\log_2(\ordering_{\revindex}^{\algbase}(\paperindex)+1)} = (2^{\similarity_{\revindex, \paperindex}}-1)\gainfunctionrev^{\ordering}(\ordering_{\revindex}^{\algbase}(\paperindex)).
\label{eq:noiseless_rgainfunc}
\end{equation}
We let
\begin{equation}
\gainfunctionrev^{\ordering}(\ordering_{\revindex}^{\algbase}(\paperindex))=\frac{1}{\log_2(\ordering_{\revindex}^{\algbase}(\paperindex)+1)}
\label{eq:noiseless_revgain_ordering}
\end{equation}
denote the component of the reviewer-side gain function that only depends on the position the paper is in. 

\subsubsection{Preliminaries}
\label{sec:noiseless_prelim}
The focus of this section is to simplify the expressions for the expected paper-side gain and expected reviewer-side gain from~\eqref{eq:noiseless_pgain} and~\eqref{eq:noiseless_rgain} respectively, using the similarity matrix structure and the given class of gain and bidding functions. There are several immediate characteristics of the reviewer bidding behavior and the relation between the similarity scores as a result of the given bidding function from~\eqref{eq:noiseless_bidmodel}
 and the noiseless community model similarity score structure from~\eqref{eq:noiseless_model} that we reference throughout the proof:
 \begin{itemize}
\item  If the reviewer-paper pair $(\revindex, \paperindex)$ is on the block diagonal of the noiseless community model similarity matrix $\similarity$ so that $\revindex\in \diagsetp_{\paperindex}$ and $\paperindex\in \diagset_{\revindex}$, then paper $\paperindex \in \paperset$ is bid on almost surely by reviewer $\revindex\in \reviewerset$ when $\ordering_{\revindex}^{\algbase}(\paperindex)=1$ and almost never when $\ordering_{\revindex}^{\algbase}(\paperindex)\neq1$.  
\item  If the reviewer-paper pair $(\revindex, \paperindex)$ is not on the block diagonal of the noiseless community model similarity matrix $\similarity$ so that $\revindex\in \diagsetp_{\paperindex}^c$ and $\paperindex\in \diagset_{\revindex}^c$, then paper $\paperindex \in \paperset$ is bid on almost never by reviewer $\revindex\in \reviewerset$ independent of $\ordering_{\revindex}^{\algbase}(\paperindex)$.
\item  If the reviewer-paper pair $(\revindex, \paperindex)$ is on the block diagonal of the noiseless community model matrix $\similarity$  so that $\revindex\in \diagsetp_{\paperindex}$ and $\paperindex\in \diagset_{\revindex}$, and the reviewer-paper pair $(\revindex, \paperindex')$ is not on the block diagonal of the noiseless community model matrix $\similarity$ so that $\revindex\in \diagsetp_{\paperindex'}^c$ and $\paperindex'\in \diagset_{\revindex}^c$, then $\similarity_{\revindex, \paperindex}>\similarity_{\revindex, \paperindex'}$.
 \end{itemize}
 Observe that the statements above further imply that each reviewer bids on at most one paper almost surely.

We now show that the expected paper-side gain from any paper $\paperindex \in \paperset$ only depends on the positions it is shown to reviewers $\revindex \in \reviewerset$ by some algorithm $\algbase$ for which the reviewer-paper pair $(\revindex, \paperindex)$ is on the block diagonal of the similarity matrix. Indeed, the expected paper-side gain for any paper $\paperindex \in \paperset$ simplifies to be
\begin{equation}
\mathbb{E}[\gainfunction(\numbids_{\paperindex})] 
=  \mathbb{E}\Big[\gainfunction\Big(\sum_{\revindex\in \reviewerset}\randombid_{\revindex, \paperindex}\Big)\Big] = \mathbb{E}\Big[\gainfunction\Big(\sum_{\revindex\in \diagsetp_{\paperindex}}\randombid_{\revindex, \paperindex}\Big)\Big] =  \sum_{\ell=0}^{\blocksize}\mathbb{P}\Big(\ell=\sum_{\revindex\in \diagsetp_{\paperindex}}\mathds{1}\{\ordering_{\revindex}^{\algbase}(\paperindex)=1\}\Big)\gainfunction(\ell).
\label{eq:noiseless_decoupled}
\end{equation}
The preceding equation follows from the fact that the the bid from any reviewer $\revindex \in \diagsetp_{\paperindex}^c$ on paper $\paperindex \in \paperset$ is zero almost surely independent of the position the paper is shown, and since any reviewer $\revindex \in \diagsetp_{\paperindex}$ bids on the paper $\paperindex \in \paperset$ almost surely if $\ordering_{\revindex}^{\algbase}(\paperindex)=1$ and almost never if $\ordering_{\revindex}^{\algbase}(\paperindex)\neq 1$.

To obtain a final simplified version of the expected paper-side gain given in~\eqref{eq:noiseless_pgain}, we sum~\eqref{eq:noiseless_decoupled} over the paper set $\paperset$ and get that
\begin{equation}
\mathbb{E}[\gain_p^{\algbase}] = \sum_{\paperindex\in \paperset}\mathbb{E}[\gainfunction(\numbids_{\paperindex})] 
=  \sum_{\paperindex\in \paperset}\sum_{\ell=0}^{\blocksize}\mathbb{P}\Big(\ell=\sum_{\revindex\in \diagsetp_{\paperindex}}\mathds{1}\{\ordering_{\revindex}^{\algbase}(\paperindex)=1\}\Big)\gainfunction(\ell).
\label{eq:noiseless_full_decoupled}
\end{equation}
It is now clear from~\eqref{eq:noiseless_full_decoupled} that to analyze the expected paper-side gain of any algorithm \algbase, we only need to determine the distribution on the number of times each paper is shown in the highest position to reviewers for which the reviewer-paper pair is on the block diagonal of the similarity matrix. 

We now turn to deriving a simplified form of the expected reviewer-side gain given in~\eqref{eq:noiseless_rgain}. Beginning from~\eqref{eq:noiseless_rgain}, we substitute in the form of the reviewer-side gain function from~\eqref{eq:noiseless_rgainfunc} and then plug in the similarity scores of the noiseless community model matrix to obtain
\begin{align}
\mathbb{E}[\gain_r^{\algbase}] = \mathbb{E}\Big[\sum_{\revindex\in \reviewerset}\sum_{\paperindex\in \paperset}(2^{\similarity_{\revindex, \paperindex}}-1)\gainfunctionrev^{\ordering}(\ordering_{\revindex}^{\algbase}(\paperindex))\Big] = \mathbb{E}\Big[(2^{\simscalar}-1)\sum_{\revindex\in \reviewerset}\sum_{\paperindex\in \diagset_{\revindex}}\gainfunctionrev^{\ordering}(\ordering_{\revindex}^{\algbase}(\paperindex))\Big].
\label{eq:noiseless_revgain}
\end{align} 
To be clear, the final equality above follows from the facts that $\diagset_{\revindex}\cup \diagset_{\revindex}^c = \paperset$ for each reviewer $\revindex \in \reviewerset$ and $\similarity_{\revindex, \paperindex}=\simscalar$ for $\paperindex \in \diagset_{\revindex}$ and $\similarity_{\revindex, \paperindex'}=0$ for $\paperindex' \in \diagset_{\revindex}^c$. It is now evident that the expected reviewer-side gain only depends on the positions the papers on the block diagonal for each individual reviewer are presented.

\subsubsection{Optimal Policy}
\label{sec:noiseless_optimal}
In this section, we characterize the optimal policy for the noiseless community model and the given class of gain and bidding functions. To do so, we independently explain how the expected paper-side and reviewer-side gain are maximized. Then, we show that they can be simultaneously maximized to obtain the optimal policy.

\paragraph*{Policy to maximize the expected paper-side gain.}
The expected paper-side gain is maximized by any policy that shows a paper among the set with the minimum number of bids within $\diagset_{\revindex}$ in the highest position to each reviewer $\revindex \in \reviewerset$. We now characterize the maximum expected paper-side gain that can be obtained and then show that the aforementioned policy achieves it.

From the characteristics of the reviewer bidding behavior given in Section~\ref{sec:noiseless_prelim}, each reviewer bids on at most one paper almost surely.
This means that the maximum number of bids that can be obtained by any policy is equal to the number of reviewers $\numrev=\numblocks\blocksize$ almost surely. The expected paper-side gain from~\eqref{eq:noiseless_pgain} for the given paper-side gain function is the sum of the expected value of a strictly concave function of the number of bids on a paper over each the $\numpapers=\numblocks\blocksize$ papers. Consequently, since the maximum number of bids that be obtained by any algorithm is equal to the number of papers almost surely, the expected paper-side gain is maximized if the bids are evenly distributed among the papers so that each paper has exactly one bid almost surely. 
It then immediately follows that the maximum expected paper-side gain that can be obtained from any algorithm \algbase is
\begin{equation}
\mathbb{E}[\gain_p^{\algbase}] = \sum_{\paperindex\in \paperset}\mathbb{E}[\gainfunction(\numbids_{\paperindex})] 
=  \sum_{\paperindex\in \paperset}\gainfunction(1)=\numblocks\blocksize.
\label{eq:noiseless_max_pgain2}
\end{equation}

We now show that any policy presenting a paper among the set with a minimum number of bids within $\diagset_{\revindex}$ in the highest position to each reviewer $\revindex \in \reviewerset$ maximizes the expected paper-side gain. 
For any given reviewer $\revindex \in \reviewerset$, the $\blocksize$ papers in $\diagset_{\revindex}$ are each in $\diagset_{\revindex'}$ for $\blocksize-1$ other reviewers $\revindex' \in \reviewerset$ and also in $\diagset_{\revindex''}^c$ for each of the remaining reviewers $\revindex''\in \reviewerset$.
If a paper from $\diagset_{\revindex}$ is shown in the highest position to reviewer $\revindex \in \reviewerset$, then it is bid on almost surely.
Moreover, any paper that is not shown in the highest position to the reviewer is bid on with probability zero.
Together, this means that upon the arrival of each reviewer $\revindex\in \reviewerset$, there is a paper in $\diagset_{\revindex}$ with zero bids that has not been shown in the highest position to any reviewer previously almost surely. 
Consequently, each paper $\paperindex \in \paperset$ is shown exactly once almost surely in the highest position to some reviewer $\revindex \in \diagsetp_{\paperindex}$. 
It then follows from the decomposition in~\eqref{eq:noiseless_decoupled} that the expected paper-side gain of this policy \algbase is
\begin{equation}
\mathbb{E}[\gain_p^{\algbase}] = \sum_{\paperindex\in \paperset}\mathbb{E}[\gainfunction(\numbids_{\paperindex})] 
=  \sum_{\paperindex\in \paperset}\gainfunction(1)=\numblocks\blocksize.
\label{eq:noiseless_max_pgain}
\end{equation}
We conclude that the policy maximizes the expected paper-side gain since it was shown in~\eqref{eq:noiseless_max_pgain2} that the maximum expected paper-side gain that can be obtained is $\numblocks\blocksize$.

\paragraph*{Policy to maximize the expected reviewer-side gain.} 
The expected reviewer-side gain as given in~\eqref{eq:noiseless_revgain} is decoupled between each of the reviewers. Moreover, the expected reviewer-side gain from any reviewer $\revindex \in \reviewerset$ only depends on the positions that papers in the set $\diagset_{\revindex}$ are shown. Since the function $\gainfunctionrev^{\ordering}$ as given in~\eqref{eq:noiseless_revgain_ordering} is decreasing on the domain $\mathbb{R}_{>0}$, as long as each paper in the set $\diagset_{\revindex}$ is shown before the papers in the set $\diagset_{\revindex}^c$, then the expected reviewer-side gain from any reviewer $\revindex \in \reviewerset$ is maximized. This means that if a policy shows papers this way for each reviewer $\revindex \in \reviewerset$, then the expected reviewer-side gain is maximized.

\paragraph*{Overall optimal policy.}
The expected paper-side and reviewer-side gains can be simultaneously maximized. Indeed, if a paper among the set with the minimum number of bids from $\diagset_{\revindex}$ is shown in the highest position to each reviewer $\revindex \in \reviewerset$, then the expected paper-side gain is maximized. Furthermore, if the remaining papers in $\diagset_{\revindex}$ are shown ahead of each paper in $\diagset_{\revindex}^c$ for each reviewer $\revindex\in \reviewerset$, then the expected reviewer-side gain is maximized. It then follows that this is the optimal policy. We refer to such a policy as \optbase in the remainder of the proof.

\subsubsection{Optimality of \super with Zero Heuristic}
\label{sec:noiseless_super}
We show in this section that \super with zero heuristic is equivalent to the optimal policy under the noiseless community model for the given class of gain and bidding functions. 

\paragraph*{Informal description of \super with zero heuristic policy.}
Recall that as explained in Section~\ref{sec:algorithm} and formally characterized in Section~\ref{sec:theoretical_properties}, \super with zero heuristic is designed to maximize the immediate expected gain from each reviewer conditioned on the history. We show that the immediate expected paper-side and reviewer-side gain from any reviewer $\revindex \in \reviewerset$ are both maximized by showing a paper with the minimum number of bids among $\diagset_{\revindex}$ in the highest position, followed by the remaining papers in $\diagset_{\revindex}$ in any order, and then the papers from $\diagset_{\revindex}^c$ in any order.

The immediate expected paper-side gain from any reviewer $\revindex\in \reviewerset$ is maximized by showing a paper with the minimum number of bids among $\diagset_{\revindex}$ in the highest position in the paper ordering. To see why, observe that the immediate expected paper-side gain from any paper that is not shown in the highest position is zero since the probability of it being bid on is zero. Moreover, the probability of a paper being bid on that is shown in the highest position is only non-zero if it is in the set of papers $\diagset_{\revindex}$. Then, since the given paper-side gain function is strictly concave so the returns of bids are diminishing, we determine that the immediate expected paper-side gain from the paper shown in the highest position of the ordering is maximized if it is a paper with the minimum number of bids among $\diagset_{\revindex}$. 

The expected reviewer-side gain from any reviewer $\revindex \in \reviewerset$ is maximized as long as papers in $\diagset_{\revindex}$ are shown ahead of $\diagset_{\revindex}^c$. This follows from the fact that the expected reviewer-side gain as given in~\eqref{eq:noiseless_revgain} is decoupled between the reviewers. Furthermore, the expected reviewer-side gain from any reviewer $\revindex \in \reviewerset$ only depends on the positions that papers in the set $\diagset_{\revindex}$ are shown. Since the function $\gainfunctionrev^{\ordering}$ as given in~\eqref{eq:noiseless_revgain_ordering} is decreasing on the domain $\mathbb{R}_{>0}$, as long as each paper in the set $\diagset_{\revindex}$ is shown before the papers in the set $\diagset_{\revindex}^c$, then the expected reviewer-side gain from the reviewer is maximized.

\paragraph*{Formal description of \super with zero heuristic policy.}
We now formally state the policy of \super with zero heuristic policy for the noiseless community model and the given gain and bidding functions. The proof of Lemma~\ref{lemma:noiseless_diag_super_policy} is given in Section~\ref{sec:noiseless_lemmas}.
\begin{lemma}
Under the assumptions of Theorem~\ref{prop:diagonal}, \super with zero heuristic shows a paper among the set with the minimum number of bids from $\diagset_{\revindex}$ in the highest position to each reviewer $\revindex \in \reviewerset$. Moreover, the remaining papers in $\diagset_{\revindex}$ are shown in an arbitrary order ahead of the papers in $\diagset_{\revindex}^c$ which are also shown in arbitrary order to each reviewer $\revindex \in \reviewerset$. 
\label{lemma:noiseless_diag_super_policy}
\end{lemma}
The policy of \super with zero heuristic given in Lemma~\ref{lemma:noiseless_diag_super_policy} is equivalent to the optimal policy derived in Section~\ref{sec:noiseless_optimal}. We conclude \super with zero heuristic is optimal for the noiseless community model with the given class of gain and bidding functions.

\subsubsection{Optimality of \simbase}
\label{sec:noiseless_sim}
The \simbase policy shows papers to each reviewer in decreasing order of the similarity scores with ties between a pair of papers broken in favor of the paper with fewer bids and any remaining ties are broken uniformly at random. The similarity score of each paper $\paperindex \in \diagset_{\revindex}$ is greater than the similarity score of each paper $\paperindex' \in \diagset_{\revindex}^c$ for each reviewer. By definition of the policy, the previous fact immediately implies that for each reviewer $\revindex \in \reviewerset$, \simbase shows each paper in $\diagset_{\revindex}$ ahead of each paper in $\diagset_{\revindex}^c$. Moreover, the tie-breaking mechanism of \simbase guarantees that a paper with the minimum number of bids among $\diagset_{\revindex}$ is shown in the highest position of the paper ordering to each reviewer $\revindex \in \reviewerset$. This policy is equivalent to the optimal policy given in Section~\ref{sec:noiseless_optimal}, and hence \simbase is optimal for the noiseless community model with the given class of gain and bidding functions.

\subsubsection{Suboptimality of \bidbase}
\label{sec:noiseless_bid}
We now prove the suboptimality of \bidbase for the noiseless community model. 

\paragraph*{Intuition and \bidbase policy.}
The \bidbase algorithm presents papers in an increasing order of the number of bids and ties between papers are broken in favor of the paper with the higher similarity score. In this section, we go on to show that this policy maximizes the expected paper-side gain. This follows from the fact that almost surely a paper with zero bids and a similarity score exceeding the threshold necessary for a reviewer to bid on a paper is shown in the highest position to each reviewer and bid on. However, for the noiseless community model similarity class, the algorithm is suboptimal for the combined objective since the expected reviewer-side gain obtained is suboptimal. The fundamental problem with \bidbase is that, except for as a tie-breaking mechanism, the similarity scores are ignored by the algorithm. For the given bidding model, papers which are not shown in the highest position are bid on with probability zero. Consequently, showing papers with fewer bids closer, but not in the highest position, cannot improve the expected paper-side gain and reduces the expected reviewer-side gain.

\paragraph*{Bounding the expected paper-side gain.}
Recall from Section~\ref{sec:noiseless_optimal} that any policy presenting a paper among the set with a minimum number of bids within $\diagset_{\revindex}$ in the highest position to each reviewer $\revindex \in \reviewerset$ maximizes the expected paper-side gain. We now follow similar arguments from Section~\ref{sec:noiseless_optimal} to determine that \bidbase shows a paper among the set with a minimum number of bids among $\diagset_{\revindex}$ in the highest position to each reviewer $\revindex \in \reviewerset$ almost surely so that is maximizes the expected paper-side gain.

For any given reviewer $\revindex \in \reviewerset$, the $\blocksize$ papers in $\diagset_{\revindex}$ are in $\diagset_{\revindex'}$ for the same $\blocksize-1$ other reviewers $\revindex' \in \reviewerset$ and also in $\diagset_{\revindex''}^c$ for each of the remaining reviewers $\revindex''\in \reviewerset$.
For any reviewer $\revindex\in \reviewerset$, any paper $\paperindex \in \diagset_{\revindex}^c$ is bid on almost never and any paper $\paperindex \in \diagset_{\revindex}$ is only bid on with non-zero probability if shown in the highest position to the reviewer.
This means that upon the arrival of each reviewer $\revindex\in \reviewerset$, there is a paper in $\diagset_{\revindex}$ with zero bids almost surely. 
Furthermore, the similarity score of any paper in $\diagset_{\revindex}$ is greater than the similarity score of any paper in $\diagset_{\revindex}^c$ for each reviewer $\revindex  \in \reviewerset$. Hence, \bidbase shows a paper in $\diagset_{\revindex}$ with zero bids in the highest position to each reviewer $\revindex \in \reviewerset$ almost surely since papers are shown in increasing order of the number of bids and ties are broken in favor of the paper with the higher similarity score. The structure of the bidding function guarantees that if a paper in $\diagset_{\revindex}$ is shown in the highest position of the paper ordering to reviewer $\revindex \in \reviewerset$, then it is bid on by the reviewer almost surely.
Consequently, each paper $\paperindex \in \paperset$ is shown exactly once almost surely in the highest position to some reviewer $\revindex \in \diagsetp_{\paperindex}$.  It then follows from the decomposition in~\eqref{eq:noiseless_decoupled} that the expected paper-side gain of \bidbase for every $\numblocks\geq2, \blocksize\geq 2$, and $\hyperparam\geq 0$, is given by
\begin{equation*}
\mathbb{E}[\gain_p^{\bidbase}] =
\sum_{\paperindex\in \paperset}\mathbb{E}[\gainfunction(\numbids_{\paperindex})] 
 = \sum_{\paperindex\in \paperset}\gainfunction(1) = \numblocks\blocksize.
\end{equation*}
From the expected paper-side gain of \optbase given in~\eqref{eq:noiseless_max_pgain}, we conclude that for every $\numblocks\geq2, \blocksize\geq2,$ and $\hyperparam\geq 0$,
\begin{equation}
\mathbb{E}[\gain_p^{\optbase}] - \mathbb{E}[\gain_p^{\bidbase}] = 0.
\label{eq:ndiag_bid_pgain}
\end{equation}

\paragraph*{Bounding the expected reviewer-side gain.}
We now show that the optimal policy \optbase obtains significantly more expected reviewer-side gain than \bidbase. This requires deriving a suitable lower bound on the following expression based on~\eqref{eq:noiseless_revgain}:
\begin{equation}
\hyperparam\mathbb{E}[\gain_r^{\optbase} - \gain_r^{\bidbase}] = \hyperparam\mathbb{E}\Big[\sum_{\revindex\in \reviewerset}(2^{\simscalar}-1)\sum_{\paperindex\in \diagset_{\revindex}}(\gainfunctionrev^{\ordering}(\ordering_{\revindex}^{\optbase}(\paperindex))-\gainfunctionrev(\ordering_{\revindex}^{\bidbase}(\paperindex)))\Big].
\label{eq:ndiag_rgain_bid}
\end{equation}
Let us begin by defining a ``good event'' for any reviewer and paper under which 
if the paper has probability zero of being bid on then it is not bid on and if the paper has probability one of being bid on then it is bid on. Formally, for any reviewer $k\in \reviewerset$, paper $\paperindex \in \paperset$, and paper ordering $\ordering_{k}^{\algbase}$ given by an algorithm \algbase, we define 
\begin{align*}
\mathcal{E}_{k, \paperindex}^{\algbase}=&\{\ordering_{k}^{\algbase}(\paperindex)=1, \similarity_{k, \paperindex}>\simscalar/2, \randombid_{k, \paperindex}=1\}\cup\{\ordering_{k}^{\algbase}(\paperindex)\neq 1,  \randombid_{k, \paperindex}=0\} \cup\{ \similarity_{k, \paperindex}<\simscalar/2, \randombid_{k, \paperindex}=0\}. 
\end{align*}
Moreover, for each reviewer $\revindex \in \reviewerset$, define the following event $\mathcal{E}_{\revindex}=\cup_{k=1}^{i-1}\cup_{\paperindex=1}^{\numpapers}\{\mathcal{E}_{k, \paperindex}^{\optbase}\cup \mathcal{E}_{k, \paperindex}^{\bidbase}\}$
which says the good event held for each reviewer that arrived previously for every paper and observe that the complement of this event occurs on a measure zero space by the structure of the bidding function given in~\eqref{eq:noiseless_bidmodel}.
Consequently, from the law of total expectation, an equivalent form of~\eqref{eq:ndiag_rgain_bid} is given by
\begin{equation}
\hyperparam\mathbb{E}[\gain_r^{\optbase} - \gain_r^{\bidbase}] = \hyperparam\mathbb{E}\Big[\sum_{\revindex\in \reviewerset}(2^{\simscalar}-1)\mathbb{E}\Big[\sum_{\paperindex\in \diagset_{\revindex}}(\gainfunctionrev^{\ordering}(\ordering_{\revindex}^{\optbase}(\paperindex))-\gainfunctionrev(\ordering_{\revindex}^{\bidbase}(\paperindex)))\Big|\mathcal{E}_{\revindex}\Big]\Big].
\label{eq:ndiag_rgain_bid2}
\end{equation}

Recall from the derivation of the expected paper-side gain of \optbase in
Section~\ref{sec:noiseless_optimal} and \bidbase in this section that each algorithm obtains exactly one bid almost surely from each reviewer and on each paper. Define $\mathcal{F}$ as the set of initial $\lfloor \numblocks\blocksize/4\rfloor$ reviewers 
for which upon arrival of such a reviewer $\revindex\in \mathcal{F}$ at least one paper on the block diagonal for the reviewer given by $\diagset_{\revindex}$ has received a bid previously.
Observe that \optbase obtains at least as much expected reviewer-side gain as \bidbase from each reviewer since it was shown in Section~\ref{sec:noiseless_optimal} that the policy maximizes the expected reviewer-side gain from each individual reviewer. As a result, we get the following lower bound on~\eqref{eq:ndiag_rgain_bid2}:
\begin{equation}
\hyperparam\mathbb{E}[\gain_r^{\optbase} - \gain_r^{\bidbase}] \geq \hyperparam\mathbb{E}\Big[\sum_{\revindex\in \mathcal{F}}(2^{\simscalar}-1)\mathbb{E}\Big[\sum_{\paperindex\in \diagset_{\revindex}}(\gainfunctionrev^{\ordering}(\ordering_{\revindex}^{\optbase}(\paperindex))-\gainfunctionrev(\ordering_{\revindex}^{\bidbase}(\paperindex)))\Big|\mathcal{E}_{\revindex}\Big]\Big].
\label{eq:ndiag_rgain_bid3}
\end{equation}

We now separate papers into relevant groups defined upon arrival for each reviewer $\revindex\in \mathcal{F}$ given the event $\mathcal{E}_{\revindex}$.
Let $T_{\revindex, 1}$ be the set of papers in $\diagset_{\revindex}$ with zero bids and $T_{\revindex, 2}$ be the set of papers in $\diagset_{\revindex}$ with one bid. Denote by $T_{\revindex, 3}$ the set of papers in $\diagset_{\revindex}^c$ with zero bids and $T_{\revindex, 4}$ as the papers in $\diagset_{\revindex}^c$ with one bid. Moreover, we let $N_{\revindex, k} = |T_{\revindex, k}|$ for $k \in \{1, 2, 3, 4\}$ denote the number of papers in each set and define $\ell_{\revindex}=N_{\revindex, 1}+1$.
Using this notation,~\eqref{eq:ndiag_rgain_bid3} is equivalently
\begin{equation}
\hyperparam\mathbb{E}[\gain_r^{\optbase} - \gain_r^{\bidbase}] \geq \hyperparam\mathbb{E}\Big[\sum_{\revindex\in \mathcal{F}}(2^{\simscalar}-1)\mathbb{E}\Big[\sum_{\paperindex\in T_{\revindex, 1}\cup T_{\revindex, 2}}(\gainfunctionrev^{\ordering}(\ordering_{\revindex}^{\optbase}(\paperindex))-\gainfunctionrev(\ordering_{\revindex}^{\bidbase}(\paperindex)))\Big|\mathcal{E}_{\revindex}\Big]\Big].
\label{eq:ndiag_rgain_bid4}
\end{equation}

As shown in Section~\ref{sec:noiseless_optimal}, \optbase shows a paper with the minimum number of bids among $\diagset_{\revindex}$ in the highest position of the paper ordering to each reviewer $\revindex \in \mathcal{F}$. Again, this paper corresponds to a paper in the set $T_{\revindex, 1}$ with zero bids. After this paper, the remaining papers in $\diagset_{\revindex}$ are shown in any arbitrary order. This group of papers contains papers among $T_{\revindex, 1}\cup T_{\revindex, 2}$. Since it has no impact on the expected gain in the analysis that follows, without loss of generality, consider that \optbase shows the papers in $T_{\revindex, 1}$ ahead of the papers in $T_{\revindex, 2}$. 

The \bidbase policy shows a paper with the minimum number of bids among $\diagset_{\revindex}$ in the highest position of the paper ordering to each reviewer $\revindex \in \mathcal{F}$ almost surely as proved earlier. See that such a paper corresponds to a paper in the set $T_{\revindex, 1}$ with zero bids. After this paper, the remaining papers with zero bids, which by definition belong to $T_{\revindex, 1}\cup T_{\revindex, 3}$, are shown with ties broken in favor of the paper with the higher similarity score. Since each paper in $\diagset_{\revindex}$ has a higher similarity score than each paper in $\diagset_{\revindex}^c$, we conclude that \bidbase shows the remaining papers in $T_{\revindex, 1}$ after the paper shown in the highest position. 

Consequently, the papers in $T_{\revindex, 1}$ are shown among the positions $\{1, \dots, N_{\revindex, 1}\}$ by both \optbase and \bidbase conditioned on the event $\mathcal{E}_{\revindex}$.
This allows us to simplify~\eqref{eq:ndiag_rgain_bid4} and get that
\begin{equation}
\hyperparam\mathbb{E}[\gain_r^{\optbase} - \gain_r^{\bidbase}] \geq \hyperparam\mathbb{E}\Big[\sum_{\revindex\in \mathcal{F}}(2^{\simscalar}-1)\mathbb{E}\Big[\sum_{\paperindex\in T_{\revindex, 2}}(\gainfunctionrev^{\ordering}(\ordering_{\revindex}^{\optbase}(\paperindex))-\gainfunctionrev(\ordering_{\revindex}^{\bidbase}(\paperindex)))\Big|\mathcal{E}_{\revindex}\Big]\Big].
\label{eq:ndiag_rgain_bid5}
\end{equation}

As we just showed, conditioned on the event $\mathcal{E}_{\revindex}$, \optbase shows the papers in $T_{\revindex, 2}$ in an arbitrary order immediately after the papers in $T_{\revindex, 1}$ to each reviewer $\revindex \in \mathcal{F}$.
This means that the papers in $T_{\revindex, 2}$ are shown among the position set $\{\ell_{\revindex}, \dots, \ell_i+N_{\revindex, 2}-1\}$ by \optbase to each reviewer $\revindex \in \mathcal{F}$ conditioned on $\mathcal{E}_{\revindex}$.

In contrast, \bidbase shows the papers in $T_{\revindex, 2}$ after the papers in $T_{\revindex, 1}\cup T_{\revindex, 3}$, but before the papers in $T_{\revindex, 4}$. Indeed, the papers in $T_{\revindex, 3}$ each have zero bids and the papers in $T_{\revindex, 2}$ each have one bid, so by definition of the policy, \bidbase shows the papers in $T_{\revindex, 3}$ ahead of the papers in $T_{\revindex, 2}$. Furthermore, by definition of the sets, the similarity score of each paper in $T_{\revindex, 2}$ is greater than the similarity score of each paper in $T_{\revindex, 4}$, which combined with the tie-breaking mechanism of \bidbase ensures that papers in $T_{\revindex, 2}$ are shown ahead of the papers in $T_{\revindex, 4}$ even though the number of bids are equal.
This means that the papers in $T_{\revindex, 2}$ are shown among the position set
 $\{\ell_{\revindex}+N_{\revindex, 3}, \dots, \ell_i+N_{\revindex, 2}+N_{\revindex, 3}-1\}$ by \bidbase to each reviewer $\revindex \in \mathcal{F}$ conditioned on $\mathcal{E}_{\revindex}$.
 
 From this set of facts and continuing from~\eqref{eq:ndiag_rgain_bid5}, we obtain
\begin{equation}
\begin{split}
\hyperparam\mathbb{E}[\gain_r^{\optbase} - \gain_r^{\bidbase}] &\geq \hyperparam\mathbb{E}\Big[(2^{\simscalar}-1)\sum_{\revindex\in \mathcal{F}}\mathbb{E}\Big[\sum_{\paperindex=\ell_{\revindex}}^{\ell_{\revindex}+N_{\revindex, 2}-1}(\gainfunctionrev^{\ordering}(\paperindex)-\gainfunctionrev^{\ordering}(\paperindex+N_{\revindex, 3}))\Big|\mathcal{E}_{\revindex}\Big]\Big].
\end{split}
\label{eq:ndiag_rgain_bid12}
\end{equation}
Minimizing over $\revindex\in \mathcal{F}$ in~\eqref{eq:ndiag_rgain_bid12} and using the definition $|\mathcal{F}| = \lfloor\numblocks\blocksize/4\rfloor$, we get the bound 
\begin{equation}
\hyperparam\mathbb{E}[\gain_r^{\optbase} - \gain_r^{\bidbase}] \geq \hyperparam\mathbb{E}\Big[(2^{\simscalar}-1)(\lfloor \numblocks\blocksize/4\rfloor)\min_{\revindex\in \mathcal{F}}\mathbb{E}\Big[\sum_{\paperindex=\ell_{\revindex}}^{\ell_{\revindex}+N_{\revindex, 2}-1}(\gainfunctionrev^{\ordering}(\paperindex)-\gainfunctionrev^{\ordering}(\paperindex+N_{\revindex, 3}))\Big|\mathcal{E}_{\revindex}\Big]\Big].
\label{eq:ndiag_rgain_bid21}
\end{equation}
Moreover, for every $\numblocks\geq2$ and $\blocksize\geq2$, it holds that
\begin{equation}
\lfloor \numblocks\blocksize/4 \rfloor=  \numblocks\blocksize/4 - (\numblocks\blocksize \ \mathrm{mod} \ 4)/4 \geq \numblocks\blocksize/8,
\label{eq:ndiag_rgain_bid27}
\end{equation}
and by definition of the noiseless community model
\begin{equation}
2^{\simscalar}-1    \geq 2^{0.01}-1 \geq 1/150.
\label{eq:ndiag_rgain_bid28}
\end{equation}
Combining~\eqref{eq:ndiag_rgain_bid21},~\eqref{eq:ndiag_rgain_bid27}, and~\eqref{eq:ndiag_rgain_bid28}, we have
\begin{equation}
\hyperparam\mathbb{E}[\gain_r^{\optbase} - \gain_r^{\bidbase}] \geq \Big(\frac{\hyperparam}{1200}\Big)\mathbb{E}\Big[\min_{\revindex\in \mathcal{F}}\mathbb{E}\Big[\sum_{\paperindex=\ell_{\revindex}}^{\ell_{\revindex}+N_{\revindex, 2}-1}(\gainfunctionrev^{\ordering}(\paperindex)-\gainfunctionrev^{\ordering}(\paperindex+N_{\revindex, 3}))\Big|\mathcal{E}_{\revindex}\Big]\Big].
\label{eq:ndiag_rgain_bidnew3}
\end{equation}

Toward the goal of bounding the right-hand side of~\eqref{eq:ndiag_rgain_bidnew3}, we now work on verifying the following claim.

\textbf{Claim 1.} For each reviewer $\revindex \in \mathcal{F}$ and conditioned on the event $\mathcal{E}_{\revindex}$,
\begin{equation}
N_{\revindex, 3}\geq \numblocks\blocksize-\blocksize-\numblocks-\lfloor\numblocks\blocksize/4\rfloor+N_{\revindex, 2}+1 \geq N_{\revindex, 2}\geq 1.
\label{eq:ndiag_rgain_bid19}
\end{equation}
Recall that $N_{\revindex, 3}$ denotes the number of papers in $\diagset_{\revindex}^c$ with zero bids upon the arrival of reviewer $\revindex \in \mathcal{F}$. By definition, the number of papers in $\diagset_{\revindex}^c$ is $\numblocks\blocksize -\blocksize$. To bound $N_{\revindex, 3}$, we need to bound the maximum number of papers $\diagset_{\revindex}^c$ that could have been bid on previously upon the arrival of the reviewer. Observe that upon the arrival of the reviewer, there could be at most $(\lfloor \numblocks\blocksize/4\rfloor-1) - (N_{\revindex, 2}-1)$ reviewers from $\mathcal{F}$ that previously arrived and bid on a paper in $\diagset_{\revindex}^c$. This follows from the fact that $|\mathcal{F}| = \lfloor \numblocks\blocksize/4\rfloor$ and each reviewer bids on at most one paper almost surely from the structure of the bidding function given in~\eqref{eq:noiseless_bidmodel},
so the total number of bids from this set of reviewers previously is at most $(\lfloor \numblocks\blocksize/4\rfloor-1)$. Furthermore, of the $(\lfloor \numblocks\blocksize/4\rfloor-1)$ bids from the reviewer set $\mathcal{F}$, the number of bids on papers which are in $\diagset_{\revindex}$ instead of $\diagset_{\revindex}^c$ is given by $(N_{\revindex, 2}-1)$ since prior to the arrival of the reviewer a paper in $\diagset_{\revindex}$ had to be bid on by definition of the reviewer set $\mathcal{F}$.
Finally, at most $(\numblocks-1)$ papers in $\diagset_{\revindex}^c$ are bid on before the arrival of the reviewer from previous reviewers which do not belong to  $\mathcal{F}$ since there are $\numblocks$ blocks in the similarity matrix. 
Accordingly, the number of papers with a bid in $\diagset_{\revindex}^c$ is at most $(\numblocks-1)+(\lfloor \numblocks\blocksize/4\rfloor-1)-(N_{\revindex, 2}-1)$. We conclude that the number of papers in $\diagset_{\revindex}^c$ without a bid given by $N_{\revindex, 3}$ for any reviewer $\revindex\in \mathcal{F}$ conditioned on $\mathcal{E}_{\revindex}$ is bounded below as follows
\begin{equation}
N_{\revindex, 3} \geq \numblocks\blocksize-\blocksize-\numblocks-\lfloor\numblocks\blocksize/4\rfloor+N_{\revindex, 2}+1.
\label{eq:ndiag_rgain_bid15}
\end{equation}
We now show 
\begin{equation}
\numblocks\blocksize-\blocksize-\numblocks-\lfloor\numblocks\blocksize/4\rfloor+N_{\revindex, 2}+1 \geq N_{\revindex, 2}.
\label{eq:ndiag_rgain_bid16}
\end{equation}
To see~\eqref{eq:ndiag_rgain_bid16}, observe that
\begin{equation}
\numblocks\blocksize-\blocksize-\numblocks-\lfloor\numblocks\blocksize/4\rfloor+1 \geq \numblocks\blocksize-\blocksize-\numblocks-\numblocks\blocksize/4+1  
= 3\numblocks\blocksize/4-\blocksize-\numblocks+1.
\label{eq:ndiag_rgain_bid17}
\end{equation}
The quantity $(3\numblocks\blocksize/4-\blocksize-\numblocks+1)$ is increasing in $\numblocks$ and $\blocksize$ for $\numblocks\geq 2$, $\blocksize\geq 2$. Using this fact, we get that for every $\numblocks\geq 2$ and $\blocksize\geq 2$,
\begin{equation}
3\numblocks\blocksize/4-\blocksize-\numblocks+1 \geq  0.
\label{eq:ndiag_rgain_bid18}
\end{equation}
Combining~\eqref{eq:ndiag_rgain_bid17} and~\eqref{eq:ndiag_rgain_bid18} immediately implies that~\eqref{eq:ndiag_rgain_bid16} holds. Finally, $N_{\revindex, 2}\geq 1$ for each reviewer $\revindex \in \mathcal{F}$ conditioned on the event $\mathcal{E}_{\revindex}$ by definition of the reviewer set, which proves the final inequality of~\eqref{eq:ndiag_rgain_bid19}. 

Using the result from~\eqref{eq:ndiag_rgain_bid19}, we now prove the following claim to bound the right-hand side of~\eqref{eq:ndiag_rgain_bidnew3}.
\newline \noindent
\textbf{Claim 2.} Conditioned on the event $\mathcal{E}_{\revindex}$, for each reviewer $\revindex\in \mathcal{F}$ it must be that
\begin{equation}
\sum_{\paperindex=\ell_{\revindex}}^{\ell_{\revindex}+N_{\revindex, 2}-1}(\gainfunctionrev^{\ordering}(\paperindex)-\gainfunctionrev^{\ordering}(\paperindex+N_{\revindex, 3})) \geq \Big(\frac{2}{5}\Big)\Big(\frac{1}{\log_2^2(\numblocks\blocksize)}\Big).
\label{eq:ndiag_rgain_bidnew2}
\end{equation}
To begin, for any $\revindex \in \mathcal{F}$ conditioned on the event $\mathcal{E}_{\revindex}$ we get that
\begin{equation}
\sum_{\paperindex=\ell_{\revindex}}^{\ell_{\revindex}+N_{\revindex, 2}-1}(\gainfunctionrev^{\ordering}(\paperindex)-\gainfunctionrev^{\ordering}(\paperindex+N_{\revindex, 3})) \geq \sum_{\paperindex=\ell_{\revindex}}^{\ell_{\revindex}+N_{\revindex, 2}-1}(\gainfunctionrev^{\ordering}(\paperindex)-\gainfunctionrev^{\ordering}(\paperindex+ \numblocks\blocksize-\blocksize-\numblocks-\lfloor\numblocks\blocksize/4\rfloor+N_{\revindex, 2}+1)).
\label{eq:ndiag_rgain_bid22}
\end{equation}
The inequality in~\eqref{eq:ndiag_rgain_bid22} relies upon the facts that $\ell_{\revindex}\geq 1$ for each reviewer $\revindex \in \mathcal{F}$ by definition and the function $\gainfunctionrev^{\ordering}$ as given in~\eqref{eq:noiseless_revgain_ordering} is decreasing on the domain $\mathbb{R}_{> 0}$. As a result of each property, we can invoke the lower bound on $N_{\revindex, 3}$ from~\eqref{eq:ndiag_rgain_bid19} to get the stated bound in~\eqref{eq:ndiag_rgain_bid22}. Moreover, $\ell_{\revindex}+N_{\revindex, 2}-1=|\diagset_{\revindex}|= \blocksize$ by definition for any $\revindex\in \mathcal{F}$ given the event $\mathcal{E}_{\revindex}$, so an equivalent form of the bound in~\eqref{eq:ndiag_rgain_bid22} is 
\begin{equation}
\sum_{\paperindex=\ell_{\revindex}}^{\ell_{\revindex}+N_{\revindex, 2}-1}(\gainfunctionrev^{\ordering}(\paperindex)-\gainfunctionrev^{\ordering}(\paperindex+N_{\revindex, 3})) \geq \sum_{\paperindex=\ell_{\revindex}}^{\blocksize}(\gainfunctionrev^{\ordering}(\paperindex)-\gainfunctionrev^{\ordering}(\paperindex+ \numblocks\blocksize-\blocksize-\numblocks-\lfloor\numblocks\blocksize/4\rfloor+N_{\revindex, 2}+1)).
\label{eq:ndiag_rgain_bid23}
\end{equation}

Now, since $\ell_{\revindex}\geq 1$ for each reviewer $\revindex \in \mathcal{F}$ by definition, the function $\gainfunctionrev^{\ordering}$ as given in~\eqref{eq:noiseless_revgain_ordering} is decreasing on the domain $\mathbb{R}_{>0}$, and $\numblocks\blocksize-\blocksize-\numblocks-\lfloor\numblocks\blocksize/4\rfloor+N_{\revindex, 2}+1\geq 1$ from~\eqref{eq:ndiag_rgain_bid19}, we determine that each summand in~\eqref{eq:ndiag_rgain_bid23} is positive. Hence, to obtain a lower bound on~\eqref{eq:ndiag_rgain_bid23}, we take the maximum $\ell_{\revindex}$ over each reviewer $\revindex \in \mathcal{F}$. 
Recall that $\ell_{\revindex}-1=N_{\revindex, 1}$, which gives the total number of papers in $\diagset_{\revindex}$ without a bid by definition. For each reviewer $\revindex \in \mathcal{F}$, there must be at least one paper with a bid in $\diagset_{\revindex}$ by definition of the reviewer set given the event $\mathcal{E}_{\revindex}$. Then, using the fact that $|\diagset_{\revindex}|=\blocksize$, we get $\ell_{\revindex}-1= N_{\revindex, 1}\leq \blocksize-1$ so that $\ell_{\revindex}\leq \blocksize$ for any reviewer $\revindex \in \mathcal{F}$ given the event $\mathcal{E}_{\revindex}$. 
Hence, \eqref{eq:ndiag_rgain_bid23} is lower bounded as follows:
\begin{equation}
\sum_{\paperindex=\ell_{\revindex}}^{\ell_{\revindex}+N_{\revindex, 2}-1}(\gainfunctionrev^{\ordering}(\paperindex)-\gainfunctionrev^{\ordering}(\paperindex+N_{\revindex, 3})) \geq \gainfunctionrev^{\ordering}(\blocksize)-\gainfunctionrev^{\ordering}(\numblocks\blocksize-\numblocks-\lfloor\numblocks\blocksize/4\rfloor+N_{\revindex, 2}+1).
\label{eq:ndiag_rgain_bid232}
\end{equation}

Combining the fact that $N_{\revindex, 2}\geq 1$ for each reviewer $\revindex \in \mathcal{F}$ conditioned on the event $\mathcal{E}_{\revindex}$ by definition of the reviewer set with~\eqref{eq:ndiag_rgain_bid19}, we obtain 
\begin{equation}
\numblocks\blocksize-\numblocks-\lfloor\numblocks\blocksize/4\rfloor+N_{\revindex, 2}+1\geq \numblocks\blocksize-\numblocks-\lfloor\numblocks\blocksize/4\rfloor+2\geq \blocksize+1.
\label{eq:boundsonmq}\end{equation}
 Since  $\gainfunctionrev^{\ordering}$ as given in~\eqref{eq:noiseless_revgain_ordering} is decreasing on the domain $\mathbb{R}_{>0}$, the inequality in \eqref{eq:boundsonmq} immediately implies
\begin{equation}
\gainfunctionrev^{\ordering}(\numblocks\blocksize-\numblocks-\lfloor\numblocks\blocksize/4\rfloor+N_{\revindex, 2}+1) \leq \gainfunctionrev^{\ordering}(\numblocks\blocksize-\numblocks-\lfloor\numblocks\blocksize/4\rfloor+2).
\label{eq:ndiag_rgain_bid24}
\end{equation}
Then, combining~\eqref{eq:ndiag_rgain_bid232} and~\eqref{eq:ndiag_rgain_bid24} results in the bound
\begin{equation}
\sum_{\paperindex=\ell_{\revindex}}^{\ell_{\revindex}+N_{\revindex, 2}-1}(\gainfunctionrev^{\ordering}(\paperindex)-\gainfunctionrev^{\ordering}(\paperindex+N_{\revindex, 3})) \geq \gainfunctionrev^{\ordering}(\blocksize)-\gainfunctionrev^{\ordering}( \numblocks\blocksize-\numblocks-\lfloor\numblocks\blocksize/4\rfloor+2).
\label{eq:ndiag_rgain_bid25}
\end{equation}
To bound $(\gainfunctionrev^{\ordering}(\blocksize)-\gainfunctionrev^{\ordering}(\numblocks\blocksize-\numblocks-\lfloor\numblocks\blocksize/4\rfloor+2))$, we need the following result proved in Section~\ref{sec:bid_diag_sum_bound}.
\begin{lemma}
Fix $\numblocks\geq2$, $\blocksize\geq 2$, and let $\gainfunctionrev^{\ordering}(x)=1/\log_2(x+1)$. Then,
\setlength\belowdisplayskip{-0pt}
\begin{equation*}
\gainfunctionrev^{\ordering}(\blocksize)-\gainfunctionrev^{\ordering}(\numblocks\blocksize-\numblocks-\lfloor\numblocks\blocksize/4\rfloor+2) \geq \Big(\frac{2}{5}\Big)\Big(\frac{1}{\log_2^2(\numblocks\blocksize)}\Big).
\end{equation*}
\label{lemma:bid_diag_sum_bound}
\end{lemma}
Applying Lemma~\ref{lemma:bid_diag_sum_bound} to~\eqref{eq:ndiag_rgain_bid25}, we arrive at the lower bound claimed in~\eqref{eq:ndiag_rgain_bidnew2}.
Then, relating~\eqref{eq:ndiag_rgain_bidnew2} back to~\eqref{eq:ndiag_rgain_bid21}, 
for every $\numblocks\geq2, \blocksize\geq 2$, and $\hyperparam\geq 0$, the following bound holds
\begin{equation}
\hyperparam\mathbb{E}[\gain_r^{\optbase} - \gain_r^{\bidbase}] \geq \Big(\frac{1}{3000}\Big)\Big(\frac{\hyperparam\numblocks\blocksize}{\log_2^2(\numblocks\blocksize)}\Big).
\label{eq:ndiag_rgain_bid282}
\end{equation}
Observe that the expectation in the right-hand side of~\eqref{eq:ndiag_rgain_bid282} is dropped since it is not a random variable.

\paragraph*{Completing the bound.}
Combining the bounds on the expected paper-side and reviewer-side gain between \optbase and \bidbase given in~\eqref{eq:ndiag_rgain_bid2} and~\eqref{eq:ndiag_rgain_bid282}, we find for every $\numblocks\geq2, \blocksize\geq 2, \hyperparam\geq 0$,
\begin{align*}
\mathbb{E}[\gain^{\optbase}-\gain^{\bidbase}] &= \mathbb{E}[\gain_p^{\optbase}-\gain_p^{\bidbase}] + \hyperparam\mathbb{E}[\gain_r^{\optbase}-\gain_r^{\bidbase}] \\
&\geq \Big(\frac{1}{3000}\Big)\Big(\frac{\hyperparam\numblocks\blocksize}{\log_2^2(\blocksize)}\Big).
\end{align*}
We conclude that there exists a constant $\constant >0$ such that for every $\numblocks\geq 2, \blocksize\geq 2$, and $\hyperparam\geq 0$, \bidbase is suboptimal by an additive factor of at least $\constant\hyperparam\numblocks\blocksize/\log_2^2(\numblocks\blocksize)$ for the noiseless community model.

\subsubsection{Suboptimality of \randbase}
\label{sec:noiseless_rand}
In this section, we show the suboptimality of \randbase for the noiseless community model.

\paragraph*{Intuition and \randbase policy.} The \randbase algorithm selects a paper ordering uniformly at random from the set of permutations of papers. For the given class of gain and bidding functions, this is problematic since to obtain a bid from a reviewer, a paper from the block diagonal for the reviewer must be shown in the highest position. Since at least half of the papers are not on the block diagonal of the similarity matrix for any reviewer, there is a significant probability that \randbase fails to induce a bid from each reviewer. This causes the algorithm to be suboptimal for the expected paper-side gain.

\paragraph*{Bounding the expected paper-side gain.}
Recall from~\eqref{eq:noiseless_decoupled} that the expected paper-side gain from any paper $\paperindex\in \paperset$ is given by
\begin{equation}
\mathbb{E}[\gainfunction(\numbids_{\paperindex})] 
=  \sum_{\ell=0}^{\blocksize}\mathbb{P}\Big(\ell=\sum_{\revindex\in \diagsetp_{\paperindex}}\mathds{1}\{\ordering_{\revindex}^{\randbase}(\paperindex)=1\}\Big)\gainfunction(\ell).
\label{eq:nrand_diag_pgain}
\end{equation}
To bound this quantity for a given paper, we need to characterize the distribution of the number of times the paper is shown in the highest position to reviewers for which it is on the block diagonal. 

The \randbase algorithm selects a paper ordering uniformly at random from the set of paper permutations. This means the probability of paper any paper $\paperindex \in \paperset$ being shown in the highest position to any reviewer $\revindex \in \reviewerset$ is $1/\numblocks\blocksize$ since there are $\numpapers=\numblocks\blocksize$ papers. Consequently, the number of times paper $\paperindex\in \paperset$ is shown in the highest position to reviewers in the set $\diagsetp_{\paperindex}$ follows a binomial distribution with $\blocksize$ trials, since the cardinality of $\diagsetp_{\paperindex}$ is $\blocksize$, and a success probability of $1/\numblocks\blocksize$. This means the expected paper-side gain from any paper $\paperindex\in \paperset$ given in~\eqref{eq:nrand_diag_pgain} for \randbase is equivalently
\begin{equation}
\mathbb{E}[\gainfunction(\numbids_{\paperindex})]  = \sum_{\ell=0}^{\blocksize}{\blocksize \choose \ell}\Big(\frac{1}{\numblocks\blocksize}\Big)^{\ell}\Big(1-\frac{1}{\numblocks\blocksize}\Big)^{\blocksize-\ell}\gainfunction(\ell).
\label{eq:nrand_diag_pgain1}
\end{equation}
To bound~\eqref{eq:nrand_diag_pgain1}, we need the following lemma that bounds the expectation of the square root of a binomial random variable with $n$ trials and success probability $p$.
\begin{lemma}
Fix $n\geq 2$ and $p\in [0, 1]$. Then, 
\setlength\belowdisplayskip{-0pt}
\begin{equation*}
\sum_{k=0}^{n}{n \choose k}p^{k}(1-p)^{n-k}\sqrt{k} \leq np(1-p)^{n-1}\Big(1-\frac{\sqrt{2}}{2}\Big) +np\Big(\frac{\sqrt{2}}{2}\Big).
\end{equation*}
\label{lemma:binom}
\end{lemma}
\noindent
The proof of Lemma~\ref{lemma:binom} is provided in Section~\ref{sec:binom}.

We can directly apply Lemma~\ref{lemma:binom} to~\eqref{eq:nrand_diag_pgain1} since the given paper-side gain function is the square root function. The number of trials is $\blocksize\geq 2$ and the success probability is $1/\numblocks\blocksize$, so for any paper $\paperindex \in \paperset$, we obtain
\begin{equation}
\mathbb{E}[\gainfunction(\numbids_{\paperindex})]  = \sum_{\ell=0}^{\blocksize}{\blocksize \choose \ell}\Big(\frac{1}{\numblocks\blocksize}\Big)^{\ell}\Big(1-\frac{1}{\numblocks\blocksize}\Big)^{\blocksize-\ell}\gainfunction(\ell) \leq \Big(\frac{1}{\numblocks}\Big)\Big(1-\frac{1}{\numblocks\blocksize}\Big)^{\blocksize-1}\Big(1-\frac{\sqrt{2}}{2}\Big) +\frac{\sqrt{2}}{2\numblocks}.
\label{eq:nrand_diag_pgain2}
\end{equation}
The bound in the right-hand side of~\eqref{eq:nrand_diag_pgain2} is decreasing in $\numblocks$ and $\blocksize$ for $\numblocks\geq 2$ and $\blocksize\geq 2$. This means for every $\numblocks\geq 2$, $\blocksize\geq 2$, and any paper $\paperindex\in \paperset$, 
\begin{equation*} 
\mathbb{E}[\gainfunction(\numbids_{\paperindex})] \leq  \frac{6+\sqrt{2}}{16}.
\end{equation*}
To get a final bound on the expected paper-side gain of the algorithm, we sum the previous bound over the number of papers and obtain
\begin{equation}
\mathbb{E}[\gain_p^{\randbase}] =
 \sum_{\paperindex\in\paperset}\mathbb{E}[\gainfunction(\numbids_{\paperindex})]  \leq \Big(\frac{6+\sqrt{2}}{16} \Big)\numblocks\blocksize.
 \label{eq:nrand_diag_pgain3}
\end{equation}
Combining~\eqref{eq:nrand_diag_pgain3} with the expected paper-side gain of the optimal policy which was shown to be $\numblocks\blocksize$ in Section~\ref{sec:noiseless_optimal}, this implies for every $\numblocks\geq2, \blocksize\geq 2$, and $\hyperparam\geq 0$, 
\begin{equation}
\mathbb{E}[\gain_p^{\optbase} - \gain_p^{\randbase}] \geq \numblocks\blocksize - \Big(\frac{6+\sqrt{2}}{16}\Big)\numblocks\blocksize.
\label{eq:nrand_diag_pgain4}
\end{equation}

\paragraph*{Bounding the expected reviewer-side gain.}
We compare the expected reviewer-side gain of the optimal algorithm \optbase and \randbase. Previously in Section~\ref{sec:noiseless_optimal} we showed that the optimal algorithm maximizes the expected reviewer-side gain. This means that the expected reviewer-side gain of \randbase cannot exceed that from the optimal policy \optbase. Consequently, for every $\numblocks\geq2, \blocksize\geq2,$ and $\hyperparam\geq 0$ we get the bound
\begin{equation}
\hyperparam\mathbb{E}[\gain_r^{\optbase}-\gain_r^{\randbase}] \geq 0.
\label{eq:ndiag_rand_rgain}
\end{equation}

\paragraph*{Completing the bound.}
Combining the bounds on the expected paper-side and reviewer-side gain between the optimal algorithm \optbase and \randbase given in~~\eqref{eq:nrand_diag_pgain4} and~\eqref{eq:ndiag_rand_rgain}, for every $\numblocks\geq2, \blocksize\geq 2$, and $\hyperparam\geq 0$, we get that
\begin{align*}
\mathbb{E}[\gain^{\optbase}-\gain^{\randbase}] &= \mathbb{E}[\gain_p^{\optbase}-\gain_p^{\randbase}] + \hyperparam\mathbb{E}[\gain_r^{\optbase}-\gain_r^{\randbase}] \\
&\geq \numblocks\blocksize - \Big(\frac{6+\sqrt{2}}{16}\Big)\numblocks\blocksize\geq \numblocks\blocksize/2.
\end{align*}
We conclude that there exists a constant $\constant > 0$ such that for every $\numblocks\geq2, \blocksize\geq 2$, and $\hyperparam\geq 0$, \randbase is suboptimal by an additive factor of at least $\constant\numblocks\blocksize$ for the noiseless community model.

\subsubsection{Proofs of Lemmas~\ref{lemma:noiseless_diag_super_policy}--\ref{lemma:binom}}
\label{sec:noiseless_lemmas}
In this section, we present the proofs of technical lemmas stated in the primary proof of Theorem~\ref{prop:diagonal}.

\paragraph{Proof of Lemma~\ref{lemma:noiseless_diag_super_policy}.}\label{sec:noiseless_diag_super_policy}
In the proof of Corollary~\ref{sec:proof_local_col} given in Section~\ref{sec:proof_local_col}, we showed in~\eqref{eq:cor_opt2} that  
\super with zero heuristic solves the problem
\begin{equation}
\ordering_{\revindex}^{\super} = \argmax_{\ordering_{\revindex}\in \symgroup_\numpapers} \quad \sum_{\paperindex\in \paperset} \bidfunction(\ordering_{\revindex}(\paperindex), \similarity_{\revindex, \paperindex})(\gainfunction(\numbids_{\revindex-1, \paperindex}  + 1) - \gainfunction(\numbids_{\revindex-1, \paperindex})) + \hyperparam \sum_{\paperindex\in \paperset}\gainfunctionrev(\ordering_{\revindex}(\paperindex), \similarity_{\revindex, \paperindex}) 
\label{eq:noiseless_diagonal_policy}
\end{equation}
in order to determine the ordering of papers $\ordering_{\revindex}^{\super}$ to present to reviewer $\revindex \in \reviewerset$ so that the immediate expected gain is maximized conditioned on the history of bids from reviewers that arrived previously. Recalling that the bidding function is $\bidfunction(\ordering_{\revindex}(\paperindex), \similarity_{\revindex, \paperindex}) = \mathds{1}\{\ordering_{\revindex}(\paperindex)=1\}\mathds{1}\{\similarity_{\revindex, \paperindex} > \simscalar/2\}$, the optimization problem in~\eqref{eq:noiseless_diagonal_policy} is equivalent to 
\begin{equation}
\ordering_{\revindex}^{\super} = \argmax_{\ordering_{\revindex}\in \symgroup_\numpapers}  \sum_{\paperindex\in \paperset}\mathds{1}\{\ordering_{\revindex}(\paperindex)=1\}\mathds{1}\{\similarity_{\revindex, \paperindex} > \simscalar/2\}(\gainfunction(\numbids_{\revindex-1, \paperindex}  + 1) - \gainfunction(\numbids_{\revindex-1, \paperindex})) + \hyperparam \sum_{\paperindex\in \paperset}\gainfunctionrev(\ordering_{\revindex}(\paperindex), \similarity_{\revindex, \paperindex}).
\label{eq:noiseless_diagonal_policy1}
\end{equation}
Observe that $\diagset_{\revindex}\cup \diagset_{\revindex}^c=\paperset$. Moreover, if $\paperindex\in \diagset_{\revindex}$, then $\similarity_{\revindex, \paperindex}>\simscalar/2$ since $\similarity_{\revindex, \paperindex}=\simscalar$ by definition of the noiseless community model similarity matrix and $\simscalar \in [0.01, 1]$. Analogously, if $\paperindex\in \diagset_{\revindex}^c$, then $\similarity_{\revindex, \paperindex}<\simscalar/2$ since $\similarity_{\revindex, \paperindex}=0$ by definition of the noiseless community model similarity matrix and $\simscalar \in [0.01, 1]$. This allows us to simplify~\eqref{eq:noiseless_diagonal_policy1} to the following problem:
\begin{equation}
\ordering_{\revindex}^{\super} = \argmax_{\ordering_{\revindex}\in \symgroup_\numpapers} \quad \sum_{\paperindex\in \diagset_{\revindex}}\mathds{1}\{\ordering_{\revindex}(\paperindex)=1\}(\gainfunction(\numbids_{\revindex-1, \paperindex}  + 1) - \gainfunction(\numbids_{\revindex-1, \paperindex})) + \hyperparam \sum_{\paperindex\in \paperset}\gainfunctionrev(\ordering_{\revindex}(\paperindex), \similarity_{\revindex, \paperindex}).
\label{eq:noiseless_diagonal_policy2}
\end{equation}
The given paper-side gain function $\gainfunction$ is such that $\gainfunction(\numbids_{\revindex-1, \paperindex}  + 1) - \gainfunction(\numbids_{\revindex-1, \paperindex})$ is decreasing as a function of the number of bids $\numbids_{\revindex-1, \paperindex}$. As a result, the expected paper-side gain term from~\eqref{eq:noiseless_diagonal_policy2}, which is given by
\begin{equation}
\sum_{\paperindex\in \diagset_{\revindex}}\mathds{1}\{\ordering_{\revindex}(\paperindex)=1\}(\gainfunction(\numbids_{\revindex-1, \paperindex}  + 1) - \gainfunction(\numbids_{\revindex-1, \paperindex})),
\label{eq:noiseless_expect_pgain}
\end{equation}
is maximized by showing a paper $\paperindex \in \diagset_{\revindex}$ with the minimum number of bids in the highest position of the paper ordering. Moreover, the given reviewer-side gain function $\gainfunctionrev$ from~\eqref{eq:noiseless_rgainfunc} is decreasing in the position $\ordering_{\revindex}(\paperindex)$ in which a paper is shown  and increasing in the similarity score $\similarity_{\revindex, \paperindex}$. Consequently, the expected reviewer-side gain term from~\eqref{eq:noiseless_diagonal_policy2}, which is given by
\begin{equation}
\sum_{\paperindex\in \paperset}\gainfunctionrev(\ordering_{\revindex}(\paperindex), \similarity_{\revindex, \paperindex}),
\label{eq:noiseless_expect_rgain}
\end{equation}
is maximized by showing papers in decreasing order of the similarity scores. The similarity score is $\similarity_{\revindex, \paperindex}=\simscalar \in [0.01, 1]$ for papers in $\diagset_{\revindex}$ and the similarity score is $\similarity_{\revindex, \paperindex}=0$ for papers in $\diagset_{\revindex}^c$ by definition of the noiseless community model. Accordingly, the expected reviewer-side gain term in~\eqref{eq:noiseless_expect_rgain} is maximized as long as each paper in $\diagset_{\revindex}$ is shown earlier in the paper ordering than each paper in $\diagset_{\revindex}^c$.   

Since the expected paper-side and reviewer-side gain terms of~\eqref{eq:noiseless_diagonal_policy2} given by~\eqref{eq:noiseless_expect_pgain} and~\eqref{eq:noiseless_expect_rgain} respectively can be simultaneously maximized by showing any of the papers with the minimum number of bids among $\diagset_{\revindex}$ in the highest position, followed by the remaining papers in $\diagset_{\revindex}$ in any arbitrary order, and then the papers in $\diagset_{\revindex}^c$ in any arbitrary order, we conclude this is the policy of \super with zero heuristic.

\paragraph{Proof of Lemma~\ref{lemma:bid_diag_sum_bound}.}\label{sec:bid_diag_sum_bound}
Recall that $\gainfunctionrev^{\ordering} = 1/\log_2(x+1)$. Moreover, fix $\numblocks\geq 2$ and $\blocksize\geq 2$. Simplifying the expression we seek to bound, we obtain
\begin{align}
\gainfunctionrev^{\ordering}(\blocksize)-\gainfunctionrev^{\ordering}(\numblocks\blocksize-\numblocks-\lfloor\numblocks\blocksize/4\rfloor+2) &= \frac{1}{\log_2(\blocksize+1)}-\frac{1}{\log_2(\numblocks\blocksize-\numblocks-\lfloor \numblocks\blocksize/4 \rfloor+2+1)} \notag \\
&= \frac{\log_2(\numblocks\blocksize-\numblocks-\lfloor \numblocks\blocksize/4 \rfloor+2+1)-\log_2(\blocksize+1)}{\log_2(\numblocks\blocksize-\numblocks-\lfloor \numblocks\blocksize/4 \rfloor+2+1)\log_2(\blocksize+1)}. \label{eq:bd1}
\end{align}
We now lower bound the numerator of the right-hand side of~\eqref{eq:bd1}. For any fixed $\numblocks\geq 2$ and $\blocksize\geq 2$, 
\begin{equation}
\log_2(\numblocks\blocksize-\numblocks-\lfloor \numblocks\blocksize/4 \rfloor+2+1)-\log_2(\blocksize+1) \geq \log_2(2\blocksize-\lfloor \blocksize/2 \rfloor+1)-\log_2(\blocksize+1).
\label{eq:bd3}
\end{equation}
To see why, observe that $\numblocks\blocksize-\numblocks-\lfloor \numblocks\blocksize/4 \rfloor$ is non-decreasing as a function of $\numblocks$ for $\numblocks\geq2$ and $\blocksize\geq 2$. The non-decreasing property follows from the fact that 
\begin{align*}
(\numblocks+1)\blocksize-(\numblocks+1)-\lfloor (\numblocks+1)\blocksize/4 \rfloor &= \numblocks\blocksize-\numblocks+\blocksize-1-\lfloor \numblocks\blocksize/4+\blocksize/4 \rfloor \\
&\geq \numblocks\blocksize-\numblocks+\blocksize-1-(\lfloor \numblocks\blocksize/4\rfloor+\lfloor\blocksize/4 \rfloor+1) \\
&= \numblocks\blocksize-\numblocks-\lfloor \numblocks\blocksize/4\rfloor+\blocksize-\lfloor\blocksize/4 \rfloor-2 \\
&\geq \numblocks\blocksize-\numblocks-\lfloor \numblocks\blocksize/4\rfloor.
\end{align*}
To get the final inequality, consider
\begin{equation}
\blocksize-\lfloor\blocksize/4 \rfloor-2=\blocksize-\blocksize/4 +(\blocksize \ \mathrm{mod} \ 4)/4-2 =3\blocksize/4 +(\blocksize \ \mathrm{mod} \ 4)/4-2
\label{eq:bd4}
\end{equation}
and notice that for $\blocksize=2$,~\eqref{eq:bd4} is zero, and for $\blocksize> 2$,~\eqref{eq:bd4} is positive.

Now, see that $\log_2(2\blocksize-\lfloor \blocksize/2 \rfloor+1)-\log_2(\blocksize+1)$ is increasing as a function of $\blocksize$ since $2\blocksize-\lfloor\blocksize/2\rfloor> \blocksize$ for $\blocksize\geq 2$. Accordingly, for every $\numblocks\geq 2$ and $\blocksize\geq 2$, 
\begin{equation}
\log_2(2\blocksize-\lfloor \blocksize/2 \rfloor+1)-\log_2(\blocksize+1) \geq \log_2(4)-\log_2(3)\geq 2/5.
\label{eq:bd5}
\end{equation}

To finish, we obtain a lower bound on~\eqref{eq:bd1} by finding an upper bound on the denominator in the right-hand side. Observe that $\log_2(\numblocks\blocksize-\numblocks-\lfloor \numblocks\blocksize/4 \rfloor+2+1)\leq \log_2(\numblocks\blocksize)$ since $\numblocks+\lfloor \numblocks\blocksize/4 \rfloor\geq 3$ and $\log_2(\blocksize+1)\leq \log_2(\numblocks\blocksize)$ for every $\numblocks\geq 2$ and $\blocksize\geq 2$. Then, combined with~\eqref{eq:bd1},~\eqref{eq:bd3}, and~\eqref{eq:bd5}, we obtain the stated result of 
\begin{equation*}
\gainfunctionrev^{\ordering}(\blocksize)-\gainfunctionrev^{\ordering}(\numblocks\blocksize-\numblocks-\lfloor\numblocks\blocksize/4\rfloor+2) \geq \Big(\frac{2}{5}\Big)\Big(\frac{1}{\log_2^2(\numblocks\blocksize)}\Big).
\end{equation*}

\paragraph{Proof of Lemma~\ref{lemma:binom}.}\label{sec:binom}
Given $n\geq 2$ and $p\in [0, 1]$, we need to prove the bound
\begin{equation*}
\sum_{k=0}^{n}{n \choose k}p^{k}(1-p)^{n-k}\sqrt{k} \leq np(1-p)^{n-1}\Big(1-\frac{\sqrt{2}}{2}\Big) +np\Big(\frac{\sqrt{2}}{2}\Big).
\end{equation*}
To begin, observe that
\begin{equation}
\sum_{k=0}^{n}{n \choose k}p^{k}(1-p)^{n-k}\sqrt{k} = \sum_{k=1}^{n}{n \choose k}p^{k}(1-p)^{n-k}\sqrt{k}.
\label{eq:binom1}
\end{equation}
Then, since
\begin{equation*}
{n \choose k}p^{k} = \Big(\frac{n!}{k!(n-k)!}\Big)p^{k} = \Big(\frac{np}{k}\Big)\Big(\frac{(n-1)!}{(k-1)!(n-k)!}\Big)p^{k-1} = \Big(\frac{np}{k}\Big){n-1 \choose k-1}p^{k-1},
\end{equation*}
we can simplify~\eqref{eq:binom1} to obtain
\begin{equation}
\sum_{k=0}^{n}{n \choose k}p^{k}(1-p)^{n-k}\sqrt{k} = np\sum_{k=1}^{n}\Big(\frac{\sqrt{k}}{k}\Big) {n-1 \choose k-1}p^{k-1}(1-p)^{n-k}.
\label{eq:binom2}
\end{equation}
From the fact that $\tfrac{\sqrt{k}}{k}\leq \tfrac{\sqrt{2}}{2}$ for $k\geq 2$, we bound~\eqref{eq:binom2} as follows:
\begin{align}
np\sum_{k=1}^{n}\Big(\frac{\sqrt{k}}{k}\Big) {n-1 \choose k-1}p^{k-1}(1-p)^{n-k} &= np(1-p)^{n-1}+ np\sum_{k=2}^{n}\Big(\frac{\sqrt{k}}{k}\Big) {n-1 \choose k-1}p^{k-1}(1-p)^{n-k} \notag \\
&\leq np(1-p)^{n-1}+ np\Big(\frac{\sqrt{2}}{2}\Big)\sum_{k=2}^{n} {n-1 \choose k-1}p^{k-1}(1-p)^{n-k}.
\label{eq:binom3}
\end{align}
Relating~\eqref{eq:binom3} back to~\eqref{eq:binom2}, we get
\begin{equation}
\sum_{k=0}^{n}{n \choose k}p^{k}(1-p)^{n-k}\sqrt{k} \leq np(1-p)^{n-1}+ np\Big(\frac{\sqrt{2}}{2}\Big)\sum_{k=2}^{n} {n-1 \choose k-1}p^{k-1}(1-p)^{n-k}.
\label{eq:binom_extra}
\end{equation}
From addition and subtraction of $np(1-p)^{n-1}\big(\tfrac{\sqrt{2}}{2}\big)$ into the right-hand side of~\eqref{eq:binom_extra}, we obtain
\begin{equation}
\sum_{k=0}^{n}{n \choose k}p^{k}(1-p)^{n-k}\sqrt{k}\leq np(1-p)^{n-1}\Big(1-\frac{\sqrt{2}}{2}\Big)+ np\Big(\frac{\sqrt{2}}{2}\Big)\sum_{k=1}^{n} {n-1 \choose k-1}p^{k-1}(1-p)^{n-k}.
\label{eq:binom4}
\end{equation}
Now, see that from an indexing manipulation
\begin{equation}
\sum_{k=1}^{n} {n-1 \choose k-1}p^{k-1}(1-p)^{n-k} = \sum_{k=0}^{n-1} {n-1 \choose k}p^{k}(1-p)^{n-1-k}.
\label{eq:binom5}
\end{equation}
Moreover, from the Binomial theorem, 
\begin{equation}
\sum_{k=0}^{n-1} {n-1 \choose k}p^{k}(1-p)^{n-1-k} = (p+(1-p))^{n-1} = 1.
\label{eq:binom6}
\end{equation}
Combining~\eqref{eq:binom4},~\eqref{eq:binom5}, and~\eqref{eq:binom6} gives the final result of 
\begin{equation*}
\sum_{k=0}^{n}{n \choose k}p^{k}(1-p)^{n-k}\sqrt{k}\leq np(1-p)^{n-1}\Big(1-\frac{\sqrt{2}}{2}\Big)+ np\Big(\frac{\sqrt{2}}{2}\Big).
\end{equation*}

\subsection{Proof of Theorem~\ref{thm:diagonal}: Noisy Community Model Result}
In this proof, we show for the noisy community model that \super with zero heuristic is near optimal and each of the baselines is significantly suboptimal with respect to \super with zero heuristic.
The organization of this proof is as follows. In Section~\ref{sec:diag_prelim}, we present notation and preliminary analysis that is needed throughout the proof. In Section~\ref{sec:diag_super}, we analyze \super with zero heuristic and compute the expected paper-side gain for the similarity matrix class. We prove the suboptimality bounds for the \simbase, \bidbase, and \randbase baselines with respect to \super with zero heuristic separately in Sections~\ref{sec:diag_sim},~\ref{sec:diag_bid}, and~\ref{sec:diag_rand} respectively. We finish the proof in Section~\ref{sec:diag_opt} by showing that \super with zero heuristic is near optimal. Combining the results in each section of this proof gives the stated result of the theorem. Proofs of technical lemmas needed only for this proof can be found in Section~\ref{sec:diag_lemmas}. The proofs of technical lemmas used in this proof, but introduced in the proof of Theorem~\ref{prop:diagonal}, are given in Section~\ref{sec:noiseless_lemmas}. Finally, we remark that a number of methods for proving this result are similar to that from the proof of Theorem~\ref{prop:diagonal} and we point out in several places where this is the case as well as where the techniques differ.

\subsubsection{Notation and Preliminaries}
\label{sec:diag_prelim}
The notation and terminology in this proof follow that from the proof of Theorem~\ref{prop:diagonal} in Section~\ref{sec:noiseless_notation} since the gain and bidding functions are shared between the results and the noisy community model is based on the noiseless community model. 
The primary adjustment is that any reference to a similarity matrix $\similarity$ refers to that from the noisy community model, which is generated by selecting some similarity matrix $\similarity'$ from the noiseless community model as given in in~\eqref{eq:noiseless_model}, and then adding noise in the manner described in~\eqref{eq:noisy_model}.
Recall that the noise in the similarity score for each reviewer-paper pair $(\revindex, \paperindex)$ denoted by $\nu_{\revindex, \paperindex}$ is drawn independently and uniformly from $(0, \unifnoise)$ where $\unifnoise \leq (1+\hyperparam)^{-1} e^{-e\numblocks\blocksize}$ for the given trade-off parameter $\hyperparam\geq 0$. We also follow the notation from the proof of Theorem~\ref{prop:diagonal} in Section~\ref{sec:noiseless_notation}, in terms of terminology of reviewers and papers on the block diagonal and keep the sets $\diagset_{\revindex}$ for all $\revindex\in \reviewerset$ and $\diagsetp_{\paperindex}$ for all $\paperindex\in \paperset$ from~\eqref{eq:diagset_i} and~\eqref{eq:diagset_j} defined in terms of the noiseless community model similarity matrix now given by $\similarity'$.

In an analogous manner to the preliminaries section of the proof of Theorem~\ref{prop:diagonal}, we present several characteristics of the reviewer bidding behavior and the similarity scores that are needed throughout the proof. This set of rather immediate results also enable a decomposition of the expected paper-side gain equivalent to that for the noiseless community model from the proof of Theorem~\ref{prop:diagonal} given in~\eqref{eq:noiseless_decoupled}. 

We begin by showing if a paper is on the block diagonal for a reviewer, then it is bid on almost surely when shown in the highest position of the paper ordering to the reviewer and almost never when it is not. 
\begin{lemma}
Under the assumptions of Theorem~\ref{thm:diagonal},
if the reviewer-paper pair $(\revindex, \paperindex)$ is on the block diagonal of the noiseless community model matrix $\similarity'$ so that $\revindex\in \diagsetp_{\paperindex}$ and $\paperindex\in \diagset_{\revindex}$, then in the noisy community model matrix $\similarity_{\revindex, \paperindex}>\simscalar/2$. Moreover, the paper $\paperindex \in \paperset$ is bid on by reviewer $\revindex\in \reviewerset$ almost surely when $\ordering_{\revindex}^{\algbase}(\paperindex)=1$ and almost never when $\ordering_{\revindex}^{\algbase}(\paperindex)\neq 1$.
\label{lemma:exceed}
\end{lemma}
\begin{proof}[Proof of Lemma~\ref{lemma:exceed}] 
If the reviewer-paper pair $(\revindex, \paperindex)$ is on the block diagonal of the noiseless community model matrix $\similarity'$, then by definition  $\similarity_{\revindex, \paperindex}'=\simscalar$ and $\similarity_{\revindex, \paperindex} = \simscalar-\nu_{\revindex, \paperindex}$ where the noise $\nu_{\revindex, \paperindex}$ is drawn uniformly at random from the interval $(0, \unifnoise)$. Moreover, recall that $\simscalar \in [0.01, 1]$ and $\unifnoise \leq (1+\hyperparam)^{-1} e^{-e\numblocks\blocksize}$. Accordingly, for $\hyperparam\geq 0, \numblocks\geq 2$, and $\blocksize\geq 2$, we obtain 
\begin{equation}
\unifnoise \leq (1+\hyperparam)^{-1} e^{-e\numblocks\blocksize} \leq e^{-4e} < 0.01/2 \leq \simscalar/2.
\label{eq:simple_noise_bound}
\end{equation}
Since $\nu_{\revindex, \paperindex} \in (0, \unifnoise)$, we immediately get $\similarity_{\revindex, \paperindex}> \simscalar-\unifnoise$. Then applying~\eqref{eq:simple_noise_bound}, we conclude that $\similarity_{\revindex, \paperindex}> \simscalar/2$. Finally, since the probability of reviewer $\revindex$ bidding on paper $\paperindex$ is given by the quantity $\bidfunction(\ordering_{\revindex}^{\algbase}(\paperindex), \similarity_{\revindex, \paperindex}) = \mathds{1}\{\ordering_{\revindex}^{\algbase}(\paperindex)=1\}\mathds{1}\{\similarity_{\revindex, \paperindex} > \simscalar/2\}$, the reviewer bids on the paper almost surely when $\ordering_{\revindex}^{\algbase}(\paperindex)=1$ and almost never when $\ordering_{\revindex}^{\algbase}(\paperindex)\neq 1$.
\end{proof}

We now show that if a paper is not on the block diagonal for a given reviewer, then the reviewer bids on the paper almost never independent of the position the paper is shown.
\begin{lemma}
Under the assumptions of Theorem~\ref{thm:diagonal}, if the reviewer-paper pair $(\revindex, \paperindex)$ is not on the block diagonal of the noiseless community model matrix $\similarity'$ so that $\revindex\in \diagsetp_{\paperindex}^c$ and $\paperindex\in \diagset_{\revindex}^c$, then in the noisy community model matrix $\similarity_{\revindex, \paperindex}< \simscalar/2$. Moreover, paper $\paperindex \in \paperset$ is bid on almost never by reviewer $\revindex \in \reviewerset$ independent of $\ordering_{\revindex}^{\algbase}(\paperindex)$.
\label{lemma:under}
\end{lemma}
\begin{proof}[Proof of Lemma~\ref{lemma:under}] 
If the reviewer-paper pair $(\revindex, \paperindex)$ is not on the block diagonal of the noiseless community model matrix $\similarity'$, then by definition $\similarity_{\revindex, \paperindex}'=0$ and $\similarity_{\revindex, \paperindex} = \nu_{\revindex, \paperindex}$ where the noise $\nu_{\revindex, \paperindex}$ is drawn uniformly at random from the interval $(0, \unifnoise)$.
Since $\nu_{\revindex, \paperindex} \in (0, \unifnoise)$, we immediately get $\similarity_{\revindex, \paperindex}< \unifnoise$. Then applying~\eqref{eq:simple_noise_bound}, we conclude that $\similarity_{\revindex, \paperindex}< \simscalar/2$. 
Finally, since the probability of reviewer $\revindex$ bidding on paper $\paperindex$ is given by the quantity $\bidfunction(\ordering_{\revindex}^{\algbase}(\paperindex), \similarity_{\revindex, \paperindex}) = \mathds{1}\{\ordering_{\revindex}^{\algbase}(\paperindex)=1\}\mathds{1}\{\similarity_{\revindex, \paperindex} > \simscalar/2\}$, the reviewer bids on the paper almost never independent of the position the paper is shown to the reviewer given from $\ordering_{\revindex}^{\algbase}(\paperindex)$.
\end{proof}
 Observe that Lemmas~\ref{lemma:exceed} and~\ref{lemma:under} imply that each reviewer bids on at most one paper almost surely. Moreover, they can also be combined to determine that any paper on the block diagonal for a reviewer is guaranteed to have a higher similarity score than any paper not on the block diagonal for the reviewer.
\begin{lemma}
Under the assumptions of Theorem~\ref{thm:diagonal}, if the reviewer-paper pair $(\revindex, \paperindex)$ is on the block diagonal of the noiseless community model matrix $\similarity'$ so that $\revindex\in \diagsetp_{\paperindex}$ and $\paperindex\in \diagset_{\revindex}$, and the reviewer-paper pair $(\revindex, \paperindex')$ is not on the block diagonal of the noiseless community model matrix $\similarity'$ so that $\revindex\in \diagsetp_{\paperindex'}^c$ and $\paperindex'\in \diagset_{\revindex}^c$, then in the noisy community model matrix $\similarity_{\revindex, \paperindex}> \similarity_{\revindex, \paperindex'}$.
\label{lemma:compare}
\end{lemma}

We now apply the preceding results to show that the expected paper-side gain from a given paper $\paperindex \in \paperset$ only depends on the positions it is shown to reviewers $\revindex \in \reviewerset$ by some algorithm $\algbase$ for which the reviewer-paper pair $(\revindex, \paperindex)$ is on the block diagonal of the similarity matrix. 
The expected paper-side gain for any paper $\paperindex \in \paperset$ simplifies to be
\begin{equation}
\mathbb{E}[\gainfunction(\numbids_{\paperindex})] 
=  \mathbb{E}\Big[\gainfunction\Big(\sum_{\revindex\in \reviewerset}\randombid_{\revindex, \paperindex}\Big)\Big] = \mathbb{E}\Big[\gainfunction\Big(\sum_{\revindex\in \diagsetp_{\paperindex}}\randombid_{\revindex, \paperindex}\Big)\Big] =  \sum_{\ell=0}^{\blocksize}\mathbb{P}\Big(\ell=\sum_{\revindex\in \diagsetp_{\paperindex}}\mathds{1}\{\ordering_{\revindex}^{\algbase}(\paperindex)=1\}\Big)\gainfunction(\ell).
\label{eq:decoupled}
\end{equation}
The above equation follows from Lemma~\ref{lemma:under}, which indicates that the the bid from any reviewer $\revindex \in \diagsetp_{\paperindex}^c$ is zero almost surely independent of the position the paper is shown, and from Lemma~\ref{lemma:exceed}, which guarantees any reviewer $\revindex \in \diagsetp_{\paperindex}$ bids on the paper $\paperindex\in \paperset$ almost surely if $\ordering_{\revindex}^{\algbase}(\paperindex)=1$ and almost never if $\ordering_{\revindex}^{\algbase}(\paperindex)\neq1$. 
As mentioned at the beginning of this section, this decomposition of the expected paper-side gain for the noisy community model in~\eqref{eq:decoupled} is equivalent to that for the noiseless community model given in~\eqref{eq:noiseless_decoupled}.

\subsubsection{Analyzing \super with Zero Heuristic}
\label{sec:diag_super}
In this section, we present a preliminary analysis of \super with zero heuristic for the noisy community model similarity matrix class with the given gain and bidding functions. We begin by characterizing the behavior of \super with zero heuristic. Following deriving the policy, the expected paper-side gain of the algorithm is computed. The analysis of the expected reviewer-side gain is deferred to the sections showing the suboptimality of the baselines with respect to \super with zero heuristic (specifically, see Sections \ref{sec:diag_sim} and \ref{sec:diag_bid}).
However, we do provide intuition in this section for why the expected reviewer-side gain is nearly optimal.

\paragraph*{Intuition and \super with zero heuristic policy.}
The given bidding function is such that only the paper shown in the highest position to a reviewer has a non-zero probability of being bid on. Intuitively Lemma~\ref{lemma:exceed} and Lemma~\ref{lemma:under} suggest that 
to optimize the expected paper-side gain, the algorithm should seek to show a paper on the block diagonal for the reviewer in the highest position. Moreover, since the given paper-side gain function exhibits diminishing returns in the number of bids, showing the paper on the block diagonal with fewest bids maximizes the immediate expected paper-side gain. 

The given reviewer-side gain function is decreasing in the position a paper is shown and increasing in the similarity score of the paper. This indicates that to maximize the immediate expected reviewer-side gain, papers should be shown in a decreasing order of the similarity scores to the reviewer. For the given noisy community model similarity class, the similarity scores of papers on the block diagonal for a given reviewer are significantly higher than the similarity scores of papers off the block diagonal for the given reviewer as was formalized in Lemma~\ref{lemma:compare}. Furthermore, the noise is bounded in a small interval. Consequently, the similarity scores for papers on the block diagonal for a reviewer are nearly identical and the similarity scores for papers off the block diagonal for a reviewer are also nearly identical. This suggests that as long as papers on the block diagonal are shown ahead of papers off the block diagonal for a reviewer, then the expected reviewer-side gain from a given reviewer should be close to the maximum that can be obtained. 

The high-level view of the objective the algorithm is optimizing indicates that the immediate expected paper-side gain can be maximized with minimal cost to the immediate expected reviewer-side gain. This can be achieved by showing the paper on the block diagonal for the reviewer with the minimum number of bids in the highest position and the remaining papers in a decreasing order of the similarity scores. The following lemma formalizes the intuition that has been given. 

\begin{lemma}
Under the assumptions of Theorem~\ref{thm:diagonal},
when reviewer $\revindex\in \reviewerset$ arrives, if there is a paper in $\diagset_{\revindex}$ with zero bids
and each paper in $\diagset_{\revindex}$ has at most one bid, then \super with zero heuristic shows the reviewer the paper with the maximum similarity score among the papers without a bid in $\diagset_{\revindex}$ at the highest position followed by the remaining papers in a decreasing order of the similarity scores.
\label{lemma:diag_super_policy}
\end{lemma}
The proof of Lemma~\ref{lemma:diag_super_policy} is provided in Section~\ref{sec:diag_super_policy}. 
It turns out that the conditions
of Lemma~\ref{lemma:diag_super_policy}, namely the existence of a paper in $\diagset_{\revindex}$ with zero bids and each paper in $\diagset_{\revindex}$ having at most one bid upon the arrival of reviewer $\revindex \in \reviewerset$, are met using \super with zero heuristic almost surely. 
We now formally characterize this statement and then compute the expected paper-side gain of the algorithm. 

\paragraph*{Computing the expected paper-side gain.}
Consider a group of $\blocksize$ reviewers denoted by $\mathcal{R}$ for which $\diagset_{\revindex}=\diagset_{\revindex'}$ for every pair of reviewers $\revindex, \revindex' \in \mathcal{R}$, meaning that the papers on the block diagonal for the reviewers are equivalent. Observe that from the structure of the noisy community model, for every reviewer $\revindex\in \mathcal{R}$, it also holds that $\diagset_{\revindex} \subset \diagset_{\revindex''}^c$ for all $\revindex''\in \mathcal{R}^c$, meaning that the papers on the block diagonal for each reviewer in $\mathcal{R}$ are off the block diagonal for all reviewers in $\mathcal{R}^c$. Moreover, there are $\numblocks$ such blocks of reviewers analogous to the given group $\mathcal{R}$. 

Upon the initial arrival of a reviewer $\revindex$ from $\mathcal{R}$, from Lemma~\ref{lemma:under} each paper in $\diagset_{\revindex}$ has zero bids almost surely since they are off the block diagonal for all reviewers that arrived previously. From Lemma~\ref{lemma:diag_super_policy}, \super with zero heuristic shows this reviewer the paper in $\diagset_{\revindex}$ with the maximum similarity score in the highest position of the paper ordering. Lemma~\ref{lemma:exceed} guarantees that this paper is bid on by the reviewer almost surely and the rest of the papers in $\diagset_{\revindex}$ are bid on almost never by the reviewer. 

We now consider the next arrival of a reviewer $\revindex'$ from $\mathcal{R}$ and note that $\diagset_{\revindex'}=\diagset_{\revindex}$ by definition of this set of reviewers. Between the arrivals of reviewers $\revindex$ and $\revindex'$, none of the papers in $\diagset_{\revindex'}$ obtain any more bids almost surely since again from Lemma~\ref{lemma:under} any paper that is off the block diagonal for a reviewer is bid on almost never independent of the position the paper is shown. This means each paper in $\diagset_{\revindex'}$ has zero bids almost surely except for the paper that has a bid from reviewer $\revindex$ almost surely. Accordingly, we apply Lemma~\ref{lemma:diag_super_policy} to determine that \super with zero heuristic shows this reviewer the paper with the maximum similarity score among the papers without a bid within $\diagset_{\revindex'}$ in the highest position of the paper ordering. Then, Lemma~\ref{lemma:exceed} guarantees that this paper is bid on by the reviewer almost surely and the rest of the papers in $\diagset_{\revindex'}$ are bid on almost never by the reviewer.

Repeatedly applying this argument, upon the final arrival of a reviewer $\revindex''$ from $\mathcal{R}$, each paper in $\diagset_{\revindex''}$ has exactly one bid except for one paper that remains without a bid almost surely. We note again that by definition of this set of reviewers $\diagset_{\revindex''}=\diagset_{\revindex}$. From Lemma~\ref{lemma:diag_super_policy}, \super with zero heuristic shows the final paper without a bid within $\diagset_{\revindex''}$ in the highest position of the paper ordering and Lemma~\ref{lemma:exceed} ensures that this paper is bid on by the reviewer almost surely and the rest of the papers in $\diagset_{\revindex'}$ are bid on almost never by the reviewer. 

Following the arrival of reviewer $\revindex''$, the papers in $\diagset_{\revindex}$ never appear on the block diagonal for a reviewer again and finish with exactly one bid almost surely after each being shown in the highest position of the paper ordering exactly once to some reviewer for which they are on the block diagonal almost surely. The line of reasoning applied to the group of reviewers $\mathcal{R}$ can be duplicated for each of the $\numblocks$ blocks of reviewers which share papers on the block diagonal. In doing so, it immediately follows from the decomposition in~\eqref{eq:decoupled} that the expected paper-side gain of \super with zero heuristic for every $\numblocks\geq2, \blocksize\geq 2$, and $\hyperparam\geq 0$, is given by
\begin{equation}
\mathbb{E}[\gain_p^{\super}] =
\sum_{\paperindex\in \paperset}\mathbb{E}[\gainfunction(\numbids_{\paperindex})]  = \sum_{\paperindex\in \paperset}\gainfunction(1) = \numblocks\blocksize.
\label{eq:diag_super_pgain}
\end{equation}
Identically as in the derivation of the optimal expected paper-side gain for the noiseless community model given in Section~\ref{sec:noiseless_optimal}, this is the optimal expected paper-side gain that can be obtained since each reviewer bids on at most one paper almost surely and the given paper-side gain function is strictly concave so evenly distributing the bids over the papers maximizes the expected paper-side gain.

\paragraph*{Properties of \super with zero heuristic for the noisy community model similarity matrix.}
Since the conditions of Lemma~\ref{lemma:diag_super_policy} are satisfied for each reviewer almost surely, \super with zero heuristic shows each reviewer $\revindex\in \reviewerset$ the paper with the maximum similarity score among the papers without a bid in $\diagset_{\revindex}$ followed by the remaining papers in a decreasing order of the similarity scores. This fact leads to several properties of the algorithm for this similarity matrix class. 
From Lemma~\ref{lemma:compare}, $\similarity_{\revindex, \paperindex}>\similarity_{\revindex, \paperindex'}$ for $\paperindex\in \diagset_{\revindex}, \paperindex'\in \diagset_{\revindex}^c$. This means the paper with the maximum similarity score among the papers without a bid in $\diagset_{\revindex}$ is equivalently the paper with the maximum similarity score among the papers without a bid. This results in the following property of the algorithm.
\begin{property}
\super with zero heuristic presents the paper with the maximum similarity score among the papers without a bid in the highest position of the paper ordering and the remaining papers in a decreasing order of the similarity scores to each reviewer $\revindex\in \reviewerset$ almost surely. 
\label{property1}
\end{property}
Similarly, using Lemma~\ref{lemma:compare}, we can determine that the algorithm shows each paper that is on the block diagonal of the similarity matrix for the reviewer ahead of each paper off the block diagonal.
\begin{property}
\super with zero heuristic shows every paper in $\diagset_{\revindex}$ ahead of every paper in $\diagset_{\revindex}^c$ to each reviewer $\revindex\in \reviewerset$ almost surely. 
\label{property2}
\end{property}
The final property that again follows from Lemma~\ref{lemma:compare} is that the algorithm shows the papers off the block diagonal in a decreasing order of the similarity scores.
\begin{property}
\super with zero heuristic shows papers among $\diagset_{\revindex}^c$ in a decreasing order of the similarity scores to each reviewer $\revindex\in \reviewerset$ almost surely. 
\label{property3}
\end{property}

The properties of \super with zero heuristic provided are going to assist the comparison of the expected reviewer-side gain with that from the baselines and the optimal policy. As discussed previously, intuitively Property~\ref{property2} should guarantee that the algorithm obtains near-optimal expected reviewer-side gain since the noise in the similarity scores is bounded in a small interval and the similarity scores of papers on the block diagonal are much higher than that for papers off the block diagonal for a reviewer. 

\subsubsection{Suboptimality of \simbase} 
\label{sec:diag_sim}

In this section, we analyze \simbase for the noisy community model similarity matrix class with the given gain and bidding functions. 

\paragraph*{Intuition and \simbase policy.}
The \simbase algorithm presents papers in a decreasing order of the similarity scores. This approach maximizes the expected reviewer-side gain since the given reviewer-side gain function is decreasing in the position a paper is shown and increasing in the similarity score of the paper. However, for the noisy community model similarity matrix class, the algorithm is suboptimal for the combined objective since the expected paper-side gain is far from optimal. As shown in the analysis of \super with zero heuristic, to maximize the expected paper-side gain, each paper should only be shown in the highest position of the paper ordering to a reviewer once almost surely. 
Moreover, this must be when the paper is on the block diagonal for a reviewer so that it obtains a bid almost surely. The problem with \simbase for this similarity matrix class is that it is oblivious to the number of bids on papers. As a result, the algorithm may show a paper in the highest position of the paper ordering to a reviewer that has only marginally higher similarity score, but many more bids, than another option. While this may result in a scarce amount  more gain from the reviewer-side objective, we show it is costly in terms of the paper-side objective. We formalize this by showing \simbase is significantly suboptimal for the expected paper-side gain, and that it only achieves a marginal amount more expected reviewer-side gain than \super with zero heuristic.

\paragraph*{Bounding the expected paper-side gain.}
Recall from~\eqref{eq:decoupled} that the expected paper-side gain from any paper $\paperindex\in \paperset$ is given by
\begin{equation}
\mathbb{E}[\gainfunction(\numbids_{\paperindex})] 
=  \sum_{\ell=0}^{\blocksize}\mathbb{P}\Big(\ell=\sum_{\revindex\in \diagsetp_{\paperindex}}\mathds{1}\{\ordering_{\revindex}^{\simbase}(\paperindex)=1\}\Big)\gainfunction(\ell).
\label{eq:sim_diag_pgain}
\end{equation}
To bound this quantity for a given paper, we need to characterize the distribution of the number of times the paper is shown in the highest position of the paper ordering to reviewers for which it is on the block diagonal. 

Toward this goal, let us consider any reviewer $\revindex\in \reviewerset$ and the probability of each paper being shown in the highest position of the paper ordering for the reviewer. 
For the given reviewer, the set of papers on the block diagonal is given by $\diagset_{\revindex}$ and this set has cardinality $\blocksize$. Moreover, from Lemma~\ref{lemma:compare}, $\similarity_{\revindex, \paperindex}>\similarity_{\revindex, \paperindex'}$ for $\paperindex\in \diagset_{\revindex}, \paperindex'\in \diagset_{\revindex}^c$. 
This result says the similarity score of any paper on the block diagonal for the reviewer is greater than the similarity score of any paper off the block diagonal for the reviewer. 

The \simbase algorithm shows papers in a decreasing order of the similarity scores, which combined with Lemma~\ref{lemma:compare}, guarantees that the probability of any paper $\paperindex \in \diagset_{\revindex}^c$ being shown in the highest position of the paper ordering is zero. For any paper $\paperindex \in \diagset_{\revindex}$, the similarity score is given by $\similarity_{\revindex, \paperindex} = \simscalar-\nu_{\revindex, \paperindex}$. The noise $\nu_{\revindex, \paperindex}$ for each reviewer-paper pair $(\revindex, \paperindex)$ is drawn independently and uniformly at random from a bounded interval. This implies that the probability of any paper $\paperindex \in \diagset_{\revindex}$ being shown in the highest position to the reviewer is $1/\blocksize$ since there are $\blocksize$ papers in the set.

Recall that $\diagsetp_{\paperindex}$ for any paper $\paperindex \in \paperset$ denotes the reviewers $\revindex \in \reviewerset$ for which the reviewer-paper pair $(\revindex, \paperindex)$ is on the block diagonal of the similarity matrix. From the preceding reasoning, the probability of paper $\paperindex\in \paperset$ being shown in the highest position to each reviewer $\revindex\in \diagsetp_{\paperindex}$ is $1/\blocksize$. Consequently, the number of times paper $\paperindex\in \paperset$ is shown in the highest position to reviewers in the set $\diagsetp_{\paperindex}$ follows a Binomial distribution with $\blocksize$ trials since the cardinality of $\diagsetp_{\paperindex}$ is $\blocksize$ and a success probability of $1/\blocksize$. This means the expected paper-side gain from any paper $\paperindex\in \paperset$ given in~\eqref{eq:sim_diag_pgain} for \simbase, is equivalently expressed as 
\begin{equation}
\mathbb{E}[\gainfunction(\numbids_{\paperindex})]  = \sum_{\ell=0}^{\blocksize}{\blocksize \choose \ell}\Big(\frac{1}{\blocksize}\Big)^{\ell}\Big(1-\frac{1}{\blocksize}\Big)^{\blocksize-\ell}\gainfunction(\ell).
\label{eq:sim_diag_pgain1}
\end{equation}
To bound~\eqref{eq:sim_diag_pgain1}, we can directly apply Lemma~\ref{lemma:binom}, which bounds the expectation of the square root of a binomial random variable, 
and obtain 
\begin{equation}
\mathbb{E}[\gainfunction(\numbids_{\paperindex})]  = \sum_{\ell=0}^{\blocksize}{\blocksize \choose \ell}\Big(\frac{1}{\blocksize}\Big)^{\ell}\Big(1-\frac{1}{\blocksize}\Big)^{\blocksize-\ell}\gainfunction(\ell) \leq \Big(1-\frac{1}{\blocksize}\Big)^{\blocksize-1}\Big(1-\frac{\sqrt{2}}{2}\Big) +\frac{\sqrt{2}}{2}.
\label{eq:sim_diag_pgain2}
\end{equation}
The bound in~\eqref{eq:sim_diag_pgain2} is decreasing in $\blocksize$ for $\blocksize\geq 2$. This means for every $\blocksize\geq 2$ and any paper $\paperindex\in \paperset$, 
\begin{equation*} 
\mathbb{E}[\gainfunction(\numbids_{\paperindex})] \leq \frac{2+\sqrt{2}}{4}.
\end{equation*}
To get the expected paper-side gain of the algorithm, we sum this bound over the number of papers and obtain
\begin{equation}
\mathbb{E}[\gain_p^{\simbase}] =
 \sum_{\paperindex\in\paperset}\mathbb{E}[\gainfunction(\numbids_{\paperindex})]  \leq \Big(\frac{2+\sqrt{2}}{4}\Big)\numblocks\blocksize.
 \label{eq:sim_diag_pgain3}
\end{equation}
Combining~\eqref{eq:sim_diag_pgain3} with the expected paper-side gain of \super with zero heuristic from~\eqref{eq:diag_super_pgain}, this implies for every $\numblocks\geq2, \blocksize\geq 2$, and $\hyperparam\geq 0$, 
\begin{equation}
\mathbb{E}[\gain_p^{\super} - \gain_p^{\simbase}] \geq \numblocks\blocksize - \Big(\frac{2+\sqrt{2}}{4}\Big)\numblocks\blocksize.
\label{eq:sim_diag_pgain4}
\end{equation}

\paragraph*{Bounding the expected reviewer-side gain.}
We now turn our attention to comparing the expected reviewer-side gain of \super with zero heuristic and \simbase. We need to bound 
\begin{equation}
\hyperparam\mathbb{E}[\gain_r^{\super} - \gain_r^{\simbase}] = \hyperparam\mathbb{E}\Big[\sum_{\revindex\in \reviewerset}\sum_{\paperindex\in \paperset}(\gainfunctionrev(\ordering_{\revindex}^{\super}(\paperindex), \similarity_{\revindex, \paperindex})-\gainfunctionrev(\ordering_{\revindex}^{\simbase}(\paperindex), \similarity_{\revindex, \paperindex}))\Big].
\label{eq:diag_rgain_sim}
\end{equation}
Toward doing so, recall Property~\ref{property2}, which says \super with zero heuristic shows every paper in $\diagset_{\revindex}$ ahead of every paper in $\diagset_{\revindex}^c$ to each reviewer $\revindex\in \reviewerset$ almost surely. Moreover, Property~\ref{property3} says the papers among $\diagset_{\revindex}^c$ are shown in decreasing order of the similarity scores to each reviewer $\revindex \in \reviewerset$ almost surely. In comparison, \simbase shows papers in decreasing order of the similarity scores to each reviewer $\revindex \in \reviewerset$. From Lemma~\ref{lemma:compare}, the similarity scores of papers in $\diagset_{\revindex}$ are greater than the similarity scores of papers in $\diagset_{\revindex}^c$ for each reviewer $\revindex\in \reviewerset$. Combining this fact with the policy of \super with zero heuristic and \simbase, we can see that the algorithms show the papers among $\diagset_{\revindex}^c$ in identical positions almost surely. This means the reviewer-side gain from this set of papers is equivalent for each of the algorithms almost surely.

Since the noise is bounded in a small interval, we expect that the ordering among papers in $\diagset_{\revindex}$ would not impact the expected reviewer-side gain significantly as long as they are shown before the papers in $\diagset_{\revindex}^c$. The following result formalizes this intuition and provides a bound.
\begin{lemma}
Let $\ordering_{\revindex}^{\ell}$ denote the paper ordering that presents papers in decreasing order of the similarity scores. Moreover, denote by $\ordering_{\revindex}^{\ell'}$ any paper ordering that shows each paper in $\diagset_{\revindex}$ ahead of each paper in $\diagset_{\revindex}^c$ and papers among $\diagset_{\revindex}^c$ in a decreasing order of the similarity scores. Then, under the assumptions of Theorem~\ref{thm:diagonal}, the following bound holds for every $\numblocks\geq 2, \blocksize\geq 2$, and $\hyperparam\geq 0$:
\begin{equation*}
\setlength\belowdisplayskip{-0pt}
\hyperparam\sum_{\paperindex\in \paperset}\big(\gainfunctionrev(\ordering_{\revindex}^{\ell'}(\paperindex), \similarity_{\revindex, \paperindex})-\gainfunctionrev(\ordering_{\revindex}^{\ell}(\paperindex), \similarity_{\revindex, \paperindex})\big)\geq -\blocksize e^{-e\numblocks\blocksize}\log(4).
\end{equation*}
\label{lemma:bound_sim_sum}
\end{lemma}
\noindent
The proof of Lemma~\ref{lemma:bound_sim_sum} is provided in Section~\ref{sec:bound_sim_sum}.

The result of Lemma~\ref{lemma:bound_sim_sum} immediately applies to~\eqref{eq:diag_rgain_sim} for each reviewer conditioned on the almost sure events given the characteristics of \super with zero heuristic and \simbase mentioned earlier and since it holds for any realization of the noise in the similarity scores. Combining~\eqref{eq:diag_rgain_sim} with Lemma~\ref{lemma:bound_sim_sum} and noting that there are $\numrev=\numblocks\blocksize$ reviewers, we obtain for every $\numblocks\geq 2, \blocksize\geq 2$, and $\hyperparam\geq 0$,
\begin{equation}
\hyperparam\mathbb{E}[\gain_r^{\super} - \gain_r^{\simbase}] = \hyperparam\mathbb{E}\Big[\sum_{\revindex\in \reviewerset}\sum_{\paperindex\in \paperset}(\gainfunctionrev(\ordering_{\revindex}^{\super}(\paperindex), \similarity_{\revindex, \paperindex})-\gainfunctionrev(\ordering_{\revindex}^{\simbase}(\paperindex), \similarity_{\revindex, \paperindex}))\Big] \geq -\numblocks\blocksize^2e^{-e\numblocks\blocksize}\log(4).
\label{eq:diag_rgain_sim2}
\end{equation}
Finally, see that for every $\numblocks\geq 2$ and $\blocksize\geq 2$,
\begin{equation}
-\numblocks\blocksize^2e^{-e\numblocks\blocksize}\log(4)\geq -8e^{-4e}\log(4)\geq -0.0001
\label{eq:diag_rgain_sim3}
\end{equation}
since $-\numblocks\blocksize^2e^{-e\numblocks\blocksize}\log(4)$ is negative and increasing as a function of $\numblocks$ and $\blocksize$ for $\numblocks\geq 2$ and $\blocksize\geq 2$. From~\eqref{eq:diag_rgain_sim2} and~\eqref{eq:diag_rgain_sim3}, we get that for every $\numblocks\geq 2, \blocksize\geq 2$, and $\hyperparam\geq 0$, 
\begin{equation}
\hyperparam\mathbb{E}[\gain_r^{\super} - \gain_r^{\simbase}]  \geq -0.0001.
\label{eq:supersimrevbound}
\end{equation}

\paragraph*{Completing the bound.}
Combining the bounds on the expected paper-side and reviewer-side gain between \super with zero heuristic and \simbase given in~\eqref{eq:sim_diag_pgain4} and~\eqref{eq:supersimrevbound}, we get that for every $\numblocks\geq2, \blocksize\geq 2$, and $\hyperparam\geq 0$,
\begin{align*}
\mathbb{E}[\gain^{\super}-\gain^{\simbase}] &= \mathbb{E}[\gain_p^{\super}-\gain_p^{\simbase}] + \hyperparam\mathbb{E}[\gain_r^{\super}-\gain_r^{\simbase}] \\
&\geq \numblocks\blocksize - \Big(\frac{2+\sqrt{2}}{4}\Big)\numblocks\blocksize-0.0001 \\
&\geq \numblocks\blocksize/10.
\end{align*}
We conclude that there is a constant $\constant >0$ such that for every $\numblocks\geq2, \blocksize\geq 2$, and $\hyperparam\geq 0$, \super with zero heuristic obtains an additive factor of at least $\constant\numblocks\blocksize$ more expected gain than \simbase in the noisy community model.

\subsubsection{Suboptimality of \bidbase} 
\label{sec:diag_bid}
We now analyze \bidbase for the noisy community model with the given gain and bidding functions. 
We remark that much of the analysis in this section follows very similarly to that for \bidbase given in Section~\ref{sec:noiseless_bid} from the proof of Theorem~\ref{prop:diagonal} regarding the noiseless community model since the reviewer bidding behavior is identical in the noisy community model as characterized by Lemmas~\ref{lemma:exceed},~\ref{lemma:under}, and~\ref{lemma:compare}. The primary adjustments are in the analysis of the expected reviewer-side gain with respect to \super with zero heuristic.

\paragraph*{Bounding the expected paper-side gain.}
For any given reviewer $\revindex \in \reviewerset$, the $\blocksize$ papers in $\diagset_{\revindex}$ are each in $\diagset_{\revindex'}$ for $\blocksize-1$ other reviewers $\revindex' \in \reviewerset$ and also in $\diagset_{\revindex''}^c$ for each of the remaining reviewers $\revindex''\in \reviewerset$.
For any reviewer $\revindex\in \reviewerset$, any paper $\paperindex \in \diagset_{\revindex}^c$ is bid on almost never from Lemma~\ref{lemma:under} and any paper $\paperindex \in \diagset_{\revindex}$ is only bid on with non-zero probability if shown in the highest position to the reviewer from Lemma~\ref{lemma:exceed}.
This means that upon the arrival of each reviewer $\revindex\in \reviewerset$, there is a paper in $\diagset_{\revindex}$ with zero bids that has not been shown in the highest position to any reviewer previously almost surely. 
Furthermore, from Lemma~\ref{lemma:compare}, the similarity score of any paper in $\diagset_{\revindex}$ is greater than the similarity score of any paper in $\diagset_{\revindex}^c$ for each reviewer $\revindex  \in \reviewerset$. Therefore, \bidbase shows a paper in $\diagset_{\revindex}$ with zero bids in the highest position to each reviewer $\revindex \in \reviewerset$ almost surely since papers are shown in increasing order of the number of bids and ties are broken in favor of the paper with the higher similarity score. Lemma~\ref{lemma:exceed} guarantees that if a paper in $\diagset_{\revindex}$ is shown in the highest position of the paper ordering to reviewer $\revindex \in \reviewerset$, then it is bid on by the reviewer almost surely.
Consequently, each paper $\paperindex \in \paperset$ is shown exactly once almost surely in the highest position to some reviewer to some reviewer $\revindex \in \diagsetp_{\paperindex}$. It then follows from the decomposition in~\eqref{eq:decoupled} that the expected paper-side gain of \bidbase for every $\numblocks\geq2, \blocksize\geq 2$, and $\hyperparam\geq 0$, is given by
\begin{equation*}
\mathbb{E}[\gain_p^{\bidbase}] =
\sum_{\paperindex\in \paperset}\mathbb{E}[\gainfunction(\numbids_{\paperindex})] 
 = \sum_{\paperindex\in \paperset}\gainfunction(1) = \numblocks\blocksize.
\end{equation*}

From the expected paper-side gain of \super with zero heuristic given in~\eqref{eq:diag_super_pgain}, we conclude that for every $\numblocks\geq2, \blocksize\geq2,$ and $\hyperparam\geq 0$,
\begin{equation}
\mathbb{E}[\gain_p^{\super}] - \mathbb{E}[\gain_p^{\bidbase}] = 0.
\label{eq:diag_bid_pgain}
\end{equation}

Before moving on to the expected reviewer-side gain, we point out that \super with zero heuristic and \bidbase show the same paper in the highest position to each reviewer almost surely as direct result of Lemma~\ref{lemma:diag_super_policy}, the definition of the \bidbase policy, and the derivations of the expected paper-side gain for each algorithm. 
\begin{property}
\super with zero heuristic and \bidbase show the same paper in the highest position of the paper ordering to each reviewer almost surely. 
\label{property1:bid}
\end{property}
As we now show, the suboptimality of \bidbase stems from the fact that after the paper shown in the highest position, the remaining papers are presented in increasing order of the number of bids, whereas \super with zero heuristic shows the remaining papers in decreasing order of the similarity scores.

\paragraph*{Bounding the expected reviewer-side gain.}
We now focus on showing \super with zero heuristic obtains significantly more expected reviewer-side gain than \bidbase. This requires deriving a suitable lower bound on the following expression:
\begin{equation}
\hyperparam\mathbb{E}[\gain_r^{\super} - \gain_r^{\bidbase}] = \hyperparam\mathbb{E}\Big[\sum_{\revindex\in \reviewerset}\sum_{\paperindex\in \paperset}(\gainfunctionrev(\ordering_{\revindex}^{\super}(\paperindex), \similarity_{\revindex, \paperindex})-\gainfunctionrev(\ordering_{\revindex}^{\bidbase}(\paperindex), \similarity_{\revindex, \paperindex}))\Big].
\label{eq:diag_rgain_bid}
\end{equation}
Let us begin by defining a ``good event'' for any reviewer and paper under which 
if the paper has probability zero of being bid on then it is not bid on and if the paper has probability one of being bid on then it is bid on. Formally, for any reviewer $k\in \reviewerset$, paper $\paperindex \in \paperset$, and paper ordering $\ordering_{k}^{\algbase}$ given by an algorithm \algbase, we define 
\begin{align*}
\mathcal{E}_{k, \paperindex}^{\algbase}=&\{\ordering_{k}^{\algbase}(\paperindex)=1, \similarity_{k, \paperindex}>\simscalar/2, \randombid_{k, \paperindex}=1\}\cup\{\ordering_{k}^{\algbase}(\paperindex)\neq 1,  \randombid_{k, \paperindex}=0\} \cup\{ \similarity_{k, \paperindex}<\simscalar/2, \randombid_{k, \paperindex}=0\}. 
\end{align*}
Moreover, for each reviewer $\revindex \in \reviewerset$, define the following event $\mathcal{E}_{\revindex}=\cup_{k=1}^{i-1}\cup_{\paperindex=1}^{\numpapers}\{\mathcal{E}_{k, \paperindex}^{\super}\cup \mathcal{E}_{k, \paperindex}^{\bidbase}\}$
which says the good event held for each reviewer that arrived previously for every paper and observe that the complement of this event occurs on a measure zero space by the structure of the bidding function given in~\eqref{eq:noiseless_bidmodel}.
Consequently, from the law of total expectation, an equivalent form of~\eqref{eq:diag_rgain_bid} is given by
\begin{equation}
\hyperparam\mathbb{E}[\gain_r^{\super} - \gain_r^{\bidbase}] = \hyperparam\mathbb{E}\Big[\sum_{\revindex\in \reviewerset}\mathbb{E}\Big[\sum_{\paperindex\in \paperset}(\gainfunctionrev(\ordering_{\revindex}^{\super}(\paperindex), \similarity_{\revindex, \paperindex})-\gainfunctionrev(\ordering_{\revindex}^{\bidbase}(\paperindex), \similarity_{\revindex, \paperindex}))\Big|\mathcal{E}_{\revindex}\Big]\Big].
\label{eq:diag_rgain_bid2}
\end{equation}

From Property~\ref{property1:bid}, \super with zero heuristic and \bidbase show the same paper in the highest position of the paper ordering to any reviewer $\revindex\in \reviewerset$ given the event $\mathcal{E}_{\revindex}$. Moreover, from Property~\ref{property1}, \super with zero heuristic presents the remainder of the papers in a decreasing order of the similarity scores given the event $\mathcal{E}_{\revindex}$. This implies that for each reviewer $\revindex \in \reviewerset$, \super with zero heuristic obtains at least as much reviewer-side gain as \bidbase given the event $\mathcal{E}_{\revindex}$ since the reviewer-side gain function is increasing in the similarity score and decreasing in the position a paper is shown. Define $\mathcal{F}$ as the initial set of $\lfloor \numblocks\blocksize/4\rfloor$ reviewers for which upon arrival of such a reviewer $\revindex\in \mathcal{F}$ at least one paper on the block diagonal for the reviewer given by $\diagset_{\revindex}$ has received a bid previously, where we recall that \super with zero heuristic and \bidbase each obtain exactly one bid from each reviewer and on each paper almost surely as proved in the analysis of the expected paper-side gains. From~\eqref{eq:diag_rgain_bid2} and the fact that \super with zero heuristic obtains at least as much expected reviewer-side gain as \bidbase from each reviewer given the event $\mathcal{E}$, we obtain 
\begin{equation}
\hyperparam\mathbb{E}[\gain_r^{\super} - \gain_r^{\bidbase}] \geq \hyperparam\mathbb{E}\Big[\sum_{\revindex\in \mathcal{F}}\mathbb{E}\Big[\sum_{\paperindex\in \paperset}(\gainfunctionrev(\ordering_{\revindex}^{\super}(\paperindex), \similarity_{\revindex, \paperindex})-\gainfunctionrev(\ordering_{\revindex}^{\bidbase}(\paperindex), \similarity_{\revindex, \paperindex}))\Big|\mathcal{E}_{\revindex}\Big]\Big].
\label{eq:diag_rgain_bid3}
\end{equation}

We now separate papers into relevant groups defined upon arrival for each reviewer $\revindex \in \mathcal{F}$. Let $T_{\revindex, 1}$ be the set of papers containing only the paper that each algorithm shows in the highest position that has zero bids and belongs to the set $\diagset_{\revindex}$. Let $T_{\revindex, 2}$ denote the remaining set of papers in $\diagset_{\revindex}$ with zero bids and $T_{\revindex, 3}$ be the set of papers in $\diagset_{\revindex}$ with one bid. Denote by $T_{\revindex, 4}$ the set of papers in $\diagset_{\revindex}^c$ with zero bids and $T_{\revindex, 5}$ as the papers in $\diagset_{\revindex}^c$ with one bid. Moreover, we let $N_{\revindex, k} = |T_{\revindex, k}|$ for $k \in \{1, 2, 3, 4, 5\}$ denote the number of papers in each set and define $\ell_{\revindex}=N_{\revindex, 1}+N_{\revindex, 2}+1$. Accordingly,~\eqref{eq:diag_rgain_bid3} is equivalently 
\begin{equation}
\hyperparam\mathbb{E}[\gain_r^{\super} - \gain_r^{\bidbase}] \geq \hyperparam\mathbb{E}\Big[\sum_{\revindex\in \mathcal{F}}\mathbb{E}\Big[\sum_{\paperindex\in T_{\revindex, 1}\cup T_{\revindex, 2}\cup T_{\revindex, 3}\cup T_{\revindex, 4}\cup T_{\revindex, 5}}(\gainfunctionrev(\ordering_{\revindex}^{\super}(\paperindex), \similarity_{\revindex, \paperindex})-\gainfunctionrev(\ordering_{\revindex}^{\bidbase}(\paperindex), \similarity_{\revindex, \paperindex}))\Big|\mathcal{E}_{\revindex}\Big]\Big].
\label{eq:diag_rgain_bid4}
\end{equation}
From Property~\ref{property1:bid}, \super with zero heuristic and \bidbase show the same paper in the highest position of the paper ordering to each reviewer $\revindex\in \reviewerset$ given $\mathcal{E}_{\revindex}$. This allows us to simplify~\eqref{eq:diag_rgain_bid4} and get that
\begin{equation}
\hyperparam\mathbb{E}[\gain_r^{\super} - \gain_r^{\bidbase}] \geq \hyperparam\mathbb{E}\Big[\sum_{\revindex\in \mathcal{F}}\mathbb{E}\Big[\sum_{\paperindex\in T_{\revindex, 2}\cup T_{\revindex, 3}\cup T_{\revindex, 4}\cup T_{\revindex, 5}}(\gainfunctionrev(\ordering_{\revindex}^{\super}(\paperindex), \similarity_{\revindex, \paperindex})-\gainfunctionrev(\ordering_{\revindex}^{\bidbase}(\paperindex), \similarity_{\revindex, \paperindex}))\Big|\mathcal{E}_{\revindex}\Big]\Big].
\label{eq:diag_rgain_bid5}
\end{equation}

Given event $\mathcal{E}_{\revindex}$, \super with zero heuristic shows papers among $T_{\revindex, 2}\cup T_{\revindex, 3}$ in decreasing order of the similarity scores followed by papers among $T_{\revindex, 4}\cup T_{\revindex, 5}$ in decreasing order of the similarity scores consequent of Properties~\ref{property1},~\ref{property2}, and~\ref{property3}. 

Now consider an algorithm \algbase that shows papers from $T_{\revindex, k}$ ahead of papers from $T_{\revindex, k+1}$ for each $k\in \{1, 2, 3,4\}$. Moreover, let this algorithm present papers among each group $T_{\revindex, k}$ for $k\in \{1, 2, 3, 4, 5\}$ in decreasing order of the similarity scores. The given reviewer-side gain function is decreasing in the position a paper is shown and increasing in the similarity score. This means the expected reviewer-side gain from any reviewer is maximized by showing papers in decreasing order of the similarity scores. Consequently, the expected reviewer-side gain of \super with zero heuristic from each reviewer is at least as much as that from \algbase. This fact leads to a lower bound on~\eqref{eq:diag_rgain_bid5} of 
\begin{equation}
\hyperparam\mathbb{E}[\gain_r^{\super} - \gain_r^{\bidbase}] \geq \hyperparam\mathbb{E}\Big[\sum_{\revindex\in \mathcal{F}}\mathbb{E}\Big[\sum_{\paperindex\in T_{\revindex, 2}\cup T_{\revindex, 3}\cup T_{\revindex, 4}\cup T_{\revindex, 5}}(\gainfunctionrev(\ordering_{\revindex}^{\algbase}(\paperindex), \similarity_{\revindex, \paperindex})-\gainfunctionrev(\ordering_{\revindex}^{\bidbase}(\paperindex), \similarity_{\revindex, \paperindex}))\Big|\mathcal{E}_{\revindex}\Big]\Big]
\label{eq:diag_rgain_bid6}
\end{equation}
where now $\mathcal{E}_{\revindex}=\cup_{k=1}^{i-1}\cup_{\paperindex=1}^{\numpapers}\{\mathcal{E}_{k, \paperindex}^{\algbase}\cup \mathcal{E}_{k, \paperindex}^{\bidbase}\}$.

The \bidbase policy shows papers in $T_{\revindex, 2} \cup T_{\revindex, 4}$ ahead of papers in $T_{\revindex, 3} \cup T_{\revindex, 5}$. Moreover, papers in $T_{\revindex, 2}$ are shown ahead of papers in $T_{\revindex, 4}$ and papers in $T_{\revindex, 3}$ ahead of papers in $T_{\revindex, 5}$. Papers among each group $T_{\revindex, k}$ for $k\in \{2, 3, 4, 5\}$ are shown in decreasing order of the similarity scores. 
This characterization of the \bidbase policy follows from definition, since papers are shown in increasing order of the number of bids with ties broken by the similarity scores. Recall that the similarity scores of papers in $T_{\revindex, 2} \cup T_{\revindex, 3}$  are greater than the similarity scores of papers in $T_{\revindex, 4}\cup T_{\revindex, 5}$ from Lemma~\ref{lemma:compare}. It is now clear that \algbase and \bidbase show papers among $T_{\revindex, 2}\cup T_{\revindex, 5}$ in identical positions given the event $\mathcal{E}_{\revindex}$. Combining this fact with~\eqref{eq:diag_rgain_bid6}, we get that 
\begin{equation}
\hyperparam\mathbb{E}[\gain_r^{\super} - \gain_r^{\bidbase}] \geq \hyperparam\mathbb{E}\Big[\sum_{\revindex\in \mathcal{F}}\mathbb{E}\Big[\sum_{\paperindex\in T_{\revindex, 3}\cup T_{\revindex, 4}}(\gainfunctionrev(\ordering_{\revindex}^{\algbase}(\paperindex), \similarity_{\revindex, \paperindex})-\gainfunctionrev(\ordering_{\revindex}^{\bidbase}(\paperindex), \similarity_{\revindex, \paperindex}))\Big|\mathcal{E}_{\revindex}\Big]\Big].
\label{eq:diag_rgain_bid7}
\end{equation}

We now separate the sum over papers in $T_{\revindex, 3}$ from the sums over papers in $T_{\revindex, 4}$ in~\eqref{eq:diag_rgain_bid7} to obtain
\begin{equation}
\begin{split}
\hyperparam\mathbb{E}[\gain_r^{\super} - \gain_r^{\bidbase}] &\geq \hyperparam\mathbb{E}\Big[\sum_{\revindex\in \mathcal{F}}\mathbb{E}\Big[\sum_{\paperindex\in T_{\revindex, 3}}(\gainfunctionrev(\ordering_{\revindex}^{\algbase}(\paperindex), \similarity_{\revindex, \paperindex})-\gainfunctionrev(\ordering_{\revindex}^{\bidbase}(\paperindex), \similarity_{\revindex, \paperindex}))\Big|\mathcal{E}_{\revindex}\Big]\\
&\qquad\quad+\sum_{\revindex\in \mathcal{F}}\mathbb{E}\Big[\sum_{\paperindex\in T_{\revindex, 4}}(\gainfunctionrev(\ordering_{\revindex}^{\algbase}(\paperindex), \similarity_{\revindex, \paperindex})-\gainfunctionrev(\ordering_{\revindex}^{\bidbase}(\paperindex), \similarity_{\revindex, \paperindex}))\Big|\mathcal{E}_{\revindex}\Big]\Big].
\end{split}
\label{eq:diag_rgain_bid8}
\end{equation}
The \algbase policy shows papers in $T_{\revindex, 3}$ followed by papers in $T_{\revindex, 4}$, with each group of papers being presented in decreasing order of the similarity scores to each reviewer given the event $\mathcal{E}_{\revindex}$. In contrast, the \bidbase policy shows papers in $T_{\revindex, 4}$ followed by papers in $T_{\revindex, 3}$, with each group of papers being presented in decreasing order of the similarity scores. This means that \bidbase shows each paper in $T_{\revindex, 3}$ later in the paper ordering by $N_{\revindex, 4}$ positions compared to \algbase to each reviewer given $\mathcal{E}_{\revindex}$. Analogously, \algbase shows each paper in $T_{\revindex, 4}$ later in the paper ordering by $N_{\revindex, 3}$ positions compared to \bidbase to each reviewer given $\mathcal{E}_{\revindex}$. This set of facts and continuing from~\eqref{eq:diag_rgain_bid4}, leads to the bound
\begin{equation}
\begin{split}
\hyperparam\mathbb{E}[\gain_r^{\super} - \gain_r^{\bidbase}] &\geq \hyperparam\mathbb{E}\Big[\sum_{\revindex\in \mathcal{F}}\mathbb{E}\Big[\sum_{\paperindex\in T_{\revindex, 3}}(\gainfunctionrev(\ordering_{\revindex}^{\algbase}(\paperindex), \similarity_{\revindex, \paperindex})-\gainfunctionrev(\ordering_{\revindex}^{\algbase}(\paperindex)+N_{\revindex, 4}, \similarity_{\revindex, \paperindex}))\Big|\mathcal{E}_{\revindex}\Big]\\
&\qquad\quad-\sum_{\revindex\in \mathcal{F}}\mathbb{E}\Big[\sum_{\paperindex\in T_{\revindex, 4}}(\gainfunctionrev(\ordering_{\revindex}^{\bidbase}(\paperindex), \similarity_{\revindex, \paperindex})-\gainfunctionrev(\ordering_{\revindex}^{\bidbase}(\paperindex)+N_{\revindex, 3}, \similarity_{\revindex, \paperindex}))\Big|\mathcal{E}_{\revindex}\Big]\Big].
\end{split}
\label{eq:diag_rgain_bid9}
\end{equation}
From the decomposed form of the reviewer-side gain function in~\eqref{eq:noiseless_revgain_ordering}, an equivalent form of~\eqref{eq:diag_rgain_bid9} is 
\begin{equation}
\begin{split}
\hyperparam\mathbb{E}[\gain_r^{\super} - \gain_r^{\bidbase}] &\geq \hyperparam\mathbb{E}\Big[\sum_{\revindex\in \mathcal{F}}\mathbb{E}\Big[\sum_{\paperindex\in T_{\revindex, 3}}(2^{\similarity_{\revindex, \paperindex}}-1)(\gainfunctionrev^{\ordering}(\ordering_{\revindex}^{\algbase}(\paperindex))-\gainfunctionrev^{\ordering}(\ordering_{\revindex}^{\algbase}(\paperindex)+N_{\revindex, 4}))\Big|\mathcal{E}_{\revindex}\Big]\\
&\qquad\quad-\sum_{\revindex\in \mathcal{F}}\mathbb{E}\Big[\sum_{\paperindex\in T_{\revindex, 4}}(2^{\similarity_{\revindex, \paperindex}}-1)(\gainfunctionrev^{\ordering}(\ordering_{\revindex}^{\bidbase}(\paperindex))-\gainfunctionrev^{\ordering}(\ordering_{\revindex}^{\bidbase}(\paperindex)+N_{\revindex, 3}))\Big|\mathcal{E}_{\revindex}\Big]\Big].
\end{split}
\label{eq:diag_rgain_bid10}
\end{equation}

The similarity scores of papers in $T_{\revindex, 3}$ are given by $\similarity_{\revindex, \paperindex} = \simscalar-\nu_{\revindex, \paperindex}$ and the similarity score of papers in $T_{\revindex, 4}$ are $\similarity_{\revindex, \paperindex} = \nu_{\revindex, \paperindex}$. Recall that the noise is bounded in the interval $(0, \unifnoise)$. Combining this with the fact that the function $\gainfunctionrev^{\ordering}$ from~\eqref{eq:noiseless_revgain_ordering} is decreasing on the domain $\mathbb{R}_{> 0}$, we bound~\eqref{eq:diag_rgain_bid10} as follows
\begin{equation}
\begin{split}
\hyperparam\mathbb{E}[\gain_r^{\super} - \gain_r^{\bidbase}] &\geq \hyperparam\mathbb{E}\Big[(2^{\simscalar-\unifnoise}-1)\sum_{\revindex\in \mathcal{F}}\mathbb{E}\Big[\sum_{\paperindex\in T_{\revindex, 3}}(\gainfunctionrev^{\ordering}(\ordering_{\revindex}^{\algbase}(\paperindex))-\gainfunctionrev^{\ordering}(\ordering_{\revindex}^{\algbase}(\paperindex)+N_{\revindex, 4}))\Big|\mathcal{E}_{\revindex}\Big]\\
&\qquad\quad-(2^{\unifnoise}-1)\sum_{\revindex\in \mathcal{F}}\mathbb{E}\Big[\sum_{\paperindex\in T_{\revindex, 4}}(\gainfunctionrev^{\ordering}(\ordering_{\revindex}^{\bidbase}(\paperindex))-\gainfunctionrev^{\ordering}(\ordering_{\revindex}^{\bidbase}(\paperindex)+N_{\revindex, 3}))\Big|\mathcal{E}_{\revindex}\Big]\Big].
\end{split}
\label{eq:diag_rgain_bid11}
\end{equation}
Now recall that \algbase shows the papers in $T_{\revindex, 3}$ after the papers in $T_{\revindex, 1}\cup T_{\revindex, 2}$. Similarly, \bidbase shows the papers in $T_{\revindex, 4}$ after the papers in $T_{\revindex, 1}\cup T_{\revindex, 2}$. Recall that $\ell_{\revindex} = N_{\revindex, 1}+N_{\revindex, 2}+1$. From this notation and~\eqref{eq:diag_rgain_bid11}, we obtain
\begin{equation}
\begin{split}
\hyperparam\mathbb{E}[\gain_r^{\super} - \gain_r^{\bidbase}] &\geq \hyperparam\mathbb{E}\Big[(2^{\simscalar-\unifnoise}-1)\sum_{\revindex\in \mathcal{F}}\mathbb{E}\Big[\sum_{\paperindex=\ell_{\revindex}}^{\ell_{\revindex}+N_{\revindex, 3}-1}(\gainfunctionrev^{\ordering}(\paperindex)-\gainfunctionrev^{\ordering}(\paperindex+N_{\revindex, 4}))\Big|\mathcal{E}_{\revindex}\Big]\\
&\qquad\quad-(2^{\unifnoise}-1)\sum_{\revindex\in \mathcal{F}}\mathbb{E}\Big[\sum_{\paperindex=\ell_{\revindex}}^{\ell_{\revindex}+N_{\revindex, 4}-1}(\gainfunctionrev^{\ordering}(\paperindex)-\gainfunctionrev^{\ordering}(\paperindex+N_{\revindex, 3}))\Big|\mathcal{E}_{\revindex}\Big]\Big].
\end{split}
\label{eq:diag_rgain_bid12}
\end{equation}

Following the exact techniques to prove Claim 1 in Section~\ref{sec:noiseless_bid} for the expected reviewer-side gain analysis of \bidbase in the noiseless community model, we get that for each reviewer $\revindex \in \mathcal{F}$ and conditioned on the event $\mathcal{E}_{\revindex}$,
\begin{equation}
N_{\revindex, 4}\geq \numblocks\blocksize-\blocksize-\numblocks-\lfloor\numblocks\blocksize/4\rfloor+N_{\revindex, 3}+1 \geq N_{\revindex, 3}\geq 1.
\label{eq:noisy_diag_rgain_bid19}
\end{equation}
Now, toward the goal of bounding~\eqref{eq:diag_rgain_bid12}, we perform the following indexing manipulations:
\begin{align}
\sum_{\paperindex=\ell_{\revindex}}^{\ell_{\revindex}+N_{\revindex, 4}-1} (\gainfunctionrev^{\ordering}(\paperindex)-\gainfunctionrev^{\ordering}(\paperindex+N_{\revindex, 3})) 
&= \sum_{\paperindex=\ell_{\revindex}}^{\ell_{\revindex}+N_{\revindex, 4}-1} \gainfunctionrev^{\ordering}(\paperindex) - \sum_{\paperindex=\ell_{\revindex}+N_{\revindex, 3}}^{\ell_{\revindex}+N_{\revindex, 3}+N_{\revindex, 4}-1} \gainfunctionrev^{\ordering}(\paperindex) \notag \\
&= \sum_{\paperindex=\ell_{\revindex}}^{\ell_{\revindex}+N_{\revindex, 3}-1} \gainfunctionrev^{\ordering}(\paperindex) - \sum_{\paperindex=\ell_{\revindex}+N_{\revindex, 4}}^{\ell_{\revindex}+N_{\revindex, 3}+N_{\revindex, 4}-1} \gainfunctionrev^{\ordering}(\paperindex) \label{eq:diag_rgain_bid13} \\
&= \sum_{\paperindex=\ell_{\revindex}}^{\ell_{\revindex}+N_{\revindex, 3}-1} (\gainfunctionrev^{\ordering}(\paperindex) - \gainfunctionrev^{\ordering}(\paperindex+N_{\revindex, 4})).  \label{eq:diag_rgain_bid14}
\end{align}
To obtain~\eqref{eq:diag_rgain_bid13}, we used the fact from~\eqref{eq:noisy_diag_rgain_bid19} that $N_{\revindex, 4}\geq N_{\revindex, 3}$ for any reviewer $\revindex \in \mathcal{F}$ given the event $\mathcal{E}_{\revindex}$. 

Then from~\eqref{eq:diag_rgain_bid11} and~\eqref{eq:diag_rgain_bid14}, we get 
\begin{equation}
\begin{split}
\hyperparam\mathbb{E}[\gain_r^{\super} - \gain_r^{\bidbase}] &\geq \hyperparam\mathbb{E}\Big[(2^{\simscalar-\unifnoise}-2^{\simscalar})\sum_{\revindex\in \mathcal{F}}\mathbb{E}\Big[\sum_{\paperindex=\ell_{\revindex}}^{\ell_{\revindex}+N_{\revindex, 3}-1}(\gainfunctionrev^{\ordering}(\paperindex)-\gainfunctionrev^{\ordering}(\paperindex+N_{\revindex, 4}))\Big|\mathcal{E}_{\revindex}\Big]\Big].
\end{split}
\label{eq:pre_min_step}
\end{equation}
Minimizing over $\revindex\in \mathcal{F}$ in~\eqref{eq:pre_min_step} and using the definition $|\mathcal{F}| = \lfloor\numblocks\blocksize/4\rfloor$, we have 
\begin{equation}
\hyperparam\mathbb{E}[\gain_r^{\optbase} - \gain_r^{\bidbase}] \geq \hyperparam\mathbb{E}\Big[(2^{\simscalar}-1)(\lfloor \numblocks\blocksize/4\rfloor)\min_{\revindex\in \mathcal{F}}\mathbb{E}\Big[\sum_{\paperindex=\ell_{\revindex}}^{\ell_{\revindex}+N_{\revindex, 3}-1}(\gainfunctionrev^{\ordering}(\paperindex)-\gainfunctionrev^{\ordering}(\paperindex+N_{\revindex, 4}))\Big|\mathcal{E}_{\revindex}\Big]\Big].
\label{eq:min_step}
\end{equation}
Moreover, for every $\numblocks\geq2$, $\blocksize\geq2$, it holds that
\begin{equation}
\lfloor \numblocks\blocksize/4 \rfloor=  \numblocks\blocksize/4 - (\numblocks\blocksize \ \mathrm{mod} \ 4)/4 \geq \numblocks\blocksize/8.
\label{eq:min_step1}
\end{equation}
By definition of the noisy community model and the given bound on $\unifnoise$, for every $\numblocks\geq2$, $\blocksize\geq2$, and $\hyperparam\geq 0$, we get
\begin{equation}
2^{\simscalar-\unifnoise}-2^{\unifnoise} = 2^{\simscalar-(1+\hyperparam)^{-1}e^{-e\numblocks\blocksize}}-2^{(1+\hyperparam)^{-1}e^{-e\numblocks\blocksize}}   \geq 2^{0.01-e^{-4e}}-2^{e^{-4e}} \geq 1/150.
\label{eq:min_step2}
\end{equation}
Combining~\eqref{eq:min_step},~\eqref{eq:min_step1}, and~\eqref{eq:min_step2}, we have
\begin{equation}
\hyperparam\mathbb{E}[\gain_r^{\super} - \gain_r^{\bidbase}] \geq \Big(\frac{\hyperparam}{1200}\Big)\mathbb{E}\Big[\min_{\revindex\in \mathcal{F}}\mathbb{E}\Big[\sum_{\paperindex=\ell_{\revindex}}^{\ell_{\revindex}+N_{\revindex, 3}-1}(\gainfunctionrev^{\ordering}(\paperindex)-\gainfunctionrev^{\ordering}(\paperindex+N_{\revindex, 4}))\Big|\mathcal{E}_{\revindex}\Big]\Big].
\label{eq:noisy_pre_final_thing}
\end{equation}
Then, applying exactly the same techniques to prove Claim 2 in Section~\ref{sec:noiseless_bid} for the expected reviewer-side gain analysis of \bidbase in the noiseless community model, for every $\revindex\in \mathcal{F}$ conditioned on the event $\mathcal{E}_{\revindex}$, we have
\begin{equation}
\min_{\revindex\in \mathcal{F}}\sum_{\paperindex=\ell_{\revindex}}^{\ell_{\revindex}+N_{\revindex, 3}-1}(\gainfunctionrev^{\ordering}(\paperindex)-\gainfunctionrev^{\ordering}(\paperindex+N_{\revindex, 4})) \geq \Big(\frac{2}{5}\Big)\Big(\frac{1}{\log_2^2(\numblocks\blocksize)}\Big).
\label{eq:noisy_final_thing}
\end{equation}
Finally, combining~\eqref{eq:noisy_final_thing} with~\eqref{eq:noisy_pre_final_thing}, 
for every $\numblocks\geq2, \blocksize\geq 2$, and $\hyperparam\geq 0$, the following bound holds
\begin{equation}
\hyperparam\mathbb{E}[\gain_r^{\optbase} - \gain_r^{\bidbase}] \geq \Big(\frac{1}{3000}\Big)\Big(\frac{\hyperparam\numblocks\blocksize}{\log_2^2(\numblocks\blocksize)}\Big).
\label{eq:diag_rgain_bid282}
\end{equation}
Observe that the expectation in the right-hand side of~\eqref{eq:diag_rgain_bid282} is dropped since it is not a random variable.

\paragraph*{Completing the bound.}
Combining the bounds on the expected paper-side and reviewer-side gain between \super with zero heuristic and \bidbase given in~\eqref{eq:diag_bid_pgain} and~\eqref{eq:diag_rgain_bid282}, we find for every $\numblocks\geq2, \blocksize\geq 2, \hyperparam\geq 0$,
\begin{align*}
\mathbb{E}[\gain^{\super}-\gain^{\bidbase}] &= \mathbb{E}[\gain_p^{\super}-\gain_p^{\bidbase}] + \hyperparam\mathbb{E}[\gain_r^{\super}-\gain_r^{\bidbase}] \\
&\geq \Big(\frac{1}{3000}\Big)\Big(\frac{\hyperparam\numblocks\blocksize}{\log_2^2(\numblocks\blocksize)}\Big).
\end{align*}
We conclude that there exists a constant $\constant >0$ such that for every $\numblocks\geq 2, \blocksize\geq 2$, and $\hyperparam\geq 0$, \super with zero heuristic obtains an additive factor of at least $\constant\hyperparam\numblocks\blocksize/\log_2^2(\numblocks\blocksize)$ more expected gain than \bidbase for the noisy community model.

\subsubsection{Suboptimality of \randbase}
\label{sec:diag_rand}
In this section, we analyze \randbase for the noisy community model similarity class with the given gain and bidding functions. A significant amount of the analysis in this section follows identically to that from analyzing \randbase in the noiseless community model from Section~\ref{sec:noiseless_rand} and the reason for the suboptimal behavior is identical. 

\paragraph*{Bounding the expected paper-side gain.}
Recall from~\eqref{eq:decoupled} that the expected paper-side gain from any paper $\paperindex\in \paperset$ is given by
\begin{equation}
\mathbb{E}[\gainfunction(\numbids_{\paperindex})] 
=  \sum_{\ell=0}^{\blocksize}\mathbb{P}\Big(\ell=\sum_{\revindex\in \diagsetp_{\paperindex}}\mathds{1}\{\ordering_{\revindex}^{\randbase}(\paperindex)=1\}\Big)\gainfunction(\ell).
\label{eq:rand_diag_pgain}
\end{equation}
We remark that the decomposition of the expected paper-side gain from any paper given in~\eqref{eq:rand_diag_pgain} for \randbase in the noisy community model is identical to that given in~\eqref{eq:nrand_diag_pgain} for \randbase in the noiseless community model. Since the \randbase policy is independent of the similarity scores and the reviewer bids, the distribution of the number of times a paper is shown in the highest position to reviewers for which it is on the block diagonal is identical in the noisy community model as it is in the noiseless community model. Accordingly, we directly bound the expected paper-side gain of \randbase in the noisy community model using the bound from~\eqref{eq:nrand_diag_pgain3} derived in Section~\ref{sec:noiseless_rand} for \randbase in the noiseless community model. 
Then, combining with the expected paper-side gain of \super with zero heuristic from~\eqref{eq:diag_super_pgain}, we get that for every $\numblocks\geq2, \blocksize\geq 2$, and $\hyperparam\geq 0$, 
\begin{equation}
\mathbb{E}[\gain_p^{\super} - \gain_p^{\randbase}] \geq \numblocks\blocksize - \Big(\frac{6+\sqrt{2}}{16}\Big)\numblocks\blocksize.
\label{eq:rand_diag_pgain4}
\end{equation}

\paragraph*{Bounding the expected reviewer-side gain.}
We now need to compare the expected reviewer-side gain of \super with zero heuristic and \randbase. The \simbase algorithm obtains the maximum expected reviewer-side gain that can be achieved since the reviewer-side gain function is increasing in the similarity score and decreasing in the position a paper is shown. Consequently, the bound on the expected reviewer-side gain from~\eqref{eq:supersimrevbound} between \super with zero heuristic and \simbase applies to \randbase. Using the bound from~\eqref{eq:supersimrevbound}, we get that for every $\numblocks\geq2, \blocksize\geq2,$ and $\hyperparam\geq 0$,
\begin{equation}
\hyperparam\mathbb{E}[\gain_r^{\super}-\gain_r^{\randbase}] \geq -0.0001.
\label{eq:diag_rand_rgain}
\end{equation}

\paragraph*{Completing the bound.}
Combining the bounds on the expected paper-side and reviewer-side gain between \super with zero heuristic and \randbase given in~\eqref{eq:rand_diag_pgain4} and~\eqref{eq:diag_rand_rgain}, for every $\numblocks\geq2, \blocksize\geq 2$, and $\hyperparam\geq 0$,
\begin{align*}
\mathbb{E}[\gain^{\super}-\gain^{\randbase}] &= \mathbb{E}[\gain_p^{\super}-\gain_p^{\randbase}] + \hyperparam\mathbb{E}[\gain_r^{\super}-\gain_r^{\randbase}] \\
&\geq \numblocks\blocksize - \Big(\frac{6+\sqrt{2}}{16}\Big)\numblocks\blocksize-0.0001\geq \numblocks\blocksize/2.
\end{align*}
We conclude that there exists a constant $\constant >0$ such that for every $\numblocks\geq 2, \blocksize\geq 2$, and $\hyperparam\geq 0$, \super with zero heuristic obtains an additive factor of at least $\constant \numblocks\blocksize$ more expected gain than \randbase in the noisy community model.

\subsubsection{Near-Optimality of \super with Zero Heuristic}
\label{sec:diag_opt}
In this section, we show that \super with zero heuristic is nearly optimal. We let \optbase denote the optimal algorithm for the expected gain. 

\paragraph*{Bounding the expected paper-side gain.}
As explained in Section~\ref{sec:diag_super}, the expected paper-side gain of \super with zero heuristic is the maximum that can be achieved. Indeed, this is consequent of the facts that each reviewer bids on at most one paper almost surely and the given paper-side gain function is strictly concave so evenly distributing the bids over the papers maximizes the expected paper-side gain. We conclude that for every $\numblocks\geq2, \blocksize\geq2,$ and $\hyperparam\geq 0$,
\begin{equation}
\mathbb{E}[\gain_p^{\super}] - \mathbb{E}[\gain_p^{\optbase}] \geq 0.
\label{eq:diag_opt_pgain}
\end{equation}

\paragraph*{Bounding the expected reviewer-side gain.}
The \simbase algorithm obtains the maximum expected reviewer-side gain that can be achieved since the reviewer-side gain function as given in~\eqref{eq:noiseless_rgainfunc} is increasing in the similarity score and decreasing in the position a paper is shown, which means showing papers in decreasing order of the similarity scores to each reviewer maximizes the expected reviewer-side gain. Consequently, the bound on the expected reviewer-side gain from~\eqref{eq:supersimrevbound} between \super with zero heuristic and \simbase applies to the optimal algorithm. Using the bound from~\eqref{eq:supersimrevbound}, we get that for  $\numblocks\geq2, \blocksize\geq2,$ and $\hyperparam\geq 0$,
\begin{equation}
\hyperparam\mathbb{E}[\gain_r^{\super}-\gain_r^{\optbase}] \geq -0.0001.
\label{eq:diag_opt_rgain}
\end{equation}

\paragraph*{Completing the bound.}
Combining the bounds on the expected paper-side and reviewer-side gain between \super with zero heuristic and \optbase given in~\eqref{eq:diag_opt_pgain} and~\eqref{eq:diag_opt_rgain}, we get that for every $\numblocks\geq2, \blocksize\geq 2$, and $\hyperparam\geq 0$,
\begin{equation*}
\mathbb{E}[\gain^{\super}-\gain^{\optbase}] = \mathbb{E}[\gain_p^{\super}-\gain_p^{\optbase}] + \hyperparam\mathbb{E}[\gain_r^{\super}-\gain_r^{\optbase}] 
\geq -0.0001.
\end{equation*}
We conclude \super with zero heuristic is always within at least an additive factor of $0.0001$ of the optimal in the noisy community model.

\subsubsection{Proofs of Lemmas~\ref{lemma:diag_super_policy}--\ref{lemma:bound_sim_sum}}
\label{sec:diag_lemmas}
In this section, we present the proofs of technical lemmas invoked in the primary proof of Theorem~\ref{thm:diagonal}.
\paragraph{Proof of Lemma~\ref{lemma:diag_super_policy}.}\label{sec:diag_super_policy}
In the proof of Corollary~\ref{sec:proof_local_col} given in Section~\ref{sec:proof_local_col}, we showed in~\eqref{eq:cor_opt2} that  
\super with zero heuristic solves the problem
\begin{equation}
\ordering_{\revindex}^{\super} = \argmax_{\ordering_{\revindex}\in \symgroup_\numpapers} \quad \sum_{\paperindex\in \paperset} \bidfunction(\ordering_{\revindex}(\paperindex), \similarity_{\revindex, \paperindex})(\gainfunction(\numbids_{\revindex-1, \paperindex}  + 1) - \gainfunction(\numbids_{\revindex-1, \paperindex})) + \hyperparam \sum_{\paperindex\in \paperset}\gainfunctionrev(\ordering_{\revindex}(\paperindex), \similarity_{\revindex, \paperindex}) 
\label{eq:diagonal_policy}
\end{equation}
in order to determine the ordering of papers $\ordering_{\revindex}^{\super}$ to present to reviewer $\revindex \in \reviewerset$ so that the immediate expected gain is maximized conditioned on the history of bids from reviewers that arrived previously. Recalling that the bidding function is $\bidfunction(\ordering_{\revindex}(\paperindex), \similarity_{\revindex, \paperindex}) = \mathds{1}\{\ordering_{\revindex}(\paperindex)=1\}\mathds{1}\{\similarity_{\revindex, \paperindex} > \simscalar/2\}$, the optimization problem in~\eqref{eq:diagonal_policy} is equivalent to 
\begin{equation}
\ordering_{\revindex}^{\super} = \argmax_{\ordering_{\revindex}\in \symgroup_\numpapers}  \sum_{\paperindex\in \paperset}\mathds{1}\{\ordering_{\revindex}(\paperindex)=1\}\mathds{1}\{\similarity_{\revindex, \paperindex} > \simscalar/2\}(\gainfunction(\numbids_{\revindex-1, \paperindex}  + 1) - \gainfunction(\numbids_{\revindex-1, \paperindex})) + \hyperparam \sum_{\paperindex\in \paperset}\gainfunctionrev(\ordering_{\revindex}(\paperindex), \similarity_{\revindex, \paperindex}).
\label{eq:diagonal_policy1}
\end{equation}
Observe that $\diagset_{\revindex}\cup \diagset_{\revindex}^c=\paperset$. Moreover, if $\paperindex\in \diagset_{\revindex}$, then $\similarity_{\revindex, \paperindex}>\simscalar/2$ from Lemma~\ref{lemma:exceed}. Analogously, if $\paperindex\in \diagset_{\revindex}^c$, then $\similarity_{\revindex, \paperindex}<\simscalar/2$ from Lemma~\ref{lemma:under}. This allows us to simplify~\eqref{eq:diagonal_policy1} to the following problem:
\begin{equation}
\ordering_{\revindex}^{\super} = \argmax_{\ordering_{\revindex} \in \symgroup_\numpapers} \quad \sum_{\paperindex\in \diagset_{\revindex}}\mathds{1}\{\ordering_{\revindex}(\paperindex)=1\}(\gainfunction(\numbids_{\revindex-1, \paperindex}  + 1) - \gainfunction(\numbids_{\revindex-1, \paperindex})) + \hyperparam \sum_{\paperindex\in \paperset}\gainfunctionrev(\ordering_{\revindex}(\paperindex), \similarity_{\revindex, \paperindex}).
\label{eq:diagonal_policy2}
\end{equation}

Given the assumption that there is a paper in $\diagset_{\revindex}$ with zero bids and each paper in $\diagset_{\revindex}$ has at most one bid, we need to prove \super with zero heuristic shows the paper with the maximum similarity score among the papers without a bid in $\diagset_{\revindex}$ followed by the remaining papers in a decreasing order of the similarity scores.
To do so, we analyze the solution to~\eqref{eq:diagonal_policy2} when the paper with the maximum similarity score has zero bids and when the paper with the maximum similarity score has one bid. 
For each scenario, we show \super with zero heuristic presents the paper with the maximum similarity score among the papers without a bid in the highest position and the remaining papers in a decreasing order of the similarity scores. This is equivalent to the stated result we seek to prove since from Lemma~\ref{lemma:compare}, $\similarity_{\revindex, \paperindex}>\similarity_{\revindex, \paperindex'}$ for $\paperindex\in \diagset_{\revindex}, \paperindex'\in \diagset_{\revindex}^c$, which guarantees the paper with the maximum similarity score belongs to the set $\diagset_{\revindex}$ and the paper with the maximum similarity score among the papers without a bid belongs to the set $\diagset_{\revindex}$.

Before analyzing each scenario, we recall some key properties of the functions in the optimization problem given in~\eqref{eq:diagonal_policy2} under the assumptions.
The given paper-side gain function $\gainfunction$ is such that the quantity $\gainfunction(\numbids_{\revindex-1, \paperindex}  + 1) - \gainfunction(\numbids_{\revindex-1, \paperindex})$ is decreasing as a function of the number of bids $\numbids_{\revindex-1, \paperindex}$. As a result, the expected paper-side gain term from~\eqref{eq:diagonal_policy2}, which is given by
\begin{equation}
\sum_{\paperindex\in \diagset_{\revindex}}\mathds{1}\{\ordering_{\revindex}(\paperindex)=1\}(\gainfunction(\numbids_{\revindex-1, \paperindex}  + 1) - \gainfunction(\numbids_{\revindex-1, \paperindex})),
\label{eq:expect_pgain}
\end{equation}
is maximized by showing the paper $\paperindex \in \diagset_{\revindex}$ with the minimum number of bids in the highest position of the paper ordering. Moreover, the given reviewer-side gain function $\gainfunctionrev$ from~\eqref{eq:noiseless_rgainfunc} is decreasing in the position $\ordering_{\revindex}(\paperindex)$ in which a paper is shown  and increasing in the similarity score $\similarity_{\revindex, \paperindex}$. Consequently, the expected reviewer-side gain term from~\eqref{eq:diagonal_policy2}, which is given by
\begin{equation}
\sum_{\paperindex\in \paperset}\gainfunctionrev(\ordering_{\revindex}(\paperindex), \similarity_{\revindex, \paperindex}),
\label{eq:expect_rgain}
\end{equation}
is maximized by showing papers in decreasing order of the similarity scores.

\paragraph*{Solution when the paper with the maximum similarity score has zero bids.} 
If the paper with the maximum similarity score has zero bids, then the solution to~\eqref{eq:diagonal_policy2} is to present the papers in decreasing order of the similarity scores. To see why this solution is optimal, observe that it maximizes each component of~\eqref{eq:diagonal_policy2} given in~\eqref{eq:expect_pgain} and~\eqref{eq:expect_rgain} since the paper with the maximum similarity score has the minimum number of bids among the set $\diagset_{\revindex}$ and papers are in decreasing order of the similarity scores. This solution is equivalent to presenting the paper with the maximum similarity score among the papers without a bid in the highest position and the remaining papers in a decreasing order of the similarity scores since the paper with the maximum similarity score has zero bids. 

\paragraph*{Solution when the paper with the maximum similarity score has one bid.}
To determine the solution to~\eqref{eq:diagonal_policy2} when the paper with the maximum similarity score has one bid, we consider groups of candidate solutions. We group potential solutions into the set of paper orderings that show a paper with at least one bid in the highest position (group 1) and the set of paper orderings that show a paper without a bid in the highest position (group 2). For each group of paper orderings, we find the solution that maximizes the objective of the optimization problem in~\eqref{eq:diagonal_policy2}. To resolve which is optimal, we compare the objective values of the solutions from each group.

\textbf{Analyzing Group 1.} For this group, the solution is constrained to the set of paper orderings that show a paper with at least one bid in the highest position. The solution among this group that maximizes the objective of~\eqref{eq:diagonal_policy2} is to show papers in decreasing order of the similarity scores. We call this candidate solution $\ordering_{\revindex}^{\ell}$. 

The candidate solution $\ordering_{\revindex}^{\ell}$ can be seen to be optimal among the group since it maximizes~\eqref{eq:expect_pgain} subject to the constraint of the group and it maximizes~\eqref{eq:expect_rgain}. Indeed, solution $\ordering_{\revindex}^{\ell}$ maximizes~\eqref{eq:expect_pgain} subject to the constraint of the group since 
the paper with the maximum similarity score has the minimum number of bids among the papers in $\diagset_{\revindex}$ with at least one bid. Moreover, solution $\ordering_{\revindex}^{\ell}$ maximizes~\eqref{eq:expect_rgain} since papers are shown in decreasing order of the similarity scores.

\textbf{Analyzing Group 2.} For this group, the solution is constrained to the set of paper orderings that show a paper without a bid in the highest position. The solution among this group that maximizes the objective of~\eqref{eq:diagonal_policy2} is to show the paper with the maximum similarity score  among the papers with zero bids in the highest position and then present the remaining papers in decreasing order of the similarity scores. We call this candidate solution $\ordering_{\revindex}^{\ell'}$. 

The candidate solution $\ordering_{\revindex}^{\ell'}$ can be seen to be optimal among the group since it maximizes~\eqref{eq:expect_pgain} and it maximizes~\eqref{eq:expect_rgain} subject to the constraint of the group as we now show. From assumption, there is at least one paper in $\diagset_{\revindex}$ with zero bids. The similarity score of any paper in $\diagset_{\revindex}$ is greater than the similarity score of any paper in $\diagset_{\revindex}^c$ from Lemma~\ref{lemma:compare}. This implies that the paper with the maximum similarity score among the papers with zero bids is in $\diagset_{\revindex}$. Therefore, we conclude solution $\ordering_{\revindex}^{\ell'}$ maximizes~\eqref{eq:expect_pgain} since the paper with the maximum similarity score among the papers with zero bids is in $\diagset_{\revindex}$ and it has the minimum number of bids among the papers in $\diagset_{\revindex}$. Moreover, solution $\ordering_{\revindex}^{\ell'}$ maximizes~\eqref{eq:expect_rgain} subject to the constraint of the group since the paper with maximum similarity among the set of papers with zero bids is shown in the highest position and the remaining papers are shown in decreasing order of the similarity scores.

\textbf{Comparing candidate solutions $\ordering_{\revindex}^\ell$ and $\ordering_{\revindex}^{\ell'}$.} We now compare the objective of~\eqref{eq:diagonal_policy2} for the candidate solutions $\ordering_{\revindex}^\ell$ and $\ordering_{\revindex}^{\ell'}$. Our goal is to show the objective given $\ordering_{\revindex}^{\ell'}$ is greater than the objective given $\ordering_{\revindex}^\ell$. This is to say, we wish to show the following quantity is positive
\begin{equation}
\sum_{\paperindex\in \diagset_{\revindex}}\big(\mathds{1}\{\ordering_{\revindex}^{\ell'}(\paperindex)=1\}-\mathds{1}\{\ordering_{\revindex}^{\ell}(\paperindex)=1\}\big)(\gainfunction(\numbids_{\revindex-1, \paperindex}  + 1) - \gainfunction(\numbids_{\revindex-1, \paperindex})) + \hyperparam \sum_{\paperindex\in \paperset}\big(\gainfunctionrev(\ordering_{\revindex}^{\ell'}(\paperindex), \similarity_{\revindex, \paperindex})-\gainfunctionrev(\ordering_{\revindex}^{\ell}(\paperindex), \similarity_{\revindex, \paperindex})\big).
\label{eq:obj_diff}
\end{equation}
To simplify notation, let the quantity in~\eqref{eq:obj_diff} be denoted by $\mathcal{C}$. Since $\ordering_{\revindex}^\ell$ shows a paper in $\diagset_{\revindex}$ with one bid in the highest position and $\ordering_{\revindex}^{\ell'}$ shows a paper in $\diagset_{\revindex}$ with zero bids in the highest position, we obtain
\begin{equation*}
\mathcal{C} = (\gainfunction(1) - \gainfunction(0))-(\gainfunction(2) - \gainfunction(1)) + \hyperparam \sum_{\paperindex\in \paperset}\big(\gainfunctionrev(\ordering_{\revindex}^{\ell'}(\paperindex), \similarity_{\revindex, \paperindex})-\gainfunctionrev(\ordering_{\revindex}^{\ell}(\paperindex), \similarity_{\revindex, \paperindex})\big).
\end{equation*}
Since $\ordering_{\revindex}^{\ell}$ presents papers in decreasing order of the similarity scores and $\ordering_{\revindex}^{\ell'}$ shows the papers in $\diagset_{\revindex}^c$ in a decreasing order of the similarity scores after the papers in $\diagset_{\revindex}$, we can apply Lemma~\ref{lemma:bound_sim_sum} to get 
\begin{equation}
\mathcal{C} \geq (\gainfunction(1) - \gainfunction(0))-(\gainfunction(2) - \gainfunction(1)) -\blocksize e^{-e\numblocks\blocksize}\log(4).
\label{eq:seq1}
\end{equation}
Observe that $-\blocksize e^{-e\numblocks\blocksize}\log(4)$ is negative and an increasing as a function of $\numblocks$ and $\blocksize$ on the domain $\numblocks\geq2$ and $\blocksize\geq 2$. This means for every $\numblocks\geq 2$ and $\blocksize\geq 2$,
\begin{equation}
-\blocksize e^{-e\numblocks\blocksize}\log(4) \geq -2 e^{-4e}\log(4) \geq -0.01.
\label{eq:seq2}
\end{equation}
Moreover, for the given paper-side gain function,
\begin{equation}
(\gainfunction(1) - \gainfunction(0))-(\gainfunction(2) - \gainfunction(1)) = 2-\sqrt{2}
\label{eq:seq3}
\end{equation}
Combining~\eqref{eq:seq1},~\eqref{eq:seq2}, and~\eqref{eq:seq3}, we obtain
\begin{equation*}
\mathcal{C} \geq 2-\sqrt{2}-0.01 > 0.
\end{equation*}
Since $\mathcal{C}>0$, we can conclude the objective of~\eqref{eq:diagonal_policy2} given $\ordering_{\revindex}^{\ell'}$ is greater than the objective of~\eqref{eq:diagonal_policy2} given $\ordering_{\revindex}^\ell$. This means that the solution when the paper with the maximum similarity score has one bid is to show the paper with the maximum similarity score among the papers with zero bids in the highest position and then present the remaining papers in decreasing order of the similarity scores.

\paragraph*{Combining the solutions.} We have now derived the solution to~\eqref{eq:diagonal_policy2} when the paper with the maximum similarity score has not obtained a bid previously and when the paper with the maximum similarity score has obtained exactly one bid previously. For each scenario, we showed \super with zero heuristic presents the paper with the maximum similarity score among the papers without a bid in the highest position and the remaining papers in a decreasing order of the similarity scores. This allows us to conclude that if there is a paper in $\diagset_{\revindex}$ with zero bids and each paper in $\diagset_{\revindex}$ has at most one bid, then \super with zero heuristic shows the paper with the maximum similarity score among the papers without a bid in $\diagset_{\revindex}$ followed by the remaining papers in a decreasing order of the similarity scores.

\paragraph{Proof of Lemma~\ref{lemma:bound_sim_sum}.}\label{sec:bound_sim_sum}
From the stated result, $\ordering_{\revindex}^{\ell}$ denotes the paper ordering that presents papers in decreasing order of the similarity scores. Recall that $\diagset_{\revindex}$ denotes the set of papers on the block diagonal of the similarity matrix for any reviewer $\revindex \in \reviewerset$. Moreover, $\ordering_{\revindex}^{\ell'}$ is any paper ordering that shows each paper in $\diagset_{\revindex}$ ahead of each paper in $\diagset_{\revindex}^c$ and papers among $\diagset_{\revindex}^c$ in a decreasing order of the similarity scores. Given this information, we need to bound
\begin{equation*}
\hyperparam\sum_{\paperindex\in \paperset}(\gainfunctionrev(\ordering_{\revindex}^{\ell'}(\paperindex), \similarity_{\revindex, \paperindex})-\gainfunctionrev(\ordering_{\revindex}^{\ell}(\paperindex), \similarity_{\revindex, \paperindex})).
\end{equation*}

Each paper ordering $\ordering_{\revindex}^\ell$ and $\ordering_{\revindex}^{\ell'}$ shows the papers in $\diagset_{\revindex}^c$ in a decreasing order of the similarity scores after the papers in $\diagset_{\revindex}$ since from Lemma~\ref{lemma:compare}, $\similarity_{\revindex, \paperindex}>\similarity_{\revindex, \paperindex'}$ for $\paperindex\in \diagset_{\revindex}, \paperindex'\in \diagset_{\revindex}^c$. This means papers in $\diagset_{\revindex}^c$ are shown in identical positions by each paper ordering, so we get 
\begin{equation*}
\hyperparam\sum_{\paperindex\in \paperset}(\gainfunctionrev(\ordering_{\revindex}^{\ell'}(\paperindex), \similarity_{\revindex, \paperindex})-\gainfunctionrev(\ordering_{\revindex}^{\ell}(\paperindex), \similarity_{\revindex, \paperindex}) =  \hyperparam \sum_{\paperindex\in \diagset_{\revindex}}(\gainfunctionrev(\ordering_{\revindex}^{\ell'}(\paperindex), \similarity_{\revindex, \paperindex})-\gainfunctionrev(\ordering_{\revindex}^{\ell}(\paperindex), \similarity_{\revindex, \paperindex})).
\end{equation*}
Equivalently, from the decomposed form of the given reviewer-side gain function from~\eqref{eq:noiseless_rgainfunc}, 
\begin{equation*}
\hyperparam\sum_{\paperindex\in \paperset}(\gainfunctionrev(\ordering_{\revindex}^{\ell'}(\paperindex), \similarity_{\revindex, \paperindex})-\gainfunctionrev(\ordering_{\revindex}^{\ell}(\paperindex), \similarity_{\revindex, \paperindex})  = \hyperparam \sum_{\paperindex\in \diagset_{\revindex}}(2^{\similarity_{\revindex, \paperindex}}-1)\gainfunctionrev^{\ordering}(\ordering_{\revindex}^{\ell'}(\paperindex))- \hyperparam \sum_{\paperindex\in \diagset_{\revindex}}(2^{\similarity_{\revindex, \paperindex}}-1)\gainfunctionrev^{\ordering}(\ordering_{\revindex}^{\ell}(\paperindex)).
\end{equation*}
By definition, the similarity score of each paper $\paperindex \in \diagset_{\revindex}$ is given by $\similarity_{\revindex, \paperindex} = \simscalar-\nu_{\revindex, \paperindex}$. 
Moreover, $\nu_{\revindex, \paperindex}$ is bounded in $(0, \unifnoise)$, so $\simscalar-\unifnoise < \simscalar-\nu_{\revindex, \paperindex} < \simscalar$. This fact leads to the lower bound
\begin{equation*}
\hyperparam\sum_{\paperindex\in \paperset}(\gainfunctionrev(\ordering_{\revindex}^{\ell'}(\paperindex), \similarity_{\revindex, \paperindex})-\gainfunctionrev(\ordering_{\revindex}^{\ell}(\paperindex), \similarity_{\revindex, \paperindex})  \geq  \hyperparam(2^{\simscalar-\unifnoise}-1) \sum_{\paperindex\in \diagset_{\revindex}}\gainfunctionrev^{\ordering}(\ordering_{\revindex}^{\ell'}(\paperindex))- \hyperparam (2^{\simscalar}-1) \sum_{\paperindex\in \diagset_{\revindex}}\gainfunctionrev^{\ordering}(\ordering_{\revindex}^{\ell}(\paperindex)).
\end{equation*}
Each paper ordering $\ordering_{\revindex}^\ell$ and $\ordering_{\revindex}^{\ell'}$ shows the papers in $\diagset_{\revindex}$ in some order among the set of positions $\{1,\dots, \blocksize\}$ since there are $\blocksize$ papers in $\diagset_{\revindex}$ by definition and each paper in $\diagset_{\revindex}$ is shown ahead of each paper in $\diagset_{\revindex}^c$. From this observation, we obtain
\begin{equation*}
\hyperparam\sum_{\paperindex\in \paperset}(\gainfunctionrev(\ordering_{\revindex}^{\ell'}(\paperindex), \similarity_{\revindex, \paperindex})-\gainfunctionrev(\ordering_{\revindex}^{\ell}(\paperindex), \similarity_{\revindex, \paperindex})  \geq  \hyperparam(2^{\simscalar-\unifnoise}-2^{\simscalar}) \sum_{\paperindex\in [\blocksize]}\gainfunctionrev^{\ordering}(\paperindex),
\end{equation*}
which is equivalently
\begin{equation*}
\hyperparam\sum_{\paperindex\in \paperset}(\gainfunctionrev(\ordering_{\revindex}^{\ell'}(\paperindex), \similarity_{\revindex, \paperindex})-\gainfunctionrev(\ordering_{\revindex}^{\ell}(\paperindex), \similarity_{\revindex, \paperindex})  \geq \hyperparam(2^{\simscalar-\unifnoise}-2^{\simscalar}) \sum_{\paperindex\in [\blocksize]}\frac{1}{\log_2(\paperindex+1)}
\end{equation*}
from the definition of $\gainfunctionrev^{\ordering}$ given in~\eqref{eq:noiseless_revgain_ordering}. Since $\hyperparam(2^{\simscalar-\unifnoise}-2^{\simscalar})\leq 0$ and $1/\log_2(\paperindex+1)\leq 1$ for each $\paperindex\in [\blocksize]$, we obtain
\begin{equation*}
\hyperparam\sum_{\paperindex\in \paperset}(\gainfunctionrev(\ordering_{\revindex}^{\ell'}(\paperindex), \similarity_{\revindex, \paperindex})-\gainfunctionrev(\ordering_{\revindex}^{\ell}(\paperindex), \similarity_{\revindex, \paperindex})  \geq  \hyperparam\blocksize(2^{\simscalar-\unifnoise}-2^{\simscalar}).
\end{equation*}
Recall that $\unifnoise\leq (1+\hyperparam)^{-1}e^{-e\numblocks\blocksize}$, which means
\begin{equation*}
\hyperparam\sum_{\paperindex\in \paperset}(\gainfunctionrev(\ordering_{\revindex}^{\ell'}(\paperindex), \similarity_{\revindex, \paperindex})-\gainfunctionrev(\ordering_{\revindex}^{\ell}(\paperindex), \similarity_{\revindex, \paperindex})  \geq  \hyperparam\blocksize(2^{\simscalar-(1+\hyperparam)^{-1}e^{-e\numblocks\blocksize}}-2^{\simscalar}).
\end{equation*}
Now, see that 
$\hyperparam(2^{\simscalar-(1+\hyperparam)^{-1}e^{-e\numblocks\blocksize}}-2^{\simscalar})$ is non-positive and a decreasing function of $\lambda$ and $\simscalar$ on the domain $\hyperparam\geq0$ and $\simscalar\geq 0.01$. This means for every $\hyperparam\geq0$ and $\simscalar\geq 0.01$, the following relation holds 
\begin{align*}
\hyperparam(2^{\simscalar-(1+\hyperparam)^{-1}e^{-e\numblocks\blocksize}}-2^{\simscalar}) &\geq \hyperparam(2^{1-(1+\hyperparam)^{-1}e^{-e\numblocks\blocksize}}-2) \\
&\geq \lim_{\hyperparam'\rightarrow\infty}\hyperparam'(2^{1-(1+\hyperparam')^{-1}e^{-e\numblocks\blocksize}}-2) \\
& =  -e^{-e\numblocks\blocksize}\log(4).
\end{align*}
Consequently, we conclude
\begin{equation*}
\hyperparam\sum_{\paperindex\in \paperset}(\gainfunctionrev(\ordering_{\revindex}^{\ell'}(\paperindex), \similarity_{\revindex, \paperindex})-\gainfunctionrev(\ordering_{\revindex}^{\ell}(\paperindex), \similarity_{\revindex, \paperindex})  \geq - \blocksize e^{-e\numblocks\blocksize}\log(4).
\end{equation*}

\section{Additional Results}
In this section, we formally state and prove a pair of results that were mentioned informally in the main paper. We characterize the time complexity per-reviewer of the \super algorithm for the general model and for a selected set of gain and bidding functions that admit a computationally efficient solution. Moreover, we show that \super with any heuristic is globally optimal given a linear paper-side gain. This result is a corollary of the fact that \super is locally optimal as shown in Theorem~\ref{prop:local}.

\subsection{Time Complexity of \super}
\label{sec:proof_time}
The following proposition characterizes the time complexity of the \super algorithm for each reviewer given the evaluations of the heuristic for the general form and a relevant class of gain and bidding functions that admits a computational efficient solution.

\begin{proposition}\label{prop:time}
\super has a time complexity of $\mathcal{O}(\numpapers^3)$ per-reviewer given the evaluations of the heuristic function. The time complexity of \super improves to $\mathcal{O}(\numpapers\log(\numpapers))$ given a bidding function that can be decomposed into the form $\bidfunction(\ordering_{\revindex}(\paperindex), \similarity_{\revindex, \paperindex}) = \bidfunction^{\ordering}(\ordering_{\revindex}(\paperindex))\bidfunction^{\similarity}(\similarity_{\revindex, \paperindex})$ where $\bidfunction^{\ordering}: \paperset \rightarrow [0, 1]$ is non-increasing and $\bidfunction^{\similarity}: [0, 1] \rightarrow [0, 1]$ is non-decreasing, along with a reviewer-side gain function that can be decomposed into the form $\gainfunctionrev(\ordering_{\revindex}(\paperindex), \similarity_{\revindex, \paperindex}) = \bidfunction^{\ordering}(\ordering_{\revindex}(\paperindex))\gainfunctionrev^{\similarity}(\similarity_{\revindex, \paperindex})$ where $\gainfunctionrev^{\similarity}: [0, 1]\rightarrow \reals_{\geq 0}$ is non-decreasing. 
\end{proposition}

\begin{proof}[Proof of Proposition~\ref{prop:time}]
We partition this proof by first examining the time complexity under the general model and then after which we consider the time complexity for the special case.
\paragraph*{General time complexity.}
We begin by showing the time complexity of \super for a reviewer given the heuristic evaluations under the general class of gain and bidding functions. 
The general form of the $\super$ algorithm calls Algorithm~\ref{alg:subprocedure} upon the arrival of a reviewer to determine the ordering of papers to show the reviewer. The optimization problem in Algorithm~\ref{alg:subprocedure} is in the form of the linear assignment problem. It is well known that the Hungarian algorithm can solve for the optimal solution of a linear assignment problem with a time complexity of $\mathcal{O}({\numpapers}^3)$ (see, e.g., Chapter 8 in~\citealp{lawler1976combinatorial}). 
As a result, \super has a time complexity of $\mathcal{O}(\numpapers^3)$ for the general class of gain and bidding functions under consideration for each reviewer given the evaluations of the heuristic function. 
\paragraph*{Special case time complexity.}
In the proof of Theorem~\ref{prop:local}, we showed that the optimal paper ordering to present the final reviewer could be obtained by solving the linear program given in~\eqref{eq:opt_linear} with the weights from~\eqref{eq:opt_final_weights}. The general version of \super determines the ordering of papers to present any reviewer by calling Algorithm~\ref{alg:subprocedure}, which solves the linear program given in~\eqref{eq:opt_linear} using the weights from~\eqref{eq:superweights}. We now show an equivalence between that solution method and a sorting algorithm for the class of gain and bidding functions given in the claim.

Prior to deriving the linear program in~\eqref{eq:opt_linear} as a method to obtain the optimal solution for the final reviewer in the proof of Theorem~\ref{prop:local}, we showed in~\eqref{eq:opt_preinteger} that the optimization problem for the final reviewer was of the form
\begin{equation*}
\max_{\ordering_{\numrev}\in \symgroup_\numpapers} \quad \sum_{\paperindex\in \paperset} \bidfunction(\ordering_{\numrev}(\paperindex), \similarity_{\numrev, \paperindex})(\gainfunction(\numbids_{\numrev-1, \paperindex} + 1) - \gainfunction(\numbids_{\numrev-1, \paperindex})) + \hyperparam \sum_{\paperindex \in \paperset}\gainfunctionrev(\ordering_{\numrev}(\paperindex), \similarity_{\numrev, \paperindex}).
\end{equation*}
Given the function forms $\bidfunction(\ordering_{\numrev}(\paperindex), \similarity_{\numrev, \paperindex}) = \bidfunction^{\ordering}(\ordering_{\numrev}(\paperindex))\bidfunction^{\similarity}(\similarity_{\numrev, \paperindex})$ where $\bidfunction^{\ordering}: \paperset \rightarrow [0, 1]$ is non-increasing and $\bidfunction^{\similarity}: [0, 1] \rightarrow [0, 1]$ is non-decreasing, and $\gainfunctionrev(\ordering_{\numrev}(\paperindex), \similarity_{\numrev, \paperindex}) = \bidfunction^{\ordering}(\ordering_{\numrev}(\paperindex))\gainfunctionrev^{\similarity}(\similarity_{\numrev, \paperindex})$ where $\gainfunctionrev^{\similarity}: [0, 1]\rightarrow \reals_{\geq 0}$ is non-decreasing, the problem can be reformulated as
\begin{equation}
\max_{\ordering_{\numrev}\in \symgroup_\numpapers} \quad \sum_{\paperindex\in \paperset} \sortvar_{\numrev, \paperindex}\bidfunction^{\ordering}(\ordering_{\numrev}(\paperindex)) 
\label{eq:super_efficient}
\end{equation}
where
\[\sortvar_{\numrev, \paperindex} = \bidfunction^{\similarity}(\similarity_{\numrev, \paperindex})(\gainfunction(\numbids_{\numrev-1, \paperindex} + 1) - \gainfunction(\numbids_{\numrev-1, \paperindex})) + \hyperparam \gainfunctionrev^{\similarity}(\similarity_{\numrev, \paperindex}) \ \forall \ \paperindex \in \paperset.\]
Consequently, for the class of gain and bidding functions given in the claim, an equivalent form of the general \super algorithm that calls Algorithm~\ref{alg:subprocedure} to determine the ordering of papers to show any reviewer $\revindex \in \reviewerset$ instead solves the problem in~\eqref{eq:super_efficient} 
using weights 
\begin{equation}
\sortvar_{\revindex, \paperindex} = \bidfunction^{\similarity}(\similarity_{\revindex, \paperindex})(\gainfunction(\numbids_{\revindex-1, \paperindex} +\proxy_{\revindex, \paperindex}+ 1) - \gainfunction(\numbids_{\revindex-1, \paperindex} +\proxy_{\revindex, \paperindex})) + \hyperparam \gainfunctionrev^{\similarity}(\similarity_{\revindex, \paperindex}) \ \forall \ \paperindex \in \paperset.
\label{eq:super_efficient_weights}
\end{equation}
The optimal solution to a problem of the form in~\eqref{eq:super_efficient} is simply to present the papers in decreasing order of their corresponding values of $\alpha_{\revindex, \paperindex}$ since the function $\bidfunction^{\ordering}$ is non-increasing.
The sorting procedure requires a time complexity of just $\mathcal{O}(\numpapers\log(\numpapers))$. 
Since Algorithm~\ref{alg:subprocedure_efficient} solves the problem in~\eqref{eq:super_efficient} using the weights in~\eqref{eq:super_efficient_weights}, we conclude that the per-reviewer time complexity of \super for the given class of gain and bidding functions and given the evaluations of the heuristic is $\mathcal{O}(\numpapers\log(\numpapers))$.
\end{proof}

\subsection{\super Optimality for Linear Paper-Side Gain}
\label{sec:proof_linear}
In this section, we show that \super with any heuristic is optimal when the paper-side gain function is linear. This property of the algorithm follows rather directly from the local optimality result in Theorem~\ref{prop:local} since for this type of paper-side gain function, the global optimization problem is decoupled between each reviewer. 
\begin{restatable}{proposition}{proplinear}\label{prop:linear}
\super, with any heuristic, is optimal when the paper-side gain function is linear. 
\end{restatable} 
\begin{proof}[Proof of Proposition~\ref{prop:linear}]
The optimization objective over the set of reviewers is defined as 
\begin{equation*}
\max_{\ordering_{1},\dots,\ordering_{\numrev} \in \symgroup_\numpapers} \quad \sum_{\paperindex\in \paperset} \mathbb{E}[\gainfunction(\numbids_{\paperindex})]+\hyperparam\sum_{\revindex \in \reviewerset}\sum_{\paperindex\in \paperset}\mathbb{E}[\gainfunctionrev(\ordering_{\revindex}(\paperindex), \similarity_{\revindex,\paperindex})],
\end{equation*}
where the expectation is taken over the randomness in the bids made by the reviewers. Under a linear paper-side gain function, the problem is equivalently formulated as  
\begin{equation*}
\max_{\ordering_{1},\dots,\ordering_{\numrev}\in \symgroup_\numpapers} \quad \sum_{\paperindex\in \paperset} \mathbb{E}\Big[\sum_{\revindex\in \reviewerset}\randombid_{\revindex, \paperindex}\Big]+\hyperparam\sum_{\revindex \in \reviewerset}\sum_{\paperindex\in \paperset}\gainfunctionrev(\ordering_{\revindex}(\paperindex), \similarity_{\revindex,\paperindex}),
\end{equation*}
where $\randombid_{\revindex, \paperindex}$ denotes the random bid of reviewer $\revindex$ on paper $\paperindex$ and the expectation on the reviewer-side gain went away since it is deterministic given a paper ordering for any reviewer. Using the structure of the bidding model, we can simplify the expectation over the paper-side gain to obtain the objective function
\begin{equation*}
\max_{\ordering_{1},\dots,\ordering_{\numrev}\in \symgroup_\numpapers} \quad \sum_{\revindex\in \reviewerset}\sum_{\paperindex\in \paperset}  \bidfunction(\ordering_{\revindex}(\paperindex), \similarity_{\revindex, \paperindex})+\hyperparam\sum_{\revindex \in \reviewerset}\sum_{\paperindex\in \paperset}\gainfunctionrev(\ordering_{\revindex}(\paperindex), \similarity_{\revindex,\paperindex}).
\end{equation*}
The paper-side and reviewer-side gains are now decoupled between the ordering presented to each reviewer. Consequently, the optimal paper-ordering to present to each reviewer $\revindex \in \reviewerset$ is given by the solution to the optimization problem 
\begin{equation}
\max_{\ordering_{\revindex}\in \symgroup_\numpapers.} \quad \sum_{\paperindex\in \paperset} \bidfunction(\ordering_{\revindex}(\paperindex), \similarity_{\revindex, \paperindex})+\hyperparam\sum_{\paperindex\in \paperset}\gainfunctionrev(\ordering_{\revindex}(\paperindex), \similarity_{\revindex,\paperindex}).
\label{eq:linear_sep}
\end{equation}

The \super algorithm solves the following problem to determine the ordering of papers to present each reviewer $\revindex \in \reviewerset$:
\begin{equation*}
\max_{\ordering_{\revindex}\in \symgroup_\numpapers} \quad \sum_{\paperindex\in \paperset} \bidfunction(\ordering_{\revindex}(\paperindex), \similarity_{\revindex, \paperindex})(\gainfunction(\numbids_{\revindex-1, \paperindex} + \proxy_{\revindex, \paperindex} + 1) - \gainfunction(\numbids_{\revindex-1, \paperindex}+ \proxy_{\revindex, \paperindex})) + \hyperparam \sum_{\paperindex \in \paperset}\gainfunctionrev(\ordering_{\revindex}(\paperindex), \similarity_{\revindex, \paperindex}).
\end{equation*}
Under a linear paper-side gain, the optimization problem the \super algorithm solves for each reviewer simplifies to the problem
\begin{equation*}
\max_{\ordering_{\revindex}\in \symgroup_\numpapers} \quad \sum_{\paperindex\in \paperset}\bidfunction(\ordering_{\revindex}(\paperindex), \similarity_{\revindex, \paperindex}) + \hyperparam \sum_{\paperindex \in \paperset}\gainfunctionrev(\ordering_{\revindex}(\paperindex), \similarity_{\revindex, \paperindex}) 
\end{equation*}
since the number of bids and the heuristic cancels.
We showed in the proof of Proposition~\ref{prop:local} that the \super algorithm solves this problem efficiently, and exactly. Since the problem is equivalent to that in~\eqref{eq:linear_sep} which gives the optimal solution for each reviewer, the \super algorithm is optimal with a linear paper-side gain function. 
\end{proof}

\end{appendices}
\end{document}